\documentclass{article}

\usepackage[verbose=true,letterpaper]{geometry}
  \newgeometry{
    textheight=9in,
    textwidth=5.5in,
    top=1in,
    headheight=12pt,
    headsep=25pt,
    footskip=30pt
  }

\usepackage{parskip}
\PassOptionsToPackage{numbers, compress}{natbib}
\usepackage{natbib}

\usepackage[utf8]{inputenc} 
\usepackage[T1]{fontenc}    
\usepackage{hyperref}       
\usepackage{url}            
\usepackage{booktabs}       
\usepackage{amsfonts}       
\usepackage{nicefrac}       
\usepackage{microtype}      
\usepackage{graphicx}
\usepackage{subfigure}
\usepackage{amsmath}
\usepackage{amsthm}
\usepackage{textcomp}
\usepackage{placeins}

\usepackage{graphbox}
\usepackage{booktabs}
\usepackage{soul}
\usepackage{tcolorbox}
\usepackage{paralist}
\usepackage{multirow}

\definecolor{lightblue}{HTML}{84C7F9}
\definecolor{lighterblue}{HTML}{D4ECFF}
\newtcolorbox{mybox}{colback=lighterblue,colframe=lightblue}

\newtheorem{theorem}{Theorem}
\newtheorem{proposition}[theorem]{Proposition}

\usepackage{color}

\usepackage{array}
\newcolumntype{P}[1]{>{\centering\arraybackslash}p{#1}}
\newcolumntype{M}[1]{>{\centering\arraybackslash}m{#1}}

\usepackage{xcolor}
\definecolor{dark-blue}{rgb}{0.15,0.15,0.4}
\definecolor{medium-blue}{rgb}{0,0,0.5}
\hypersetup{
   colorlinks, linkcolor={dark-blue},
   citecolor={dark-blue}, urlcolor={medium-blue}
}

\date{}

\title{Why Normalizing Flows Fail to Detect Out-of-Distribution Data}

\author{%
\normalsize \textbf{Polina Kirichenko\footnote{Equal contribution.} , Pavel Izmailov$^*$, Andrew Gordon Wilson} \\
\normalsize New York University
}

\begin{document}

\maketitle

\begin{abstract}

\normalsize

\noindent Detecting out-of-distribution (OOD) data is crucial for robust machine learning systems. 
Normalizing flows are flexible deep generative models that often surprisingly fail to distinguish between in- and out-of-distribution data: a flow trained on pictures of clothing assigns higher likelihood to handwritten digits. We investigate why normalizing flows perform poorly for OOD detection. We demonstrate that flows learn local pixel correlations and generic image-to-latent-space transformations which are not specific to the target image dataset. We show that by modifying the architecture of flow coupling layers we can bias the flow towards learning the semantic structure of the target data, improving OOD detection. Our investigation reveals that properties that enable flows to generate high-fidelity images can have a detrimental effect on OOD detection.
\end{abstract}

\section{Introduction}

Normalizing flows \citep{tabak2013family, dinh2014nice, dinh2016density} seem to be ideal candidates for out-of-distribution detection, since they are simple generative models that provide an exact likelihood. However, \citet{nalisnick2018deep} revealed the puzzling result that flows often assign higher likelihood to out-of-distribution data than the data used for maximum likelihood training. In Figure \ref{fig:intro}(a), we show the log-likelihood histogram for a RealNVP flow model \citep{dinh2016density} trained on the ImageNet dataset \cite{russakovsky2015imagenet} subsampled to $64\times 64$ resolution.
The flow assigns higher likelihood to both the CelebA dataset of celebrity photos, and the SVHN dataset of images of house numbers, compared to the target ImageNet dataset.

While there has been empirical progress in improving OOD detection with flows \citep{nalisnick2018deep, choi2018waic, nalisnick2019detecting, serra2019input, song2019unsupervised, zhang2020out}, the fundamental reasons for why flows fail at OOD detection in the first place are not fully understood. In this paper, we 
show how the \emph{inductive biases} \citep{mitchell1980need, wilson2020bayesian} of flow models --- implicit assumptions in the architectures and training procedures --- can hinder OOD detection.

In particular, our contributions are the following:

\begin{figure}[t]
    
	\def \panelheight {0.23\textwidth}
	\def \panelskip {.2cm}

    \centering
    \subfigure[Log-likelihoods]{
    	\includegraphics[height=\panelheight]{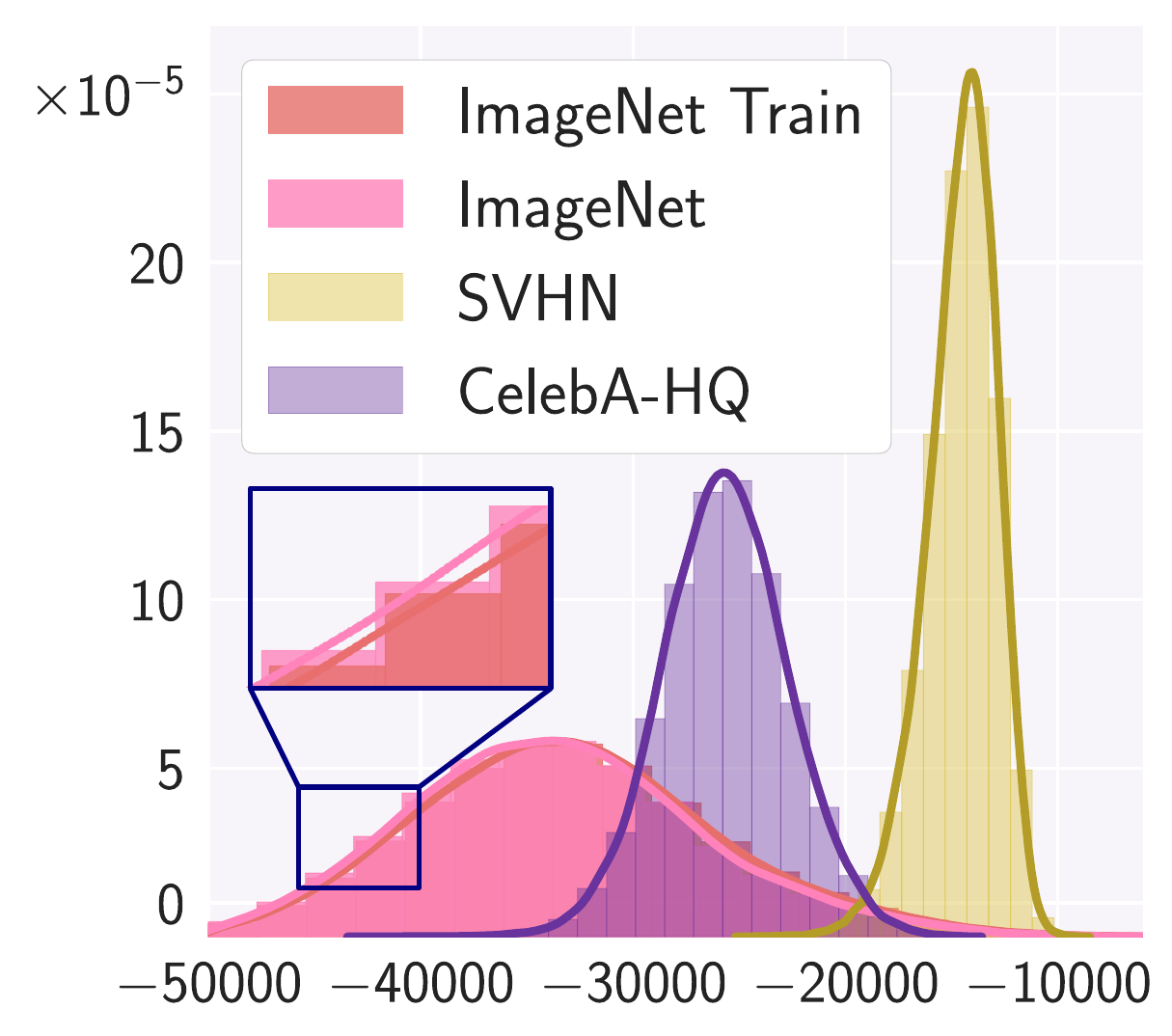} 
    }
    \hspace{\panelskip}
    \subfigure[ImageNet input, in-distribution]{
    	\includegraphics[height=\panelheight]{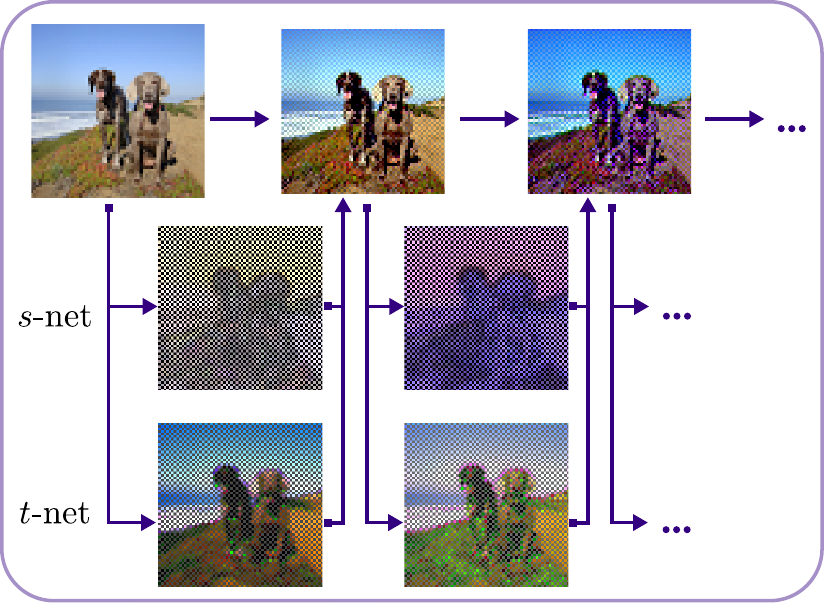} 
    }
    \hspace{\panelskip}
    \subfigure[CelebA input, OOD]{
    	\includegraphics[height=\panelheight]{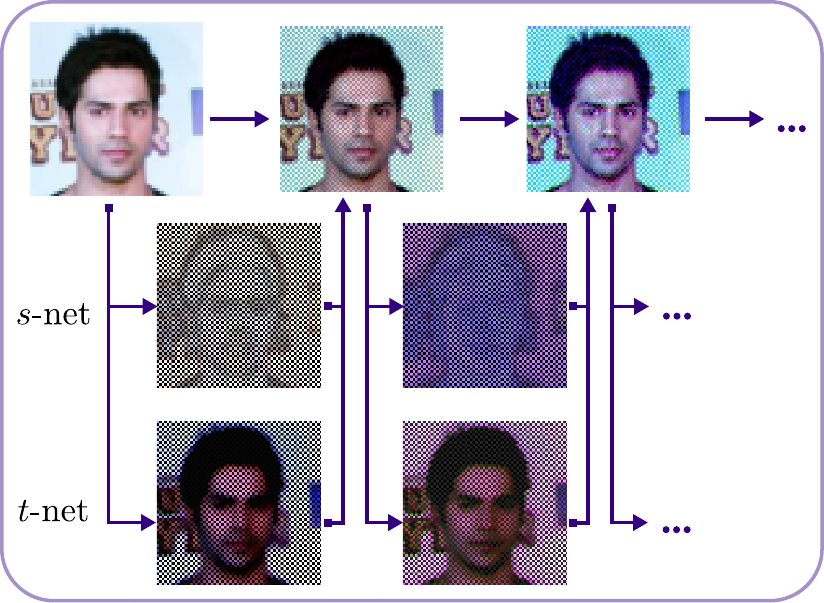} 
    }
	\caption{ 
	\textbf{RealNVP flow on in- and out-of-distribution images.}
	(\textbf{a}): A histogram of log-likelihoods that a RealNVP flow trained on ImageNet assigns to ImageNet, SVHN and CelebA. 
	The flow assigns higher likelihood to out-of-distribution data.
	(\textbf{b}, \textbf{c}): A visualization of the intermediate layers of a RealNVP model on an (b) in-distribution image and (c) OOD image.
	The first row shows the coupling layer activations, the second and third rows show the scale $s$ and shift $t$ parameters predicted by a neural network applied to the corresponding coupling layer input. Both on in-distribution and out-of-distribution images, $s$ and $t$ accurately approximate the structure of the input, even though the model has not observed inputs (images) similar to the OOD image during training. \emph{Flows learn generic image-to-latent-space transformations that leverage local pixel correlations and graphical details rather than the semantic content needed for OOD detection}.
	}
	\label{fig:intro}
\end{figure}

\begin{itemize}
    
    \item We show that flows learn latent representations for images largely based on local pixel correlations, rather than semantic content, 
    making it difficult to detect data with anomalous semantics.
    
    \item 
    We identify mechanisms through which normalizing flows can simultaneously increase likelihood for all structured images.
    For example, in Figure \ref{fig:intro}(b, c), we show that the coupling layers of RealNVP transform the in-distribution ImageNet in the 
    same way as the OOD CelebA.
    
    \item We show that by changing the architectural details of the coupling layers, we can encourage flows to learn transformations specific to the target data, improving OOD detection.
    
    \item We show that OOD detection is improved when flows are trained on high-level features which contain semantic information extracted from image datasets.
    
\end{itemize}

We also provide code at \url{https://github.com/PolinaKirichenko/flows_ood}.

\vspace{-2mm}
\section{Background}
\label{sec:background}
\vspace{-2mm}

We briefly introduce normalizing flows based on coupling layers.
For a more detailed introduction, see \citet{papamakarios2019normalizing} and 
\citet{kobyzev2019normalizing}.

\textbf{Normalizing flows}\quad
Normalizing flows \cite{tabak2013family} are a flexible class of deep generative models that model a target distribution $p^*(x)$ as an invertible transformation $f$ of a base distribution $p_Z(z)$ in the latent space.
Using the change of variables formula, the likelihoods for an input $x$ and a dataset $\mathcal D$ are
\begin{equation}\label{eq:change_of_variable}
    p_X(x) = p_Z(f^{-1}(x)) \left| \det \frac{\partial f^{-1}}{\partial x} \right|,
    \quad
    p(\mathcal D) = \prod_{x \in \mathcal D} p_X(x).
\end{equation}
The latent space distribution $p_Z(z)$ is commonly chosen to be a standard Gaussian.
Flows are typically trained by maximizing the log-likelihood \eqref{eq:change_of_variable}
of the training data with respect to the parameters of the invertible transformation 
$f$.

\textbf{Coupling layers}\quad
We focus on normalizing flows based on \textit{affine coupling layers}.
In these flows, the transformation performed by each layer is given by
\begin{equation}\label{eq:coupling}
f_{\text{aff}}^{-1} (x_\text{id}, x_\text{change}) = (y_\text{id}, y_\text{change}),
\quad 
\begin{cases}
   &y_{\text{id}} = x_{\text{id}} \\
   &y_{\text{change}} = (x_{\text{change}} + t(x_{\text{id}})) \odot \exp(s(x_{\text{id}}))
\end{cases}
\end{equation}
where $x_{\text{id}}$ and $x_{\text{change}}$ are disjoint parts of the input $x$,
$y_{\text{id}}$ and $y_{\text{change}}$ are disjoint parts of the output $y$, 
and the scale and shift parameters $s(\cdot)$ and $t(\cdot)$ are usually 
implemented by a neural network (which we will call the \textit{$st$-network}).
The split of the input into $x_{\text{id}}$ and $x_{\text{change}}$ is defined
by a \textit{mask}: a coupling layer transforms the masked part $x_\text{change} = \text{mask}(x)$ of the input based
on the remaining part $x_\text{id}$.
The transformation \eqref{eq:coupling} is invertible and allows for efficient
Jacobian computation in \eqref{eq:change_of_variable}:
\begin{equation}
\log \left| \det \frac{\partial f_{\text{aff}}^{-1}}{\partial x} \right| = \sum_{i=1}^{\text{dim}(x_{\text{change}})} s(x_{\text{id}})_i.
\end{equation}
\vspace{-3mm}

\textbf{Flows with coupling layers}\quad
Coupling layers can be stacked together into flexible normalizing flows:
$f = f^K \circ f^{K-1} \circ \ldots \circ f^1$.
Examples of flows with coupling layers include NICE \citep{dinh2014nice}, RealNVP \citep{dinh2016density}, Glow \citep{kingma2018glow}, and many others
\citep[e.g.,][]{bhattacharyya2020normalizing, chen2020vflow, durkan2019neural, ho2019flow++, hoogeboom2019emerging, kim2018flowavenet, ma2019macow,  prenger2019waveglow}.

\textbf{Out-of-distribution detection using likelihood}
\quad
Flows can be used for out-of-distribution detection based on the likelihood they
assign to the inputs.
One approach is to choose a likelihood threshold $\epsilon$ on a validation
dataset, e.g.\ to satisfy a desired false positive rate, 
and during test time identify inputs which have likelihood lower than $\epsilon$ as OOD.
Qualitatively, we can estimate the performance of the flows for OOD detection
by plotting a histogram of the log-likelihoods such as Figure \ref{fig:intro}(a):
the likelihoods for in-distribution data should generally be higher compared to OOD.
Alternatively, we can treat OOD detection as a binary classification problem using 
likelihood scores, and compute accuracy with a fixed likelihood threshold $\epsilon$,
or AUROC (area under the receiver operating characteristic curve).

\vspace{-3mm}
\section{Related Work}
\label{sec:related_work}
\vspace{-2mm}

Recent works have shown that normalizing flows, among other deep generative models, 
can assign higher likelihood to out-of-distribution data \citep{nalisnick2018deep, choi2018waic}.
The work on OOD detection with deep generative models falls into two distinct categories. 
In group anomaly detection (GAD), the task is to label a batch of $n > 1$ datapoints as in- or out-of-distribution.
Point anomaly detection (PAD) involves the more challenging task of labelling single points as
out-of-distribution.

\textbf{Group anomaly detection} \quad 
\citet{nalisnick2019detecting} introduce the typicality test which distinguishes between a high density set and a typical set of a distribution induced by a model. 
However, the typicality test cannot detect OOD data if the flow assigns it with a similar likelihood distribution to that of in-distribution data.
\citet{song2019unsupervised}~showed that out-of-distribution datasets have lower likelihoods when batch normalization statistics are computed from a current batch instead of accumulated over the train set, and proposed a test based on this observation.
\citet{zhang2020out} introduce a GAD algorithm based on measuring correlations of flow's latent representations corresponding to the input batch. 
The main limitation of GAD methods is that for most practical applications the assumption that the data comes in batches of inputs that are all in-distribution or all OOD is not realistic.

\textbf{Point anomaly detection} \quad \citet{choi2018waic} proposed to estimate the Watanabe-Akaike Information Criterion using an ensemble of generative models, showing accurate OOD detection on some of the challenging dataset pairs. 
\citet{ren2019likelihood} explain the poor OOD detection performance of deep generative models by the fact that the likelihood is dominated by background statistics.
They propose a test based on the ratio of the likelihoods for the image and background likelihood estimated using a separate \textit{background model}.
\citet{serra2019input} show that normalizing flows assign higher likelihoods to simpler datasets and propose to normalize the flow's likelihood by an image complexity score.

In this work we argue that it is the inductive biases of the model that determine its OOD performance.
While most work treats flows as black-box density estimators, we conduct a careful study of the latent representations and 
image-to-latent-space transformations learned by the flows.
Throughout the paper, we connect our findings with prior work and provide new insights.

\section{Why flows fail to detect OOD data}
\label{sec:reasoning}

Normalizing flows consistently fail at out-of-distribution detection when applied to common benchmark datasets (see Appendix \ref{appendix:baseline_likelihood_auroc}).
In this paper, we discuss the reasons behind this surprising phenomenon.
We summarize our thesis as follows:

\begin{mybox}
    The maximum likelihood objective has a limited influence on OOD detection, relative to the \textit{inductive biases} of the flow, captured by the modelling assumptions of the architecture.
\end{mybox}

\textbf{Why should flows be able to detect OOD inputs?} \quad
Flows are trained to maximize the likelihood of the training data.
Likelihood is a probability density function $p(\mathcal D)$ defined on the image space and hence has to be normalized.
Thus, likelihood cannot be simultaneously increased for all the inputs (images).
In fact, the optimal maximizer of \eqref{eq:change_of_variable} would only assign positive density to the datapoints in the training set, and, in particular, would not 
even generalize to the test set of the same dataset.
In practice, flows do not seem to overfit, assigning similar likelihood distributions to train and 
and test (see e.g. Figure \ref{fig:intro}(a)).
Thus, despite their flexibility, flows are not maximizing the likelihood \eqref{eq:change_of_variable} to values close to the global optimum. 

\textbf{What is OOD data?} \quad
There are infinitely many distributions that give rise to any value of the likelihood objective in \eqref{eq:change_of_variable} except the global optimum.
Indeed, any non-optimal solution assigns probability mass outside of the training
data distribution; we can arbitrarily re-assign this probability mass to get a new
solution with the same value of the objective (see Appendix \ref{sec:app_whatisood} for a detailed discussion).
Therefore the inductive biases of a model determines which specific solution is found 
through training. In particular, the inductive biases will affect what data is assigned
 high likelihood (in-distribution) and what data is not (OOD).

\textbf{What inductive biases are needed for OOD detection?} \quad
The datasets in computer vision are typically defined by the semantic content of the images.
For example, the CelebA dataset consists of images of faces, and SVHN contains 
images of house numbers.
In order to detect OOD data, the inductive biases
of the model have to be aligned with learning the semantic structure of the data,
i.e. what objects are represented in the data.

\textbf{What are the inductive biases of normalizing flows?} \quad
In the remainder of the paper, we explore the inductive biases of normalizing flows.
We argue that flows are biased towards learning \textit{graphical}
properties of the data such as local pixel correlations (e.g.\ nearby pixels usually
have similar colors) rather than semantic properties of the data 
(e.g.\ what objects are shown in the image).

\textbf{Flows have capacity to distinguish datasets} \quad
In Appendix \ref{sec:app_capacity}, we show that if we explicitly train flows to distinguish between a pair of datasets, they can assign large likelihood to one dataset 
and low likelihood to the other.
However, when trained with the standard maximum likelihood objective, flows do not learn to make this distinction. The inductive biases of the flows prefer solutions that assign high likelihood to most structured datasets simultaneously.

\section{Flow latent spaces}
\label{sec:latent_space}

Normalizing flows learn highly non-linear image-to-latent-space mappings often using hundreds of millions of parameters.
One could imagine that the learned latent representations have a complex structure, encoding high-level semantic information about the inputs.
In this section, we visualize the learned latent representations on both in-distribution and out-of-distribution data and demonstrate that they encode simple graphical structure rather than semantic information.

\begin{mybox}
    \textbf{Observation}:
    There exists a correspondence between the coordinates in
    an image and in its learned representation.
    We can recognize edges of the inputs in their latent representations.
    
    \textbf{Significance for OOD detection:}
    In order to detect OOD images, a model has to assign likelihood 
    based on the semantic content of the image (see Sec. \ref{sec:reasoning}).
    Flows do not represent images based on their semantic contents, 
    but rather directly encode their visual appearance.
\end{mybox}

In the first four columns of Figure \ref{fig:latent_reprs}, we show latent 
representations\footnote{
For the details of the visualization procedure and the training setup please see Appendices \ref{sec:app_visualizations} and \ref{sec:app_details}.
}
of a RealNVP model trained on \mbox{FashionMNIST}  for an in-distribution FashionMNIST image and an out-of-distribution MNIST digit.
The first column shows the original image $x$, and the second column shows the corresponding latent~$z$. 
The latent representations appear noisy both for in- and out-of-distribution samples, but the edges of the MNIST digit can be recognized in the latent.
In the third column of Figure \ref{fig:latent_reprs}, we show latent representations averaged over $K=40$ samples of dequantization noise\footnote{
When training flow models on images or other discrete data, we use dequantization to avoid pathological solutions~\citep{uria2013rnade, theis2015note}: 
we add uniform noise $ \epsilon \sim U[0; 1]$ to each pixel $x_i \in \{0, 1, \dots, 255\}$. 
Every time we pass an image through the flow $f(\cdot)$, the resulting latent representation $z$ will be different. 
} $\epsilon_k$:
$\frac{1}{K} \sum_{k=1}^K f^{-1}(x + \epsilon_k)$.
In the averaged representation, we can clearly see the edges from the original image.
Finally, in the fourth column of Figure \ref{fig:latent_reprs}, we visualize the latent representations (for a single sample of dequantization noise) from a flow when batch normalization layers are in train mode \citep{ioffe2015batch}. 
In train mode, batch normalization layers use the activation statistics of the current batch, and in evaluation mode they use the statistics accumulated over the train set.
While for in-distribution data there is no structure visible in the latent representation, the out-of-distribution latent clearly preserves the shape of the $7$-digit from the input image.
In the remaining panels of Figure \ref{fig:latent_reprs}, we show an analogous
visualization for a RealNVP trained on CelebA using an SVHN image as OOD.
In the third panel of this group, we visualize the blue channel of the latent representations.
Again, the OOD input can be recognized in the latent representation;
some of the edges from the in-distribution CelebA image can also be seen in 
the corresponding latent variable. Additional visualizations (e.g.\ for Glow) are in Appendix \ref{sec:app_latents}.

\textbf{Insights into prior work} \quad
The group anomaly detection algorithm proposed in \citet{zhang2020out} uses correlations of the latent representations as an OOD score. 
\citet{song2019unsupervised} showed that normalizing flows with batch normalization layers in train mode assign much lower likelihood to out-of-distribution images than they do in evaluation mode, while for in-distribution data the difference is not significant.
Our visualizations explain the presence of correlations in the latent space and shed light into the difference between the behaviour of the flows in train and test mode.

\vspace{-2mm}
\section{Transformations learned by coupling layers}
\label{sec:coupling_layers}

To better understand the inductive biases of coupling-layer based flows, we study
the transformations learned by individual coupling layers.

\textbf{What are coupling layers trained to do?} \quad
Each coupling layer updates
the masked part $x_{\text{change}}$ of the input $x$ to be $x_{\text{change}} \leftarrow (x_{\text{change}} + t(x_{\text{id}})) \cdot \exp(s(x_{\text{id}}))$,
where $x_{\text{id}}$ is the non-masked part of $x$, and $s$ and $t$ are the outputs of the $st$-network given $x_{\text{id}}$ (see Section \ref{sec:background}).
The flow is encouraged to predict high values for $s$ since for a given coupling layer the Jacobian term in the 
likelihood of Eq. \eqref{eq:change_of_variable} is given by $\sum_j s(x_{\text{id}})_j$ (see Section \ref{sec:reasoning}). 
Intuitively, to afford large values for scale $s$ without making the latent representations large in norm and hence decreasing 
the density term $p_{\mathcal{Z}}(z)$ in \eqref{eq:change_of_variable}, the shift $-t$ has to be an accurate 
approximation of the masked input $x_{\text{change}}$.
For example, in Figure \ref{fig:intro}(b, c) the $-t$ outputs of the first coupling
layers are a very close estimate of the input to the coupling layer.
The likelihood for a given image will be high whenever the coupling layers can accurately predict masked pixels.
To the best of our knowledge, this intuition has not been discussed in any previous work.

\begin{mybox}
    \textbf{Observation}:
    We describe two mechanisms through which coupling layers learn to predict the masked pixels: (1) leveraging local color correlations and (2) using information about the masked pixels encoded by the previous coupling layer (coupling layer co-adaptation).
    
    \textbf{Significance for OOD detection}:
    These mechanisms allow the flows to predict the masked pixels equally accurately on in- and out-of-distribution datasets. 
    As a result, flows assign high likelihood to OOD data.
\end{mybox}

\begin{figure*}[t]
	\def \panelwidth {0.09\textwidth}
	\def \panelskip {-0.7cm}

    \centering
    \subfigure{
    \begin{tabular}{c}
        \includegraphics[width=\panelwidth]{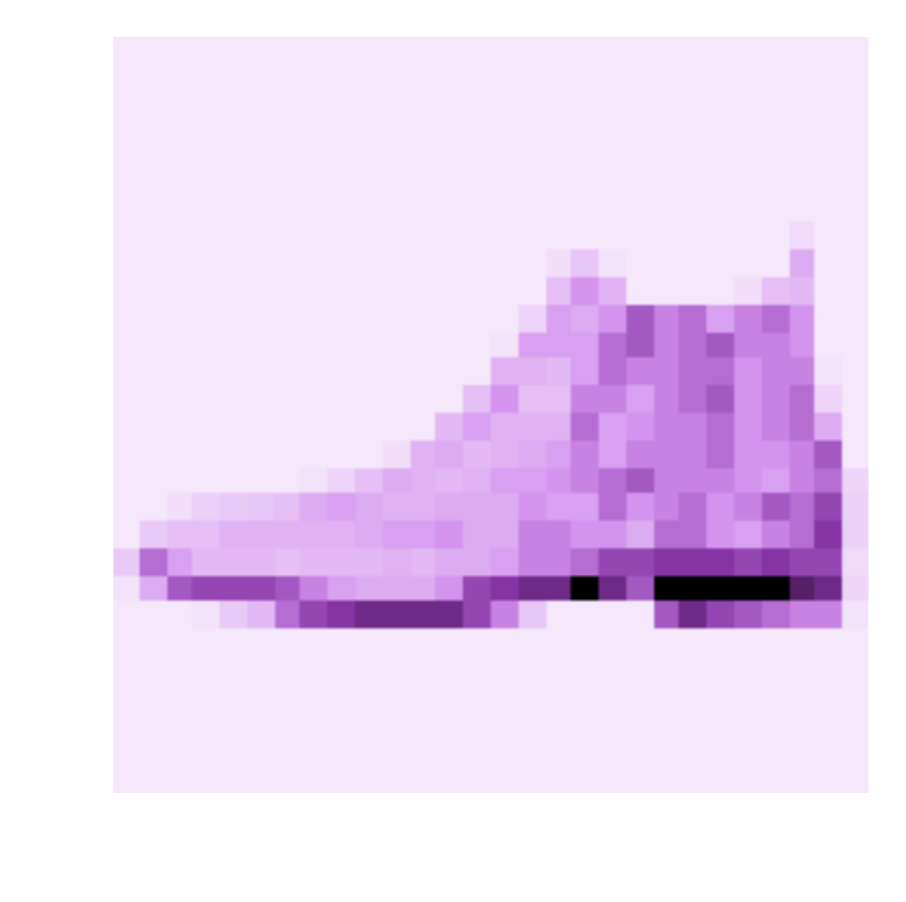}\\
        \includegraphics[width=\panelwidth]{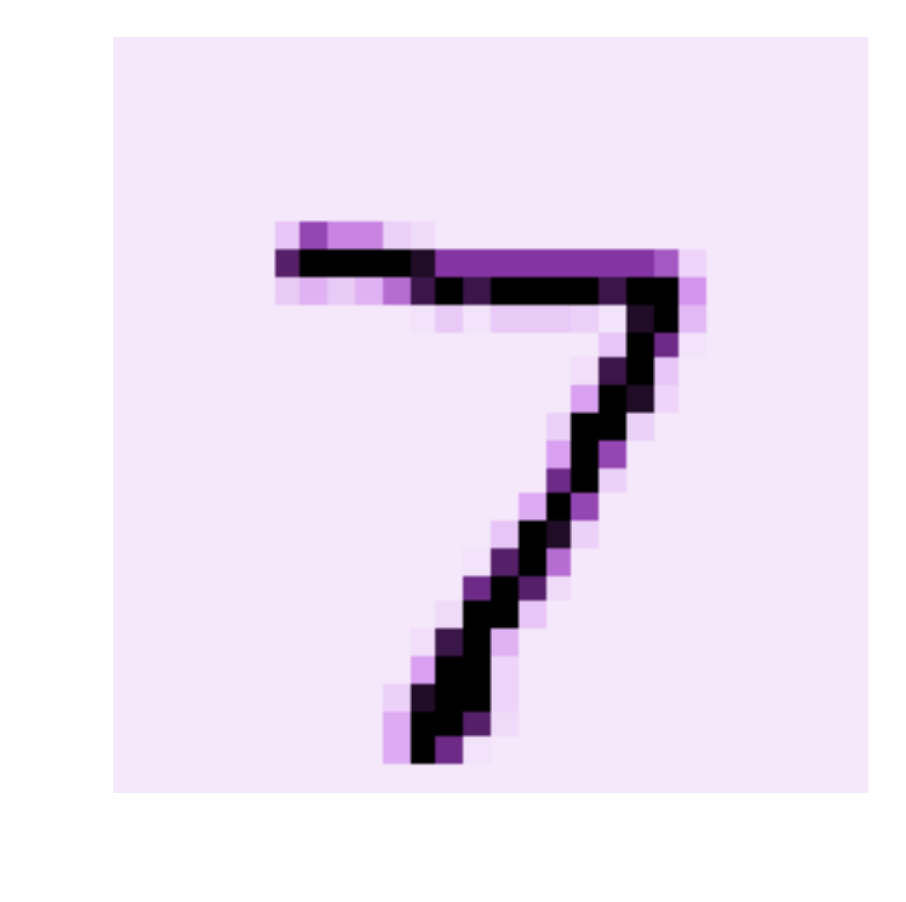}\\
        {\small Input}
    \end{tabular}
    }
    \hspace{\panelskip}
    \subfigure{
    \begin{tabular}{c}
        \includegraphics[width=\panelwidth]{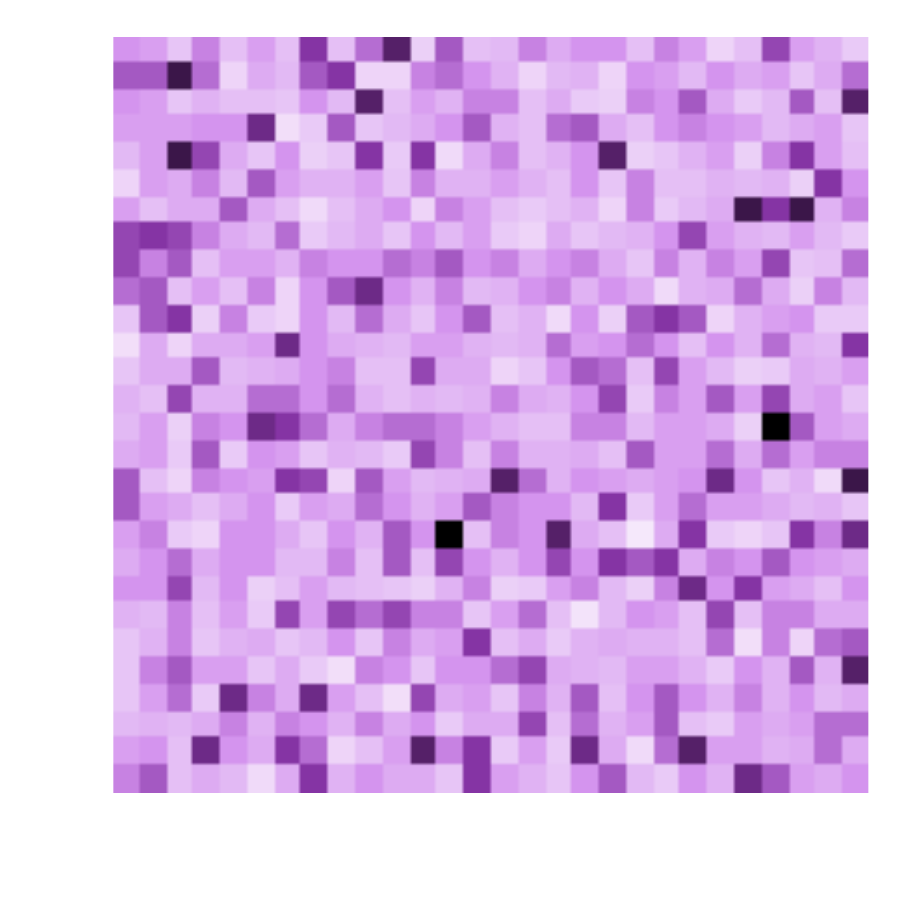}\\
        \includegraphics[width=\panelwidth]{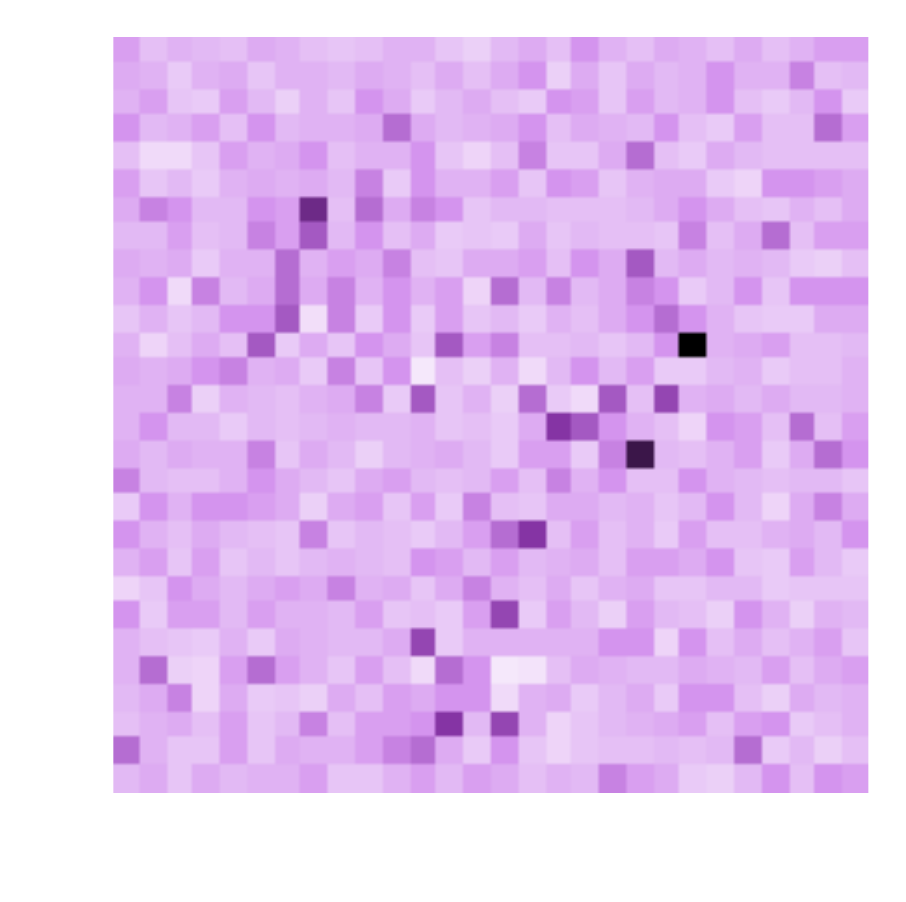}\\
        {\small Latent}
    \end{tabular}
    }
    \hspace{\panelskip}
    \subfigure{
    \begin{tabular}{c}
        \includegraphics[width=\panelwidth]{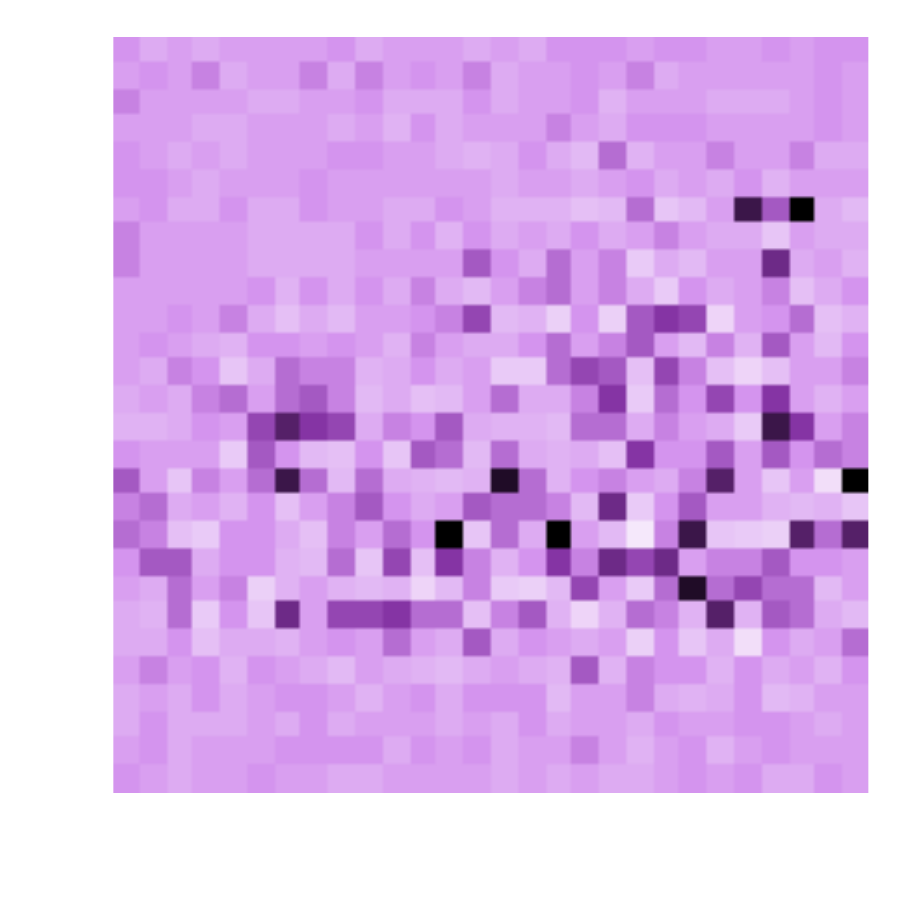}\\
        \includegraphics[width=\panelwidth]{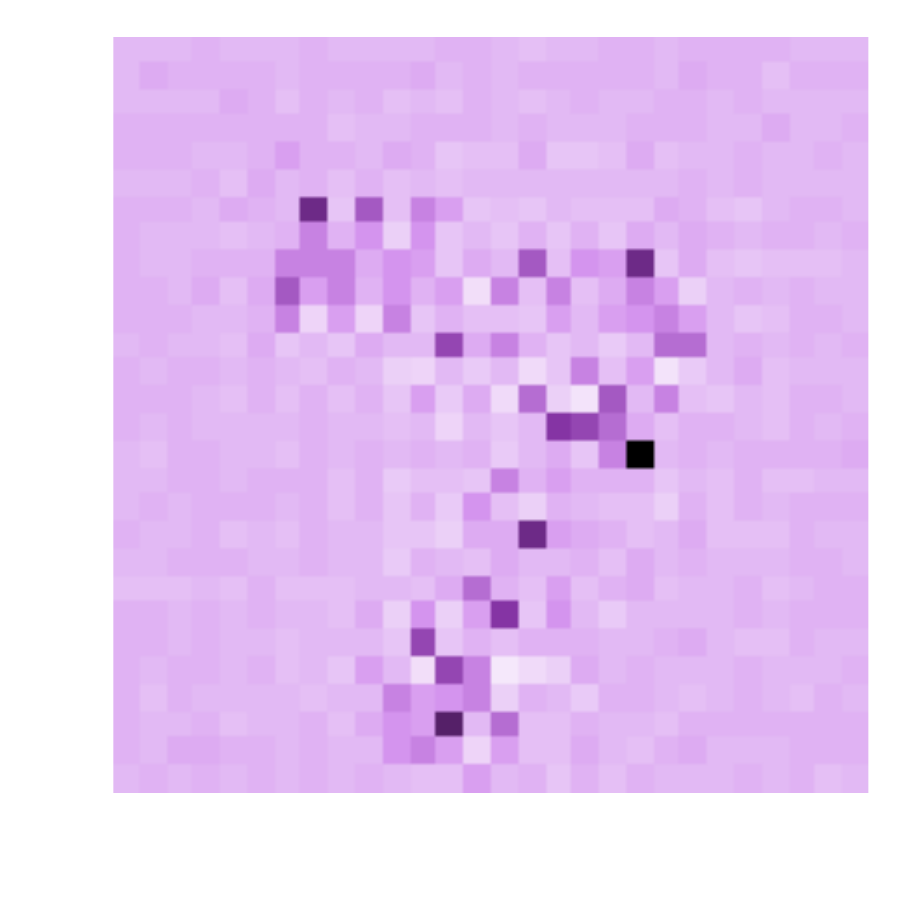}\\
        {\small Avg Latent}
    \end{tabular}
    }
    \hspace{\panelskip}
    \subfigure{
    \begin{tabular}{c}
        \includegraphics[width=\panelwidth]{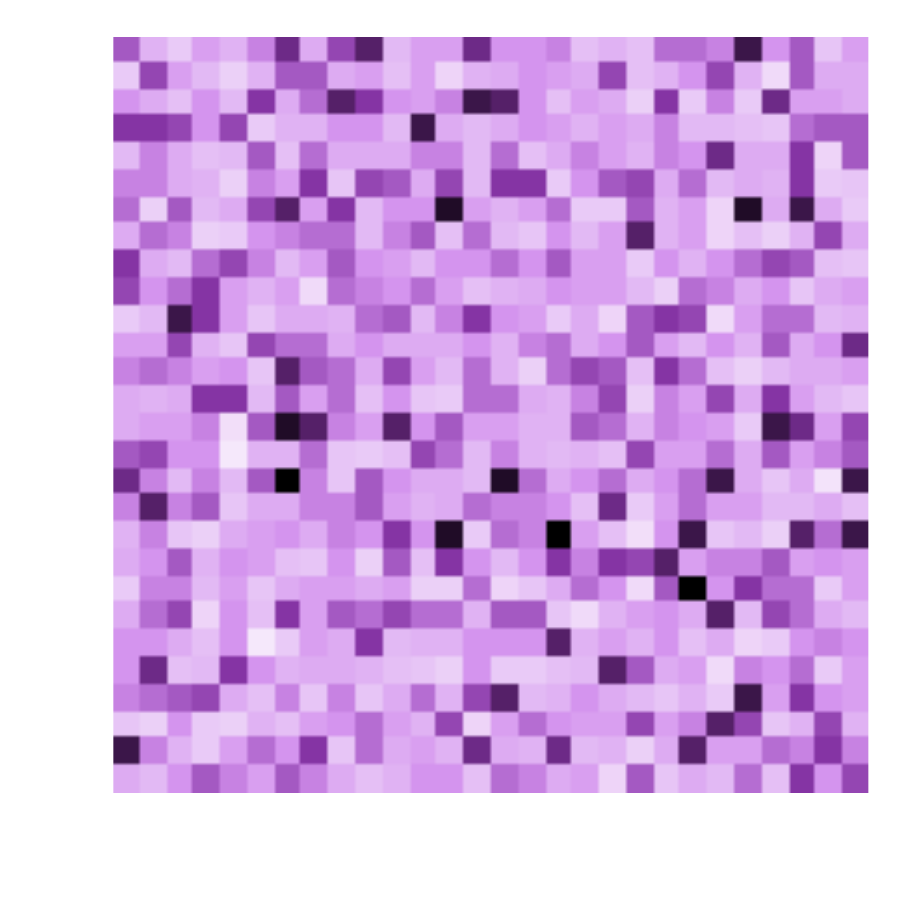}\\
        \includegraphics[width=\panelwidth]{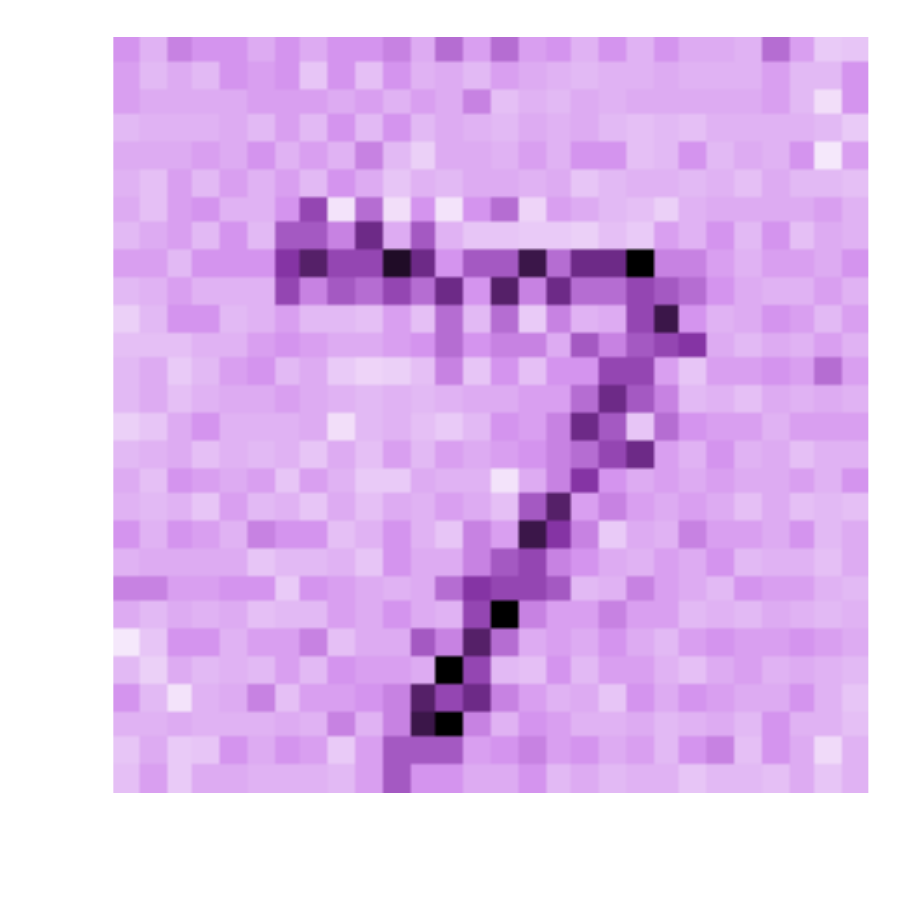}\\
        {\small BN Train}
    \end{tabular}
    }
    \hspace{0.2cm}
    \subfigure{
    \begin{tabular}{c}
        \includegraphics[width=\panelwidth]{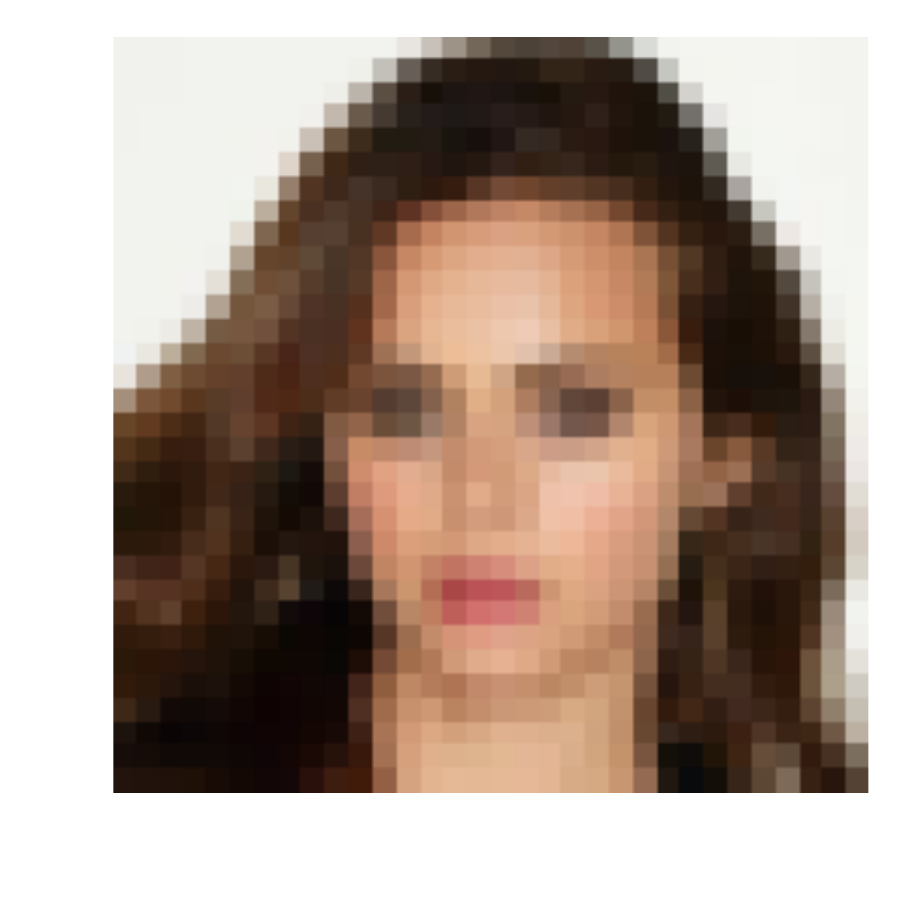}\\
        \includegraphics[width=\panelwidth]{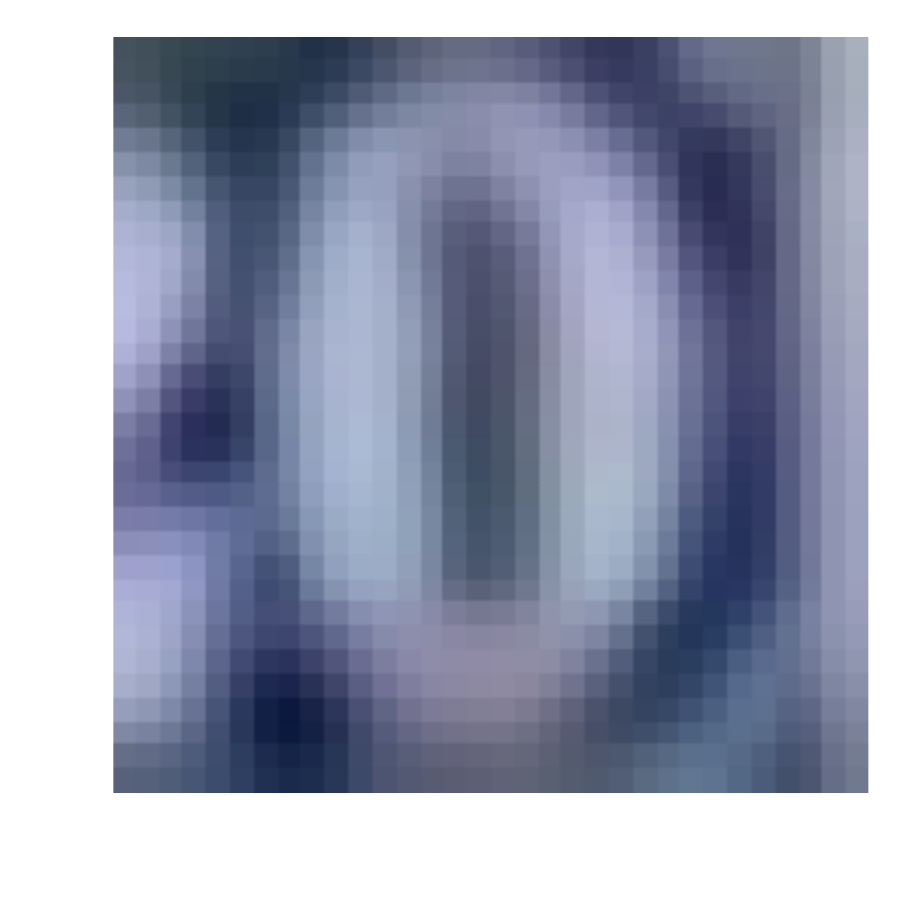}\\
        {\small Input}
    \end{tabular}
    }
    \hspace{\panelskip}
    \subfigure{
    \begin{tabular}{c}
        \includegraphics[width=\panelwidth]{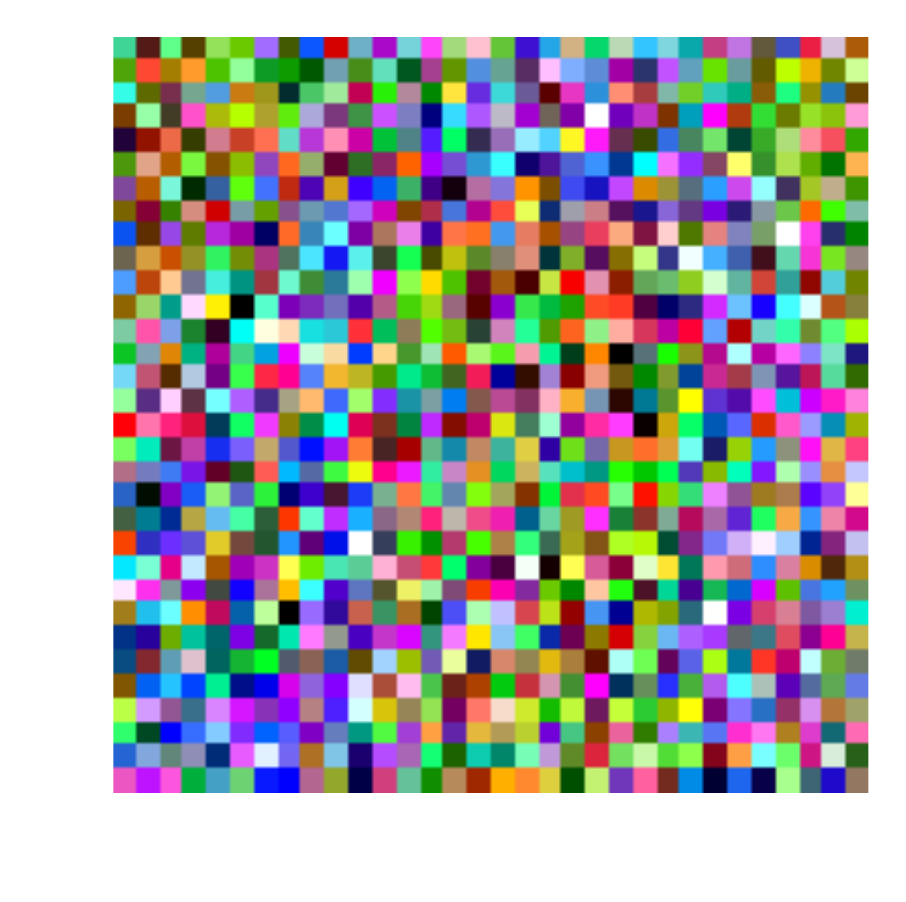}\\
        \includegraphics[width=\panelwidth]{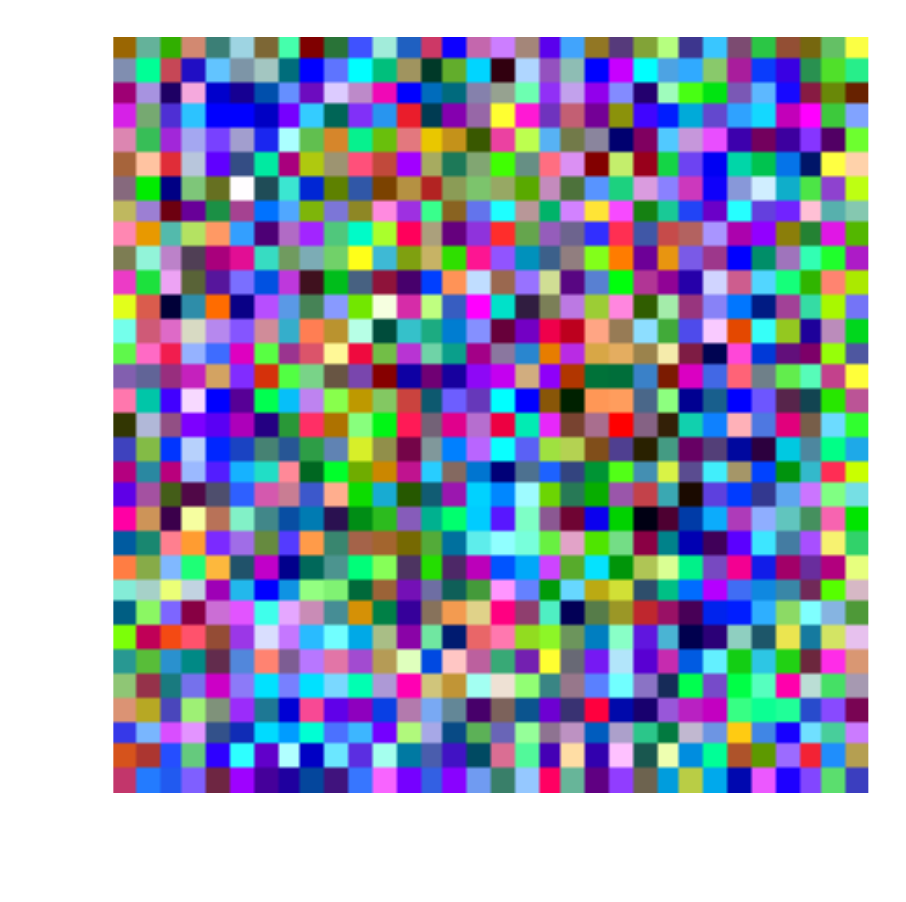}\\
        {\small Latent}
    \end{tabular}
    }
    \hspace{\panelskip}
    \subfigure{
    \begin{tabular}{c}
        \includegraphics[width=\panelwidth]{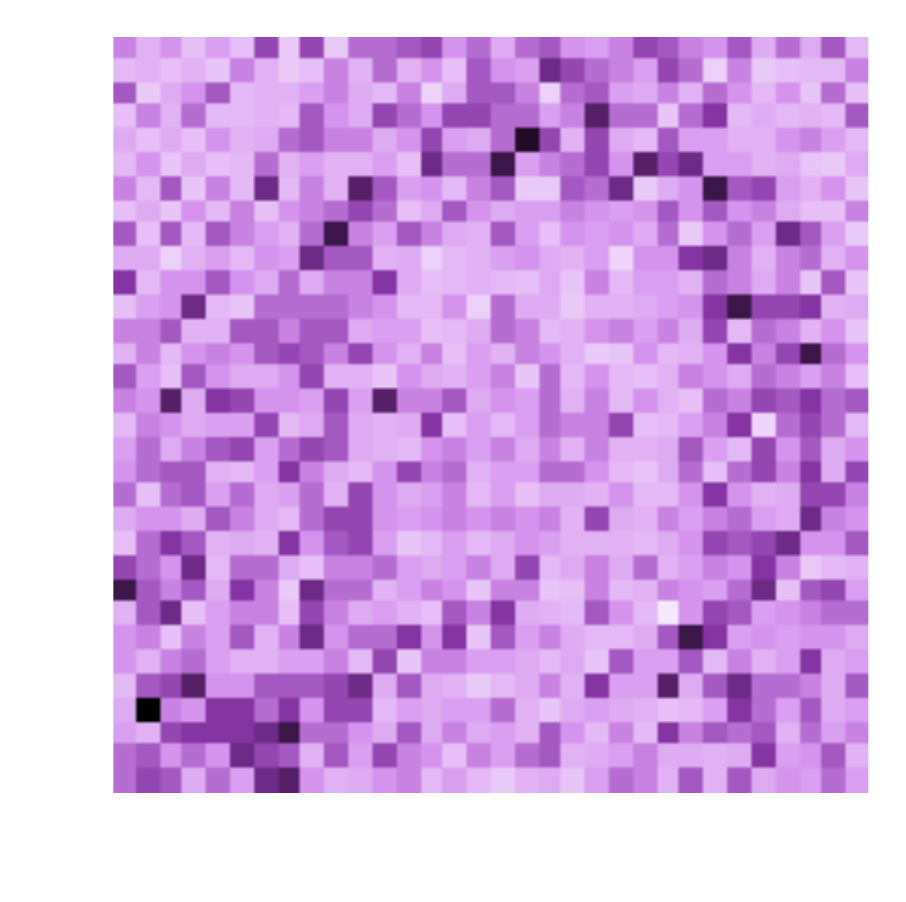}\\
        \includegraphics[width=\panelwidth]{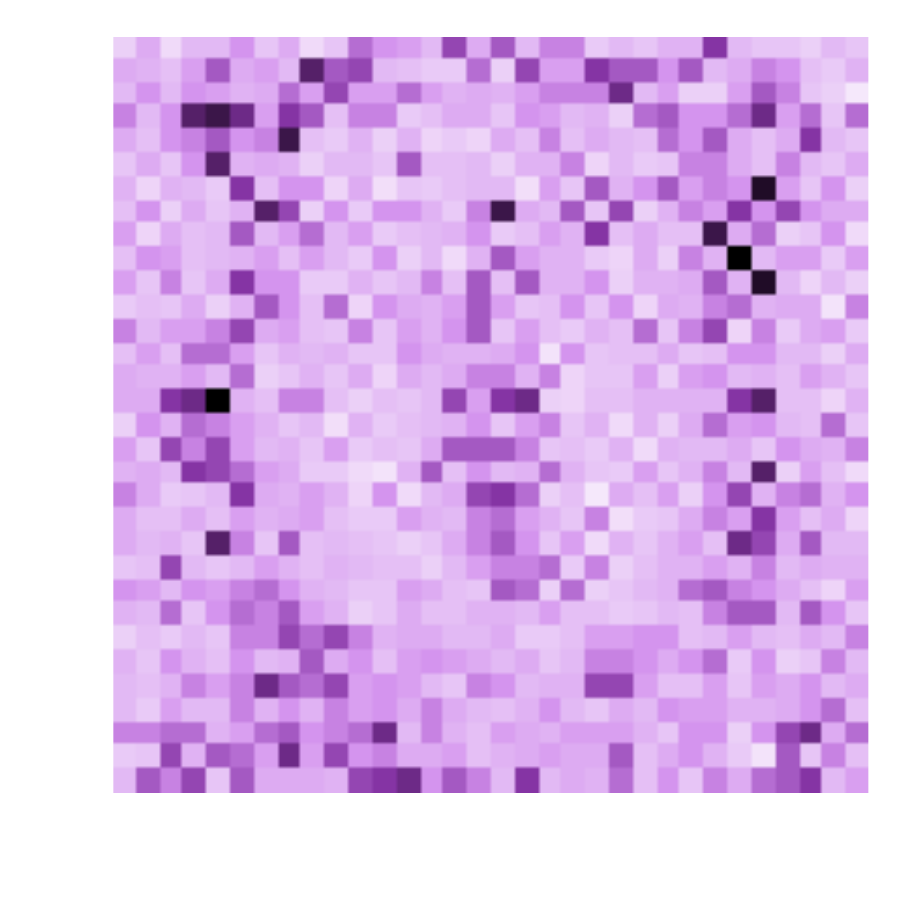}\\
        {\small Latent Blue}
    \end{tabular}
    }
    \hspace{\panelskip}
    \subfigure{
    \begin{tabular}{c}
        \includegraphics[width=\panelwidth]{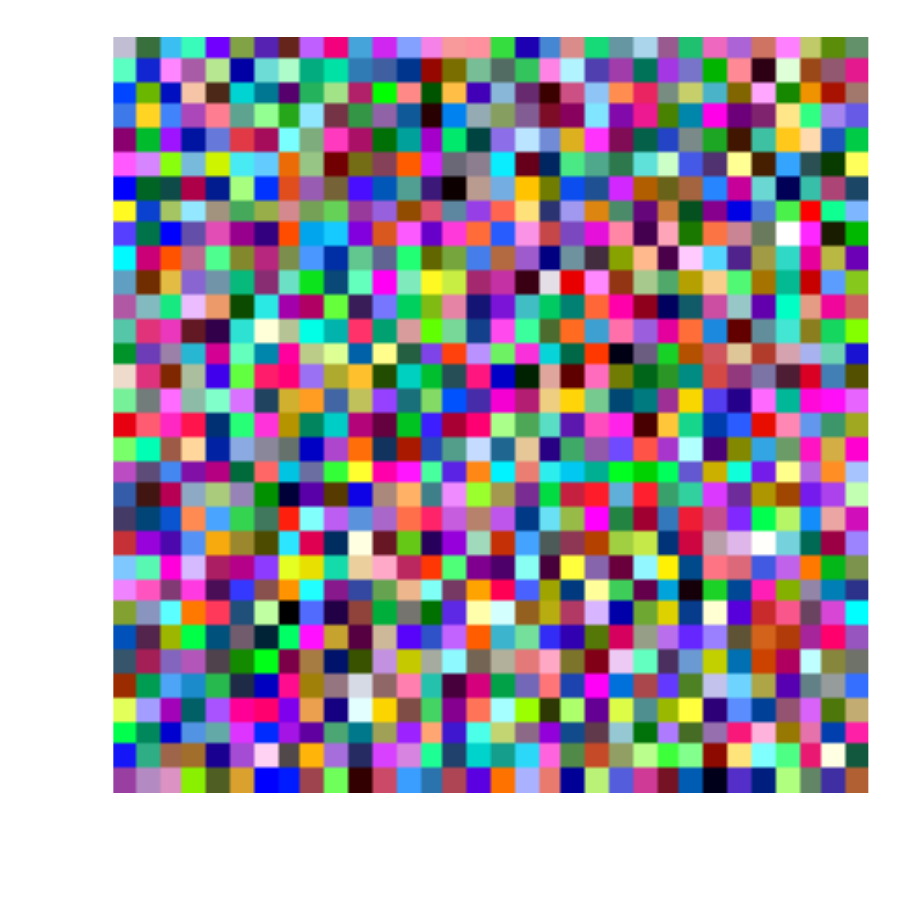}\\
        \includegraphics[width=\panelwidth]{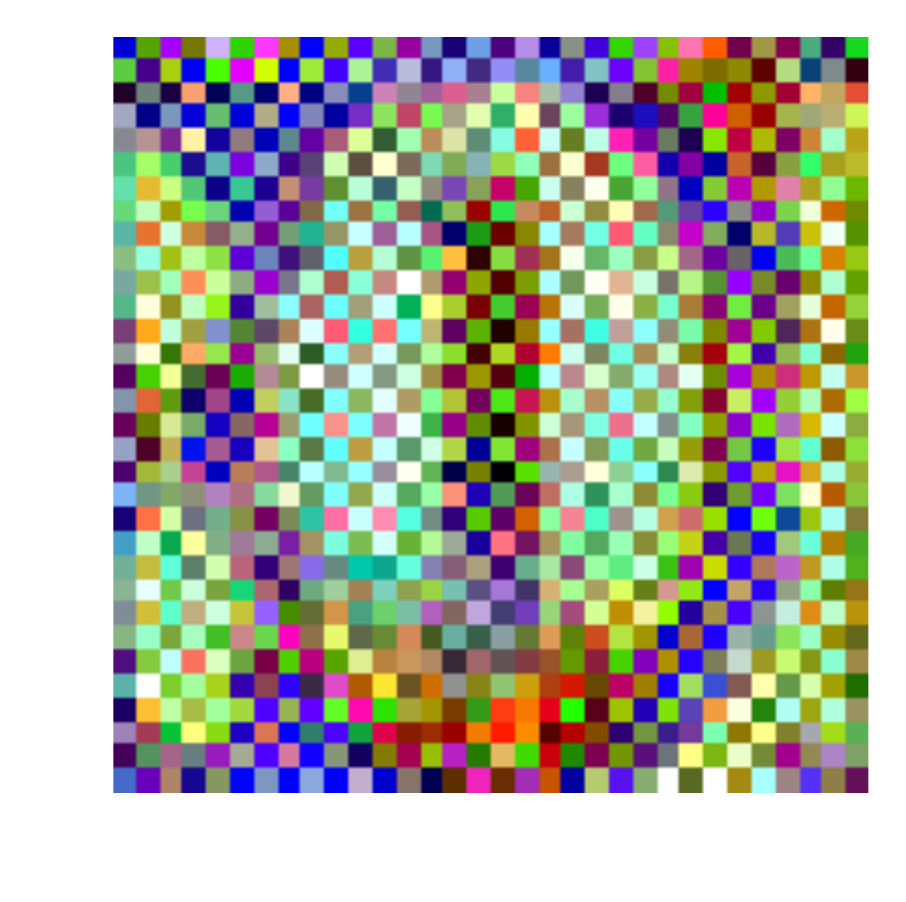}\\
        {\small BN Train}
    \end{tabular}
    }
    
	\caption{
	\textbf{Latent spaces.}
	Visualization of latent representations for the RealNVP model on
	in-distribution and out-of-distribution inputs. Panels 1-4: original images, latent representations, 
	latent representation averaged over $40$ samples of dequantization noise, and latent representations for batch normalization in train mode for a flow trained on FashionMNIST and using MNIST for OOD data.
	Panels 5-8: same as 1-4 but for a model trained on CelebA with SVHN for OOD, except in panel 7 we show the blue channel of the latent representation from panel 6 instead of an averaged latent representation.
	For both dataset pairs, we can recognize the shape of the input image in the latent representations.
	The flow represents images based on their graphical appearance rather than semantic content.
		\vspace{-2mm}
	}
	\label{fig:latent_reprs}
\end{figure*}

\subsection{Leveraging local pixel correlations}

\begin{figure}[t]
    
	\def \panelwidth {0.21\textwidth}
	\def \panelskip {-0.75cm}

    \centering
    \subfigure[Checkerboard]{
    	\includegraphics[width=\panelwidth]{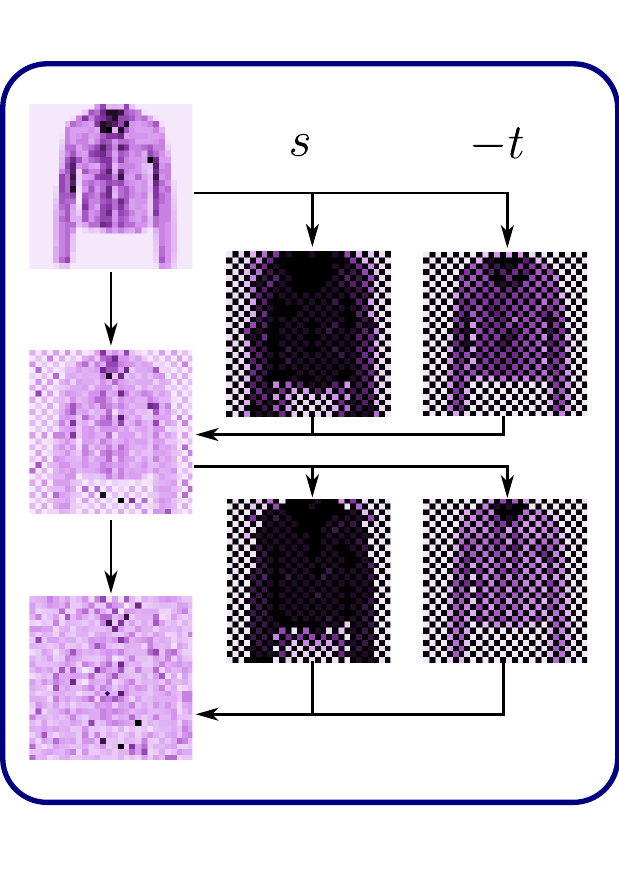} 
    }~~~
    \subfigure[Checkerboard, OOD]{
    	\includegraphics[width=\panelwidth]{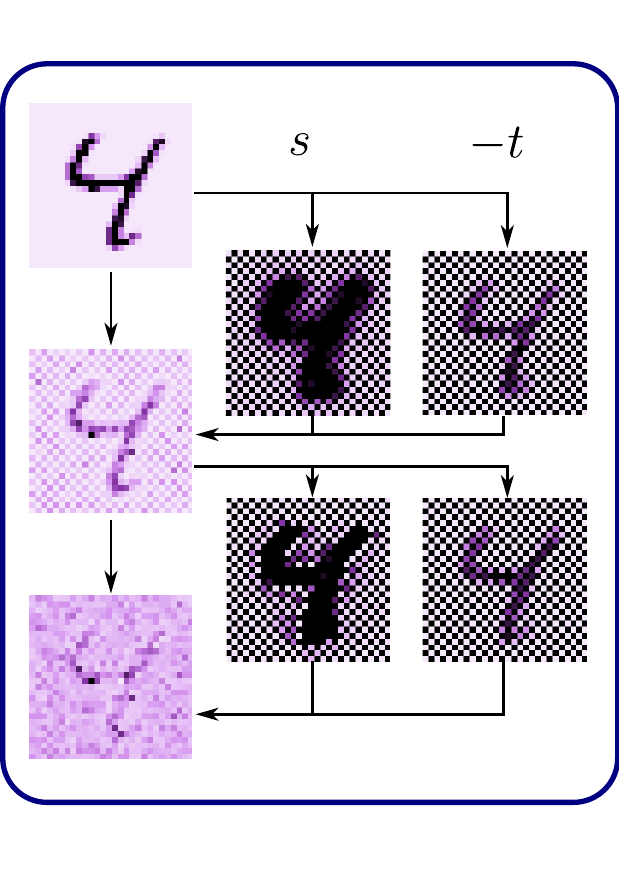} 
    }~~~
    \subfigure[Horizontal]{
    	\includegraphics[width=\panelwidth]{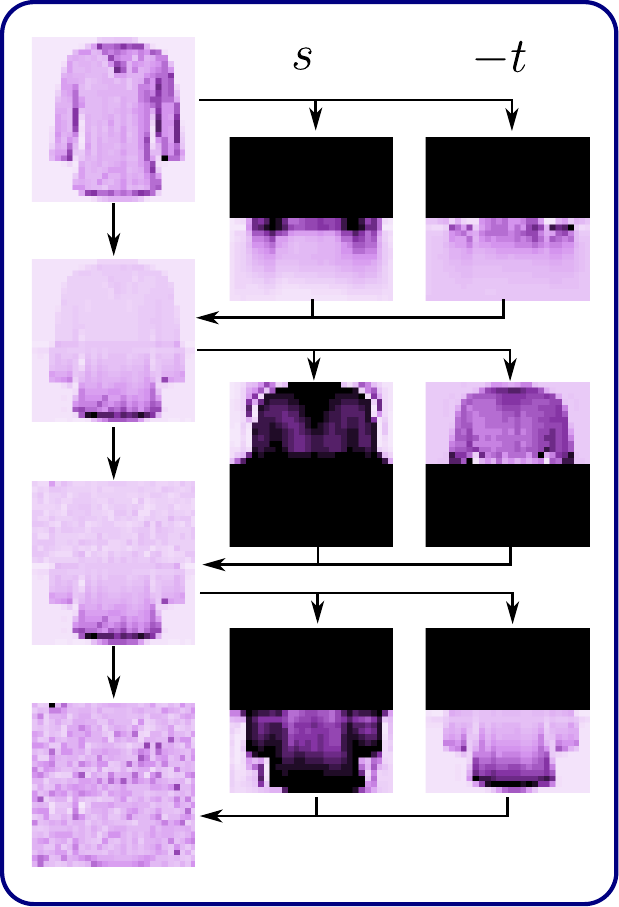} 
    }~~~
    \subfigure[Horizontal, OOD]{
    	\includegraphics[width=\panelwidth]{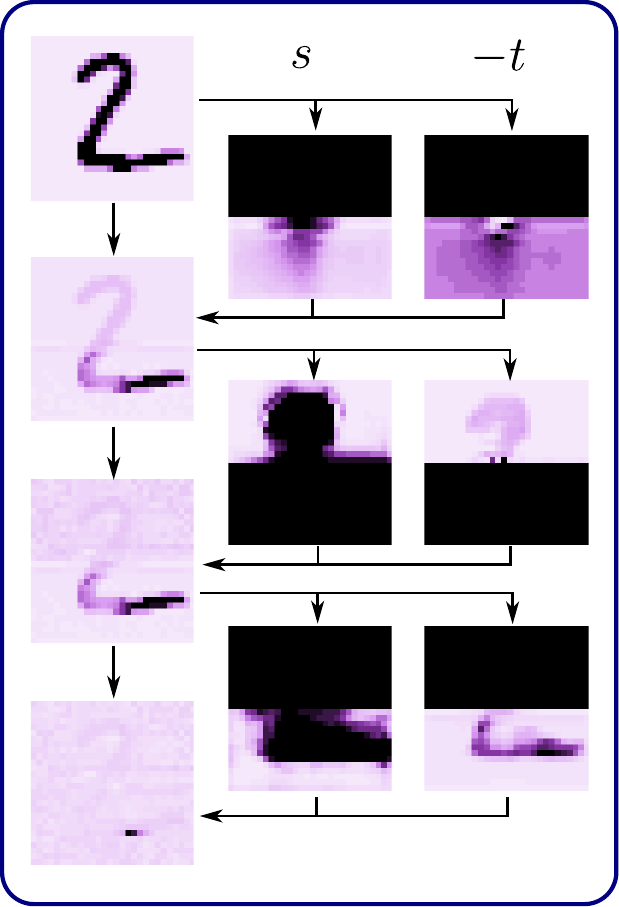} 
    }
	\caption{ 
	\textbf{Coupling layers.}
	Visualization of RealNVP's intermediate coupling layer activations, as well as scales $s$ and shifts $t$ 
	predicted by each coupling layer on in-distribution (panels a, c) and out-of-distribution inputs (panels b, d).
	RealNVP was trained on FashionMNIST.
	(a), (b): RealNVP with a standard checkerboard masks.
	The $st$-networks are able to predict the masked pixels well both on in-distribution and
	OOD inputs from neighbouring pixels.
	(c), (d): RealNVP with a horizontal mask. 
	Despite being trained on FashionMNIST, the $st$-networks are able to correctly
	predict the bottom half of MNIST digits in the second coupling layer due 
	to coupling layer co-adaptation.
	}
	\label{fig:coupling_layer_activation}
    \vspace{-5mm}
\end{figure}

In Figure \ref{fig:coupling_layer_activation}(a, b), we visualize intermediate coupling layer 
activations of a small RealNVP model with 2 coupling layers and checkerboard masks trained on FashionMNIST. 
For the masked inputs, the outputs of the $st$-network are shown in black.
Even though the flow was trained on FashionMNIST and has never seen an MNIST digit, 
the $st$-networks can easily predict masked from observed pixels on both FashionMNIST and MNIST.
Figure \ref{fig:intro} shows the same behaviour in the first coupling layers of RealNVP trained on ImageNet.

With the checkerboard mask, the $st$-networks predict the masked pixels from neighbouring pixels (see Appendix \ref{sec:app_masks} for a discussion of different masks).
Natural images have local structure and correlations: with a high probability, a particular pixel value will be similar to its neighbouring pixels.
The checkerboard mask creates an inductive bias for the flow to pick up on these local correlations.
In Figure \ref{fig:coupling_layer_activation}, we can see that the outputs of the $s$-network are especially 
large for the background pixels and large patches of the same color
(larger values are shown with lighter color), 
where the flow simply predicts for example that a pixel surrounded by black pixels would itself be black. 

In addition to the checkerboard mask, RealNVP and Glow also use channel-wise masks. 
These masks are applied after a squeeze layer, which puts different subsampled versions of the image in different channels.
As a result, 
the $st$-network is again trained to predict pixel values from neighbouring pixels. 
We provide additional 
visualizations for RealNVP and Glow in Appendix \ref{sec:app_coupling}.

\vspace{-2mm}
\subsection{Coupling layer co-adaptation}
\label{sec:coadaptation}
\vspace{-2mm}

To better understand the transformations learned by the coupling layers, we replaced 
the standard masks in RealNVP with a sequence of \textit{horizontal masks}
that cover one half of the image (either top or bottom). 
For example, the first coupling layer of the flow shown in panels (c, d) of Figure \ref{fig:coupling_layer_activation} 
transforms the bottom half of the image based on the top half, 
the second layer transforms the top half based on the bottom half, and so on.
In Figure \ref{fig:coupling_layer_activation}(c, d) we visualize the coupling
layers for a $3$-layer RealNVP with horizontal masks on in-distribution (FashionMNIST)
and OOD (MNIST) data.

In the first coupling layer, the shift output $-t$ of the $st$-network predicts the bottom half of the image poorly and the layer does not seem to transform the input significantly.
In the second and third layer, $-t$ presents an almost ideal reconstruction of the masked part of the image on both the in-distribution and, surprisingly, the OOD input.
It is not possible for the $st$-network that was only trained on FashionMNIST
to predict the top half of an MNIST digit based on the other half.
The resolution is that
the first layer encodes information about the top half into the bottom half of the
image; the second layer then decodes this information to accurately predict the top half.
Similarly, the third layer leverages information about the bottom half of the image encoded by the second layer.
We refer to this phenomenon as \textit{coupling layer co-adaptation}. 
Additional visualizations are in Appendix \ref{sec:app_coupling}.

Horizontal masks allow us to conveniently visualize the coupling layer co-adaptation, but we hypothesize that the same mechanism applies to standard checkerboard and channel-wise masks in combination with local color correlations.

\textbf{Insights into prior work} \quad
Prior work showed that the likelihood score is heavily affected by the input complexity \citep{serra2019input} and background statistics \citep{ren2019likelihood}; however, prior work does not explain \textit{why} flows exhibit such behavior.
Simpler images (e.g. SVHN compared to CIFAR-10) and background often contain large patches
of the same color, which makes it easy to predict masked pixels from their neighbours and to encode and decode the information via coupling layer co-adaptation.

\vspace{-2mm}
\section{Changing biases in flows for better OOD detection}
\label{sec:changing_biases}
\vspace{-2mm}

\begin{figure}[t]
	\def \panelwidth {0.18\linewidth}
    \centering
	\subfigure[Baseline]{
		\includegraphics[width=\panelwidth]{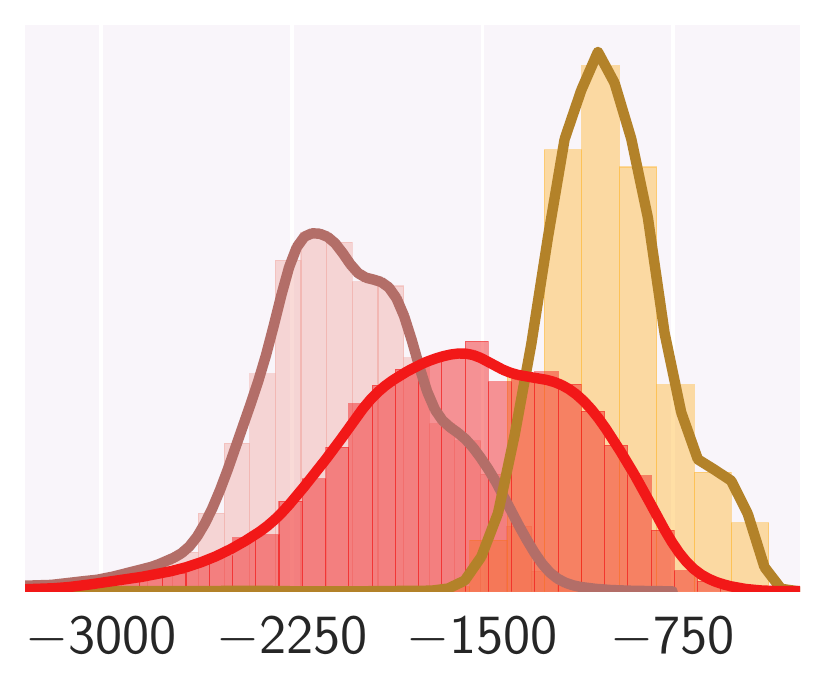}
    }	
	\subfigure[$l = 100$]{
		\includegraphics[width=\panelwidth]{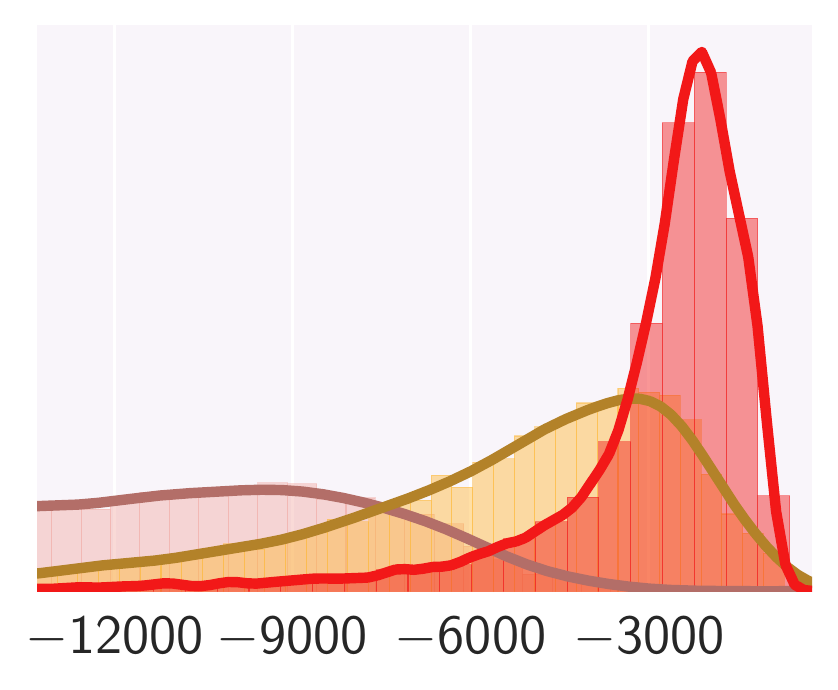}
	}
	\subfigure[$l = 50$]{
		\includegraphics[width=\panelwidth]{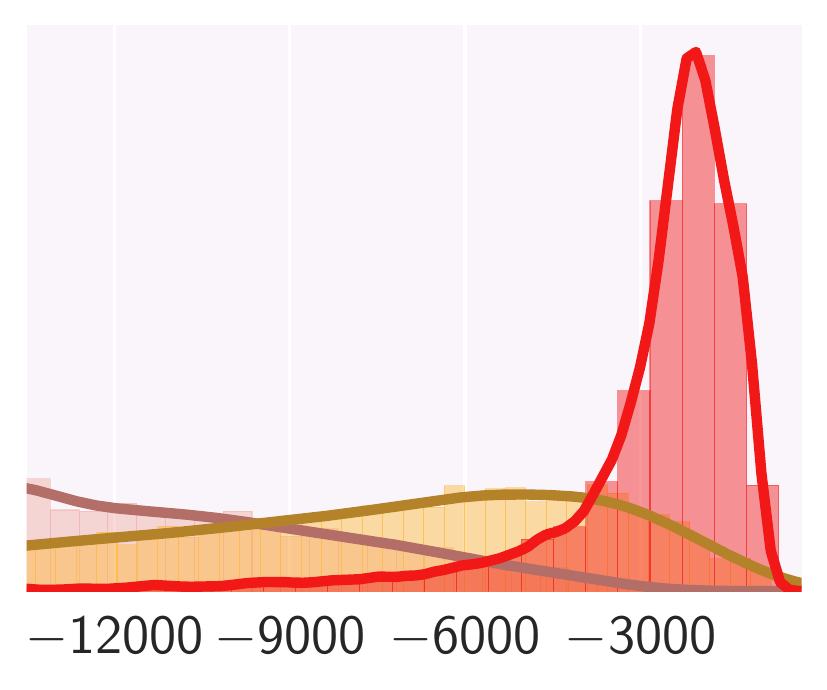}
	}
	\subfigure[$l = 10$]{
		\includegraphics[width=\panelwidth]{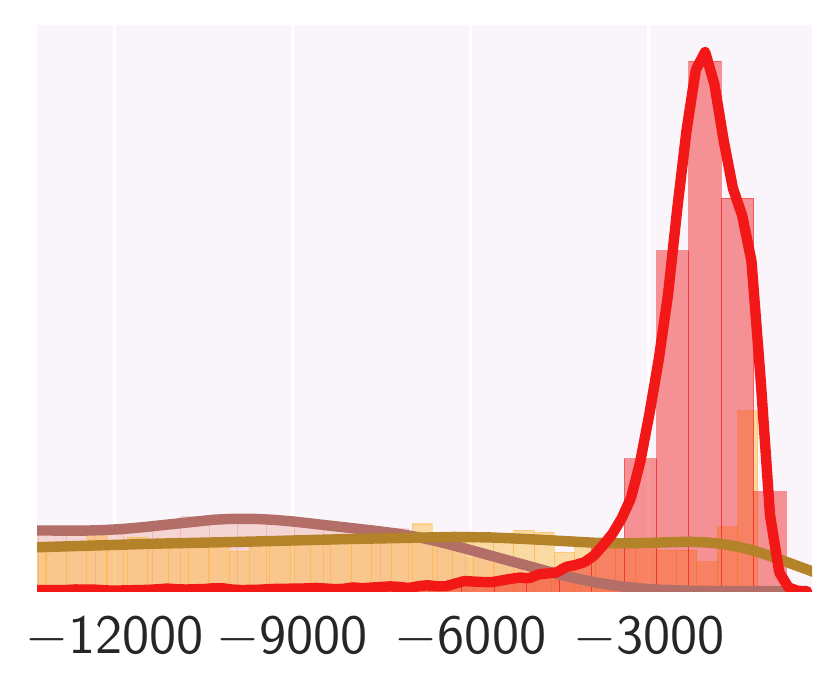}
	}
	\subfigure{
		\includegraphics[width=0.16\linewidth]{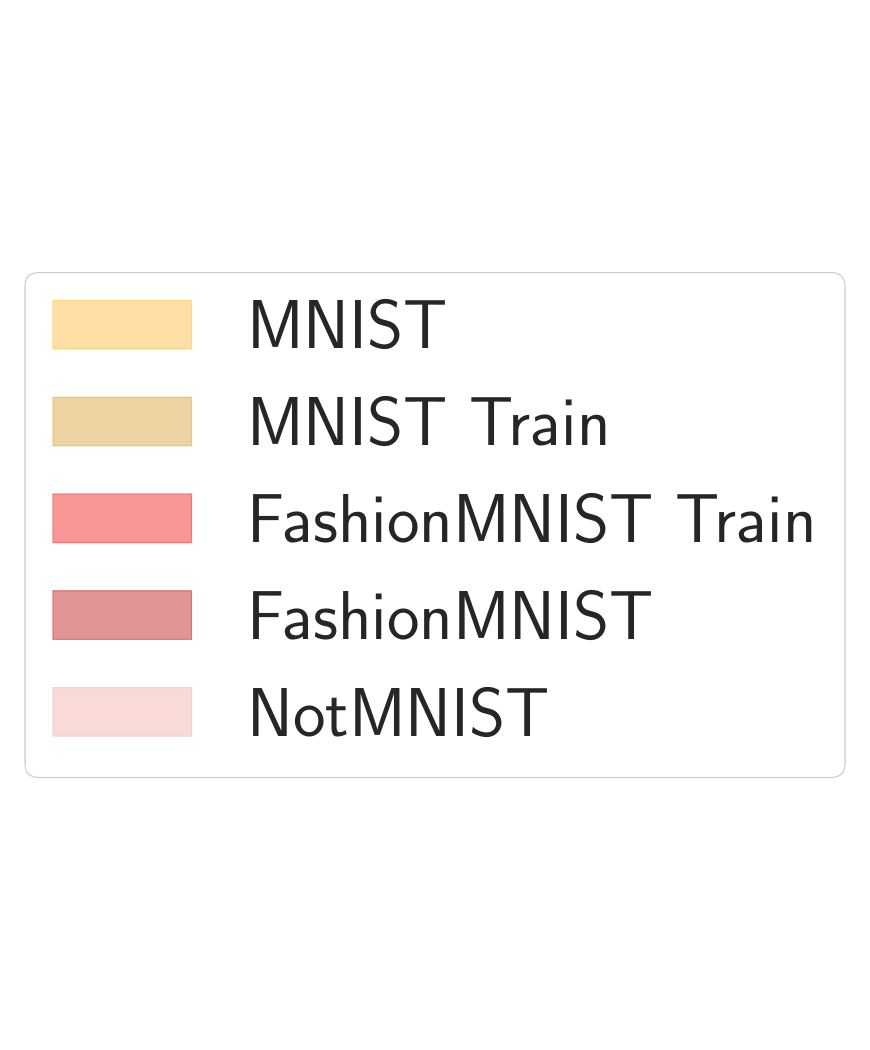}
	}
	\caption{
	\textbf{Effect of $st$-networks capacity.}
	Histograms of log-likelihoods of in- and out-of-distribution data for RealNVP trained on FashionMNIST, varying the dimension $l$ of the bottleneck in the $st$-networks.
    Flows with lower $l$ work better for OOD detection: the baseline assigns higher likelihood to the out-of-distribution MNIST images, while
    the flows with $l=50$ and $l=10$ assign much higher likelihood to in-distribution FashionMNIST data.
    With $l=100$ the flow assigns higher likelihood to in-distribution data, but the overlap
    of the likelihood distribution with OOD MNIST is higher than for $l = 50$ and $l=10$.
    }
	\label{fig:ll_hists_st_fashion}
    \vspace{-.3cm}
\end{figure}

Our observations in Sections \ref{sec:latent_space} and \ref{sec:coupling_layers} 
suggest that normalizing flows are biased towards learning transformations that increase likelihood simultaneously for all structured images. 
We discuss two simple ways of changing the inductive biases for better OOD detection.

\begin{mybox}
    By changing the masking strategy or the architecture of $st$-networks in flows we can
    improve OOD detection based on likelihood.
\end{mybox}

\textbf{Changing masking strategy}\quad
We consider two three types of masks. We introduced the horizontal mask in Section \ref{sec:coadaptation}: in each coupling layer the flow updates the bottom half of the image based on the top half or vice versa. 
With a horizontal mask, flows cannot simply use the information from neighbouring pixels when predicting a given pixel,
but they exhibit coupling layer co-adaptation (see Section \ref{sec:coadaptation}).
To combat coupling layer co-adaptation, we additionally introduce the \textit{cycle-mask},
a masking strategy where the information about a part of the image has to 
travel through three coupling layers before it can be used to update the same part of the image (details in Appendix \ref{sec:app_cyclemask}).
To compare the performance of the checkerboard mask, horizontal mask and cycle-mask, we construct flows of exactly the same size and architecture (RealNVP with 8 coupling layers and no squeeze layers) with each of these masks, trained on CelebA and FashionMNIST. 
We present the results in the Appendix \ref{sec:app_cyclemask}.
As expected, for the checkerboard mask, the flow assigns higher likelihood to the simpler OOD datasets (SVHN for CelebA and MNIST for FashionMNIST). 
With the horizontal mask, the OOD data still has higher likelihood on average, but the relative ranking of the in-distribution data is improved. 
Finally, for the cycle-mask, on FashionMNIST the likelihood is higher compared to MNIST on average. On CelebA the likelihood is similar but slightly lower compared to SVHN.

\textbf{$st$-networks with bottleneck}\quad
Another way to force the flow to learn global structure rather than local pixel correlations and to prevent coupling layer co-adaptation is to restrict the capacity of the $st$-networks.
To do so, we introduce a \textit{bottleneck} to the $st$-networks: a pair of fully-connected layers projecting to a space of dimension $l$ and back to the original input dimension. 
We insert these layers after the middle layer of the $st$-network. 
If the latent dimension $l$ is small, the $st$-network cannot simply reproduce its input as its output, and thus cannot exploit the local pixel correlations discussed in Section \ref{sec:coupling_layers}. Passing information through multiple layers with a low-dimensional bottleneck also reduces the effect of coupling layer co-adaptation.
We train a RealNVP flow 
varying the latent dimension $l$ on CelebA and on FashionMNIST. 
We present the results in Figure \ref{fig:ll_hists_st_fashion} and Appendix \ref{sec:app_biases}.
On FashionMNIST, introducing the bottleneck forces the flow to assign lower likelihood to OOD data (Figure \ref{fig:ll_hists_st_fashion}).
Furthermore, as we decrease $l$, the likelihood of the OOD data decreases but FashionMNIST likelihood stays the same.
On CelebA the relative ranking of likelihood for in-distribution data is similarly improved when we decrease the dimension $l$ of the bottleneck, but SVHN is still assigned slightly higher likelihood than CelebA.
See Appendix \ref{sec:app_biases} for detailed results.

While the proposed modifications do not completely resolve the issue of OOD 
data having higher likelihood, the experiments support our
observations in Section \ref{sec:coupling_layers}: 
preventing the flows from leveraging local color correlations and coupling layer co-adaptation, we improve the relative likelihood ranking for in-distribution data.

\section{Out-of-distribution detection using image embeddings}
\label{sec:embeddings}

\begin{figure}[t]
    \renewcommand{\arraystretch}{0.01}
    \begin{tabular}{cccccc} 
        \hspace{-0.4cm}
	    \includegraphics[height=0.15\linewidth]{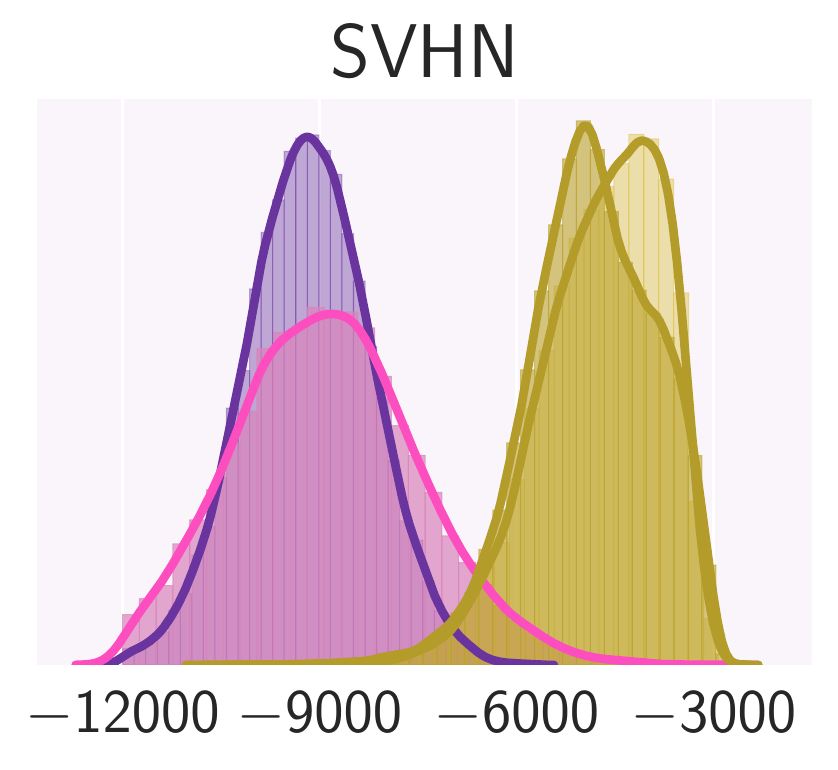}&
        \hspace{-0.4cm}
	    \includegraphics[height=0.15\linewidth]{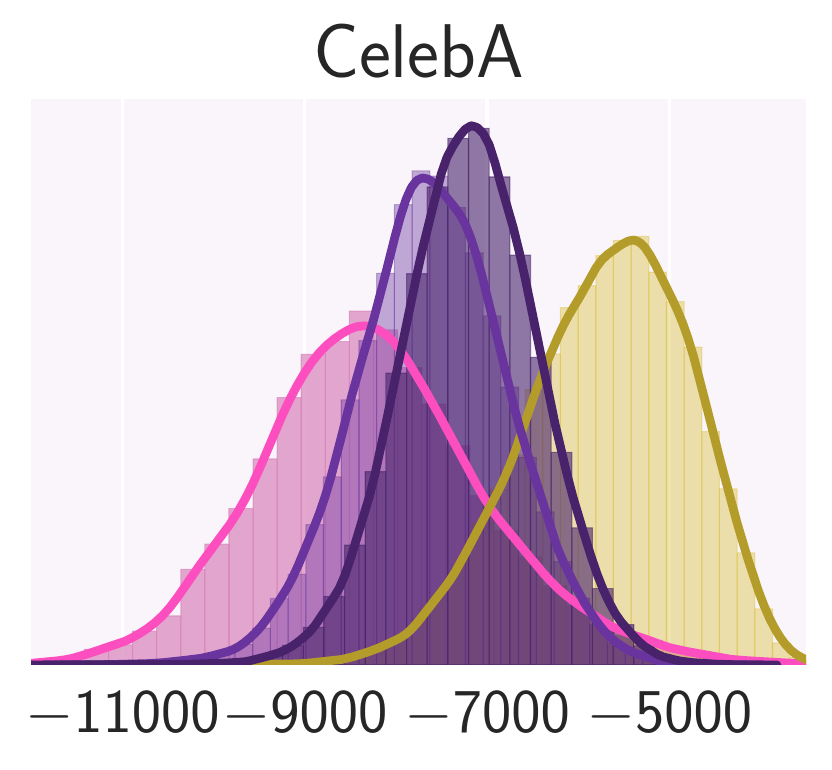}&
        \hspace{-0.4cm}
	    \includegraphics[height=0.15\linewidth]{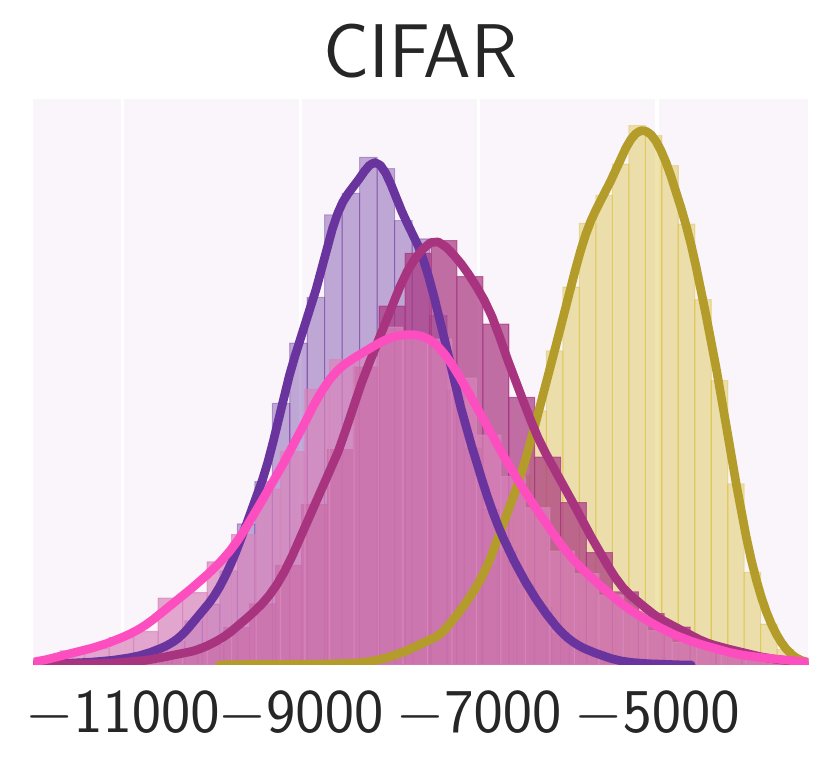} &
        \hspace{-0.4cm}
	    \includegraphics[height=0.15\linewidth]{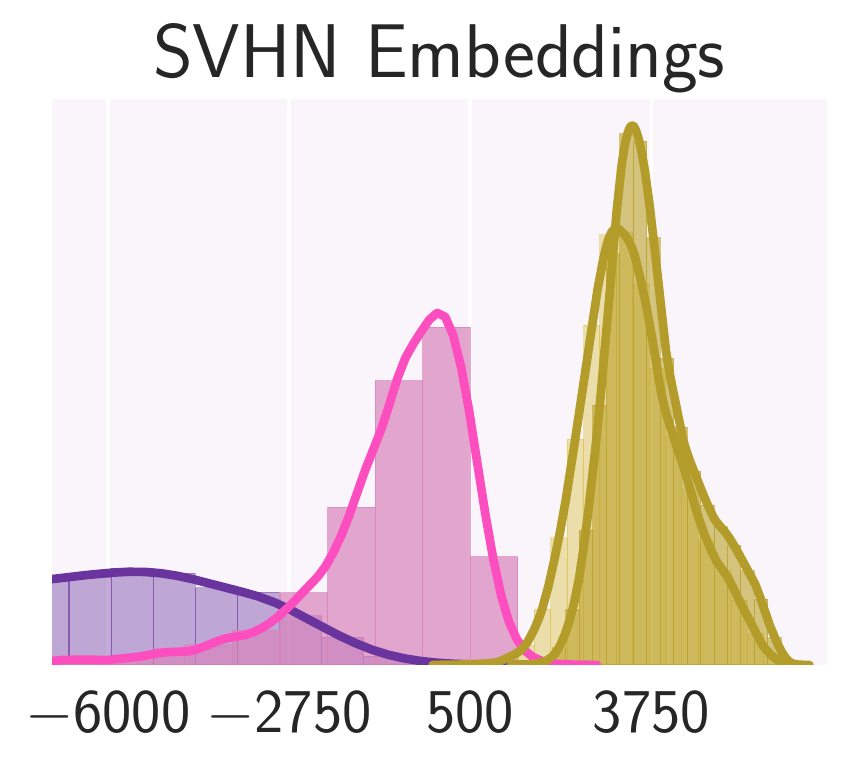}&
        \hspace{-0.4cm}
	    \includegraphics[height=0.15\linewidth]{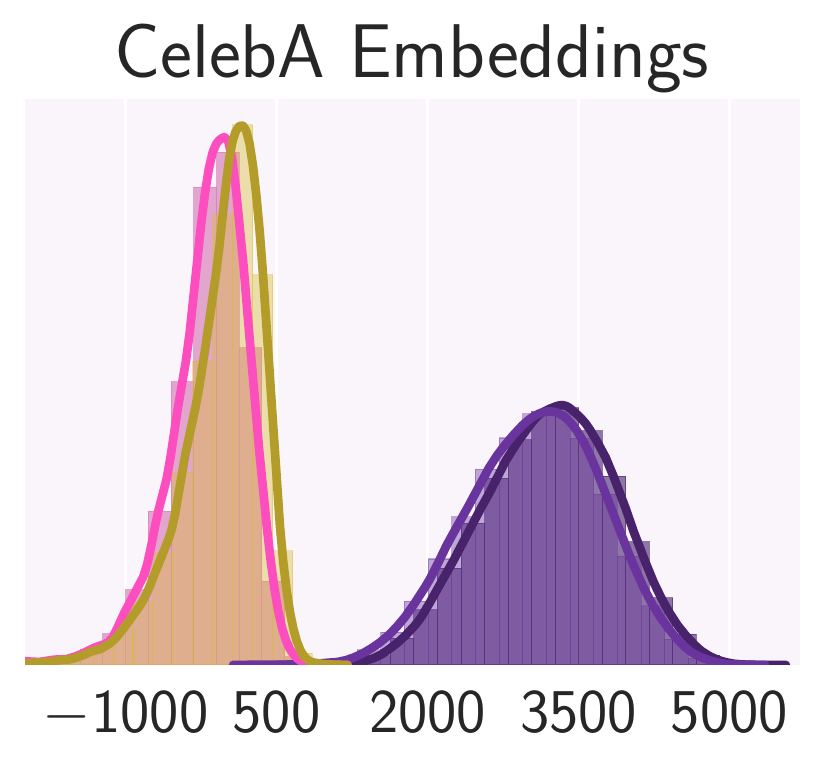}&
        \hspace{-0.4cm}
	    \includegraphics[height=0.15\linewidth]{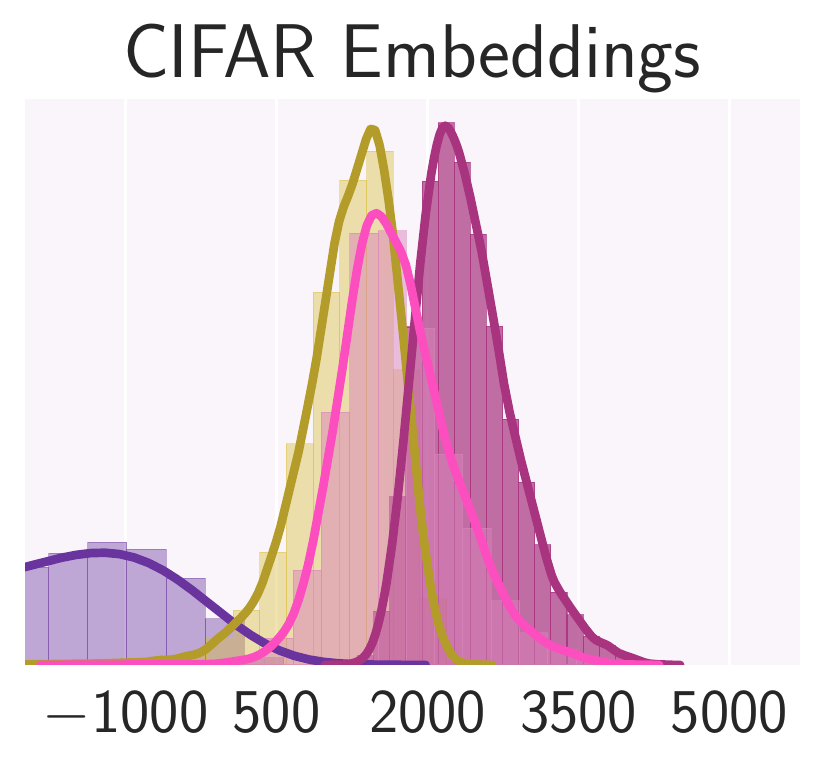} \\
        \multicolumn{6}{c}{
	    \includegraphics[height=0.032\linewidth, trim={0cm 0.9cm 0cm 0.7cm},clip]{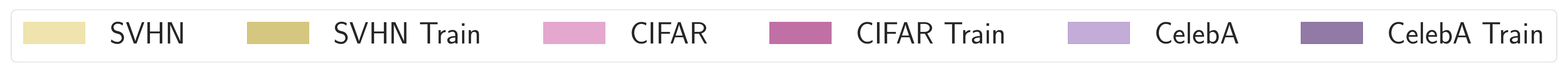}
        }
    \end{tabular}
	\caption{
	\textbf{Image embeddings.}
	Log-likelihood histograms for RealNVP trained on raw pixel data (first three panels) and embeddings extracted for the same 
    image datasets using EfficientNet trained on ImageNet.
    On raw pixels, the flow assigns the highest likelihood to SVHN regardless of its training dataset.
    On image embeddings, flows always assign higher likelihood to in-distribution data.
    When trained on features capturing the semantic content of the
    input, flows can detect OOD. 
	}
	\label{fig:ll_hists_tabular}
    \vspace{-.3cm}
\end{figure}

In Section \ref{sec:reasoning} we argued that in order to detect OOD data the model has to assign likelihood based on high-level semantic features of the data, which the flows fail to do when trained on images.
In this section, we test out-of-distribution detection using image representations from a deep neural network.

\begin{mybox}
    Normalizing flows can detect OOD images when trained on high-level semantic representations instead of raw pixels.
\end{mybox}

We extract embeddings for CIFAR-10, CelebA and SVHN using an EfficientNet \citep{tan2019efficientnet} 
pretrained on ImageNet \citep{russakovsky2015imagenet} which yields 1792-dimensional features\footnote{
The original images are $3072$-dimensional, so the dimension of the embeddings
is only two times smaller. 
Thus, the inability to detect OOD images \textit{cannot} be explained just by the high dimensionality of the data.
}. 
We train RealNVP on each of the representation datasets, considering the other two datasets as OOD. 
We present the likelihood histograms for all datasets in Figure \ref{fig:ll_hists_tabular}(b). 
Additionally, we report AUROC scores in Appendix Table \ref{tab:efficientnet_auroc}. 
For the models trained on SVHN and CelebA, both OOD datasets have lower likelihood and the AUROC scores are close to 100\%. 
For the model trained on CIFAR-10, CelebA has lower likelihood. Moreover, the likelihood distribution on SVHN, while significantly overlapping with CIFAR-10, still has a lower average: the AUROC score between CIFAR-10 and SVHN is 73\%.
Flows are much better at OOD detection on image embeddings than on the original image datasets. 
For example, a flow trained on CelebA images assigns higher 
likelihood to SVHN, while a flow trained on CelebA embeddings assigns low likelihood 
to SVHN embeddings (see Appendix \ref{appendix:baseline_likelihood_auroc} for likelihood 
distribution and AUROC scores on image data). 

\textbf{Non-image data}\quad
In Appendix \ref{appendix:tabular} we evaluate flows on tabular UCI datasets, where the features are relatively high-level compared to images.
On these datasets, normalizing flows assign higher likelihood to in-distribution data.

\section{Conclusion}

Many of the puzzling phenomena in deep learning can be boiled down to a matter of \emph{inductive biases}. Neural networks in many cases have the flexibility
to overfit datasets, but they do not because the biases of the architecture and training procedures can guide us towards reasonable solutions. In performing OOD
detection, the biases of normalizing flows can be more of a curse than a blessing. Indeed, we have shown that flows tend to learn representations that achieve
high likelihood through generic graphical features and local pixel correlations, rather than discovering semantic structure that would be specific to the training 
distribution.

To provide insights into prior results \citep[e.g.,][]{nalisnick2018deep, choi2018waic, nalisnick2019detecting, song2019unsupervised, zhang2020out, serra2019input}, part of 
our discussion has focused on an in-depth exploration of the popular class of normalizing flows based on affine coupling layers. We hypothesize 
that many of our conclusions about coupling layers extend at a high level to other types of normalizing flows \citep[e.g.,][]{behrmann2018invertible, chen2019residual, finzi2019invertible, karami2019icf, grathwohl2018ffjord, papamakarios2017masked, song2019mintnet, huang2018neural, de2019block}. A full study of these other types of flows is a promising
direction for future work.

\subsection*{Acknowledgements}
PK, PI, and AGW are supported by an Amazon Research  Award,  Amazon Machine Learning Research Award, Facebook  Research, NSF I-DISRE 193471, NIH R01 DA048764-01A1, NSF IIS-1910266, and NSF 1922658 NRT-HDR: FUTURE Foundations, Translation, and Responsibility for Data Science. We thank Marc Finzi, Greg Benton, Wesley Maddox, and Alex Wang for helpful discussions.

\bibliography{main}
\bibliographystyle{plainnat}

\newpage\null

\appendix

\section*{Appendix outline}

This appendix is organized as follows.
\begin{itemize}
    \item In Section \ref{sec:app_whatisood}, we provide additional discussion and a formal statement of the argument presented in Section \ref{sec:reasoning}.
    \item In Section \ref{sec:app_capacity}, we show that normalizing flows can be trained to assign high likelihood to the target data and low likelihood to a given OOD dataset.
    \item In Section \ref{sec:app_details}, we provide the hyperparameters that we used for the experiments in this paper.
    \item In Section \ref{appendix:baseline_likelihood_auroc}, we report the log-likelihood histograms and OOD detection AUROC scores for the baseline RealNVP and Glow models on various datasets.
    \item In Section \ref{sec:app_visualizations}, we explain the visualization procedure that we use to visualize the latent representations and coupling layers of normalizing flows.
    \item In Section \ref{sec:app_latents}, we provide additional latent representation visualizations.
    \item In Section \ref{sec:app_masks}, we explain the different masking strategies for coupling layers of normalizing flows.
    \item In Section \ref{sec:app_coupling}, we provide additional coupling layer visualizations.
    \item In Section \ref{sec:app_biases}, we provide additional details on the experiments of Section \ref{sec:changing_biases}.
    \item In Section \ref{sec:app_samples}, we provide samples from baseline models on various datasets. 
    We also discuss an experiment on resampling parts of the latent variables corresponding to different images with normalizing flows.
    \item In Section \ref{appendix:tabular}, we provide additional details and results for the experiments on image embeddings and tabular data from Section \ref{sec:embeddings}.
\end{itemize}

\section{Maximum likelihood objective is agnostic to what data is OOD}
\label{sec:app_whatisood}

\begin{figure}[t]
    \centering
	\includegraphics[width=0.85\linewidth]{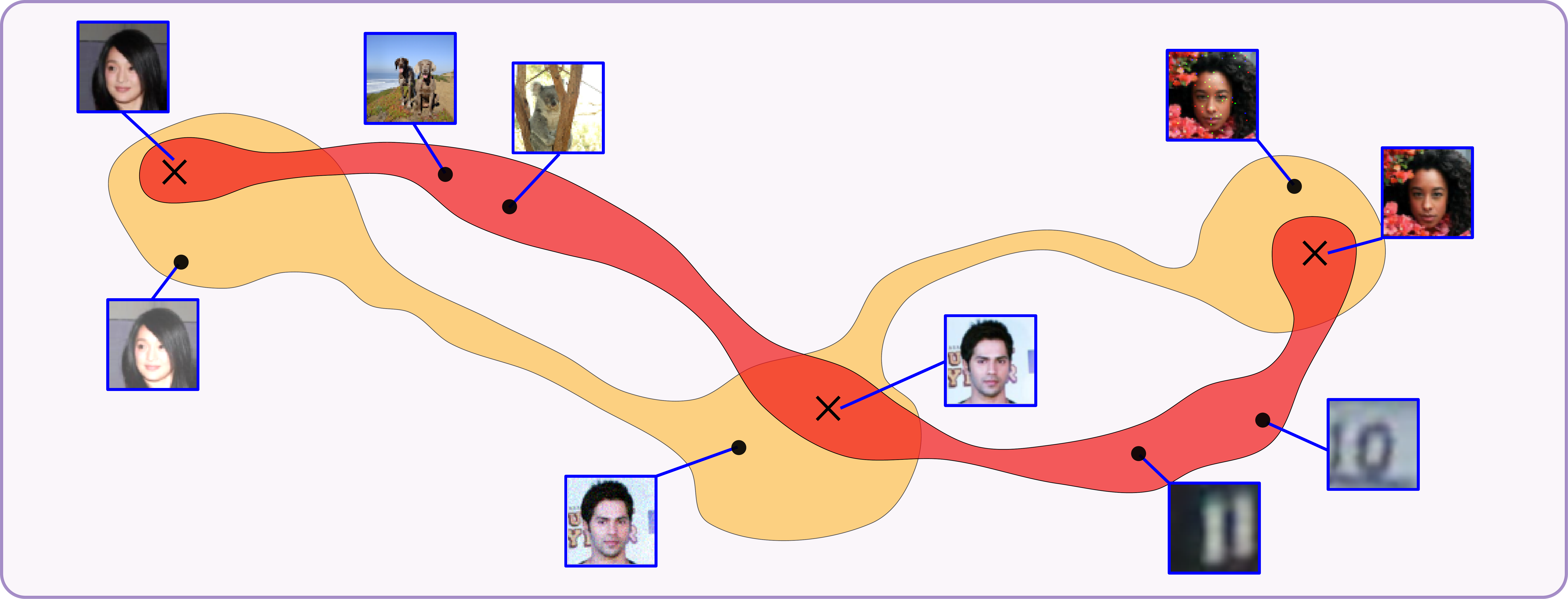}
	\caption{
	\textbf{Inductive biases define what data is OOD.}
    A conceptual visualization of two distributions in the image space (shown in yellow and red),
    training CelebA data is shown with crosses, and other images are shown with circles.
    The distribution shown in yellow could represent inductive biases of a human:
    it assigns high likelihood to all images of human faces,
    regardless of small levels of noise, and small brightness changes.
    The second distribution, shown in red, could represent a normalizing flow:
    it assigns high likelihood to all smooth structured images, including images from SVHN and ImageNet.
    Both distributions assign the same likelihood to the training set, but 
    their high-probability sets are different.
	}
	\label{fig:app_conceptual}
    \vspace{-.3cm}
\end{figure}

In Section \ref{sec:reasoning} we argued that the maximum likelihood objective by itself
does not define out-of-distribution detection prefromance of a normalizing flow.
Instead, it is the inductive biases of the flow that define what data will be 
assigned with high or low likelihood.
We illustrate this point in Figure \ref{fig:app_conceptual}.

The yellow and red shaded regions illustrate the high-probability regions of two
distributions defined on the image space $\mathcal X$.
The distribution in yellow assigns high likelihood to the train (CelebA)
images corrupted by a small level of noise, or brightness adjustments.
This distribution represents how a human could describe the target dataset.
The red distribution on the other hand assigns high likelihood to all structured
images including those from ImageNet and SVHN, but does not support noisy train images.
The red distribution represents a distribution learned by normalizing flow.

For simplicity, we could think that the distributions are uniform on the highlighted
sets, and the sets have the same volume.
Then, both distributions assign equally high likelihood to the training data, 
but the split of the data into in-distribution and OOD is different.
As both distributions provide the same density to the target data, the value of the
maximum likelihood objective in Equation \eqref{eq:change_of_variable} would be the
same for the corresponding models.

More generally, for any distribution that only assigns finite density to the train
set, we can construct another distribution that assigns the same density to the train data,
but also high density to a given set of (OOD) datapoints.
In particular, the new distribution will achieve the same value of the maximum likeihood
objective in Equation \eqref{eq:change_of_variable}.
We formalize our reasoning in the following simple proposition.

\begin{proposition}
    Let $p(\cdot)$ be a probability density on the space $\mathcal X$, and let
    $\mathcal D = \{x_i\}_{i=1}^N$ be the training dataset, where $x_i \in \mathcal X$ for $i=1, \ldots, N$.
    Assume for simplicity that $p$ is upper bounded: for any $x$  $p(x) \le u$.
    Let $\mathcal D_{\text{OOD}}$ be an arbitraty finite set of points.
    Then, for any $c \ge 0$ there exists a distribution with density $p'(\cdot)$ such that 
    $p'(x) =  p(x)$ for all $x \in \mathcal D$, and $p'(x') \ge c$ for all
    $x' \in \mathcal D_{\text {OOD}}$.
\end{proposition}
\begin{proof}
Consider the set $\mathcal S(r) = \cup_{x_i \in \mathcal D} B(x_i, r)$, where $B(x, r)$ is a ball of radius $r$ centered at $x$.
The probability mass of this set $P(\mathcal S(r)) = \int_{x \in \mathcal S(r)} p(x) dx$.
As $r \rightarrow 0$, the volume $V(\mathcal S(r))$ of the set $\mathcal S(r)$ goes to zero.
We have
\begin{equation}
    P(\mathcal S(r)) = \int_{x \in \mathcal S(r)} p(x) dx \le V(\mathcal S(r)) \cdot u \xrightarrow{r \rightarrow 0} 0.
\end{equation}
Hence, there exists $r_0$ such that $P(\mathcal S(r_0)) \le \frac 1 2$.

Now, define the a neighborhood of the set $\mathcal D_\text{OOD}$ as
\begin{equation}
    \mathcal S_{\text{OOD}} = 
    \cup_{x' \in \mathcal D_{\text{OOD}}} B(x, \hat r),
\end{equation}
where $\hat r$ is selected so that the total volume of set $\mathcal S_{\text{OOD}}$ is $1 / 2c$.
Then, we can define a new density $p'$ by redistributing the mass in $p(\cdot)$ from
outside the set $\mathcal S(r_0)$ to the neighborhood $\mathcal S_{\text{OOD}}$ as follows:
\begin{equation}
    p'(x) =
    \begin{cases}
       p(x),\quad&\text{if}~x \in \mathcal S(r_0),\\
       2c \cdot\big(1 -  P(\mathcal S(r_0))\big),\quad&\text{if}~x \in \mathcal S_{\text{OOD}},\\
       0,\quad&\text{otherwise}.
    \end{cases}
\end{equation}
The density $p'(\cdot)$ integrates to one, coincides with $p$ on the training data,
and assigns density of at least $c$ to points in $\mathcal D_{\text{OOD}}$.
\end{proof}

\begin{figure}[t]
    \centering
    \subfigure[CIFAR $\uparrow$, SVHN $\downarrow$]{
	\includegraphics[width=0.23\linewidth]{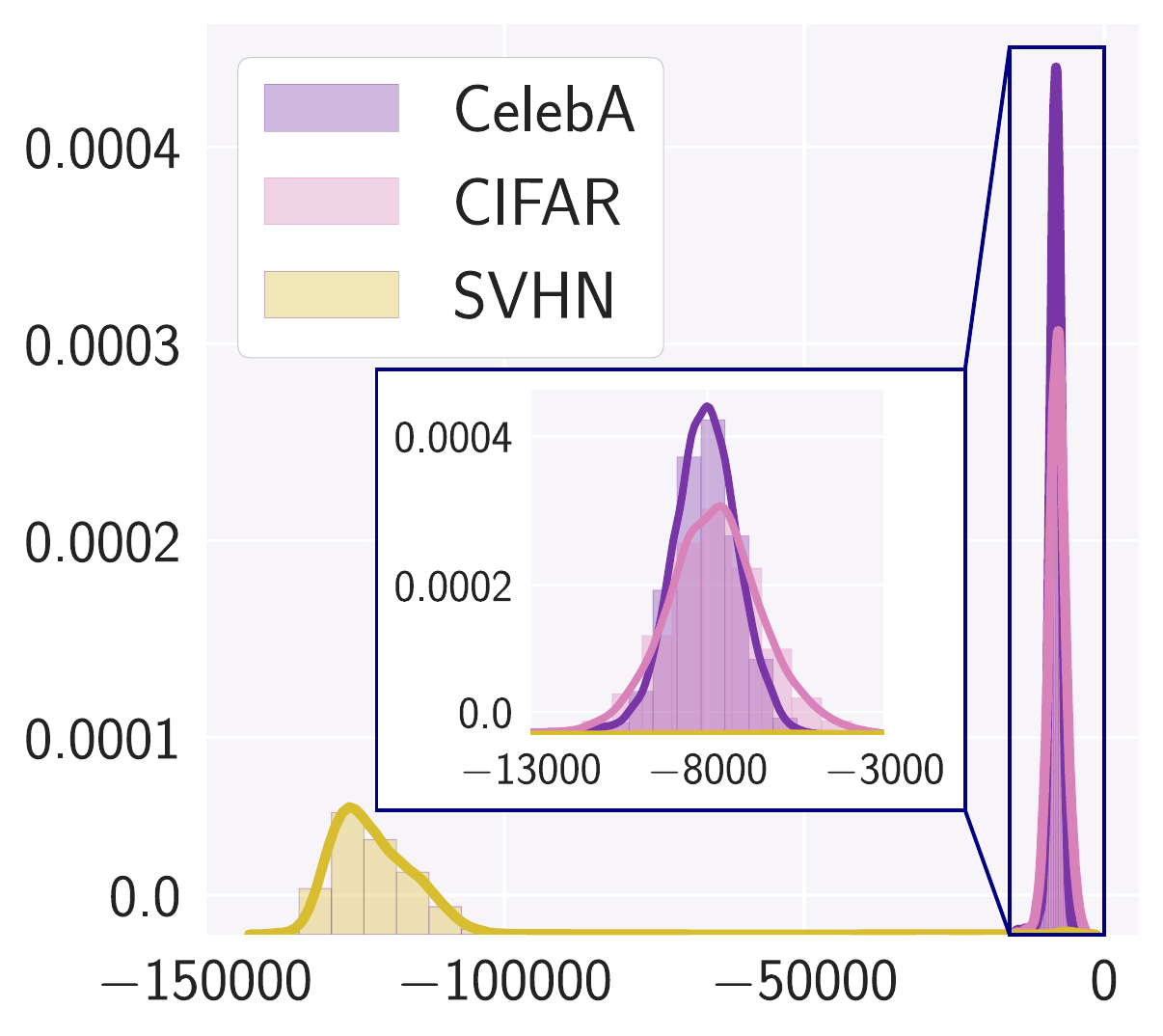}
	}
	\subfigure[SVHN $\uparrow$, CIFAR $\downarrow$]{
	\includegraphics[width=0.23\linewidth]{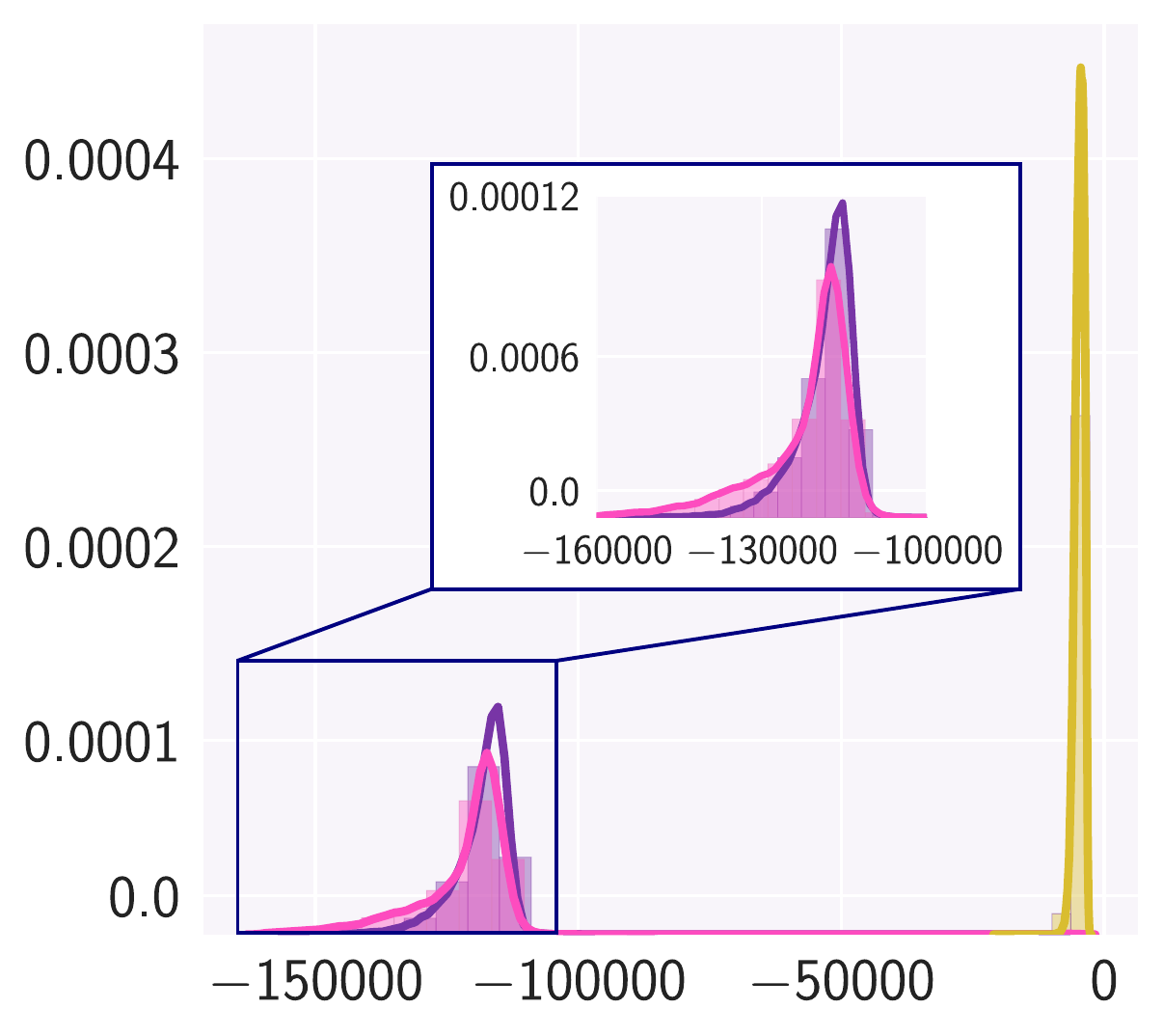}
	}
    \subfigure[CIFAR $\uparrow$, CelebA $\downarrow$]{
	\includegraphics[width=0.23\linewidth]{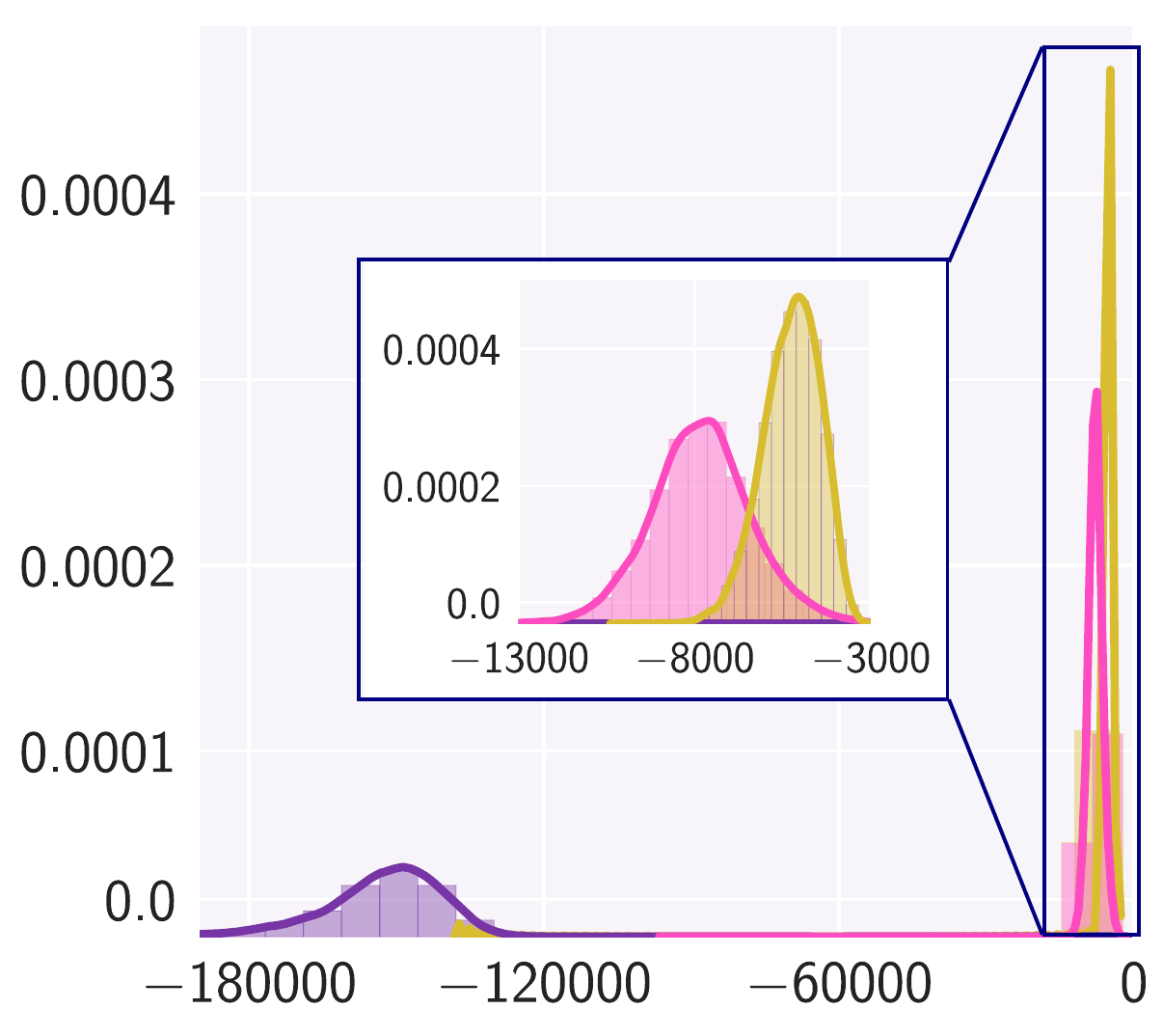}
	}
	\subfigure[CelebA $\uparrow$, CIFAR $\downarrow$]{
	\includegraphics[width=0.23\linewidth]{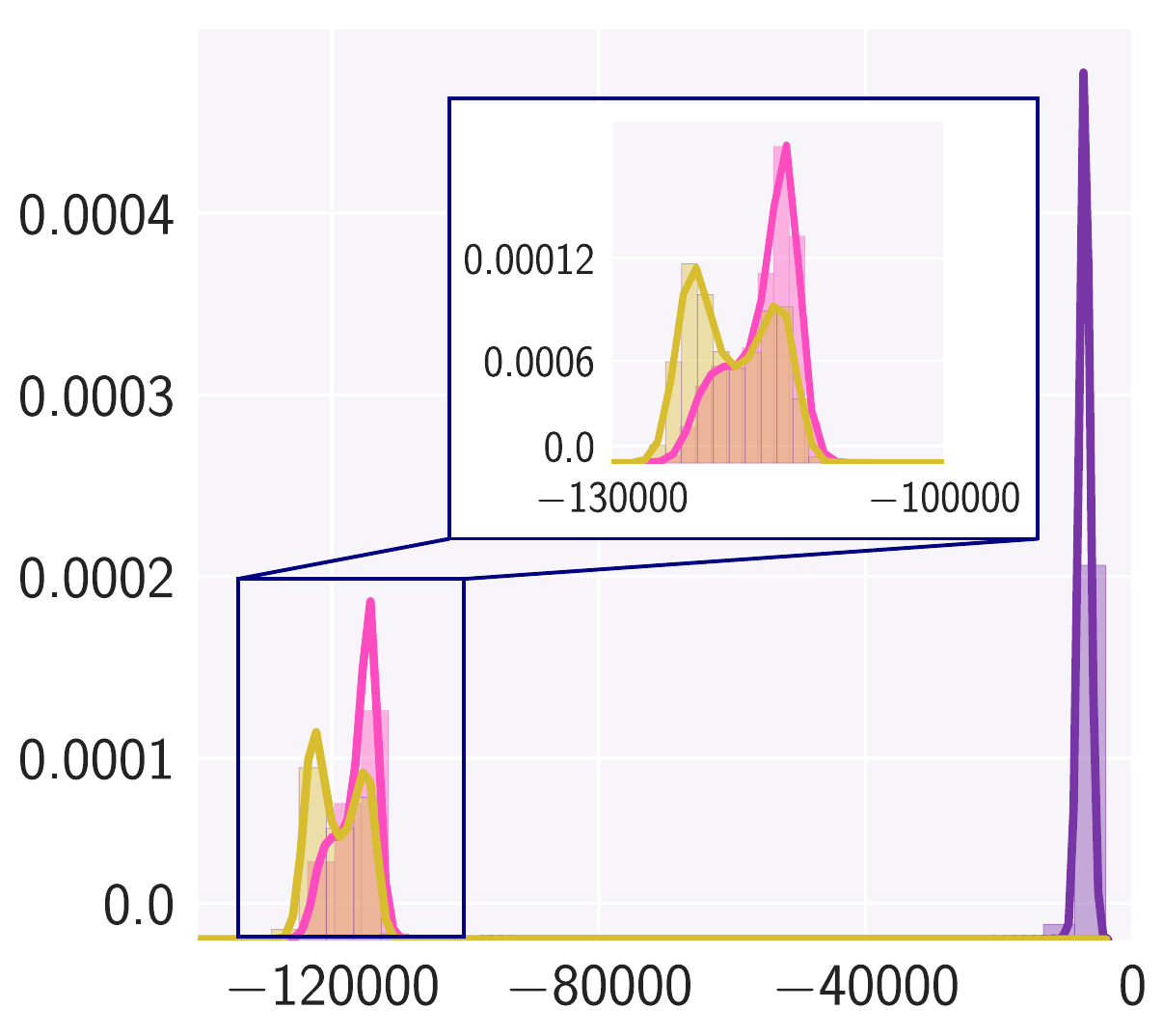}
	}
    \subfigure[Fashion $\uparrow$, MNIST $\downarrow$]{
	\includegraphics[width=0.23\linewidth]{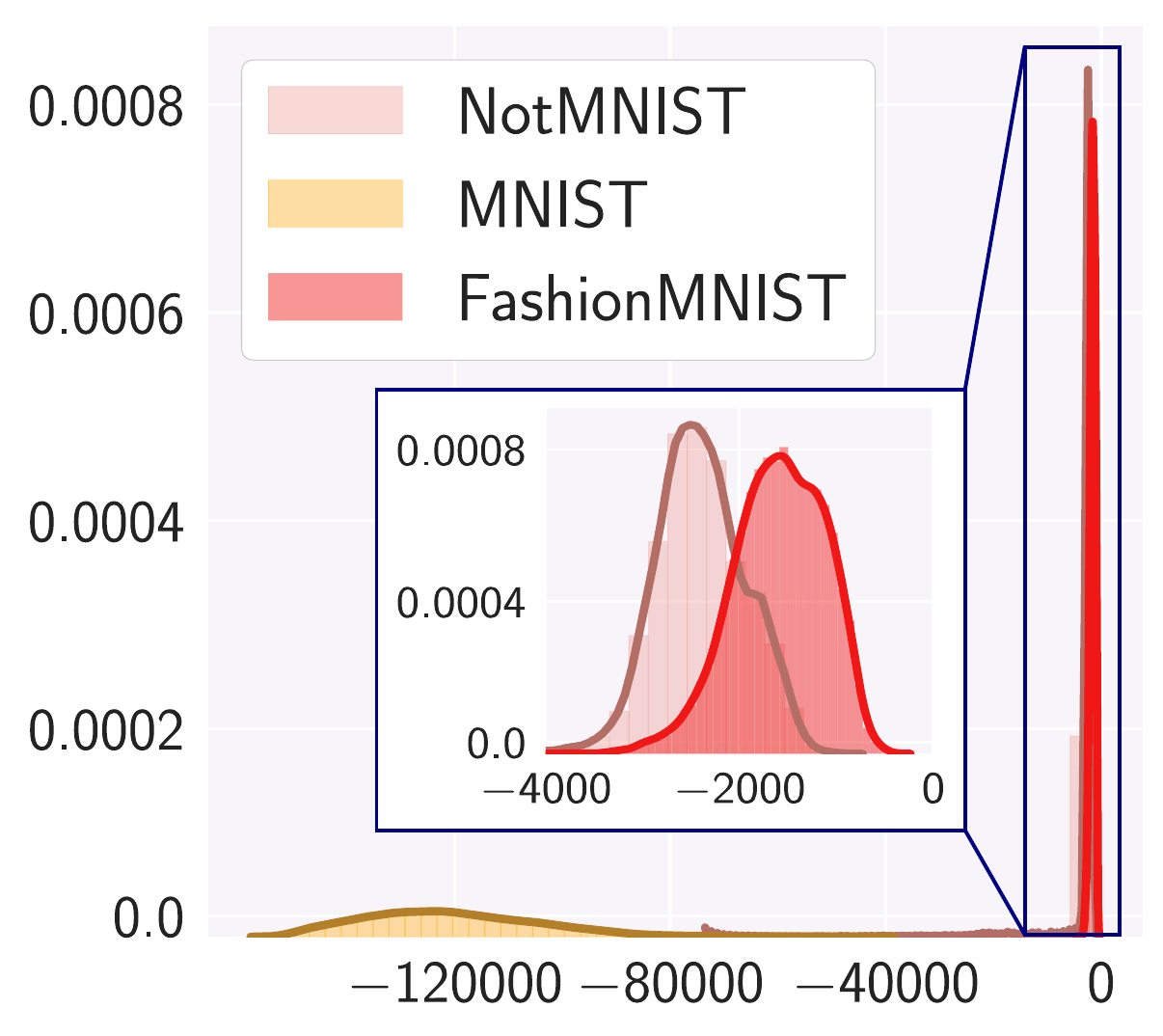}
	}
    \subfigure[MNIST $\uparrow$, Fashion $\downarrow$]{
	\includegraphics[width=0.23\linewidth]{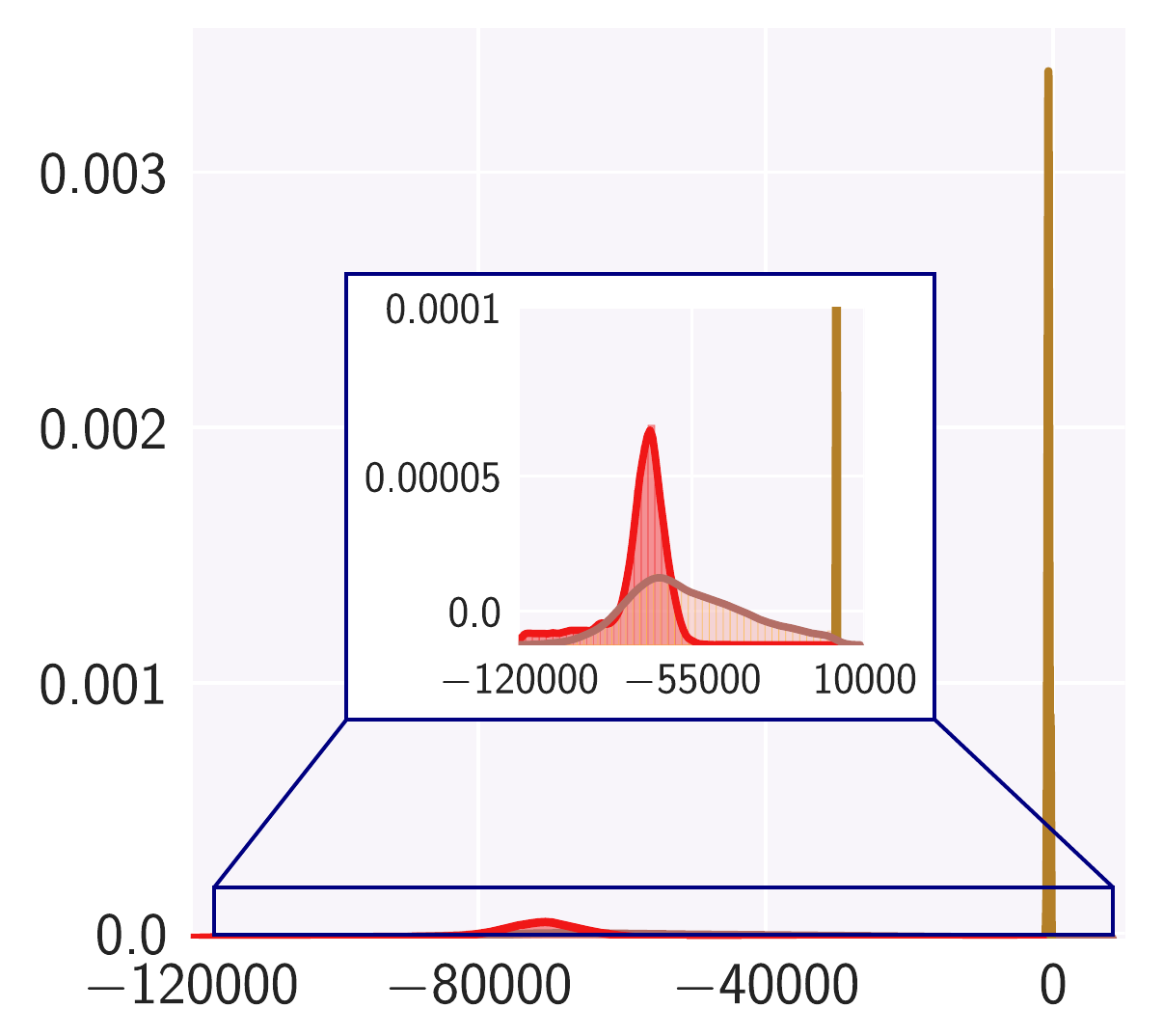}
	}
	\caption{
	\textbf{Negative training.}
	The histograms of log-likelihood for RealNVP when in training likelihood is maximized on one dataset and minimized on another dataset: 
	(a) maximized on CIFAR, minimized on SVHN; 
	(b) maximized on SVHN, minimized on CIFAR; 
	(c) maximized on CIFAR, minimized on CelebA; 
	(d) maximized on CelebA, minimized on CIFAR. 
	(e) maximized on FashionMNIST, minimized on MNIST; 
	(f) maximized on MNIST, minimized on FashionMNIST; 
	}
	\label{fig:ll_hists_maxmin}
    \vspace{-.3cm}
\end{figure}

\section{Flows have capacity to distinguish datasets}
\label{sec:app_capacity}

Normalizing flows are unable to detect OOD image data when trained to maximize likelihood on the train set.
It is natural to ask whether these models are at all capable of distinguishing different image datasets.
In this section we demonstrate the following:

\begin{mybox}
    \textbf{Observation}:
    Flows can assign high likelihood to the train data and low likelihood to 
    a given OOD dataset if they are explicitly trained to do so.
    
    \textbf{Relevance to OOD detection}:
    While flows have sufficient capacity to distinguish different data, they
    are biased towards learning solutions that assign high likelihood to all
    structured data and consequently fail to detect OOD inputs.
\end{mybox}

We introduce an objective that encouraged the flow to maximize likelihood on the
target dataset and to minimize likelihood on a specific OOD dataset. 
The objective we used is
\begin{equation}
    \label{eq:minmax}
    \frac 1 {N_{\mathcal D}} \sum_{x \in \mathcal{D}}\log p(x) - \frac 1 {N_{\text{OOD}}} \sum_{x \in \mathcal{D}_{\text{OOD}}}\log p(x) \cdot I[\log p(x) > c],
\end{equation}
where $I[\cdot]$ is an indicator function and the constant $c$ allows us to encourage 
the flow to only push the likelihood of OOD data to a threshold rather than decreasing it to $-\infty$; 
$N_{\mathcal D}$ is the number of train datapoints and $N_{\text{OOD}} = \sum_{x \in \mathcal{D}_{\text{OOD}}} I[\log p(x) > c]$
is the number of OOD datapoints that have likelihood above the threshold $c$. 

We trained a RealNVP flow with the objective \eqref{eq:minmax} using different pairs of target and OOD datasets: CIFAR-10 vs CelebA, CIFAR-10 vs SVHN and FashionMNIST vs MNIST.
We present the results in Figure~\ref{fig:ll_hists_maxmin}.
In each case, the flow is able to push the likelihood of the OOD dataset to very low values, and simultaneously maximize the likelihood on the target dataset creating a clear separation between the two.

\textbf{Hyper-parameters} \quad
For the flow architecture and training used the same hyper-parameters as we did
for the baselines, described in Appendix \ref{sec:app_details}. 
For CelebA, CIFAR and SVHN models we set $c = -100000$, and for 
MNIST, FashionMNIST and NotMNIST we set $c = -30000$.

\textbf{Connection with prior work} \quad
Flows can be used as classifiers separating different classes of the same dataset \citep{nalisnick2019detecting, izmailov2019semi, atanov2019semi},
which further highlights the fact that flows can distinguish images based on their
contents when trained to do so.
A similar experiment for the PixelCNN model \citep{oord2016pixel} was presented in \citet{hendrycks2018deep}.
The authors maximized the likelihood of CIFAR-10 and minimized the likelihood of the TinyImages dataset \citep{torralba200880}.
In their experiments, this procedure consistently led to CIFAR-10 having higher likelihood than any of the other benchmark datasets.
In Figures \ref{fig:ll_hists_maxmin}, for each experiment in addition to the two datasets that were used in training we show the log-likelihood distribution on another OOD dataset.
For example, when we train the flow to separate CIFAR-10 from CelebA (panels c, d), the flow successfully does so but assigns SVHN with likelihood similar to that of CIFAR.
When we train the flow to separate CIFAR-10 from SVHN (panels c, d), the flow successfully does so but assigns CelebA with likelihood similar to that of CIFAR.
Similar observations can be made for MNIST, FashionMNIST and notMNIST.
At least for normalizing flows, minimizing the likelihood on a single OOD dataset 
does not lead to all the other OOD datasets achieving low-likelihood.

\section{Details of the experiments}
\label{sec:app_details}

\begin{figure}[t]
	\def \panelwidth {0.18\linewidth}
	\def \panelskip {-0.3cm}
    \centering
	
	\hspace{0.11\linewidth}
    \subfigure[RNVP, ImageNet]{
	\includegraphics[height=\panelwidth]{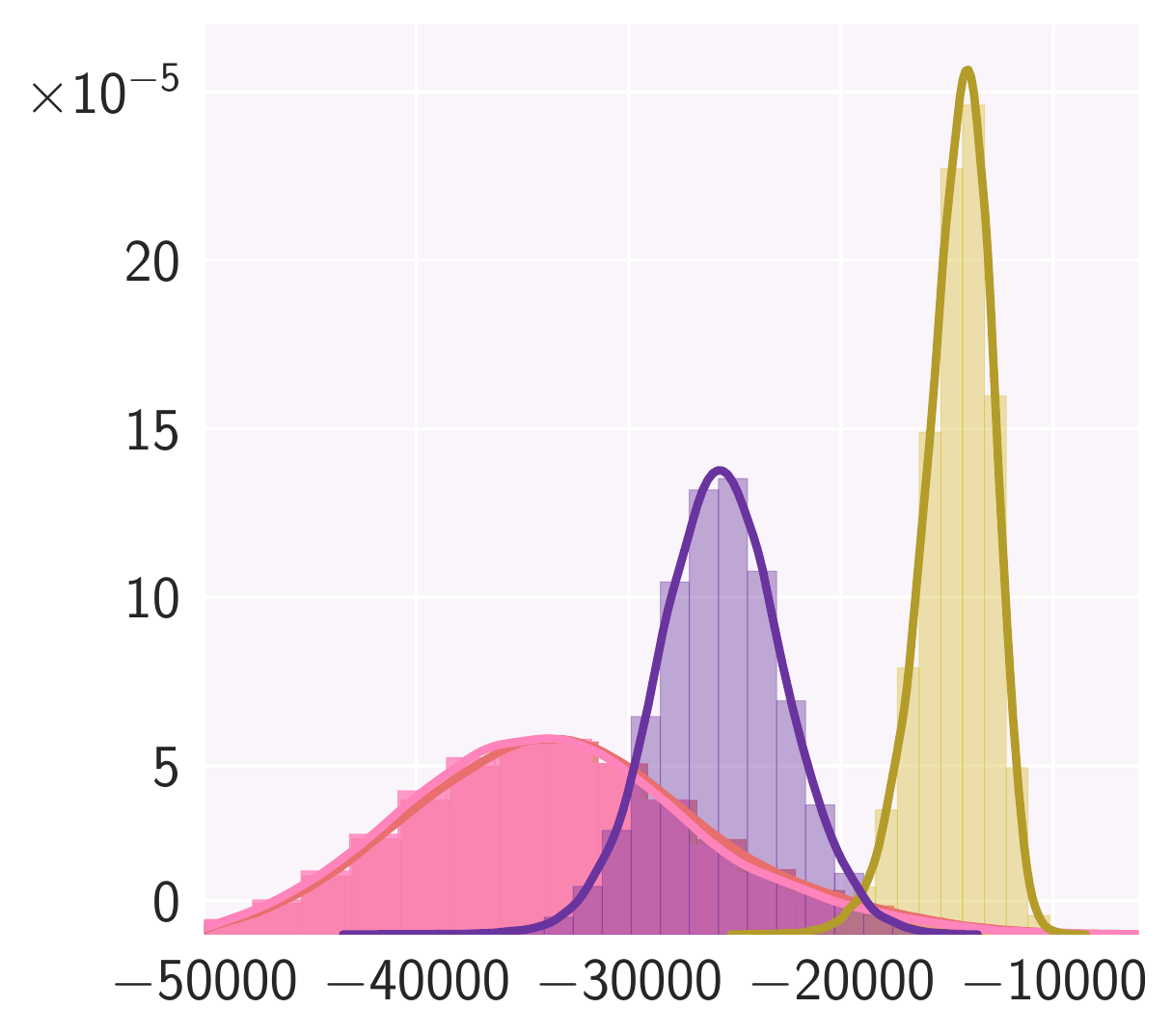}
	}\quad
    \subfigure[RNVP, CelebA-HQ]{
	\includegraphics[height=\panelwidth]{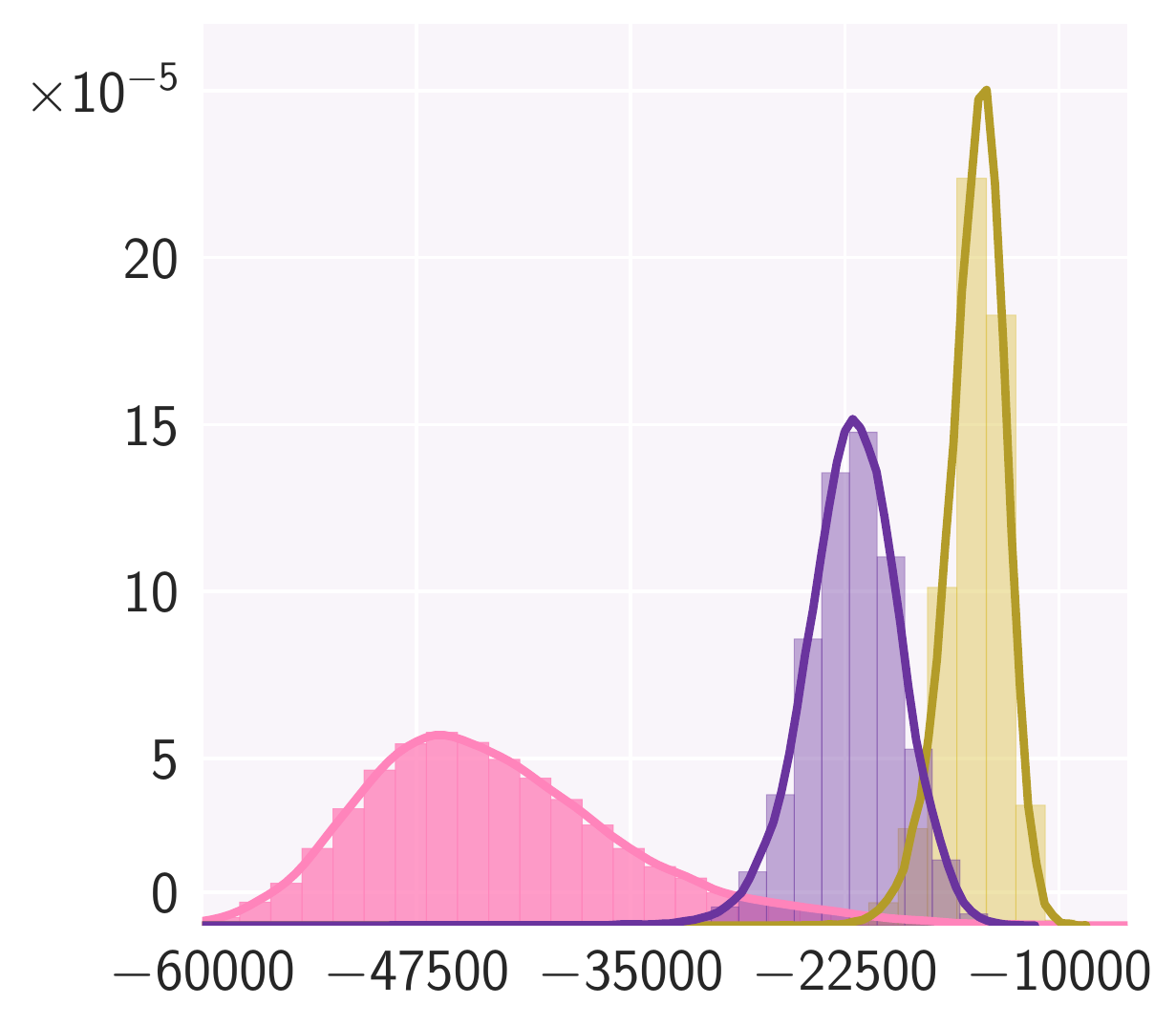}
	}
	\includegraphics[height=\panelwidth]{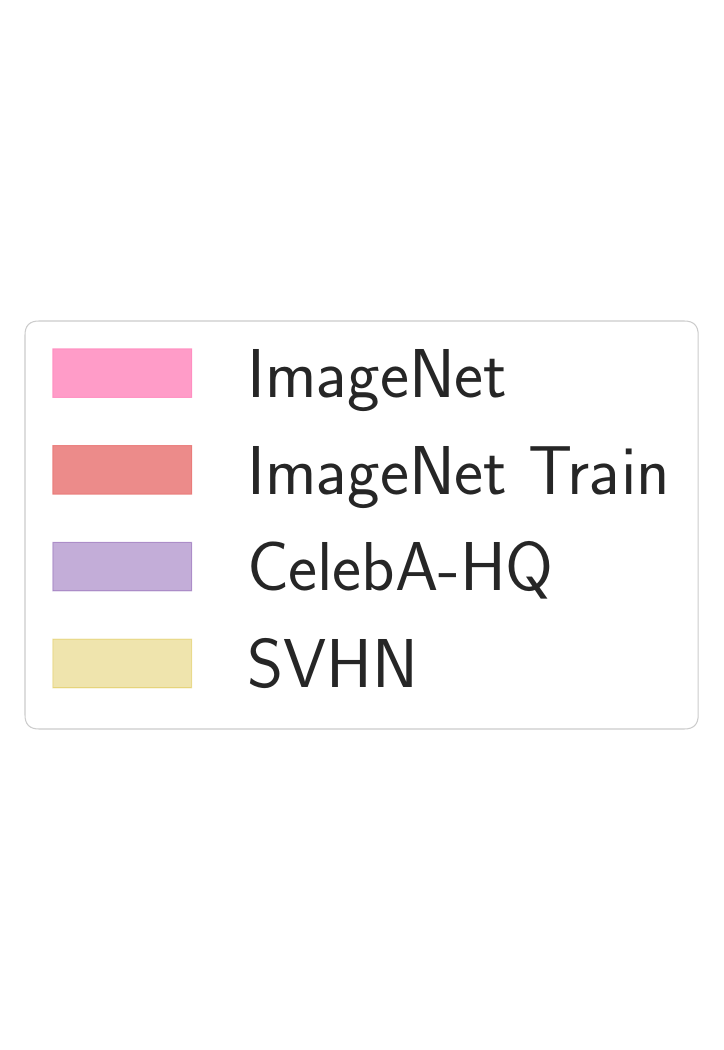}

    \subfigure[RNVP, MNIST]{
	    \includegraphics[height=\panelwidth]{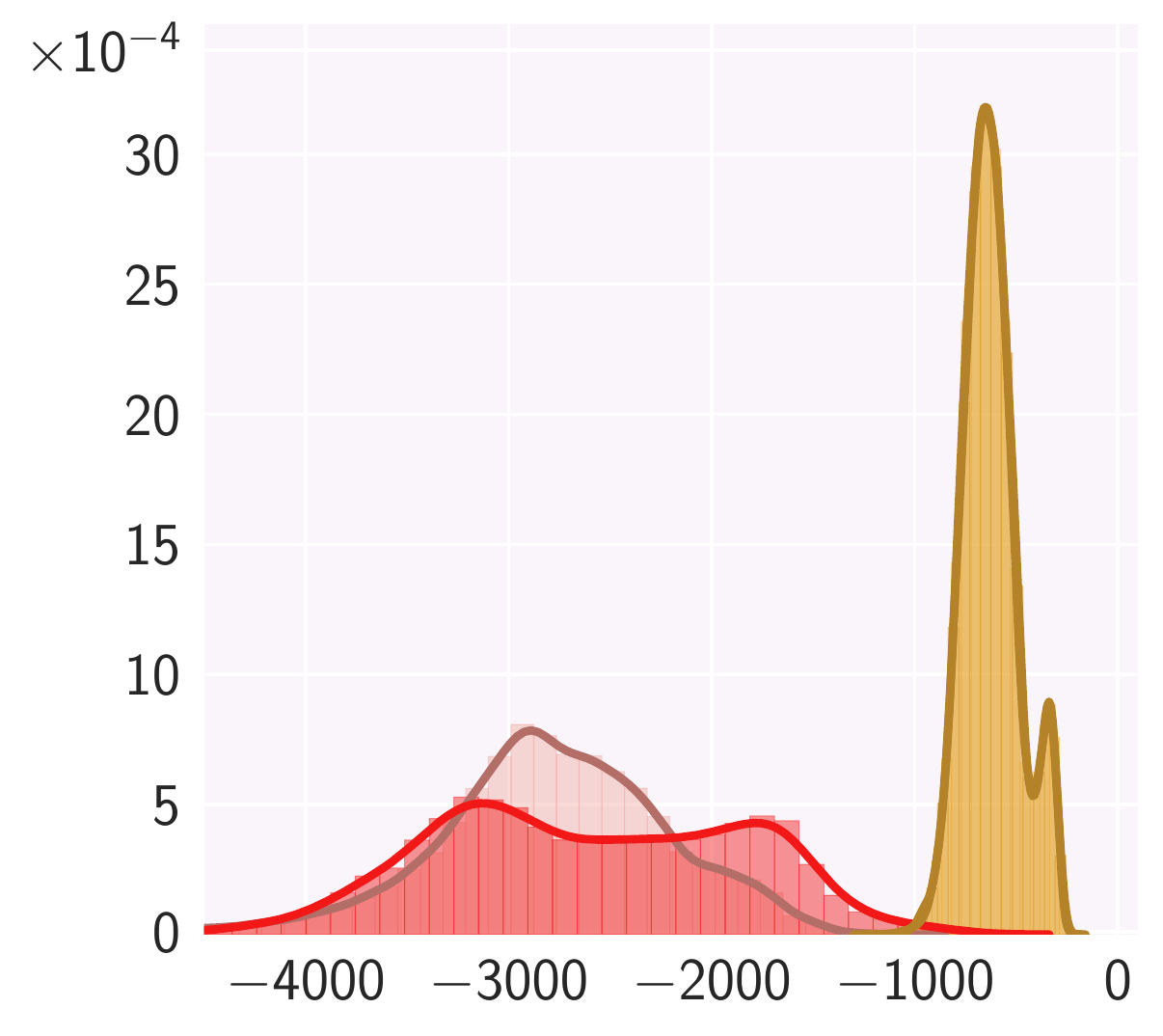}
	}
	\hspace{\panelskip}
    \subfigure[RNVP, Fashion]{
	    \includegraphics[height=\panelwidth]{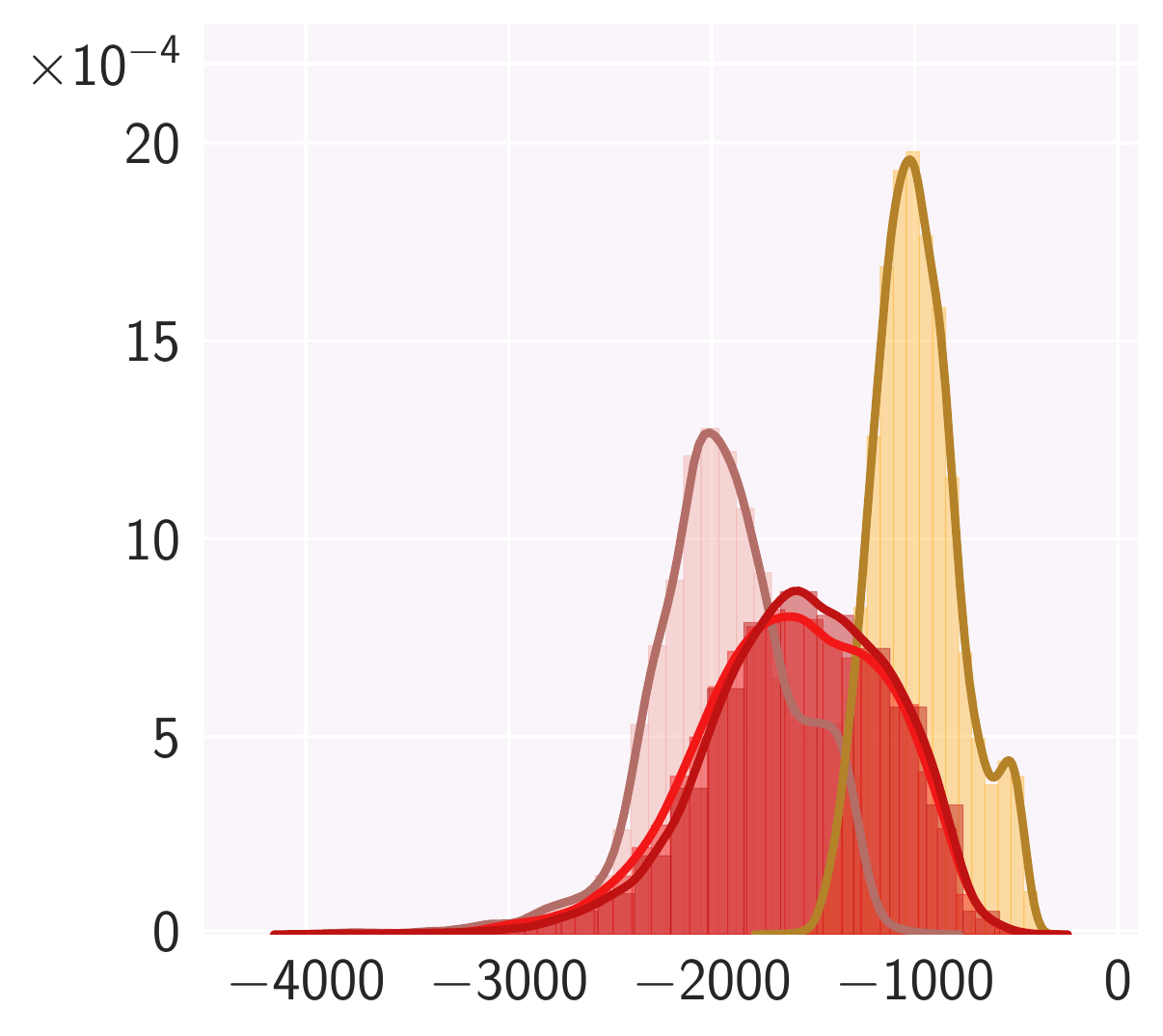}
	}
	\hspace{\panelskip}
	\subfigure[Glow, MNIST]{
	    \includegraphics[height=\panelwidth]{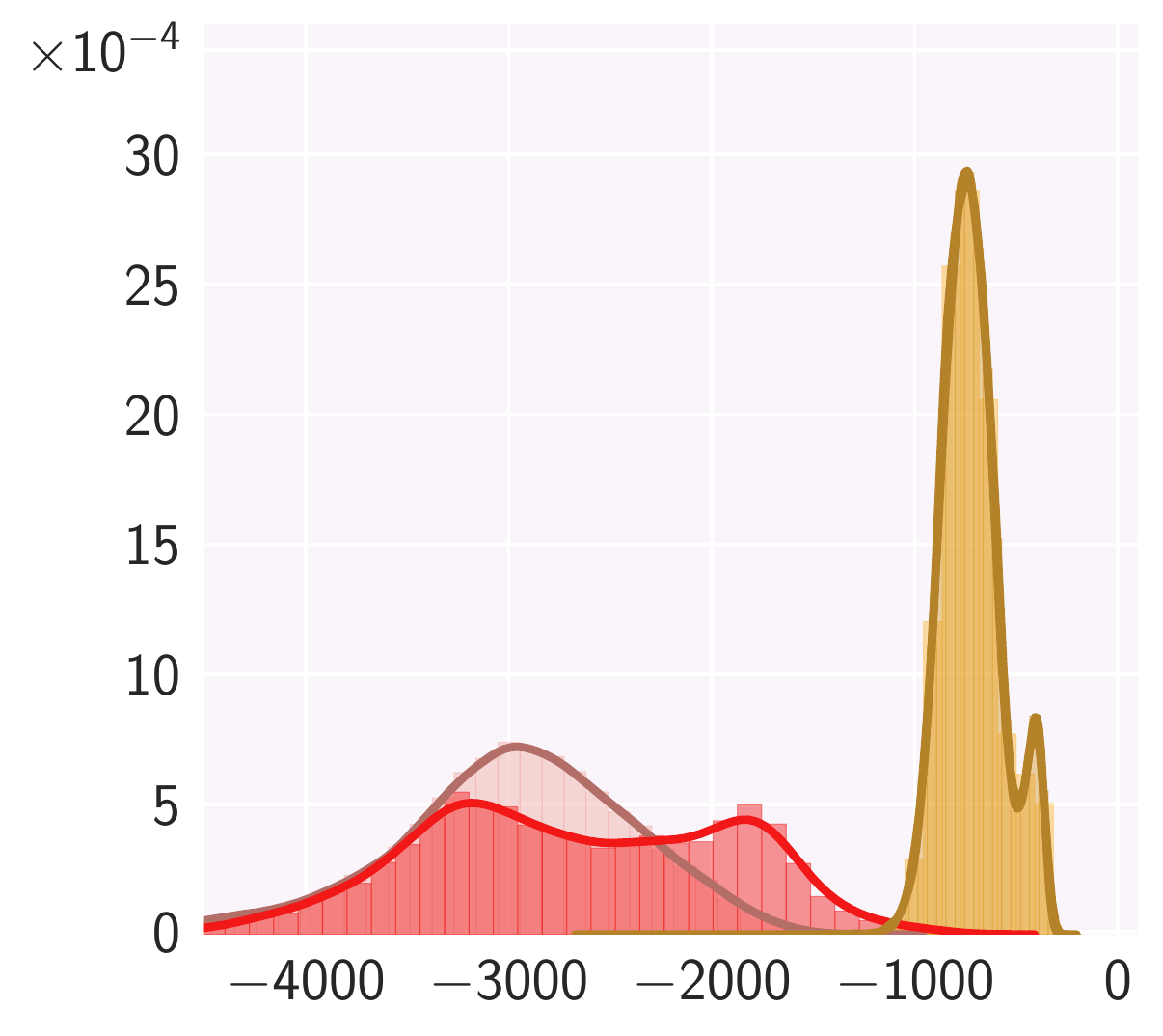}
	}
	\hspace{\panelskip}
    \subfigure[Glow, Fashion]{
	    \includegraphics[height=\panelwidth]{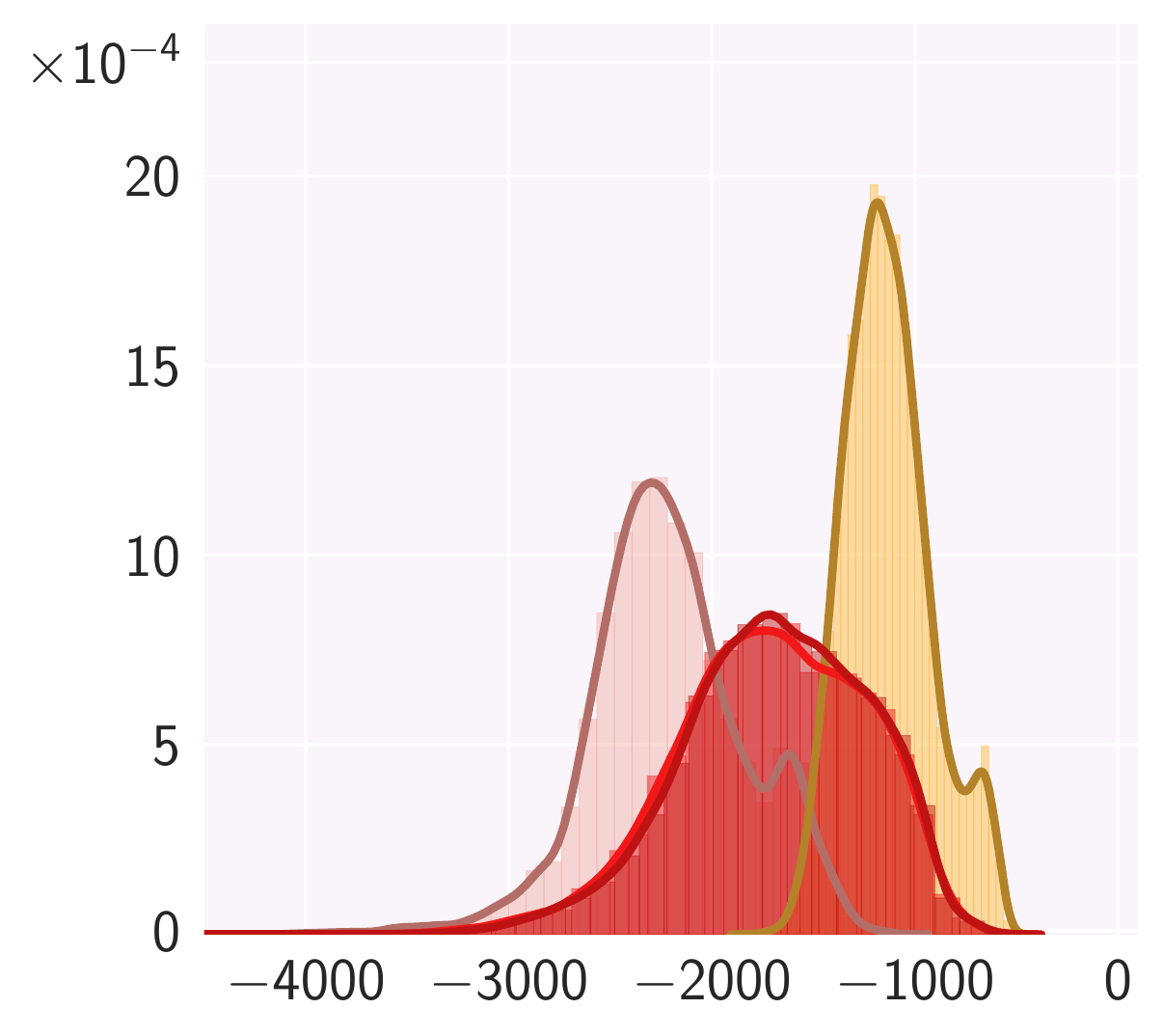}
	}
	\hspace{\panelskip}
	\includegraphics[height=\panelwidth]{figs/baseline/mnist_legend_1col.pdf}
	
 \subfigure[RNVP, CelebA]{
	    \includegraphics[height=\panelwidth]{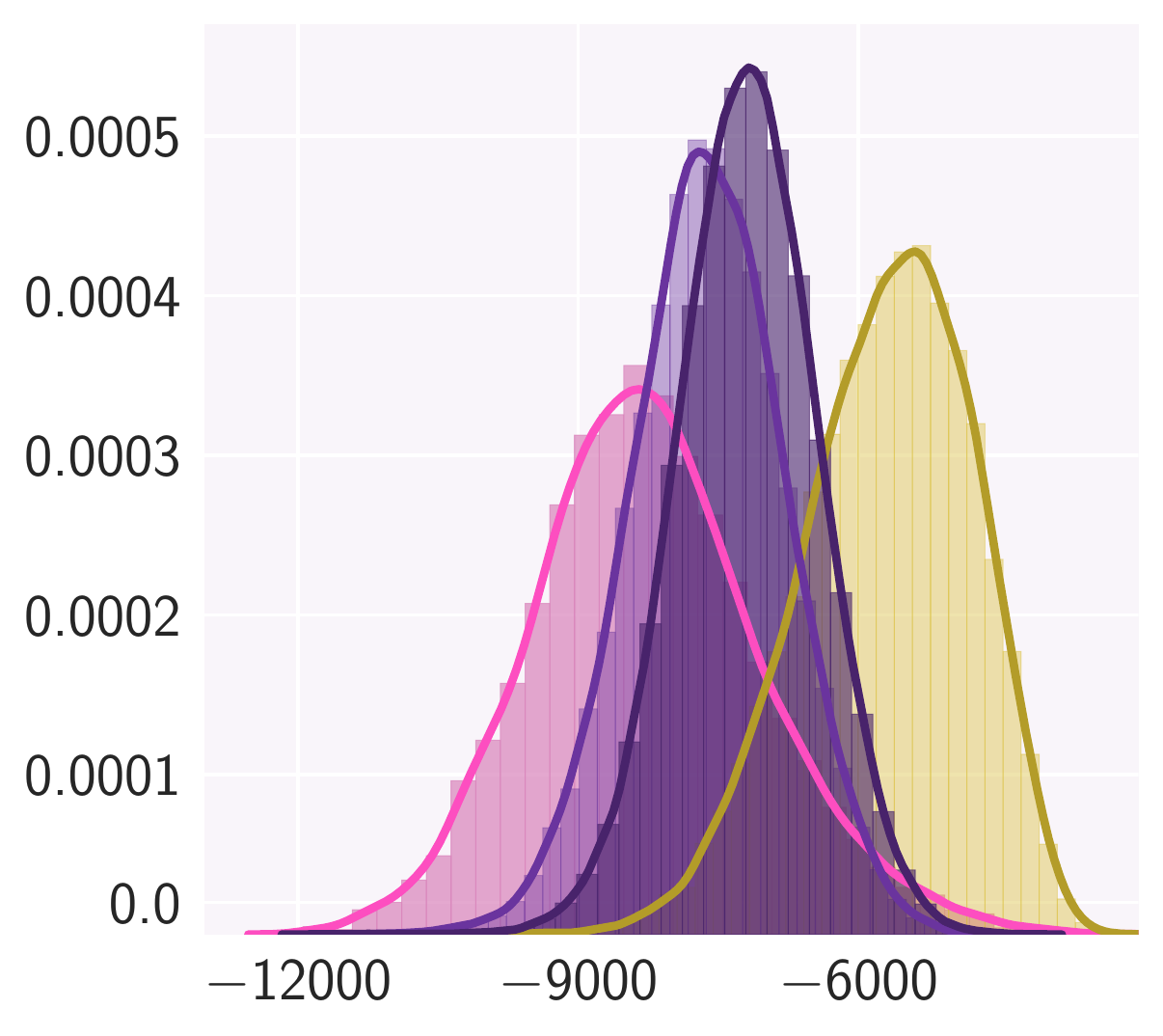}
	}
    \subfigure[RNVP, CIFAR-10]{
	    \includegraphics[height=\panelwidth]{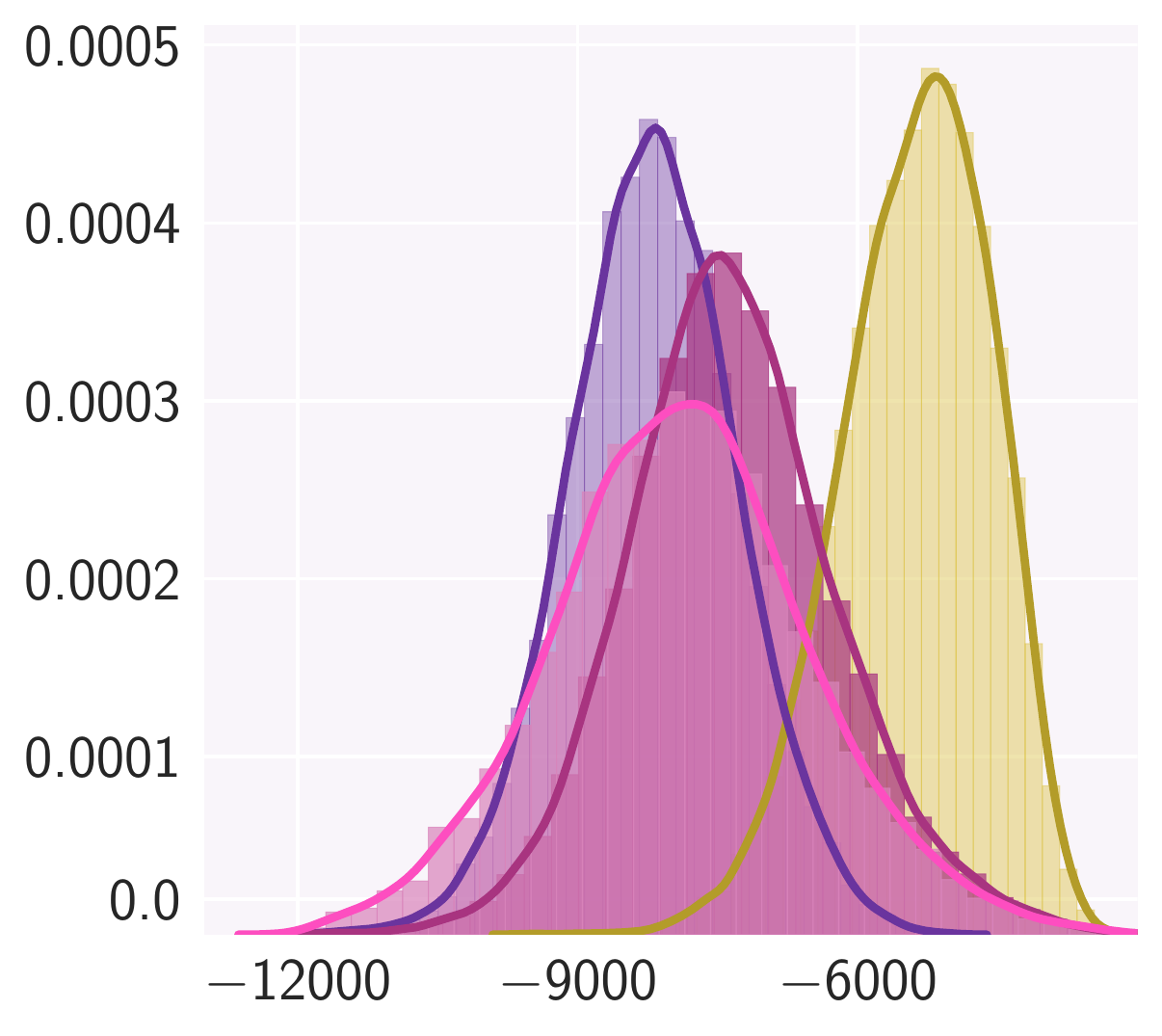}
    }
    \subfigure[RNVP, SVHN]{
	    \includegraphics[height=\panelwidth]{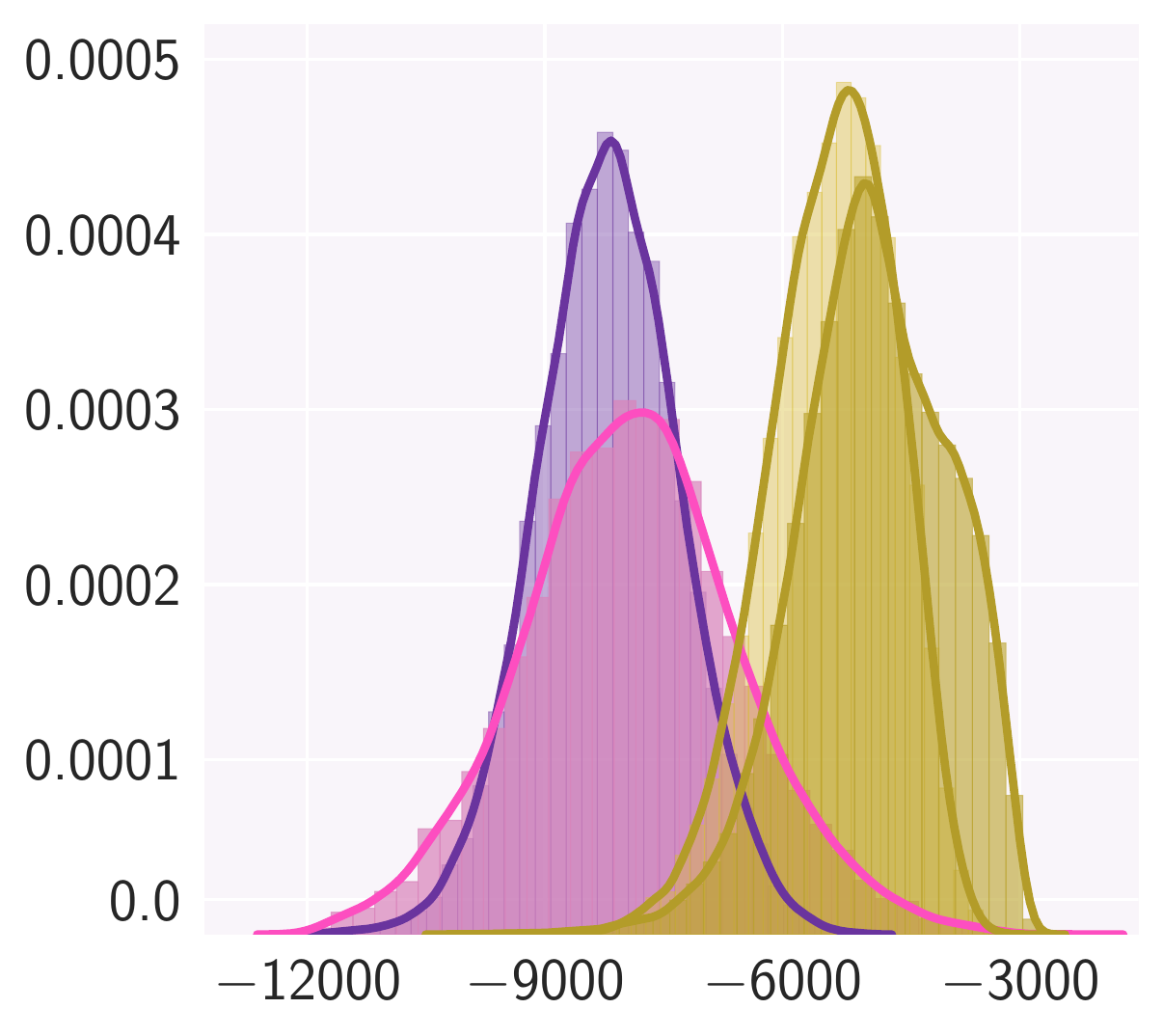}
	}
	
	\hspace{0.11\linewidth}
    \subfigure[Glow, CelebA]{
	    \includegraphics[height=\panelwidth]{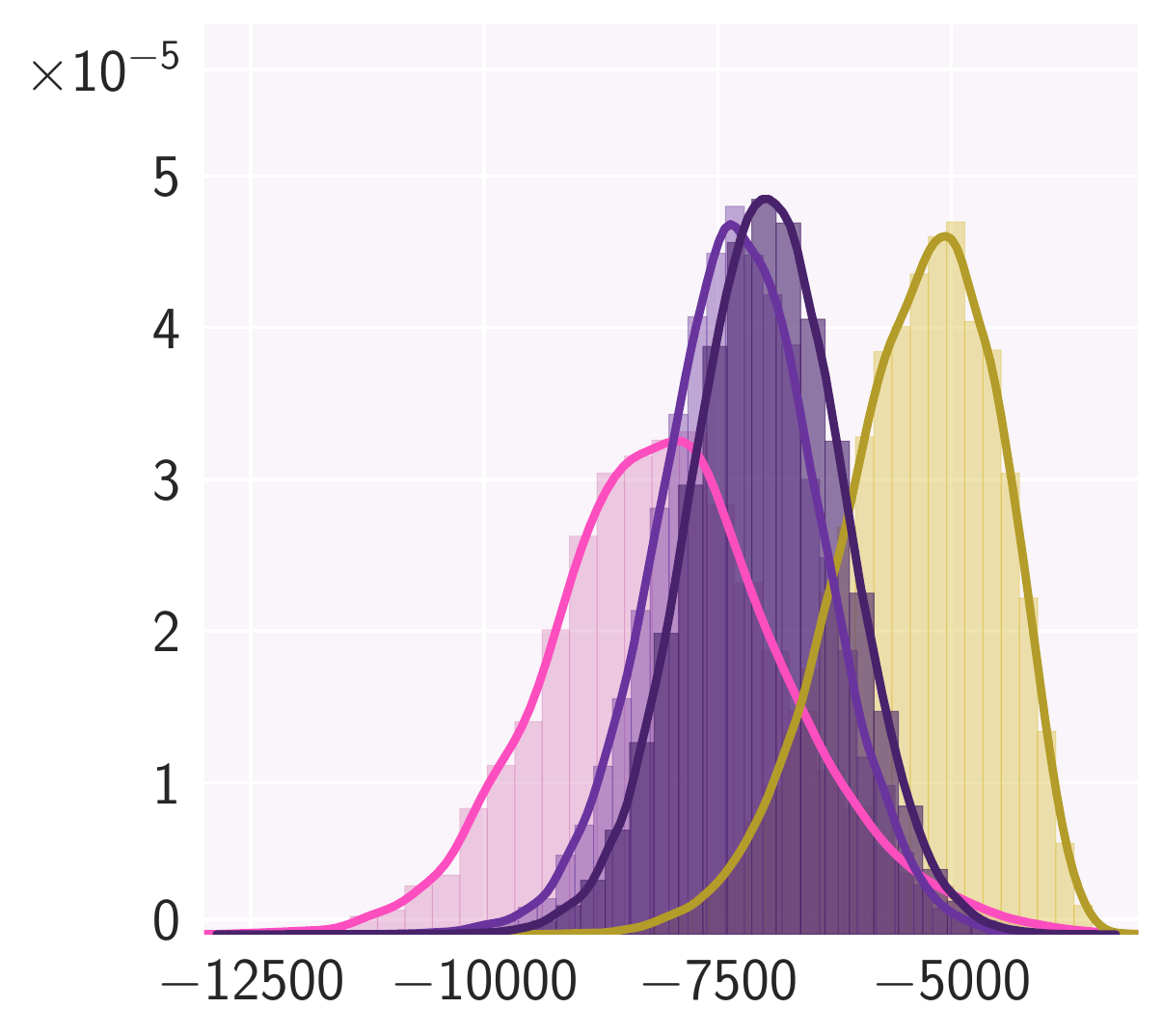}
	}
	\subfigure[Glow, CIFAR-10]{
	    \includegraphics[height=\panelwidth]{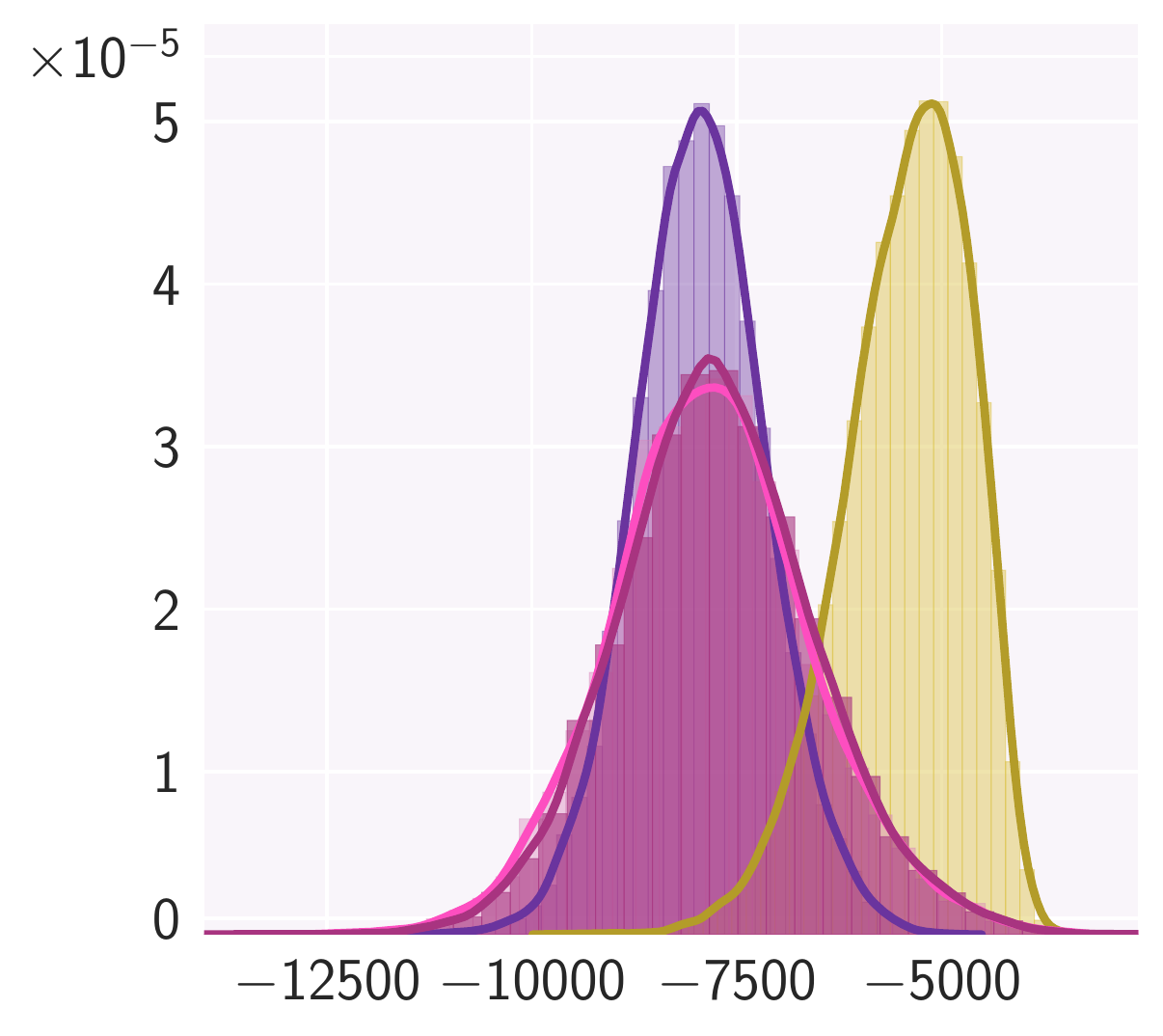}
	}
    \subfigure[Glow, SVHN]{
	    \includegraphics[height=\panelwidth]{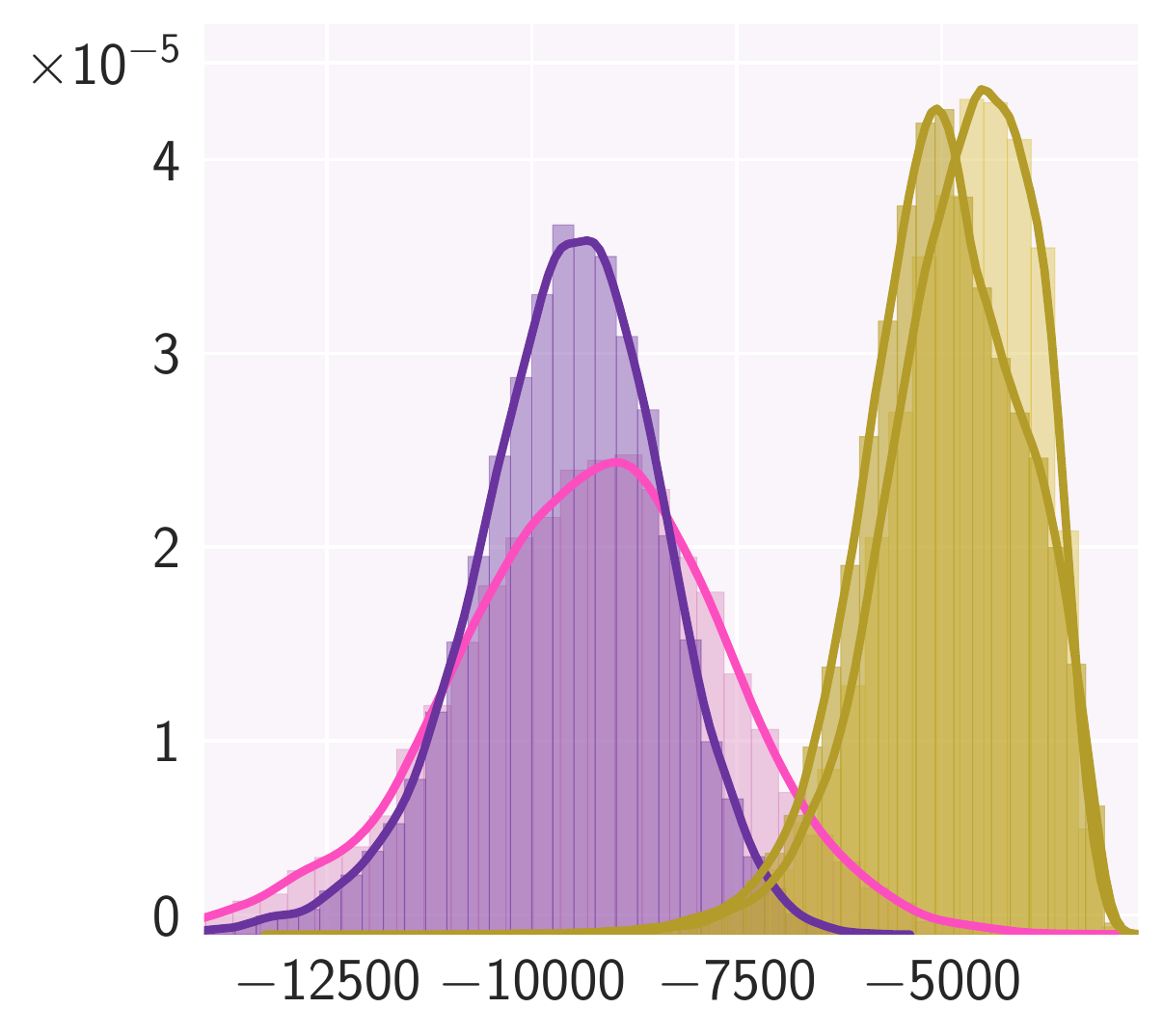}
	}
	\includegraphics[height=\panelwidth]{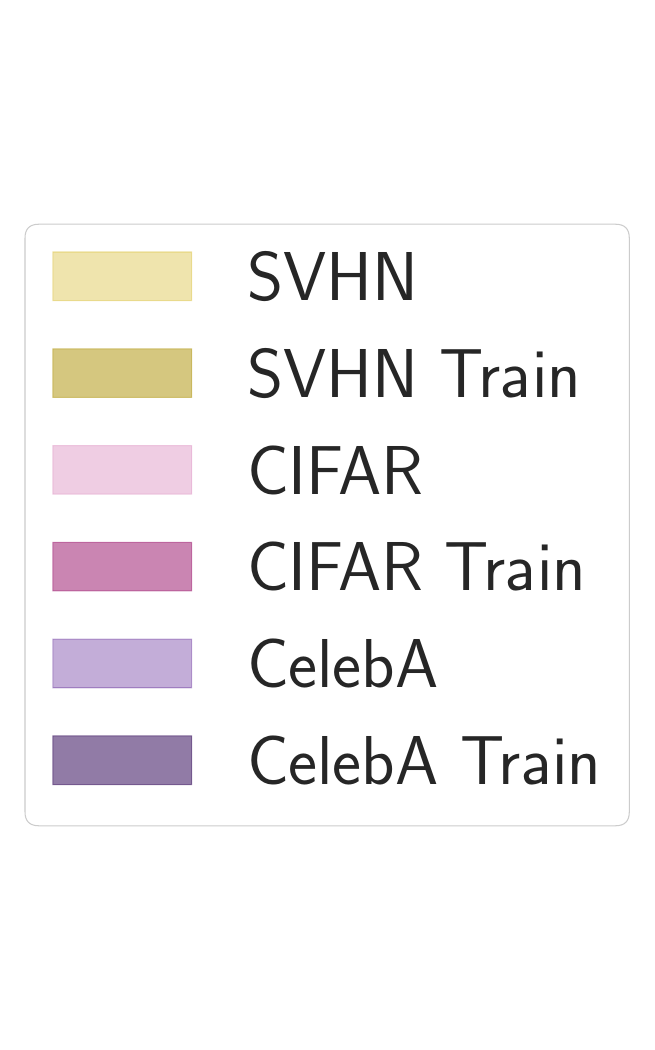}

	\caption{
    \textbf{Baseline log-likelihoods.}
	The histograms of log-likelihood for RealNVP and Glow models trained on various datasets. 
    Both flows consistently assign similar or higher likelihood to OOD data compared to the
    target dataset.
    The likelihood distribution for train and test sets of the target data is 
    typically very similar.
	}
	\label{fig:app_ll_hists_baseline}
\end{figure}

\textbf{RealNVP} \quad For all RealNVP models, we generally follow the architecture design of \citet{dinh2016density}.
We use multi-scale architecture where after a block of coupling layers half of the variables are factored out and copied forward directly to the latent representation.
Each scale consists of 3 coupling layers with checkerboard mask, followed by a squeeze operation and 3 coupling layers with channel-wise mask (see Figure \ref{fig:app_squeeze}).
For the $st$-network we use deep convolutional residual
networks with additional skip connections following \citet{dinh2016density}. In all experiments, we use Adam optimizer. 
On grayscale images (MNIST, FashionMNIST), we used 2 scales in RealNVP, 6~blocks in residual $st$-network, learning rate $5 \times 10^{-5}$, batch size 32 and trained model for 80 epochs.
On CIFAR-10, CelebA and SVHN, we used 3 scales, 8 blocks in $st$-network, learning rate $10^{-4}$, batch size 32, weight decay $5 \times 10^{-5}$ and trained the model for 100 epochs.
On ImageNet, we used 5 scales, 2 blocks in $st$-network, learning rate $10^{-3}$, batch size 64, weight decay $5 \times 10^{-5}$ and trained the model for 42 epochs. On CelebA $64 \times 64$, we used 4 scales, 4 blocks in $st$-network, learning rate $10^{-4}$, batch size 64, weight decay $5 \times 10^{-5}$ and trained the model for 100 epochs.

\textbf{Glow} \quad We follow the training details of \citet{nalisnick2018deep} for multi-scale Glow models. Each scale consists of a sequence of actnorm, invertible $1 \times 1$ convolution and coupling layers \citep{kingma2018glow}. The squeeze operation is applied before each scale, and half of the variables are factored out after each scale. In all experiments, we use RMSprop optimizer.
On grayscale images (MNIST, FashionMNIST), we used 2 scales with 16 coupling layers, a 3-layer Highway network with 200 hidden units for $st$-network, learning rate $5 \times 10^{-5}$, batch size 32 and trained model for 80 epochs.
On color images (CIFAR-10, CelebA, SVHN), we used 3 scales with 8 coupling layers, a 3-layer Highway network with 400 hidden units for $st$-network, learning rate $5 \times 10^{-5}$, batch size 32 and trained model for 80 epochs.

\section{Baseline models likelihood distributions and AUROC scores}
\label{appendix:baseline_likelihood_auroc}

In Figure \ref{fig:app_ll_hists_baseline}, we plot the histograms of the log likelihoods on in-distribution and out-of-distribution datasets RealNVP and Glow models.
In Table \ref{tab:baseline_auroc} we report AUROC scores for OOD detection with
these models.
As reported in prior work, Glow and RealNVP consistently fail at OOD detection.

\begin{table}[!h]
    \centering
    \scriptsize
    \begin{tabular}{c c ccc c ccc}
        \toprule
        && \multicolumn{3}{c}{OOD data} & \multicolumn{3}{c}{OOD data}\\
        \cmidrule(r){3-5} 
        \cmidrule(r){6-8} 
        Model &Train data & CelebA & CIFAR-10 & Data & SVHN & MNIST & Fashion &  NotMNIST\\
        \midrule

        \multirow{3}{*}{RealNVP}
        &CelebA   & -- & 67.7 & 6.3 & MNIST        & -- & 99.99 & 99.99\\
        &CIFAR-10 & 56.0 & -- & 6.0 & Fashion & 10.8 & -- & 72.1\\
        &SVHN     & 99.0 & 98.4 & -- \\
        \midrule
        \multirow{3}{*}{Glow}
        &CelebA   & -- & 69.1 & 6.4 & MNIST        & -- & 99.96 & 100.0\\
        &CIFAR-10 & 52.9 & -- & 5.5 & Fashion & 13.3 & -- & 80.2\\
        &SVHN     & 99.9 & 99.1 & -- \\
         \bottomrule
    \end{tabular}
    \vspace{0.5cm}
    \caption{
    \textbf{Baseline AUROC.}
    AUROC scores on OOD detection for RealNVP and Glow models trained on various image data.
    Flows consistently assign higher likelihoods to OOD dataset except when trained
    on MNIST and SVHN.
    The AUROC scores for RealNVP and Glow are close.
    }
    \label{tab:baseline_auroc}
\end{table}

\section{Visualization implementation}
\label{sec:app_visualizations}

\begin{figure}[t]
	\def \panelwidth {0.17\linewidth}
	\def \panelskip {-0.5cm}
    \centering
    
    \subfigure[Squeeze layer]{
 	    \includegraphics[height=\panelwidth]{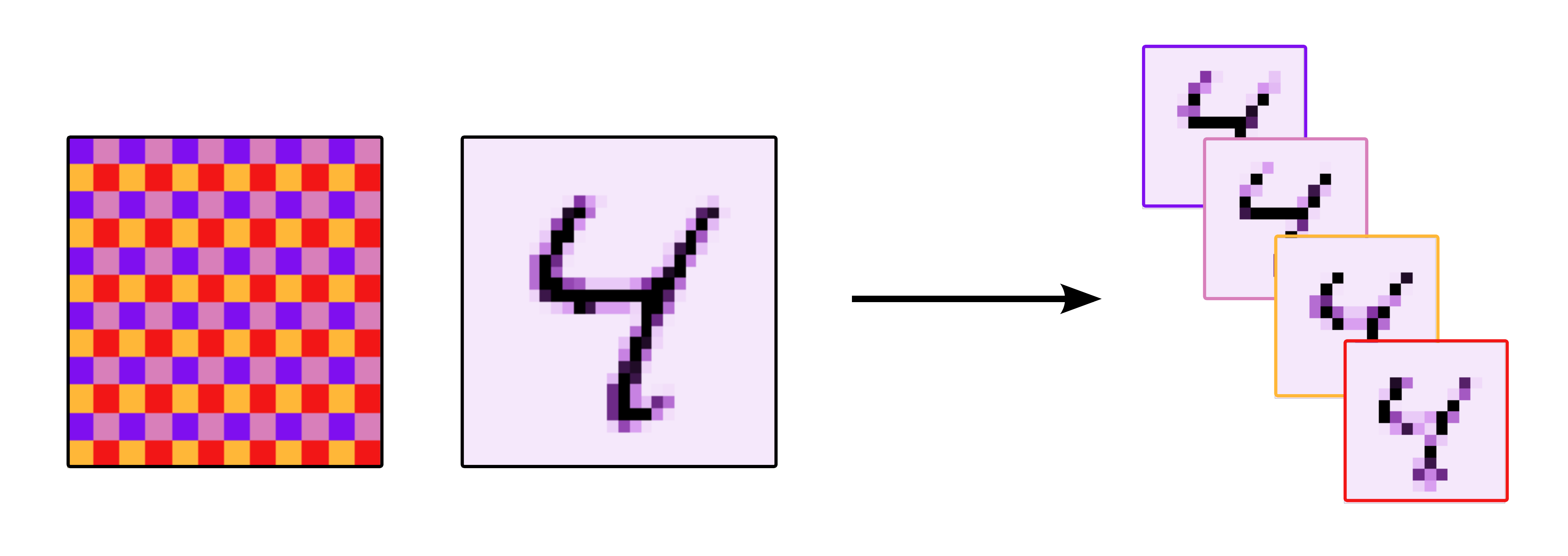}
 	}
 	\subfigure[CB mask]{
	    \includegraphics[height=\panelwidth]{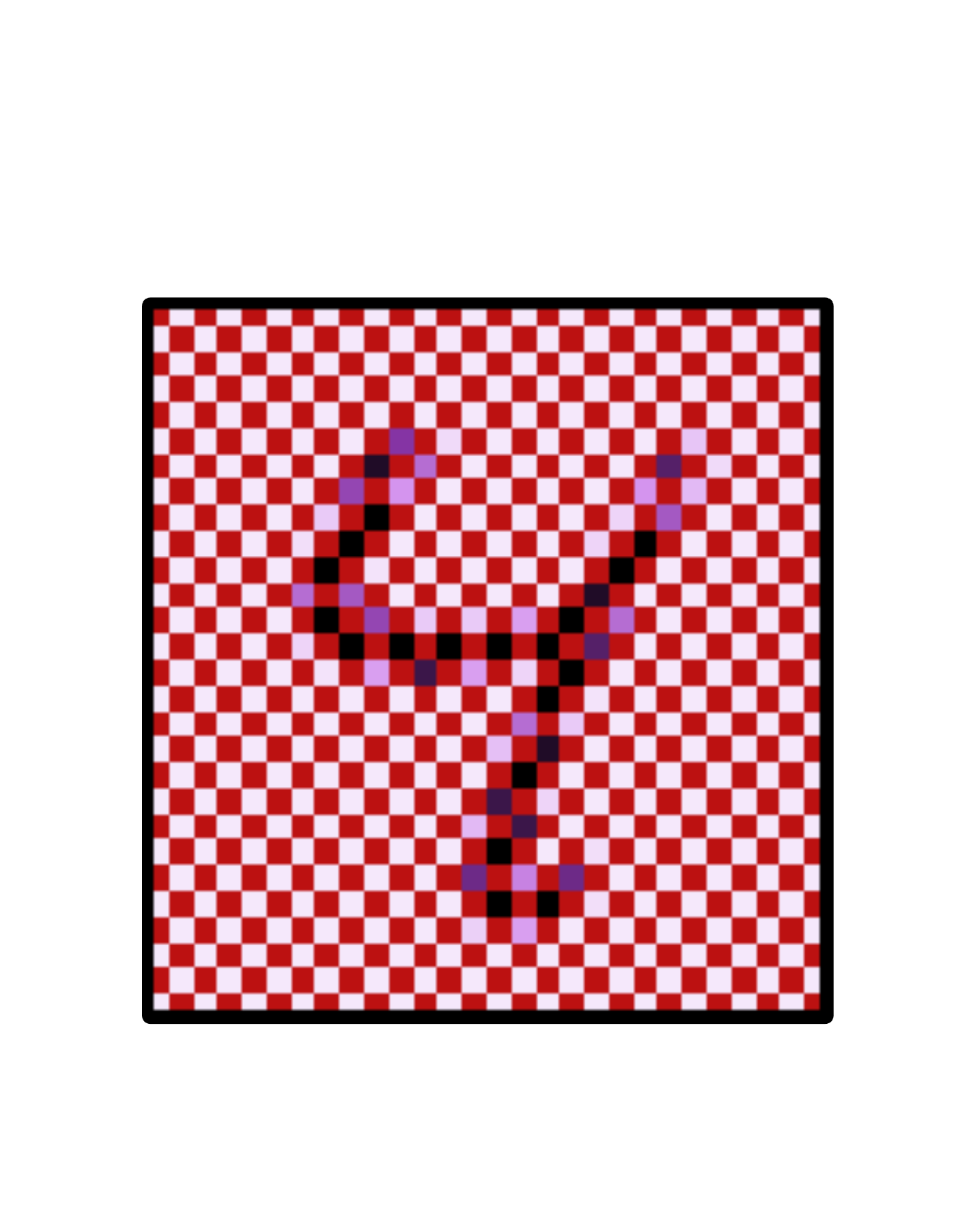}
	}
    \subfigure[CW mask]{
	    \includegraphics[height=\panelwidth]{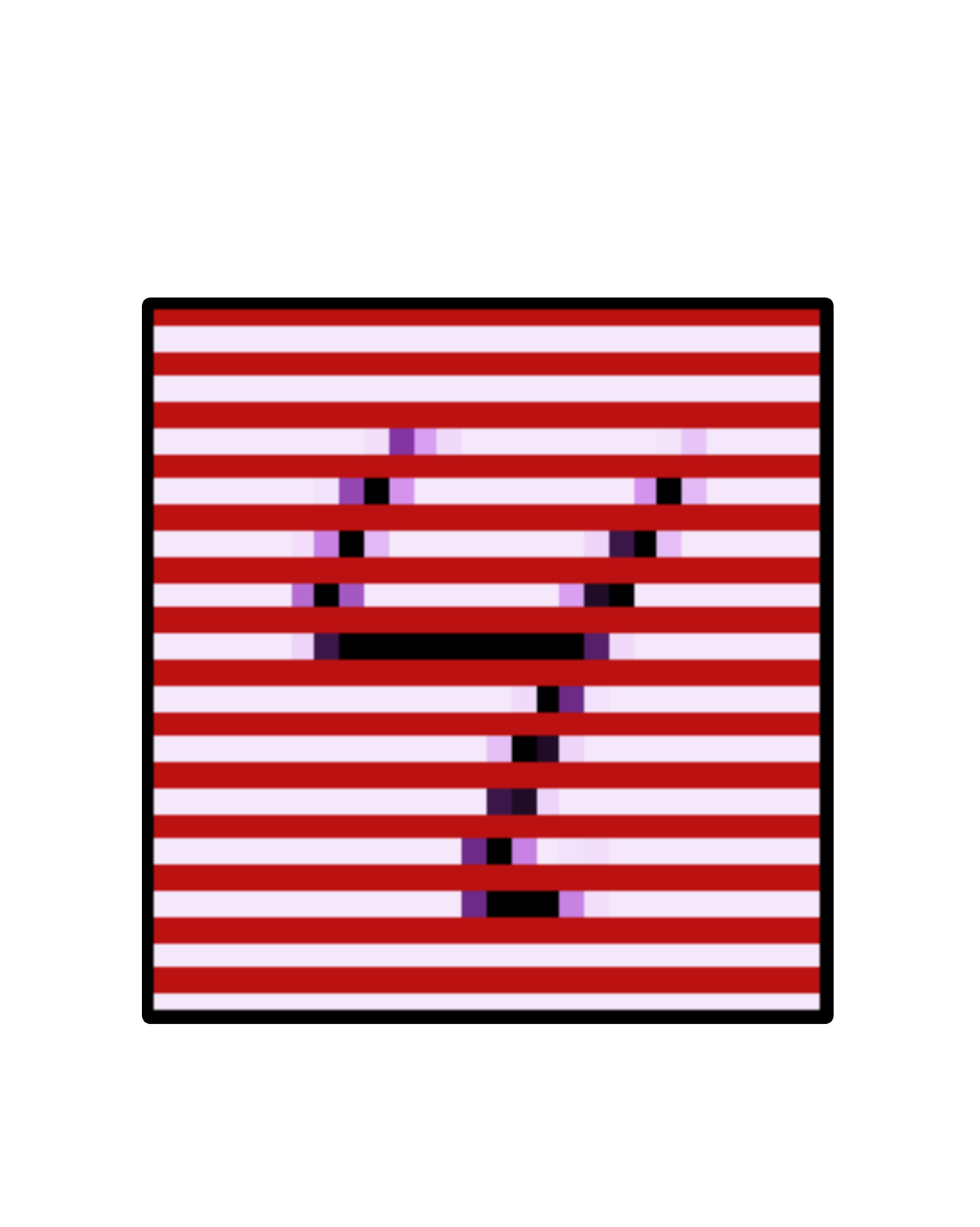}
	}
    \subfigure[Hor. mask]{
	    \includegraphics[height=\panelwidth]{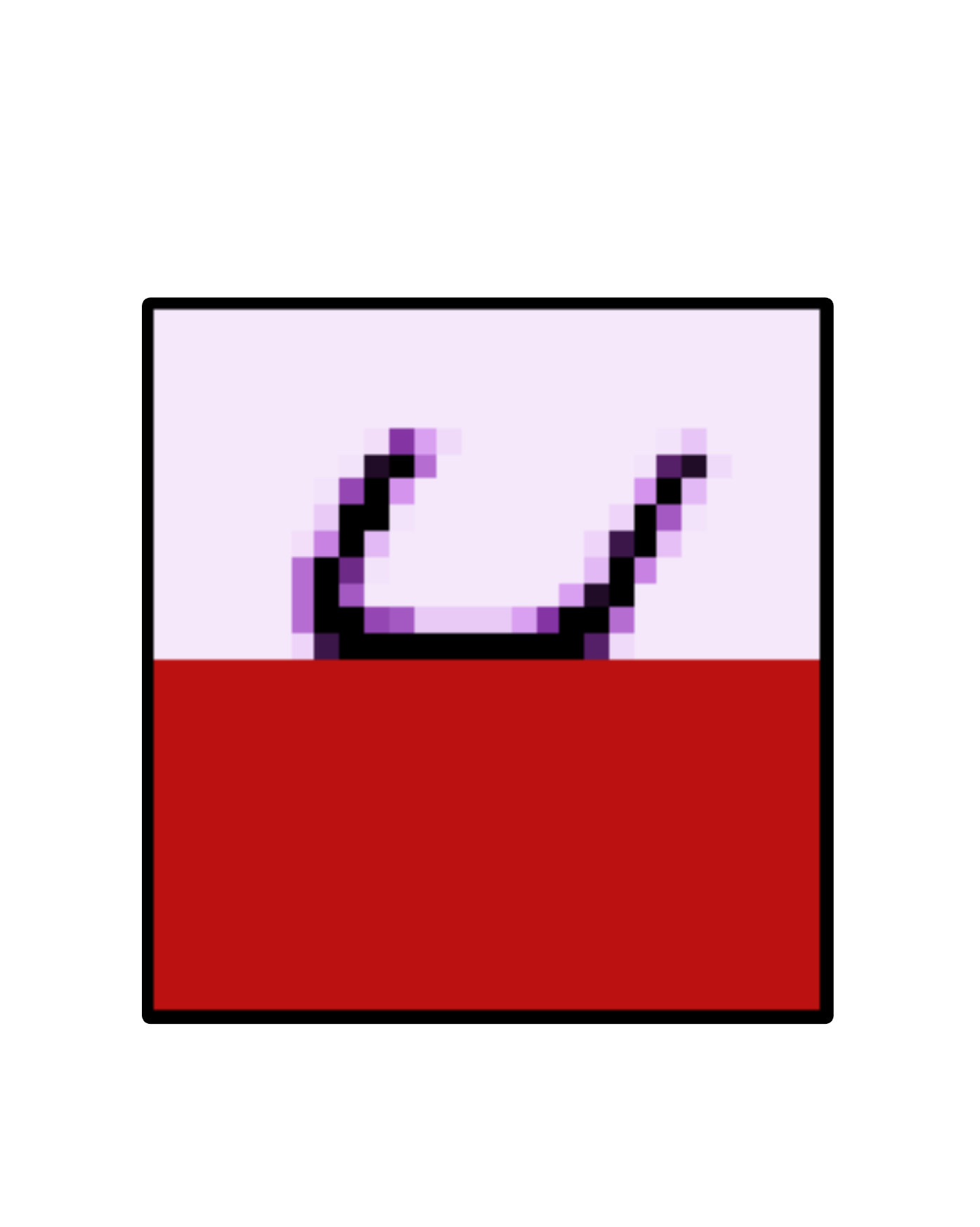}
	}
	\caption{
    \textbf{Squeeze layers and masks.}
    \textbf{(a)}:
    A squeeze layer squeezes an image of size $c \times h \times w$ into $4c \times h / 2 \times w / 2$. 
    The first panel shows the mask, where each color corresponds to a channel 
    added by the squeeze layer (for visual clarity we show the mask for a $12 \times 12$ image).
    The second panel shows a $1 \times 28 \times 28$ MNIST digit, and the last panel
    shows the $4$ channels produced by the squeeze layer. 
    The colors of the boundaries of the channel visualizations correspond to the
    colors of the pixels in the mask.
    Each channel produced by the squeeze layer is a subsampled version of the input image.
    \textbf{(b)-(d)}:
    Checkerboard, channel-wise and horizontal masks applied to the same input image.
    Masked regions are shown in red.
    Channel-wise mask is obtained by applying a squeeze layer 
    and masking two of the channels (e.g. the last two); here we show the masked
    pixels in the un-squeezed image.
    Masks are typically alternated: in the subsequent layers the masked and observed
    positions are swapped. 
	}
	\label{fig:app_squeeze}
\end{figure}

Normalizing flows such as RealNVP and Glow consist of a sequence of coupling layers which change the content of the input and squeeze layers (see Figure \ref{fig:app_squeeze}) which reshape it.
Due to the presence of squeeze layers, the latent representations of the flow have a different shape compared to the input. 
In order to visualize latent representations, we revert all squeezing operations of the flow and visualize \texttt{unsqueeze}$(z)$. Similarly, for visualization of coupling layer activations and scale and shift parameters predicted by $st$-network, we revert all squeezing operations and join all factored out tensors in the case of multi-scale architecture (i.e., we feed the corresponding tensor through inverse sub-flow without applying coupling layers or invertible convolutions).

\section{Additional latent representation visualizations}
\label{sec:app_latents}

\begin{figure*}[t]
	\def \panelwidth {0.085\textwidth}
	\def \panelskip {-0.6cm}

    \setlength{\tabcolsep}{1pt}
    \renewcommand{\arraystretch}{0.2}
    \centering

    \subfigure[RealNVP trained on FashionMNIST]{
    \begin{tabular}{cccccc}
		\rotatebox{90}{\quad~ $x$}&
        \hspace{-0.1cm} \includegraphics[width=\panelwidth]{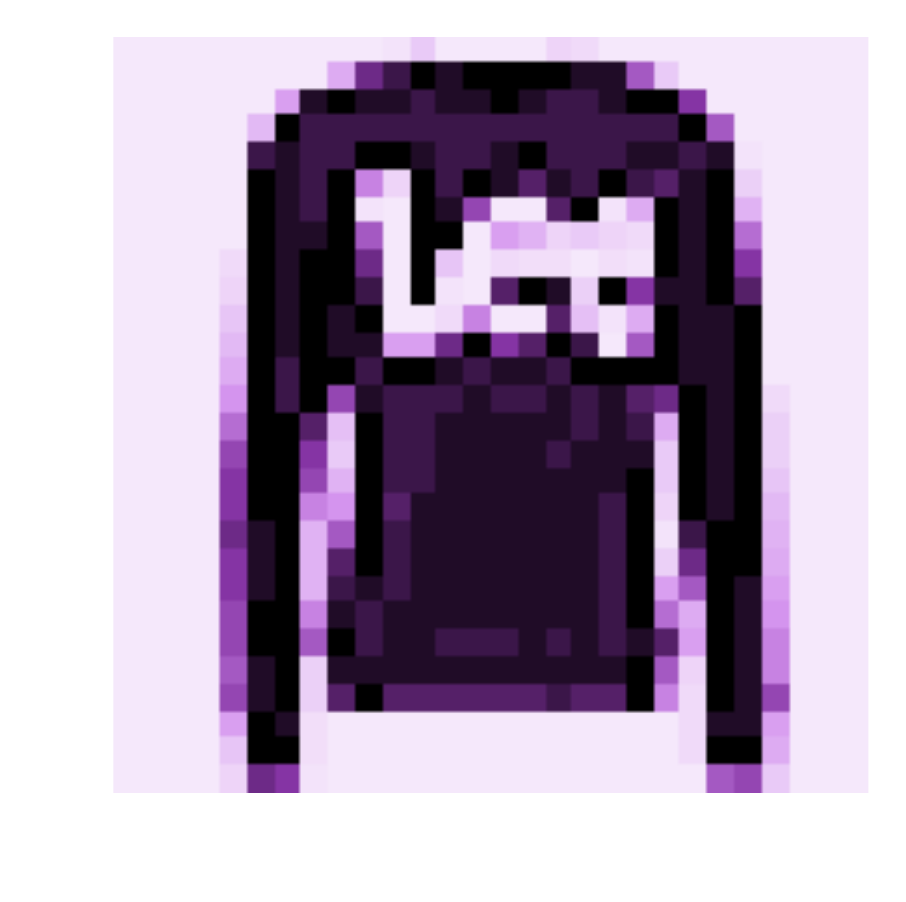}&
        \includegraphics[width=\panelwidth]{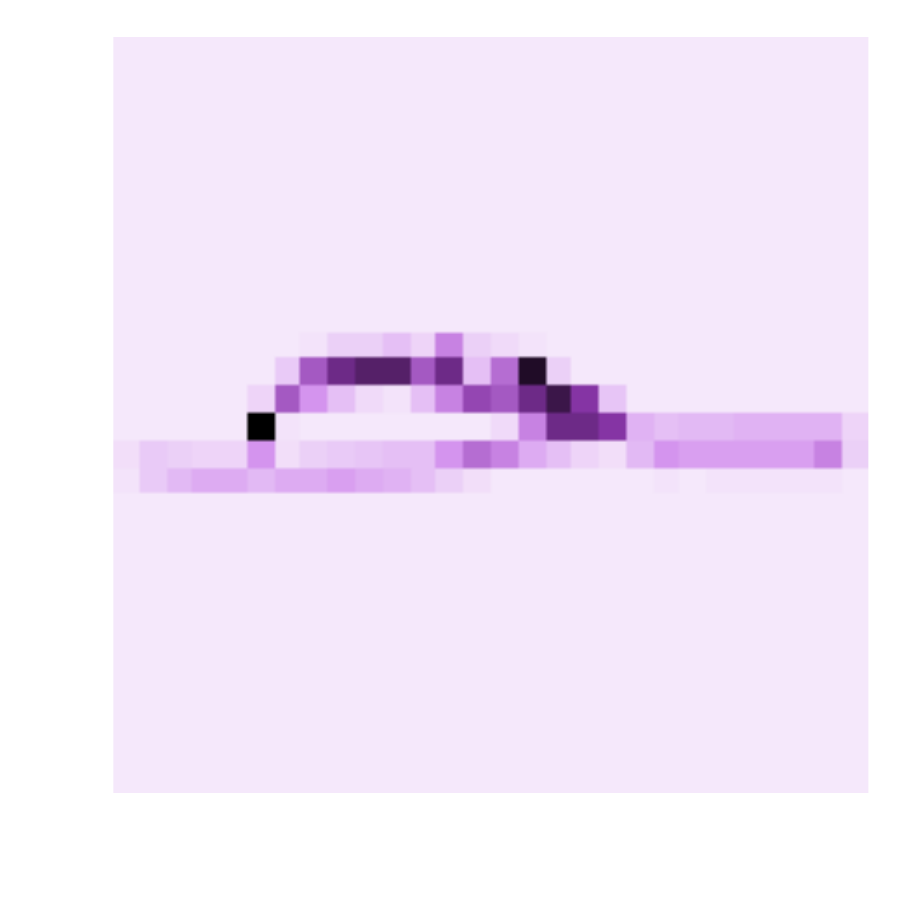}&
        \includegraphics[width=\panelwidth]{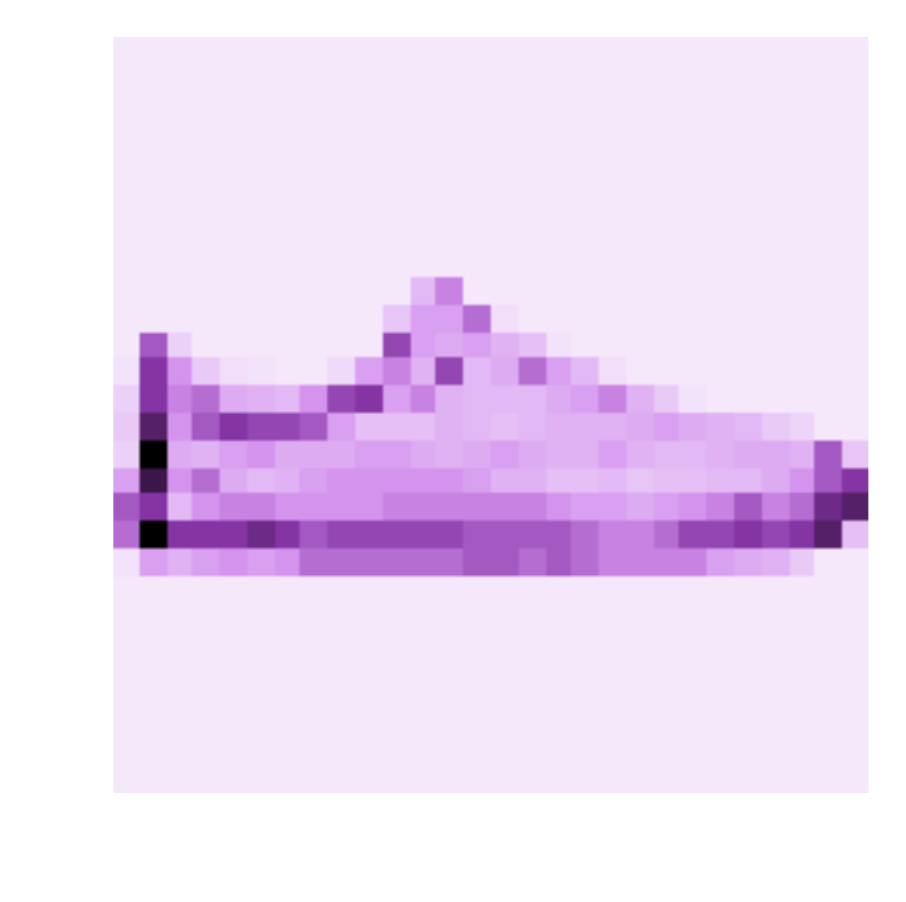}&
        \includegraphics[width=\panelwidth]{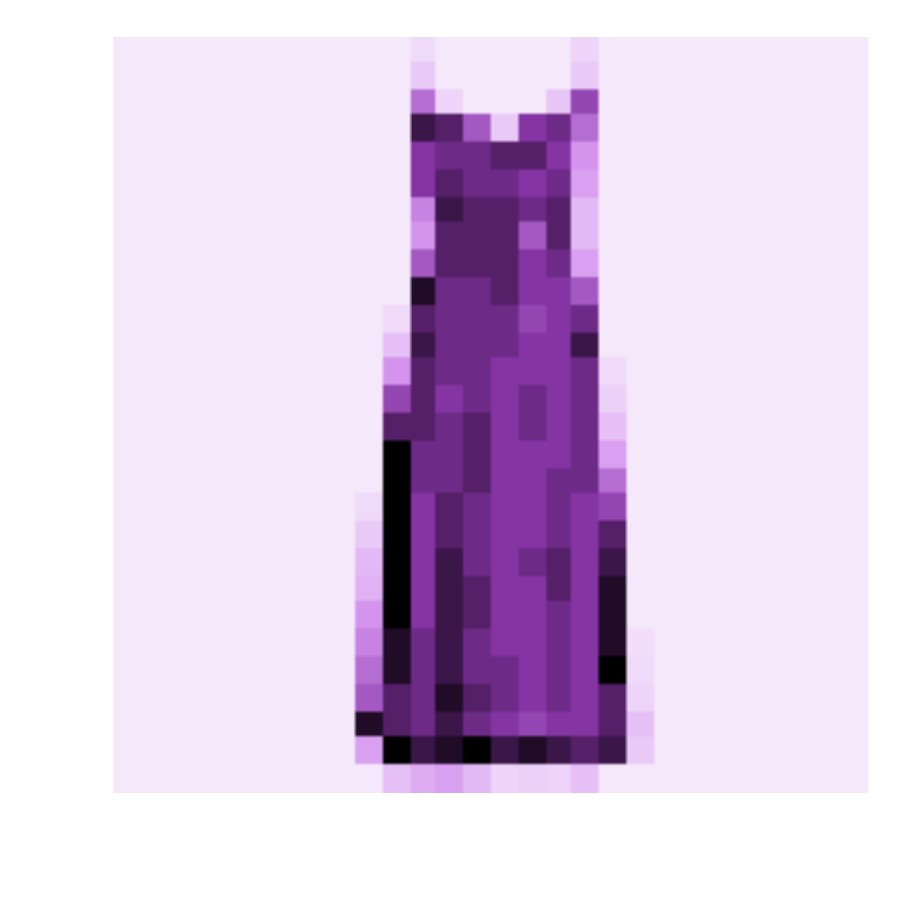}&
        \includegraphics[width=\panelwidth]{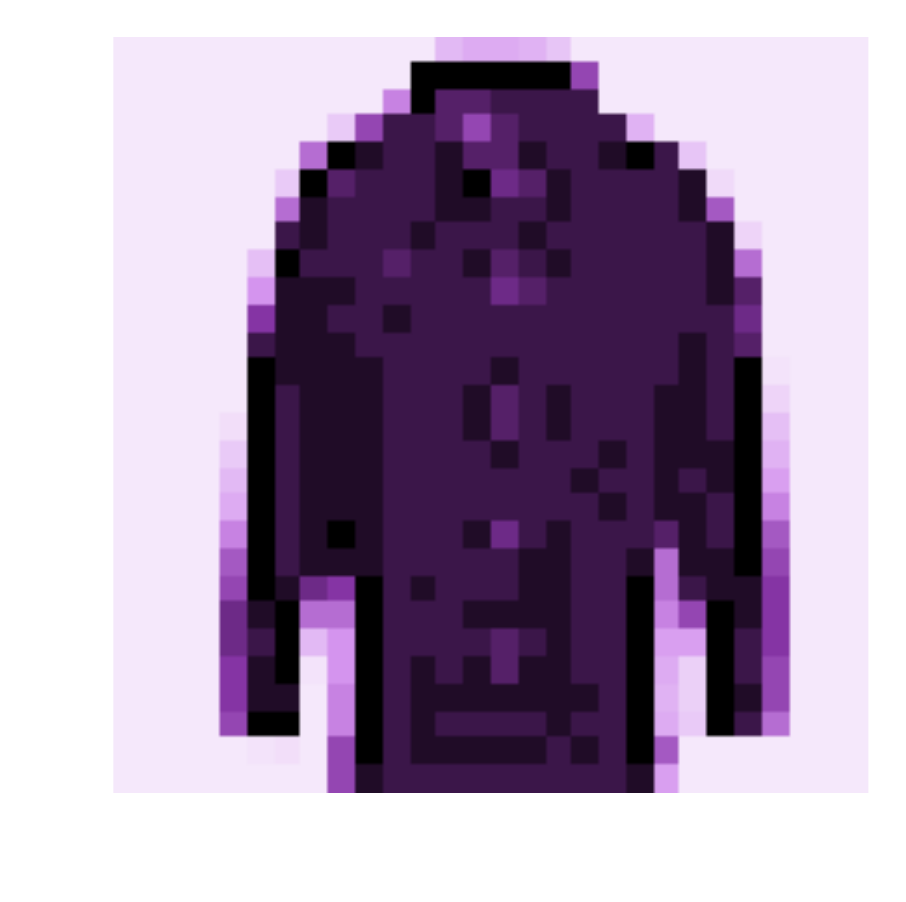}
        \\
		\rotatebox{90}{\quad~ $z$}&
        \hspace{-0.1cm} \includegraphics[width=\panelwidth]{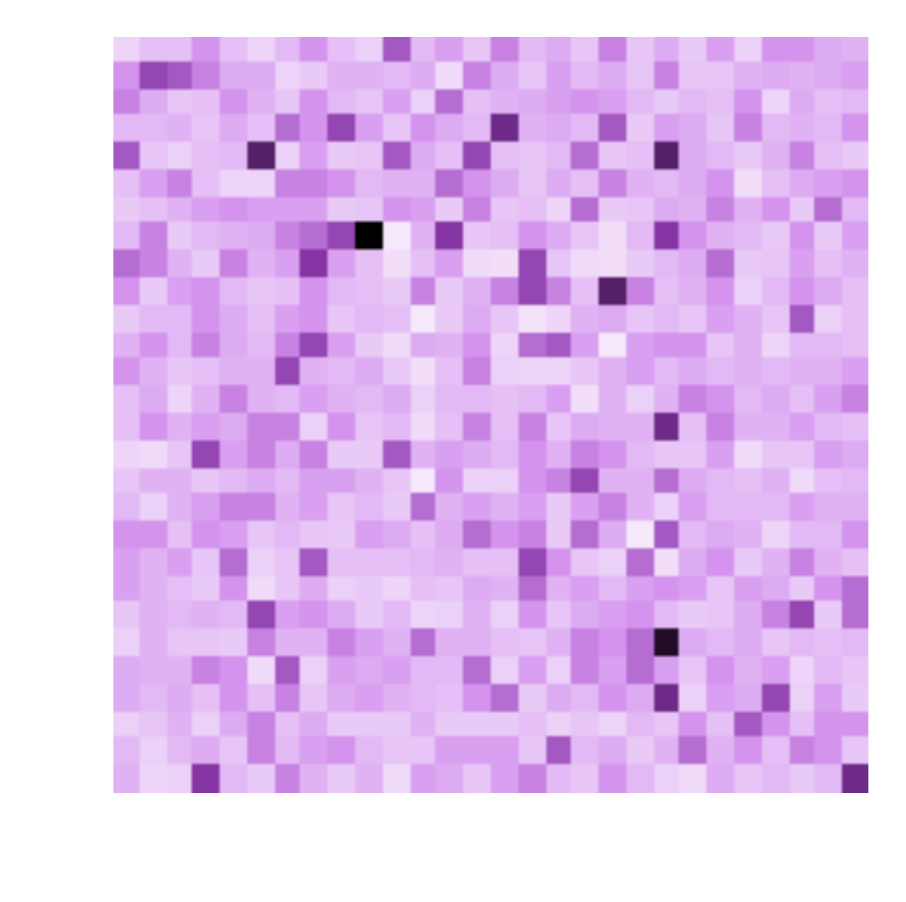}&
        \includegraphics[width=\panelwidth]{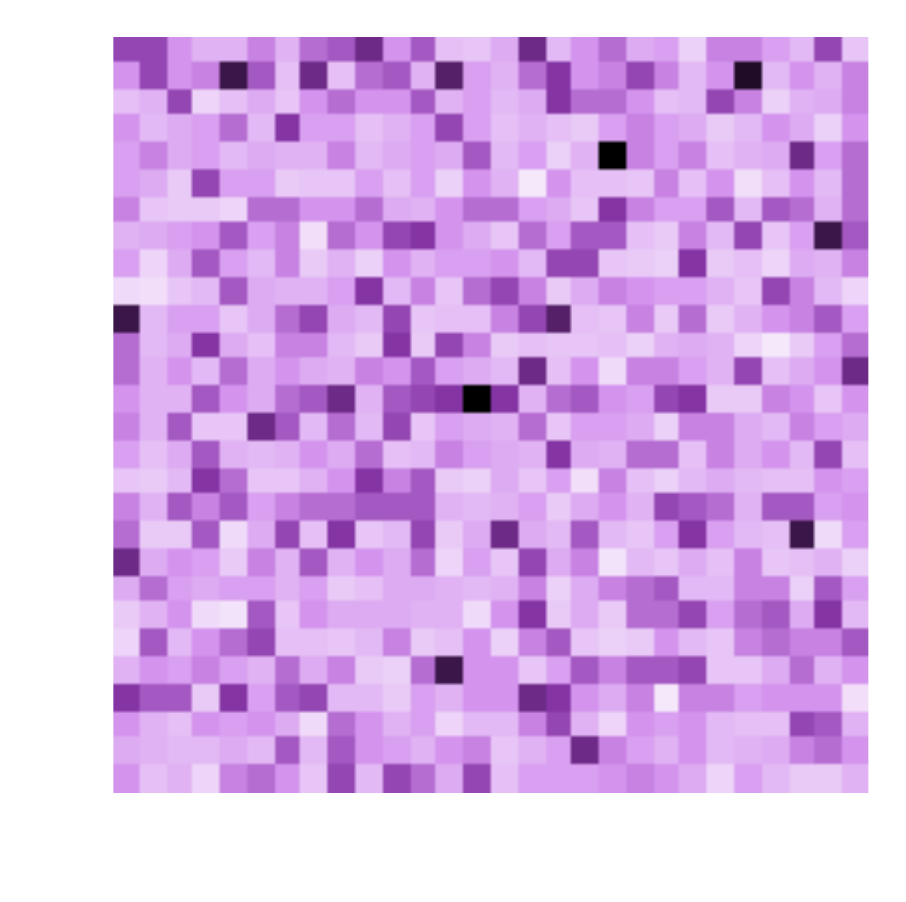}&
        \includegraphics[width=\panelwidth]{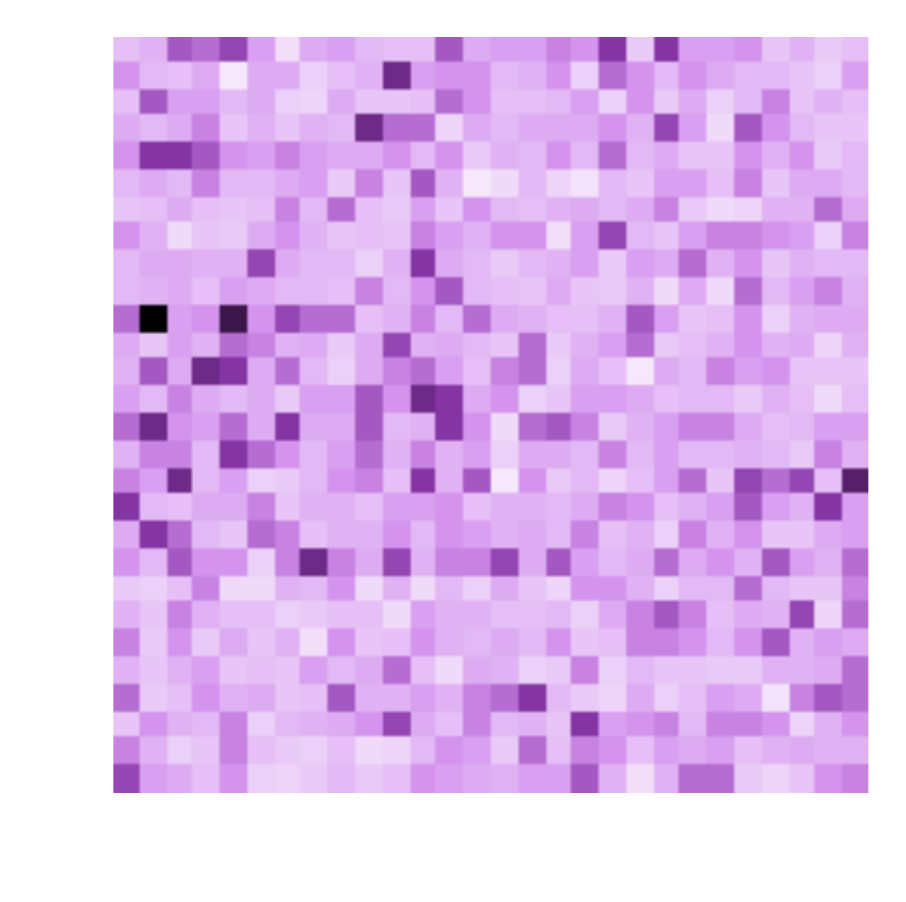}&
        \includegraphics[width=\panelwidth]{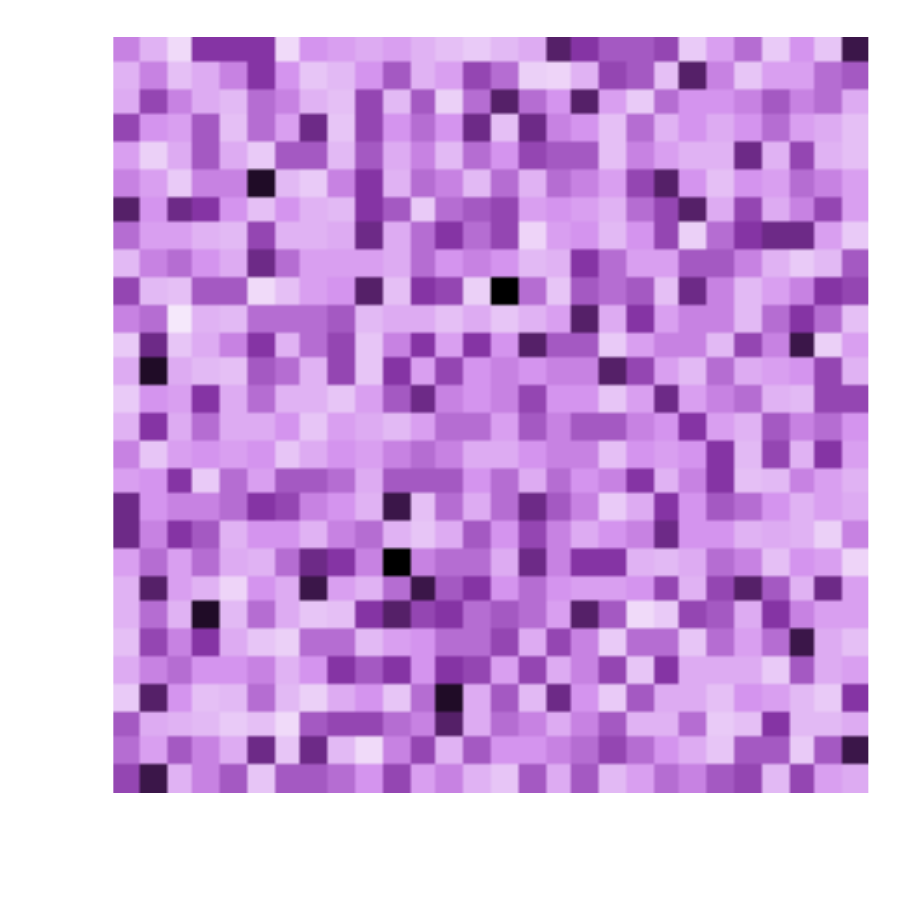}&
        \includegraphics[width=\panelwidth]{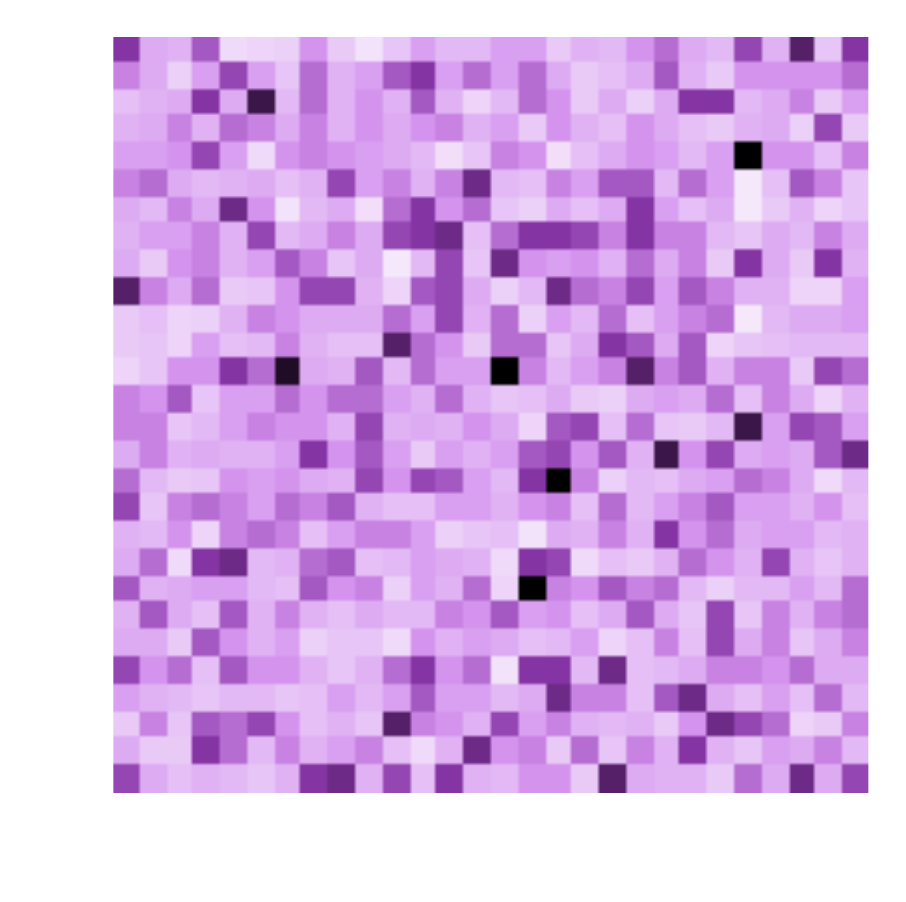}
        \\
		\rotatebox{90}{~~~Avg $z$}&
        \hspace{-0.1cm} \includegraphics[width=\panelwidth]{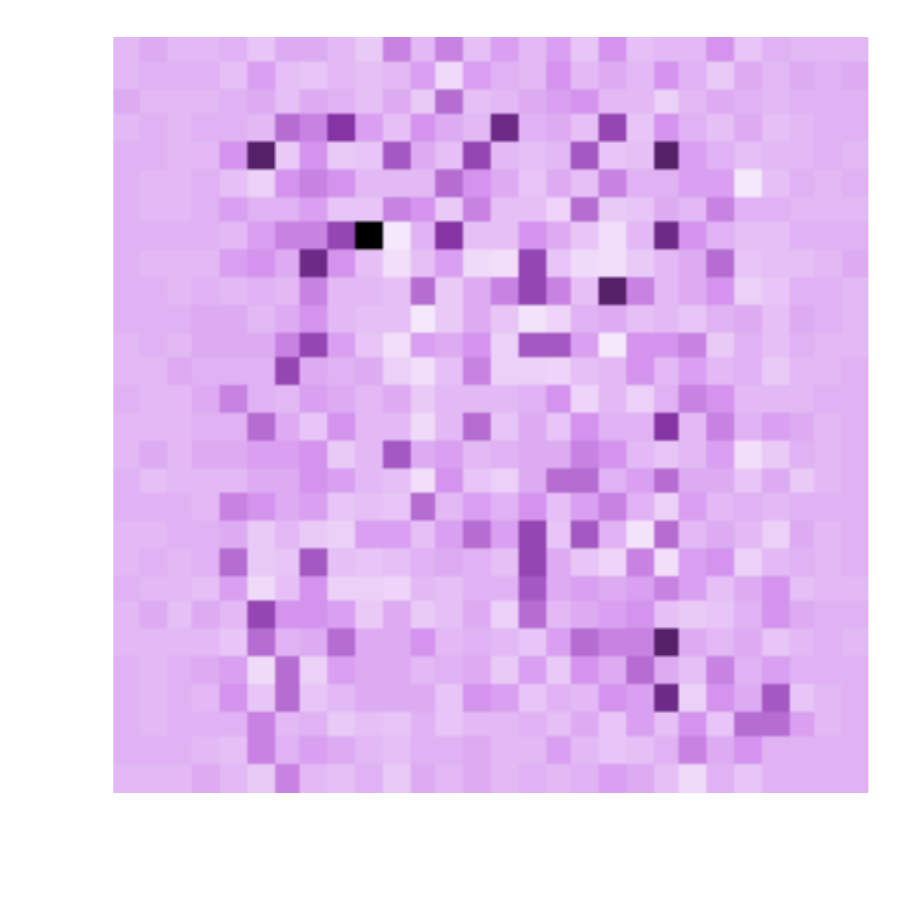}&
        \includegraphics[width=\panelwidth]{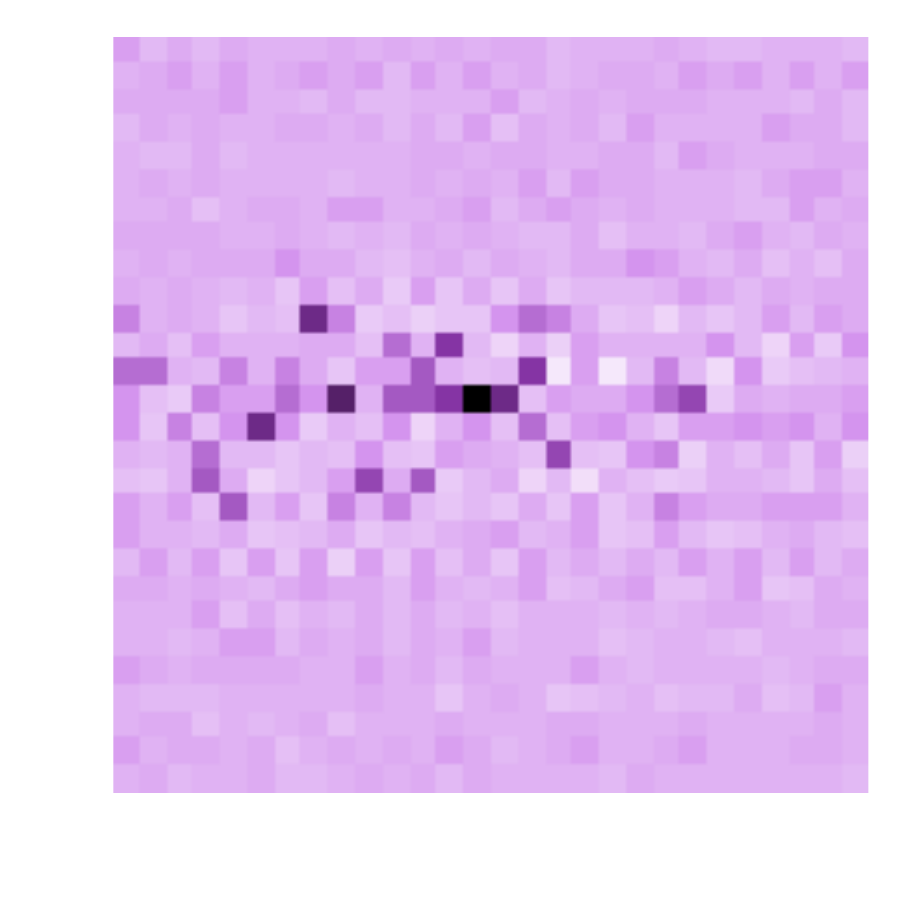}&
        \includegraphics[width=\panelwidth]{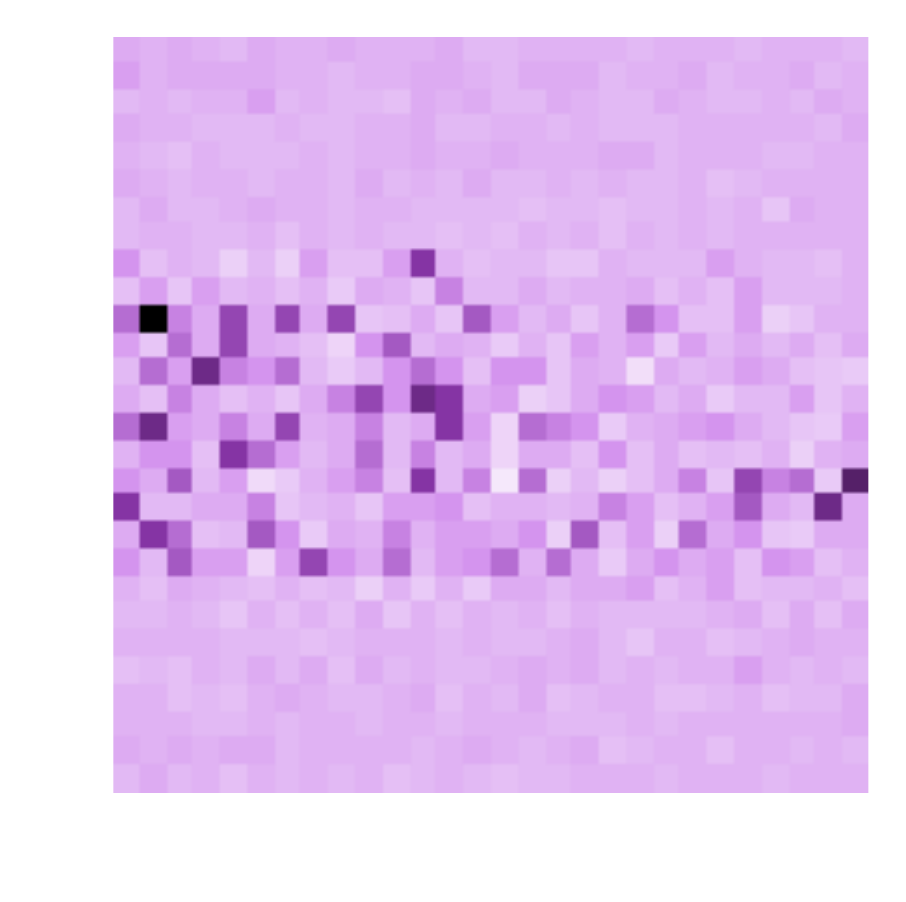}&
        \includegraphics[width=\panelwidth]{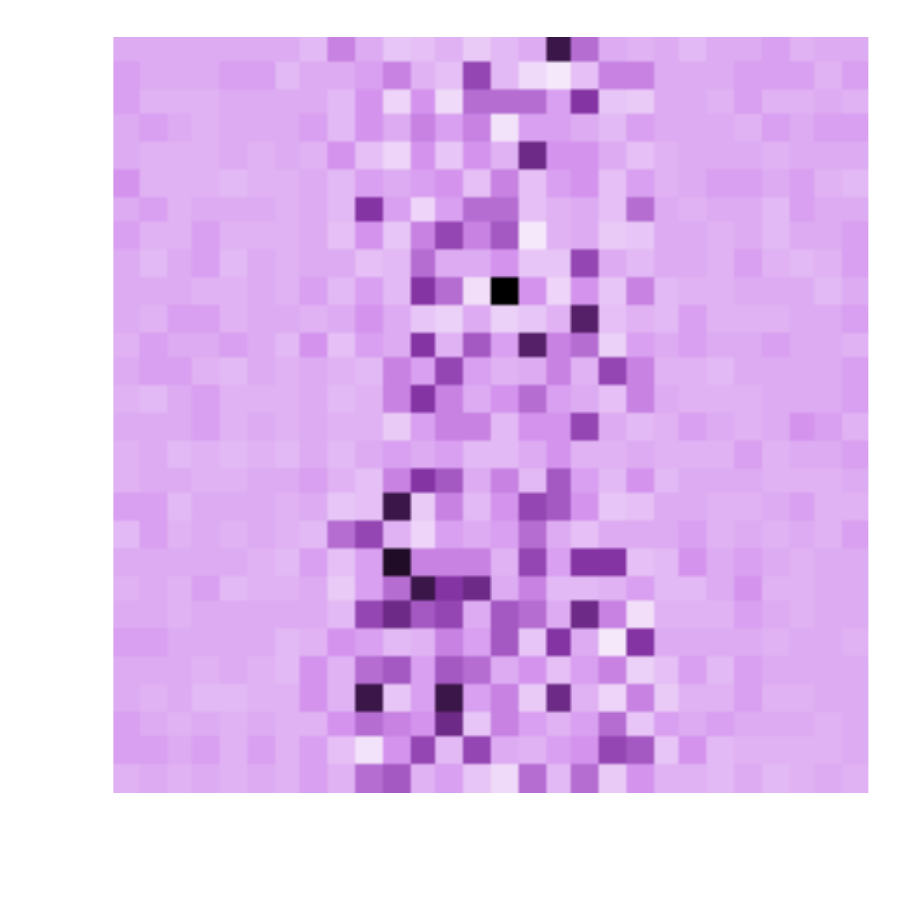}&
        \includegraphics[width=\panelwidth]{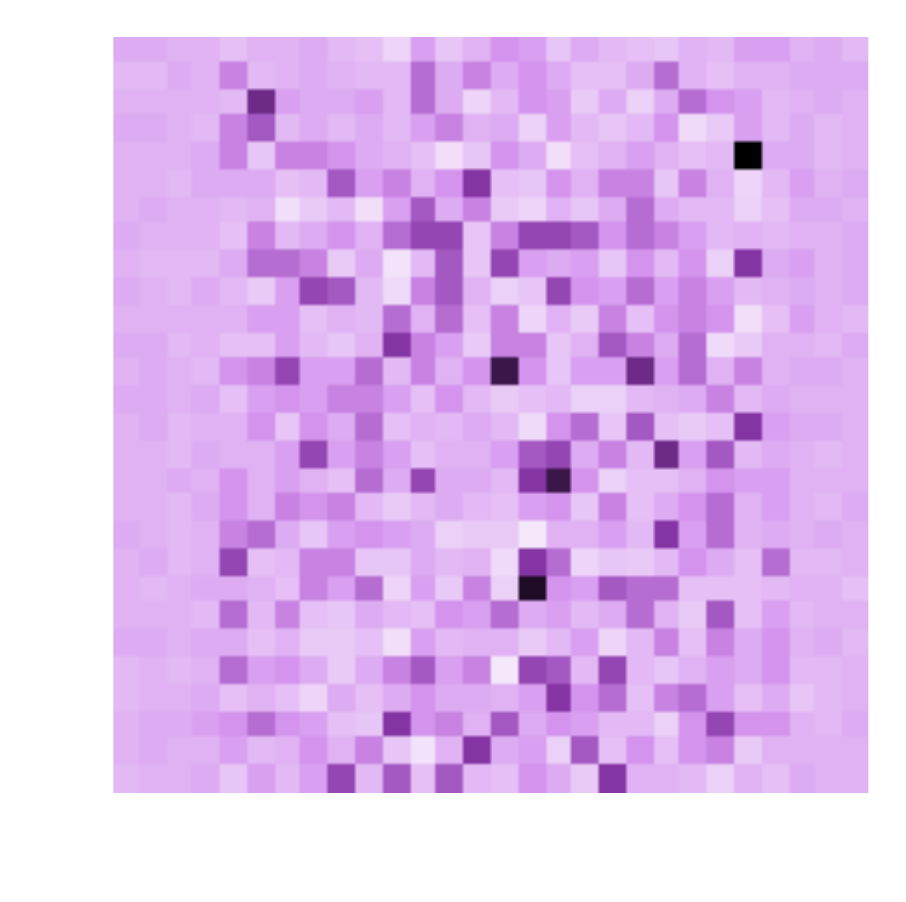}
        \\
		\rotatebox{90}{~~~BN $z$}&
        \hspace{-0.1cm} \includegraphics[width=\panelwidth]{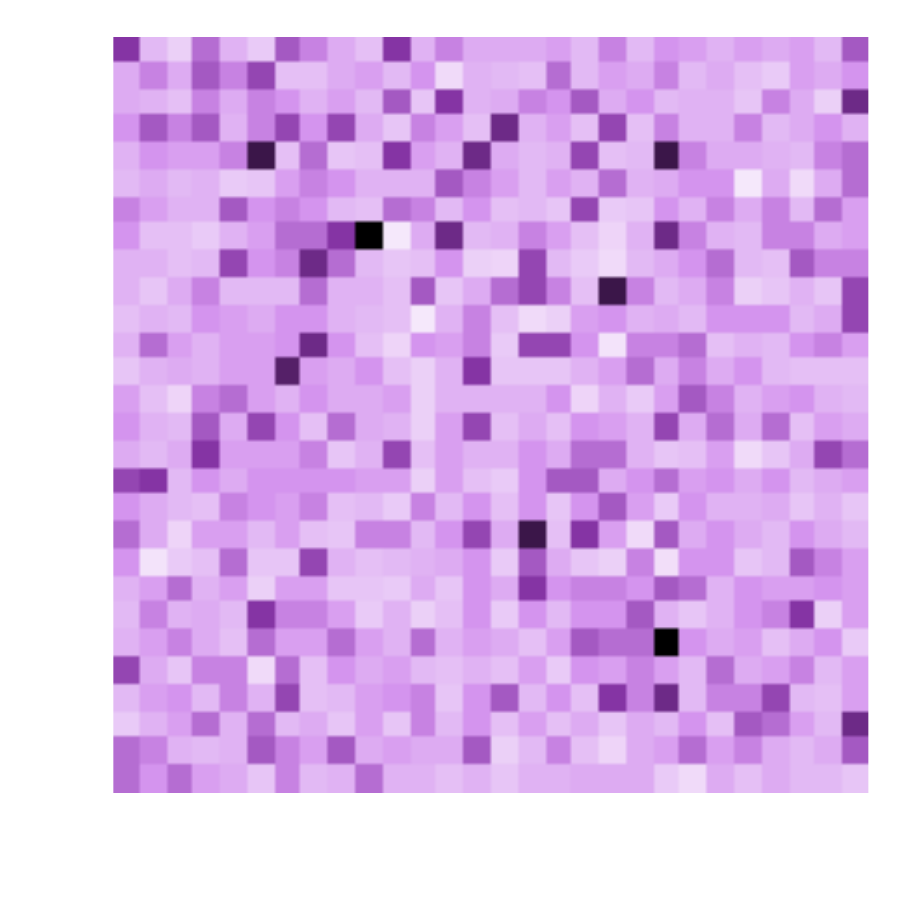}&
        \includegraphics[width=\panelwidth]{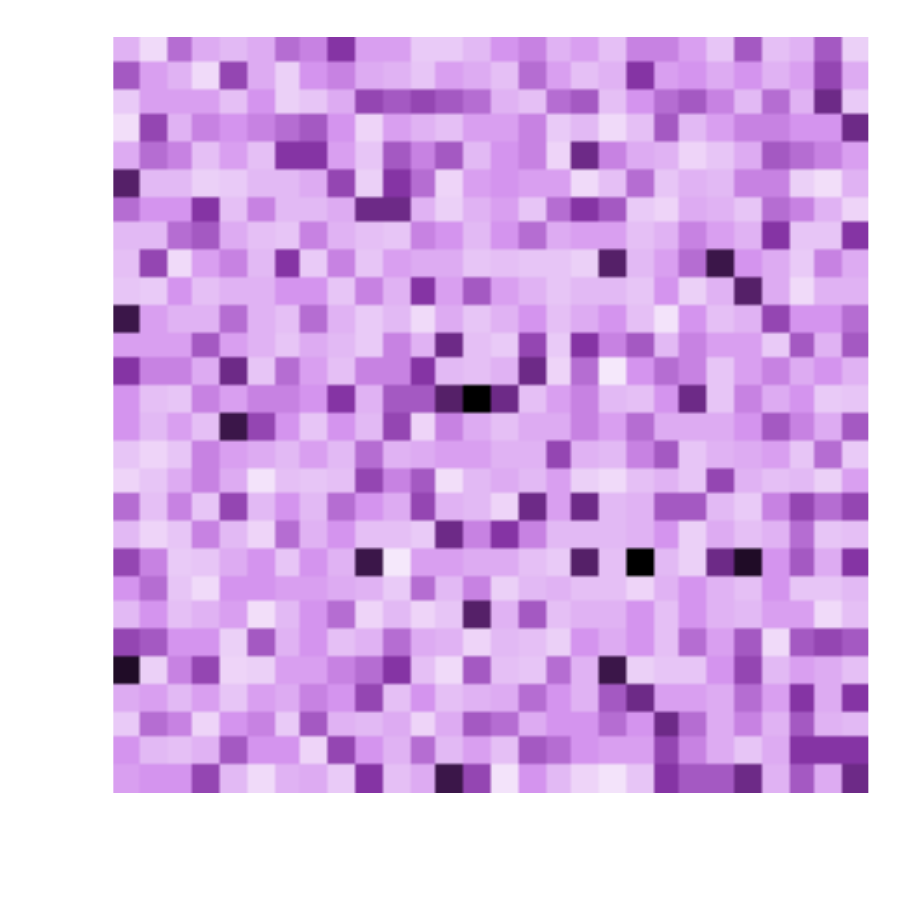}&
        \includegraphics[width=\panelwidth]{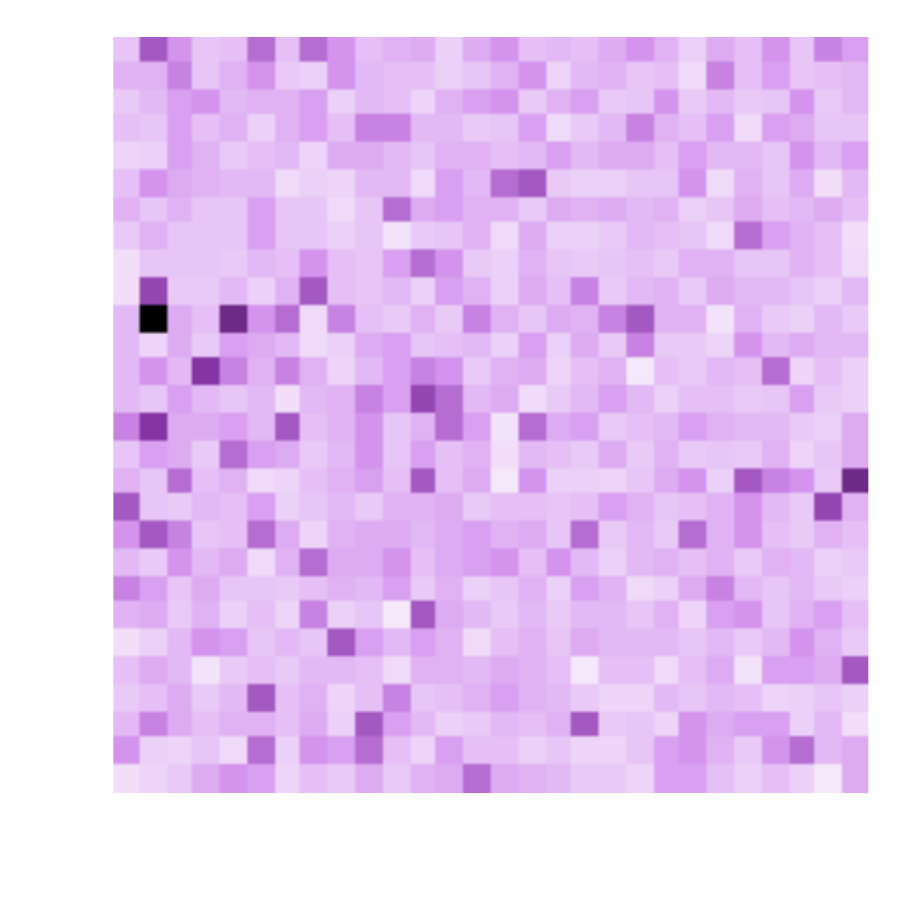}&
        \includegraphics[width=\panelwidth]{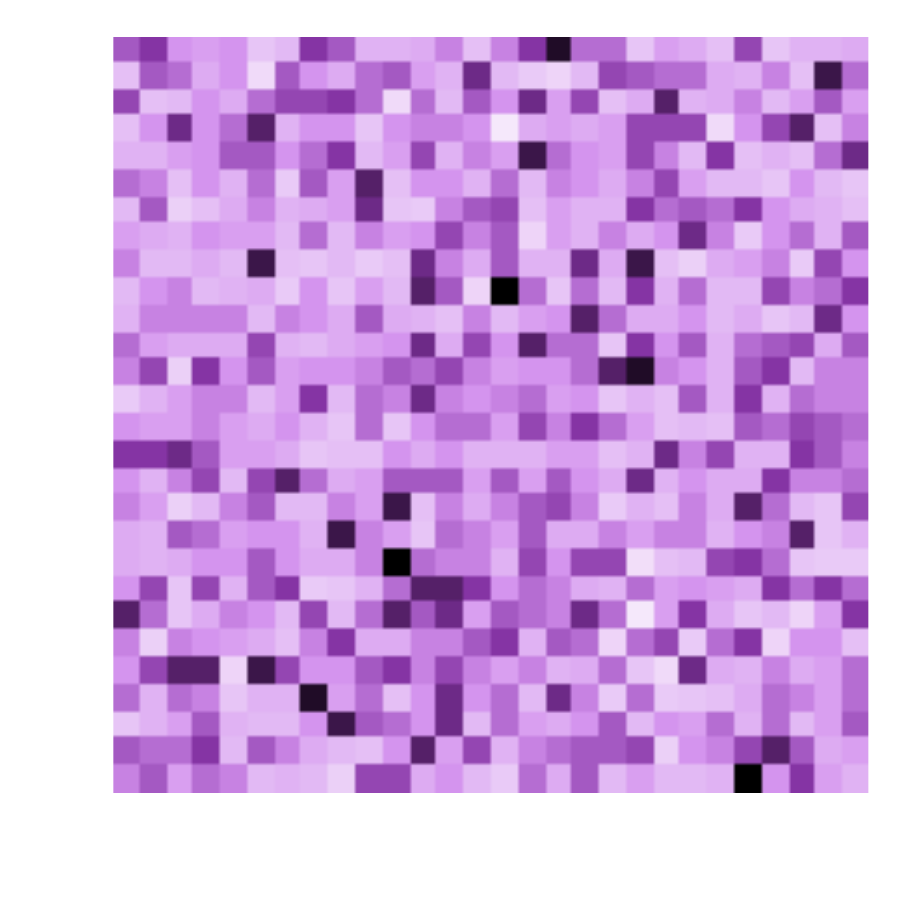}&
        \includegraphics[width=\panelwidth]{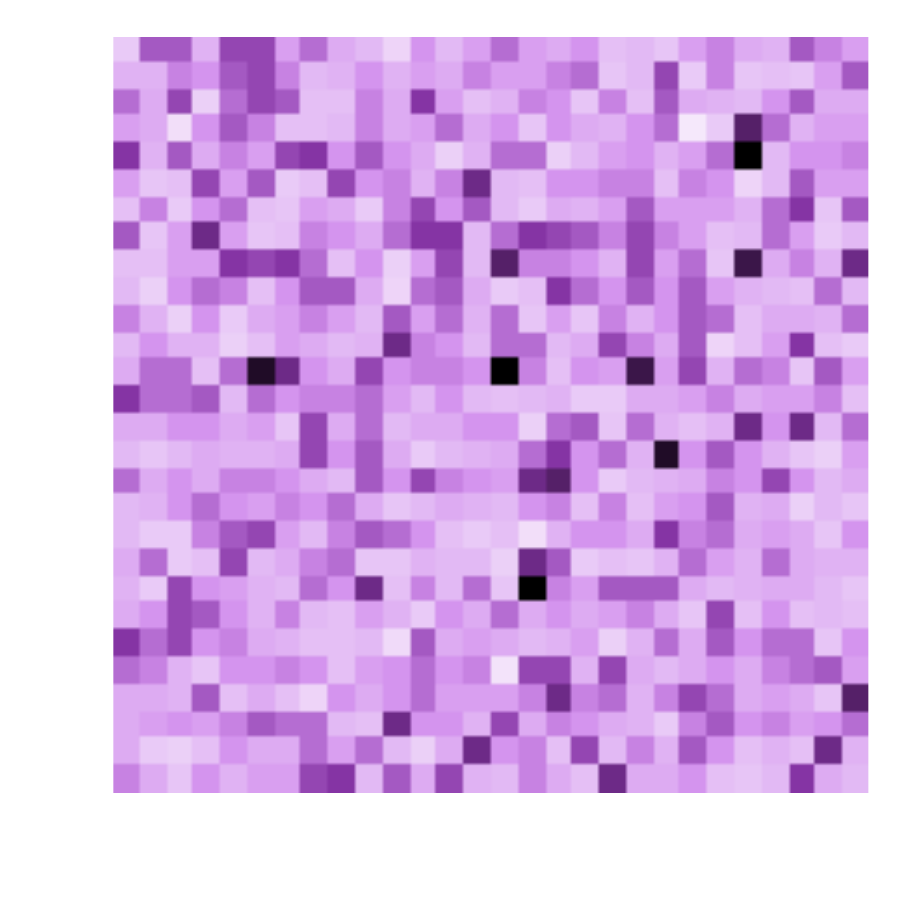}
    \end{tabular}
    \quad
    \begin{tabular}{ccccc}
        \includegraphics[width=\panelwidth]{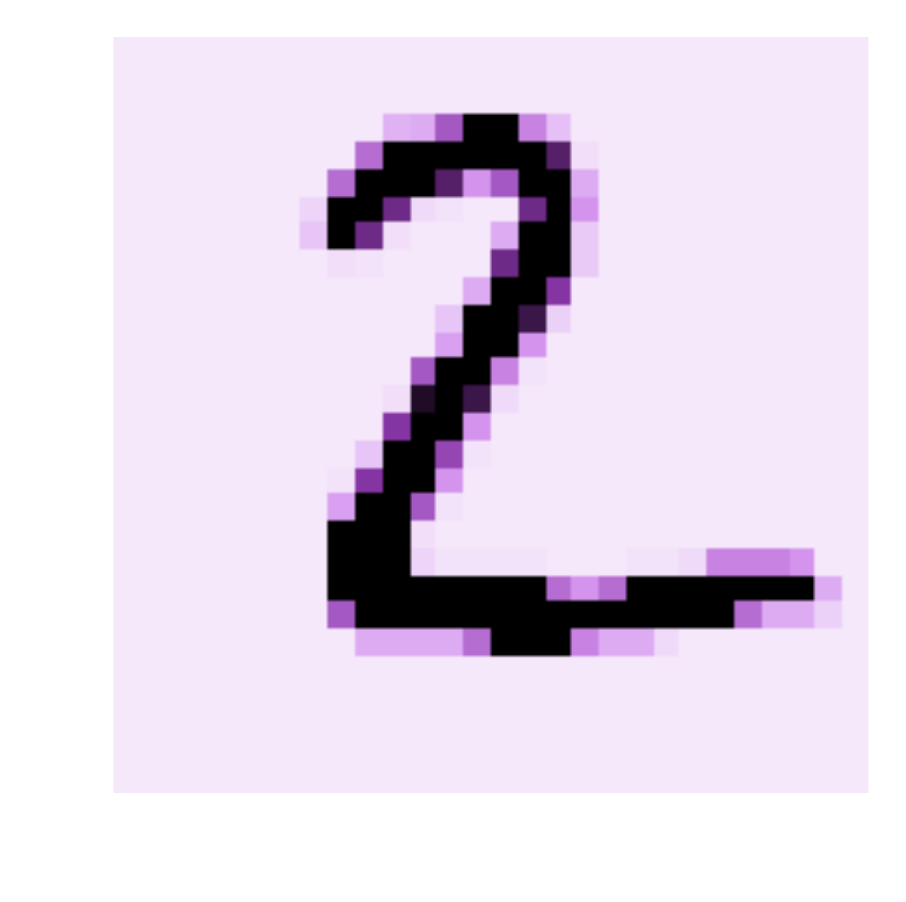}&
        \includegraphics[width=\panelwidth]{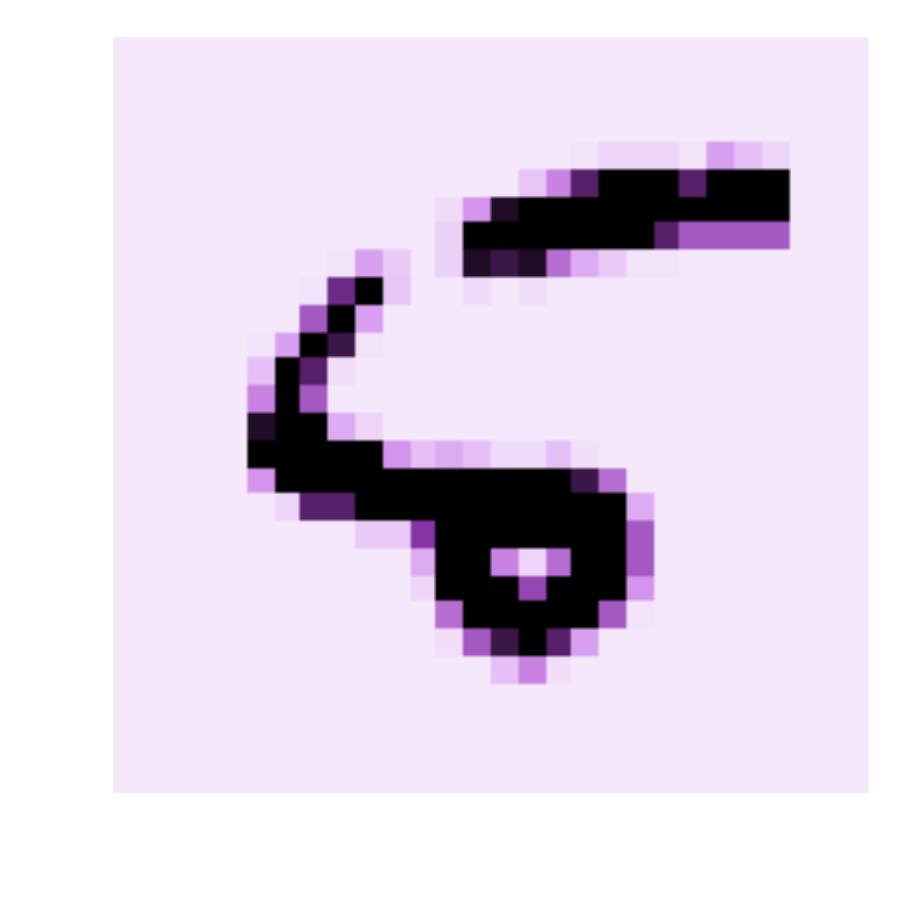}&
        \includegraphics[width=\panelwidth]{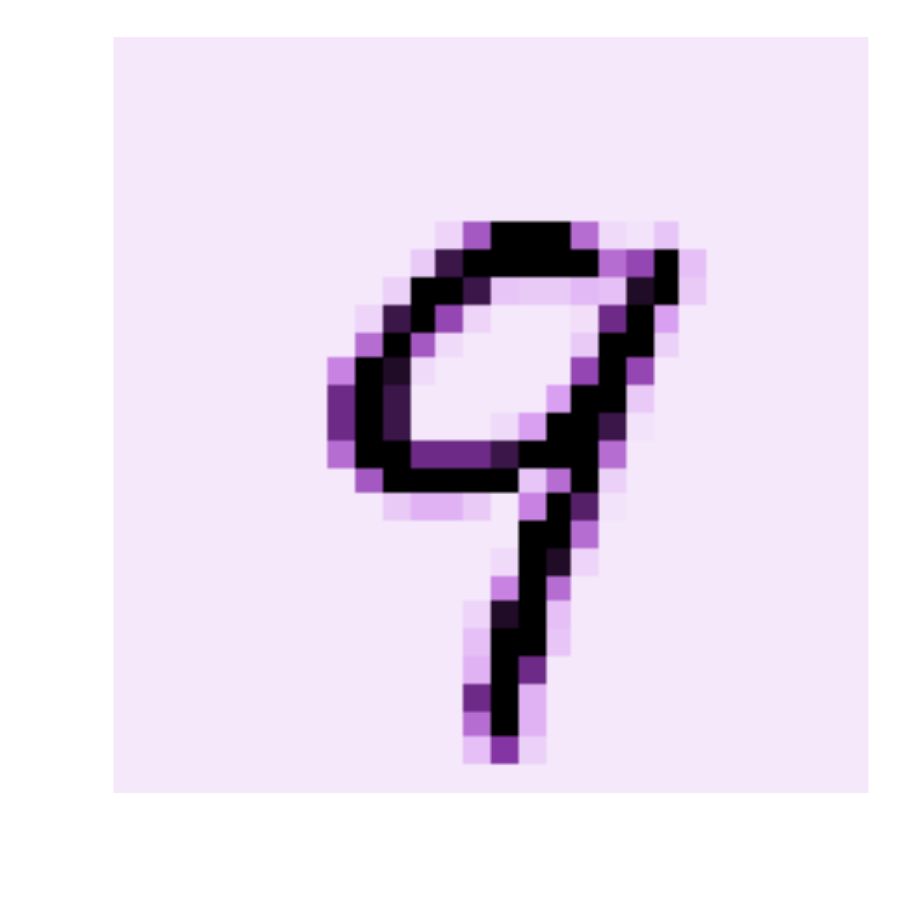}&
        \includegraphics[width=\panelwidth]{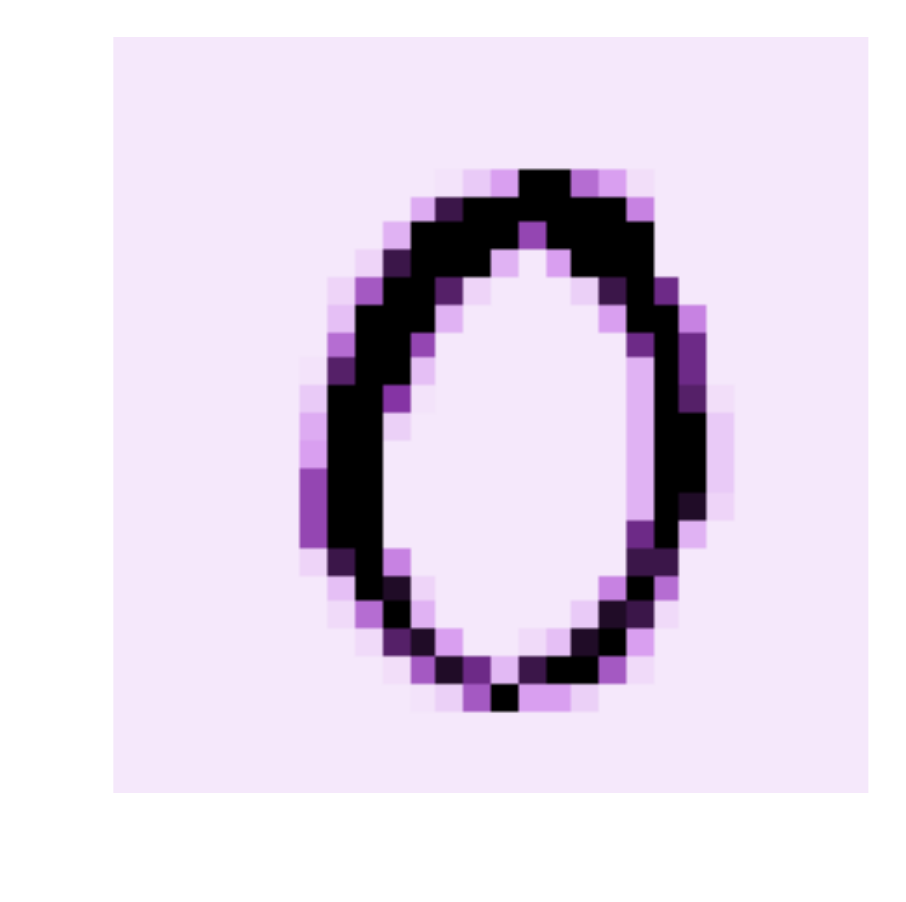}&
        \includegraphics[width=\panelwidth]{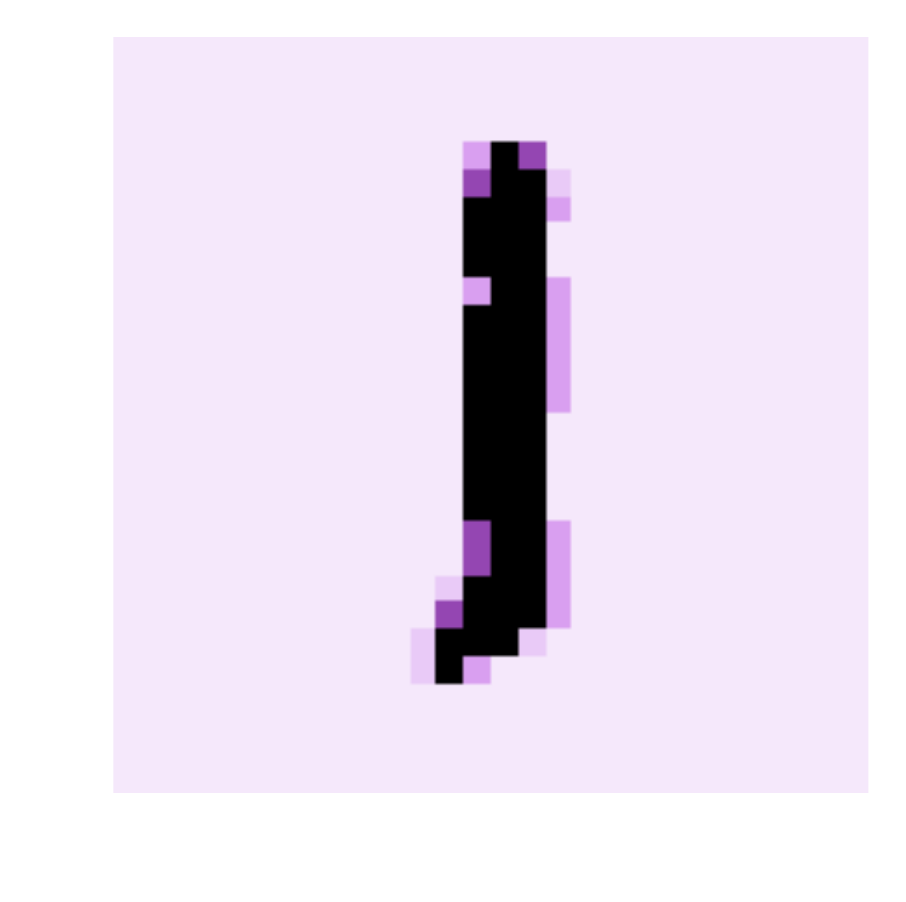}
        \\
        \includegraphics[width=\panelwidth]{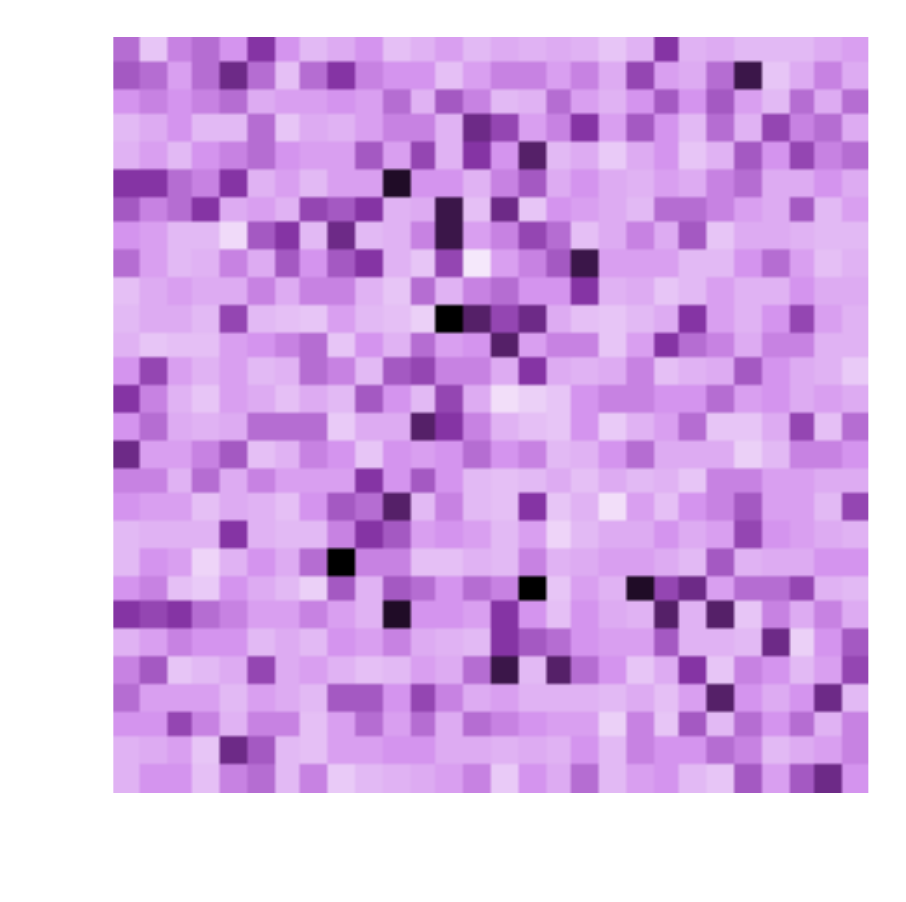}&
        \includegraphics[width=\panelwidth]{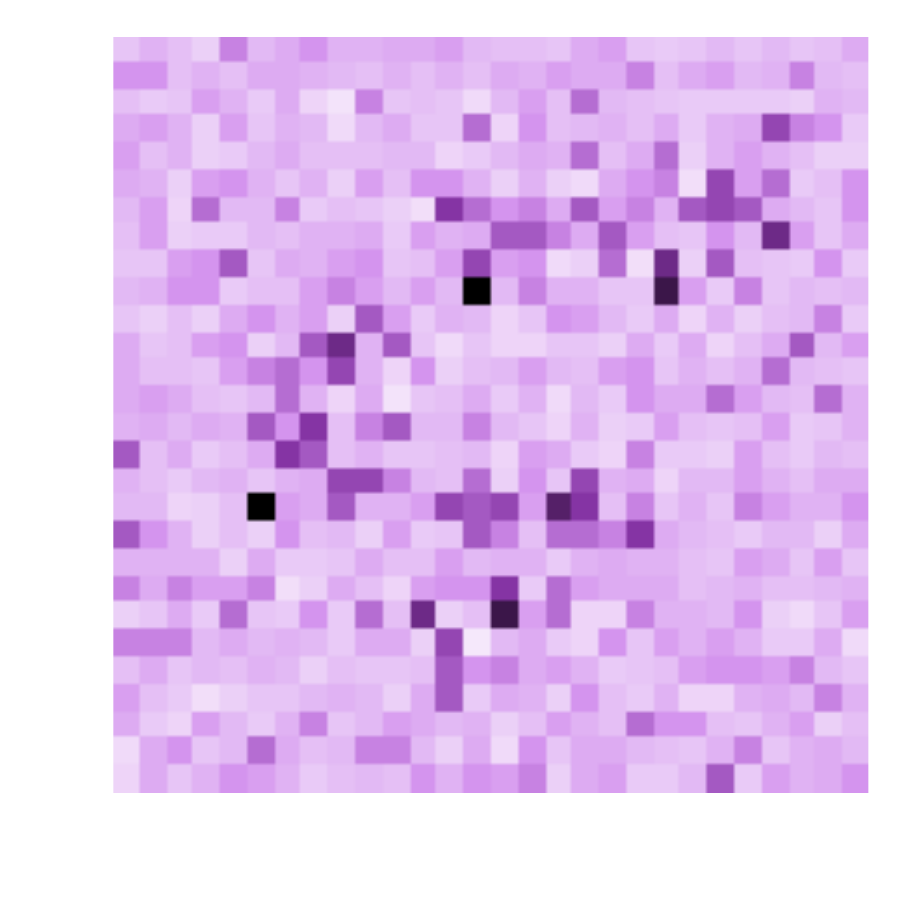}&
        \includegraphics[width=\panelwidth]{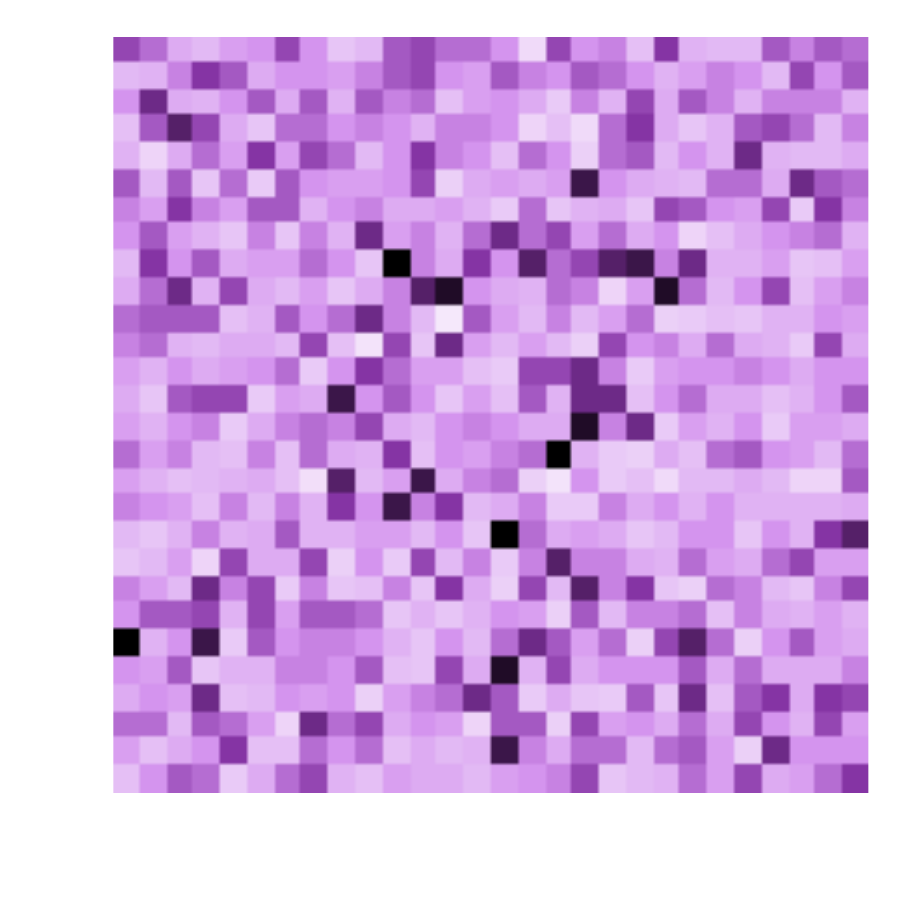}&
        \includegraphics[width=\panelwidth]{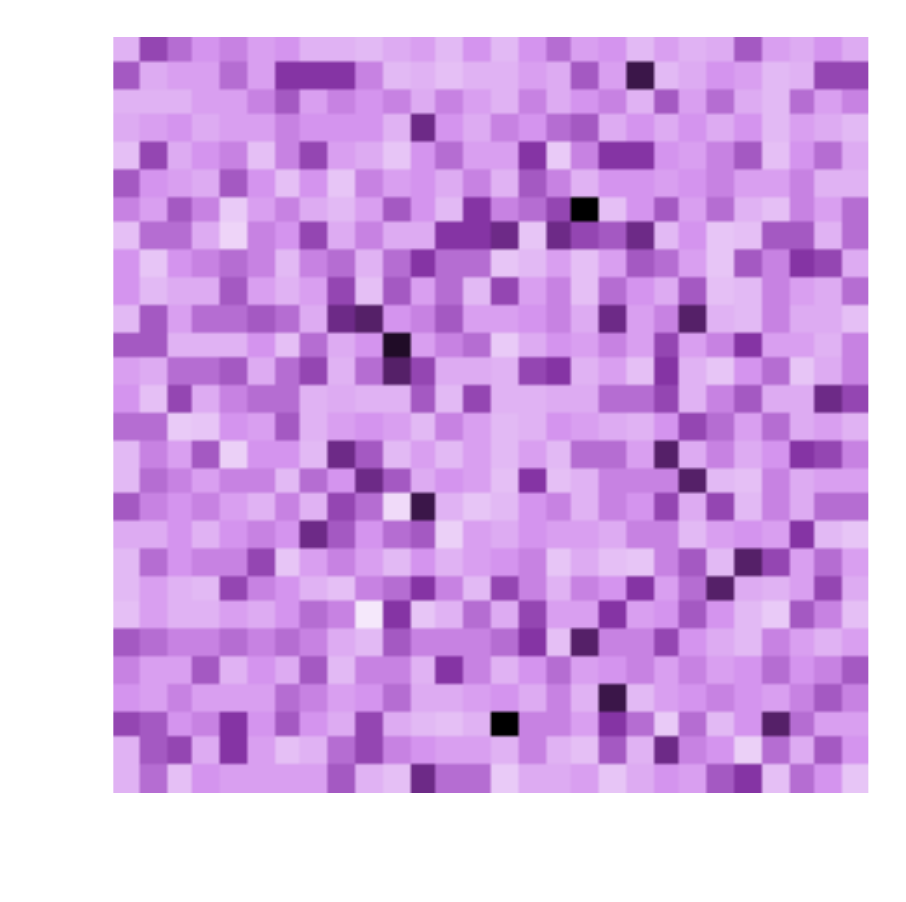}&
        \includegraphics[width=\panelwidth]{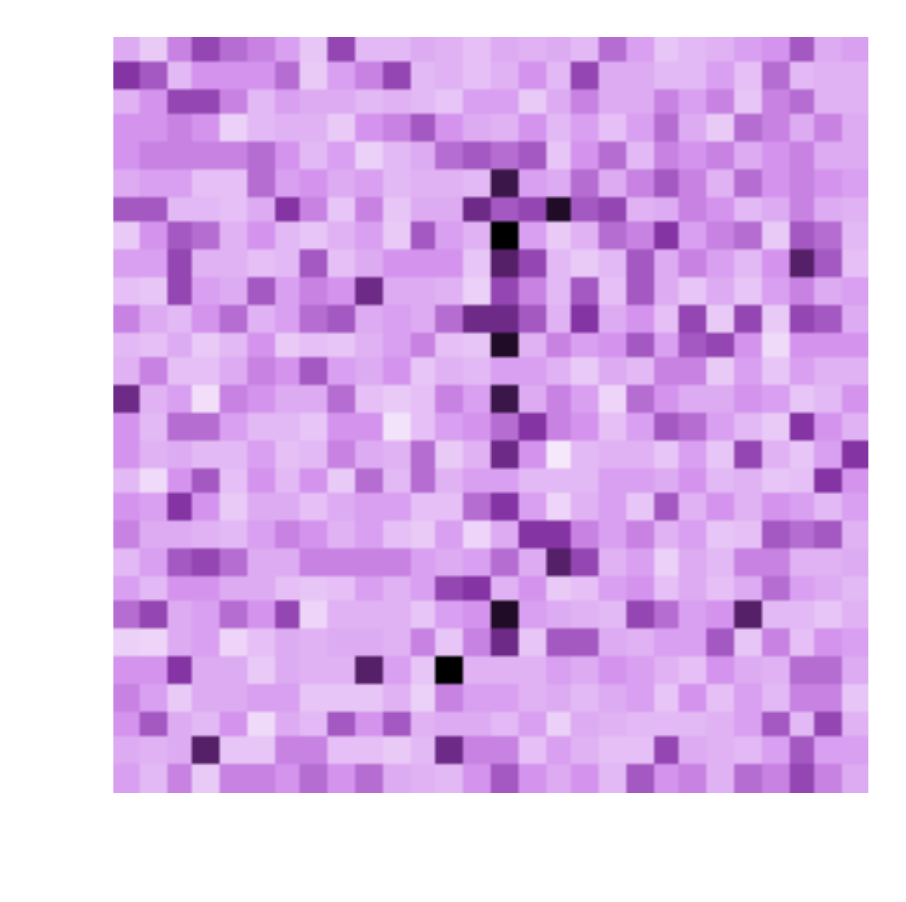}
        \\
        \includegraphics[width=\panelwidth]{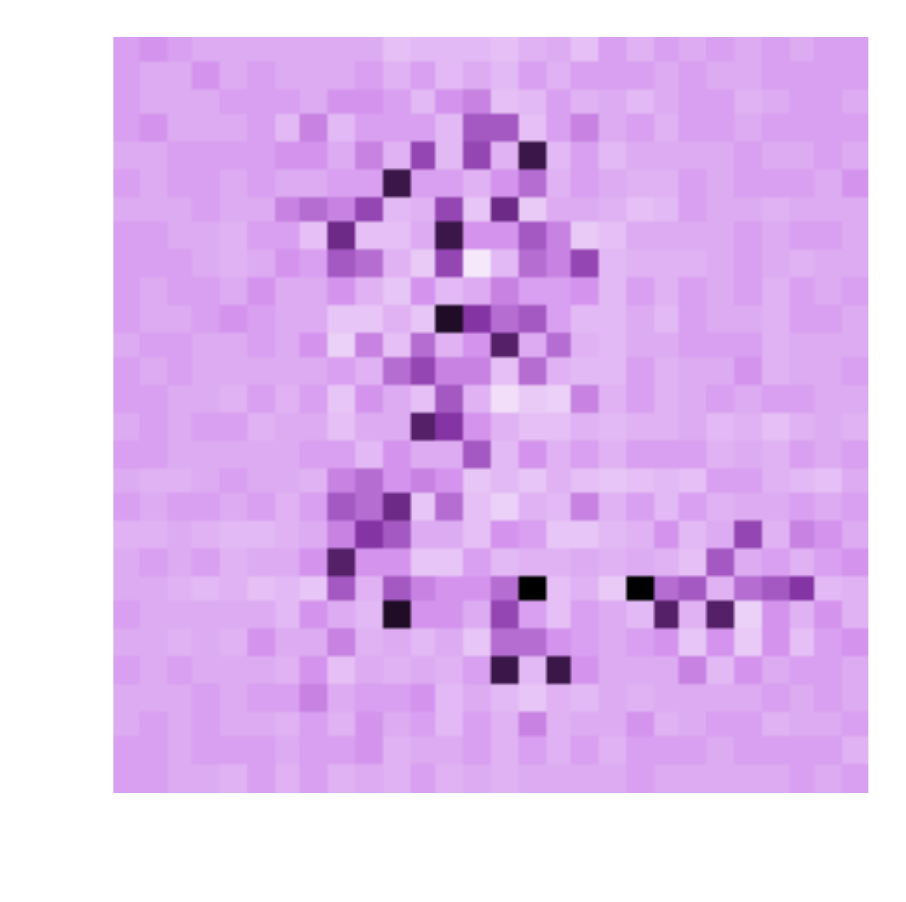}&
        \includegraphics[width=\panelwidth]{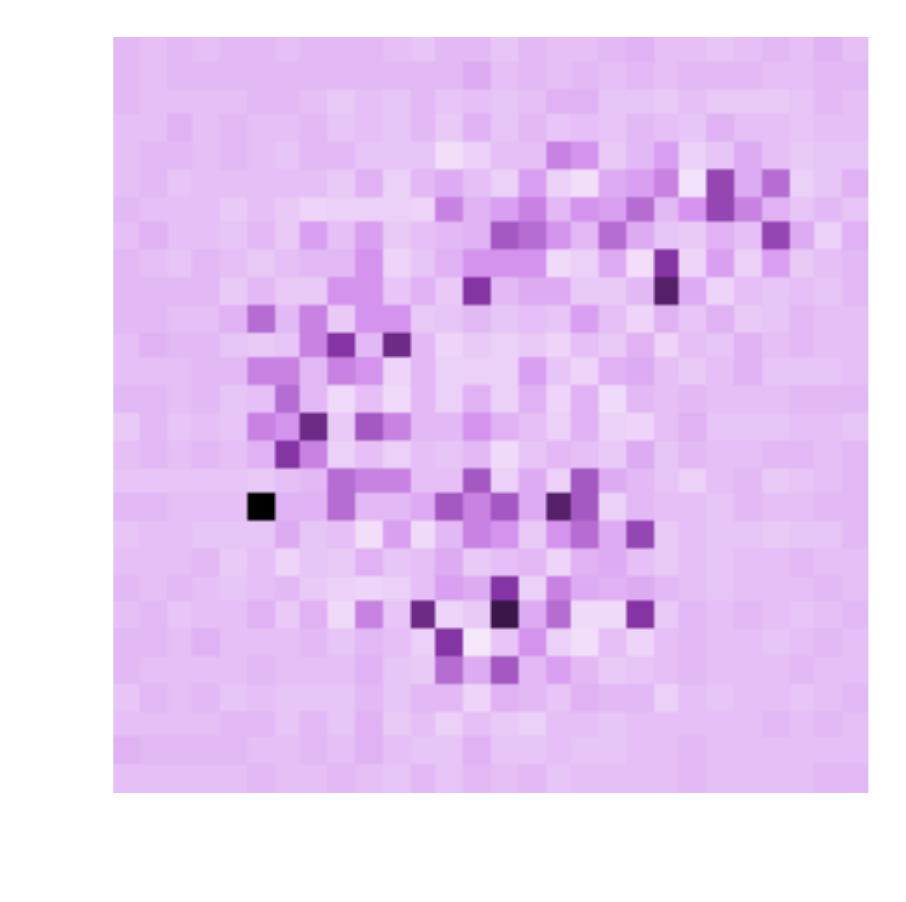}&
        \includegraphics[width=\panelwidth]{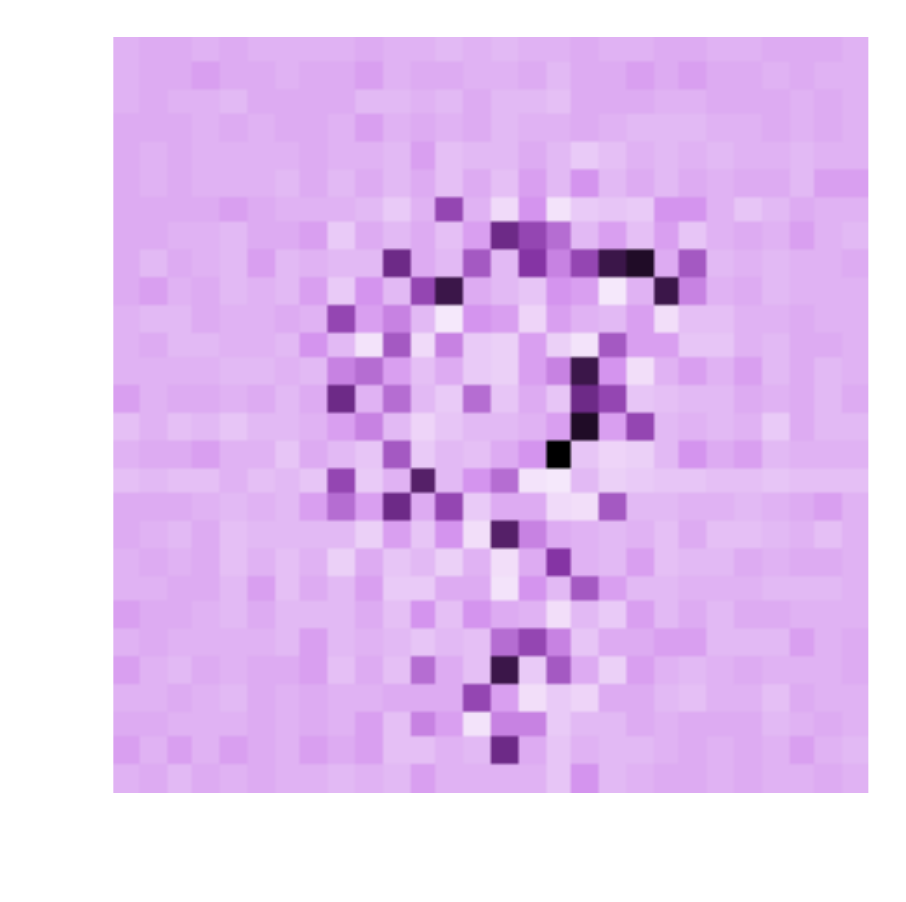}&
        \includegraphics[width=\panelwidth]{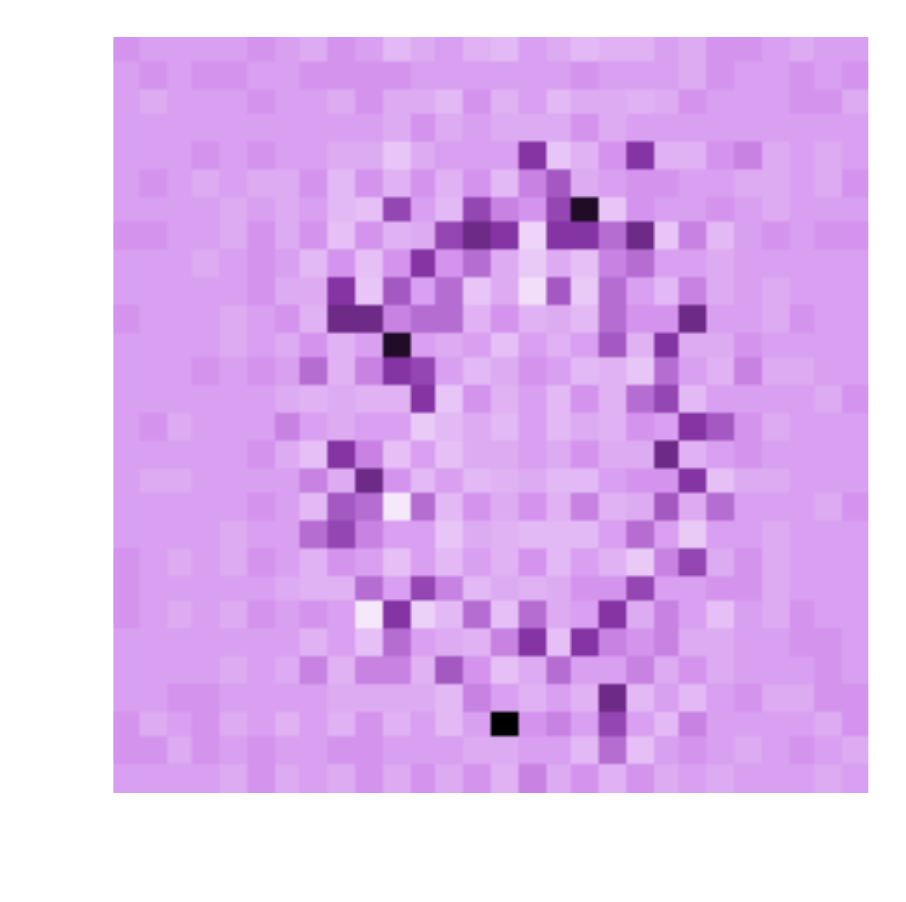}&
        \includegraphics[width=\panelwidth]{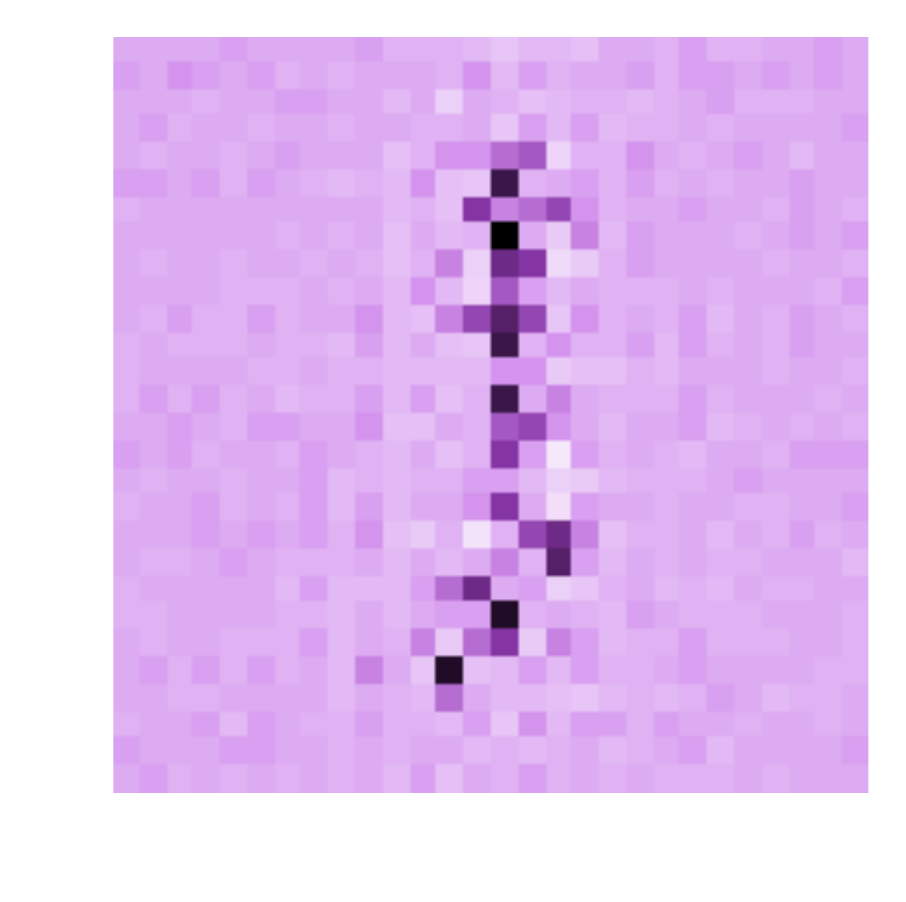}
        \\
        \includegraphics[width=\panelwidth]{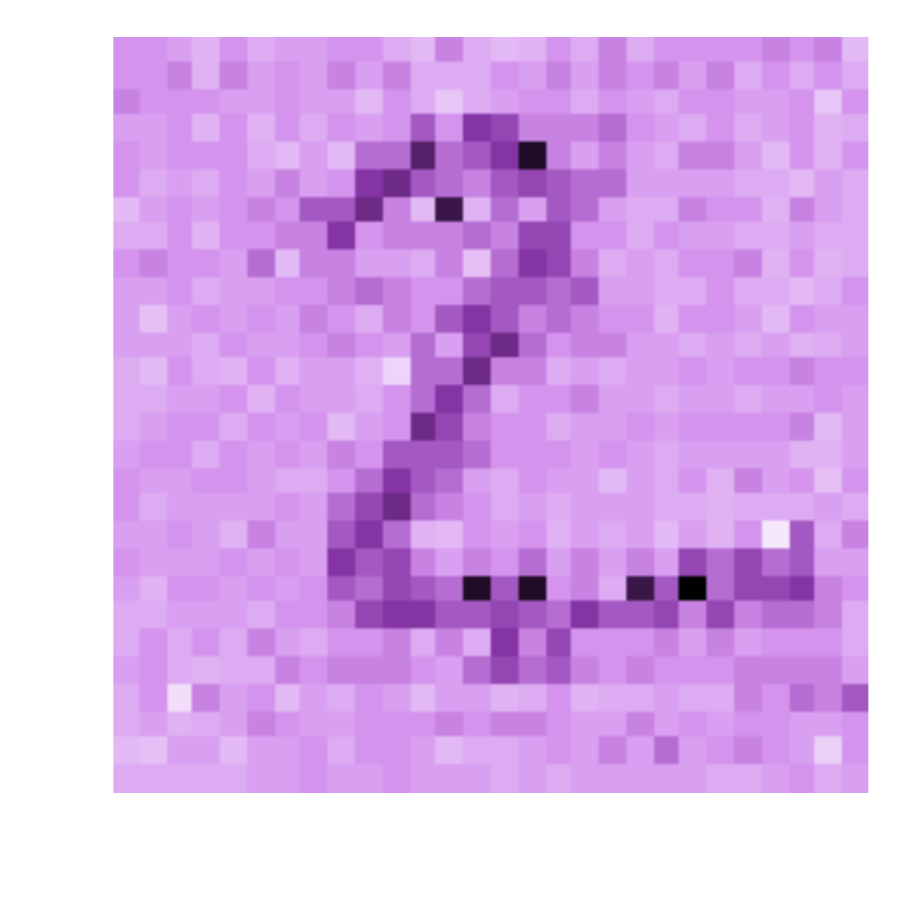}&
        \includegraphics[width=\panelwidth]{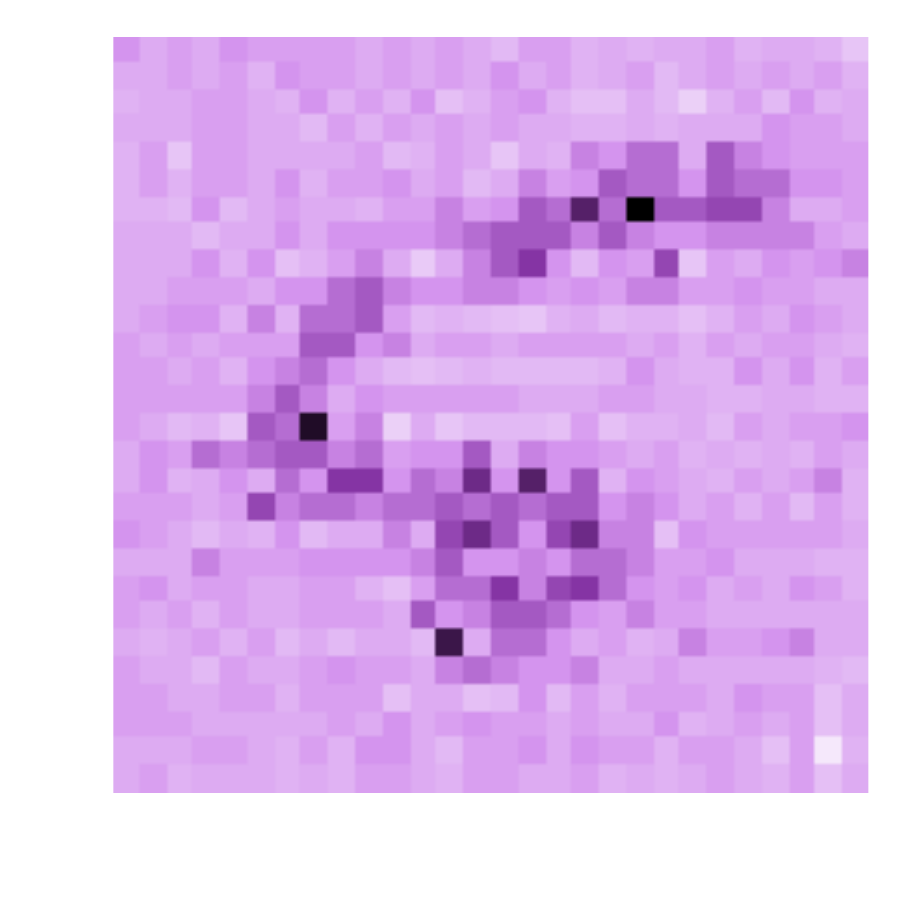}&
        \includegraphics[width=\panelwidth]{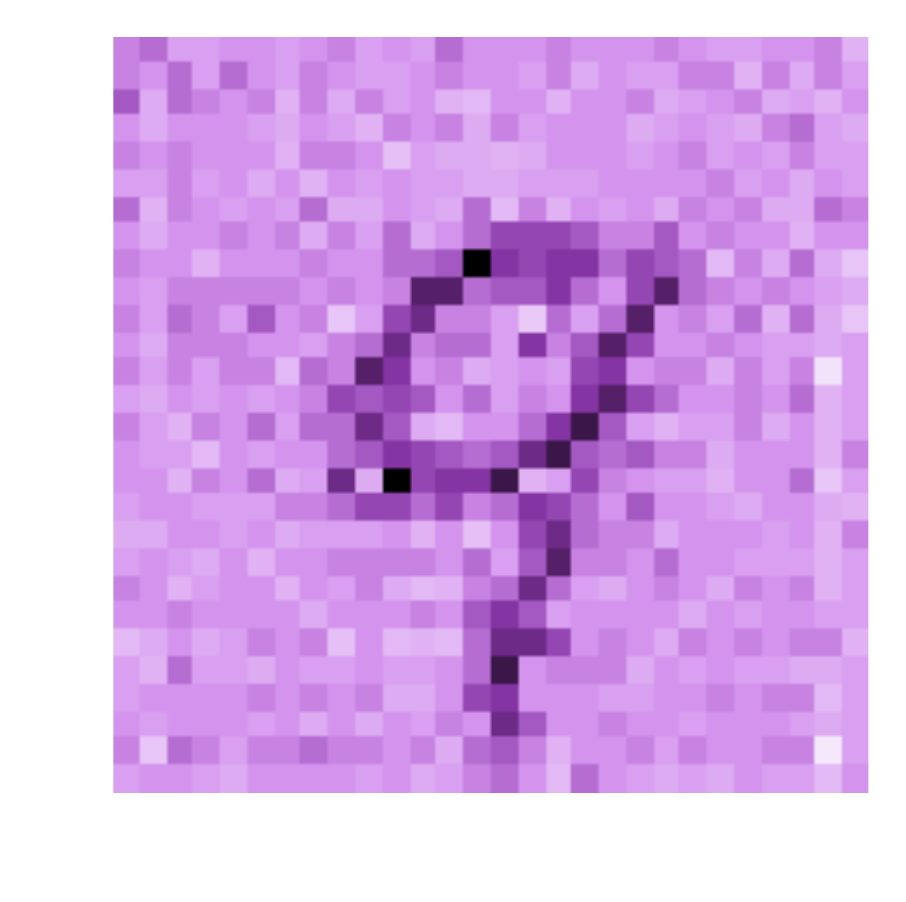}&
        \includegraphics[width=\panelwidth]{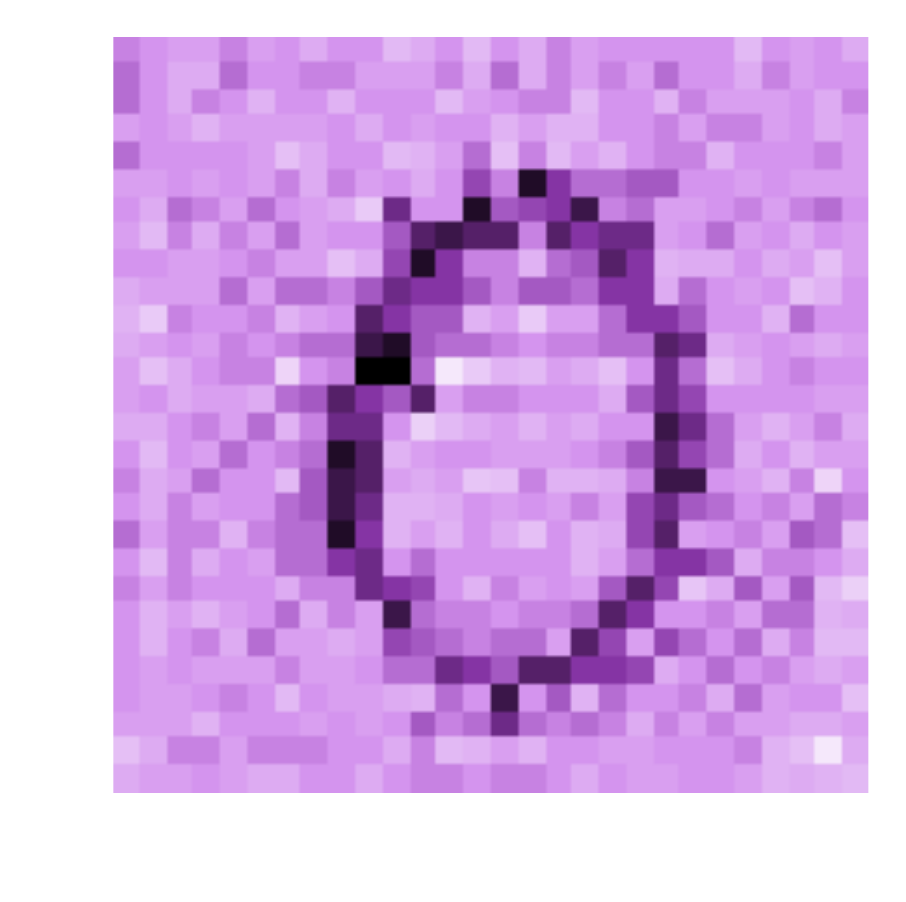}&
        \includegraphics[width=\panelwidth]{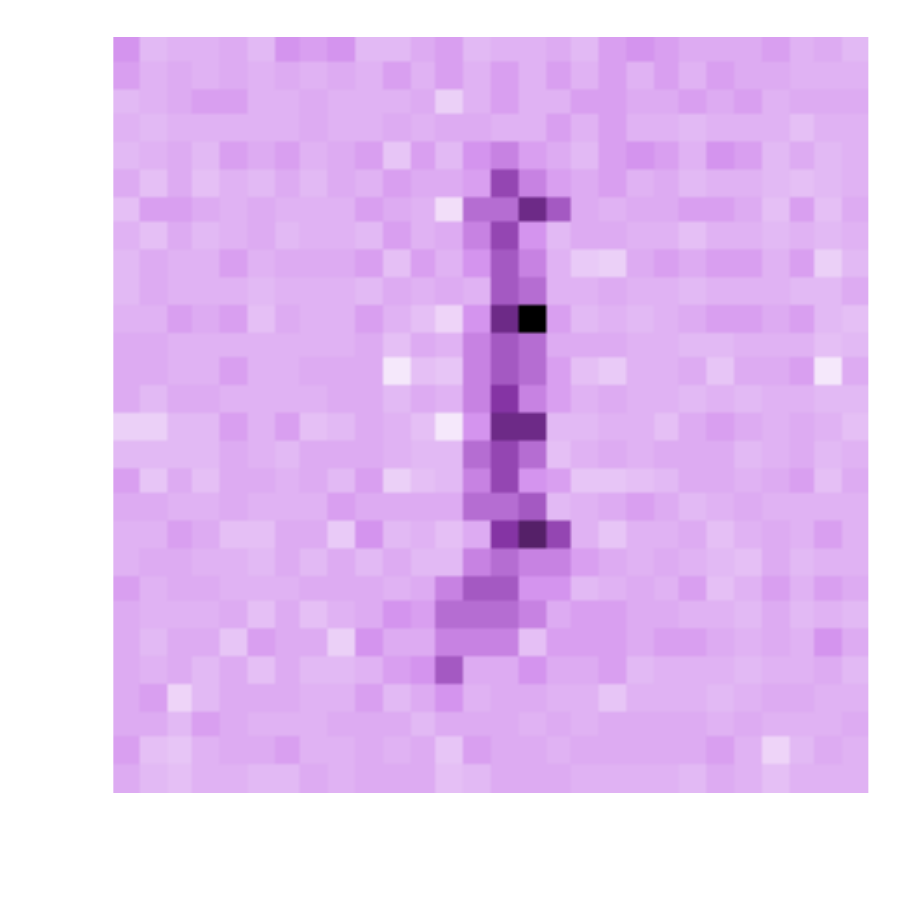}
    \end{tabular}
    }

    \subfigure[Glow trained on FashionMNIST]{
    \begin{tabular}{cccccc}
		\rotatebox{90}{\quad~ $x$}&
        \hspace{-0.1cm} \includegraphics[width=\panelwidth]{figs/latents/appendix_fashion_data_0.pdf}&
        \includegraphics[width=\panelwidth]{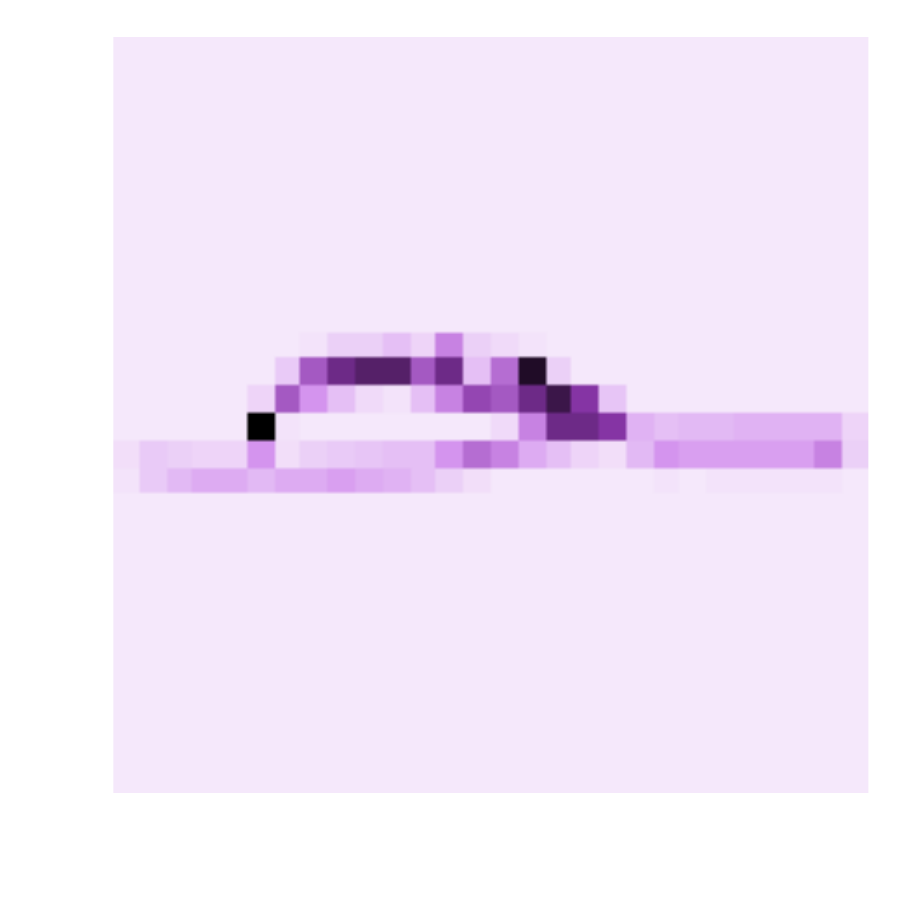}&
        \includegraphics[width=\panelwidth]{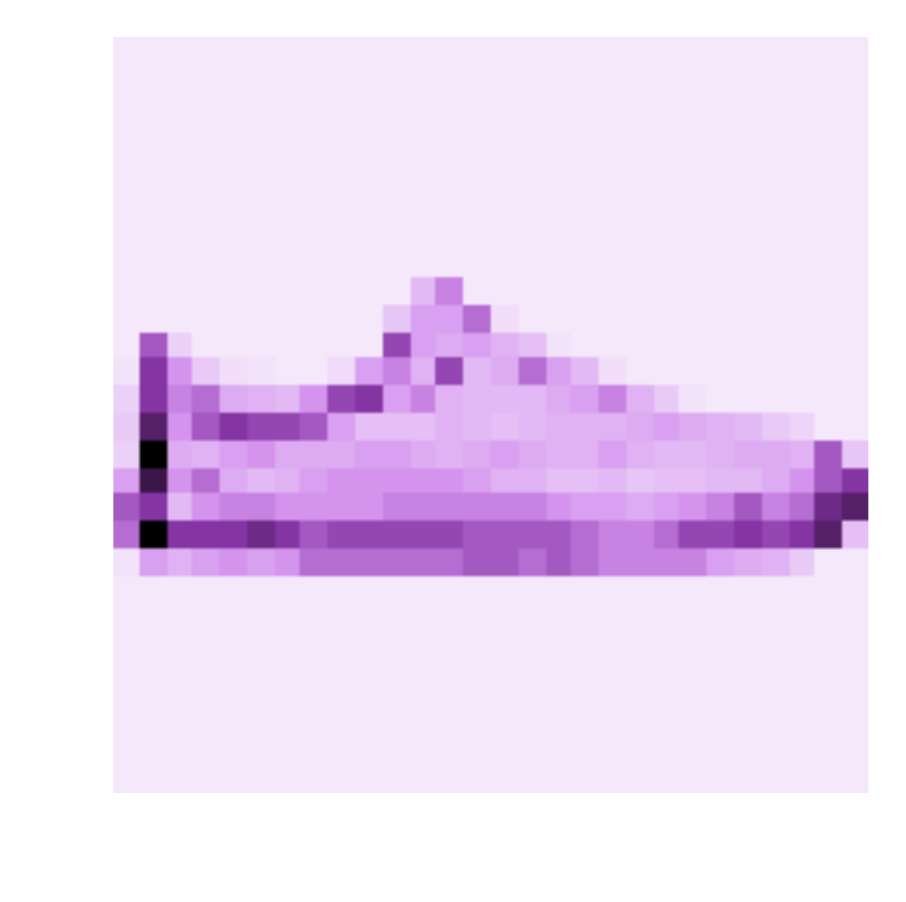}&
        \includegraphics[width=\panelwidth]{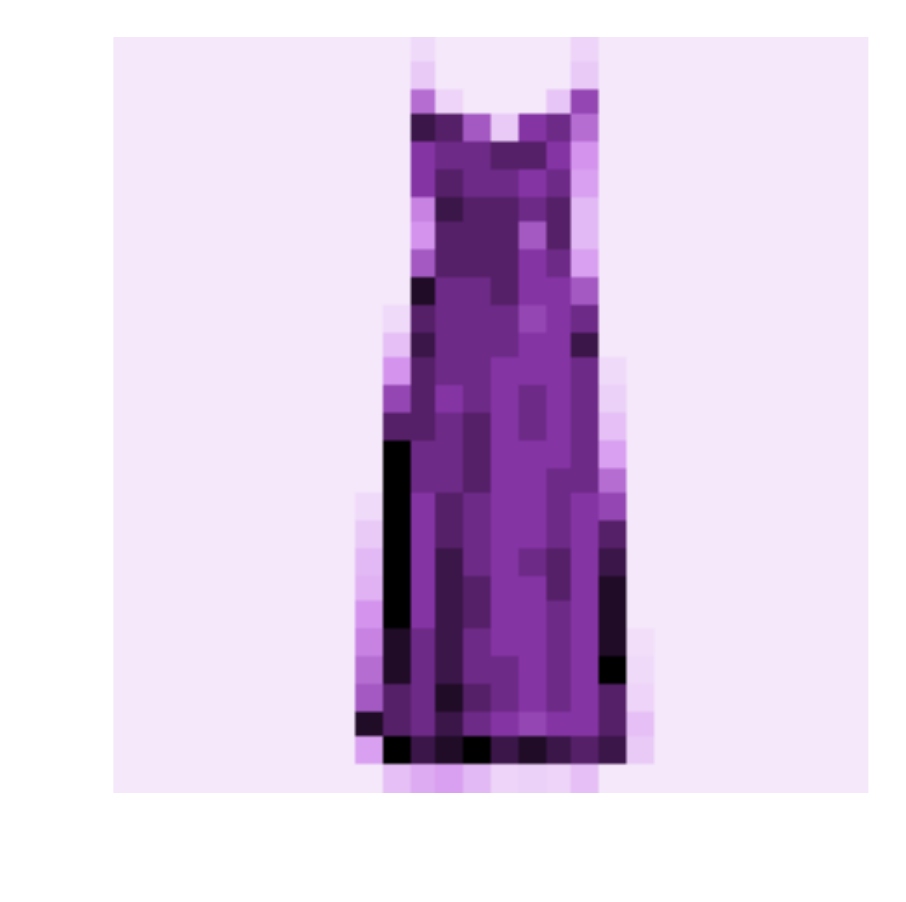}&
        \includegraphics[width=\panelwidth]{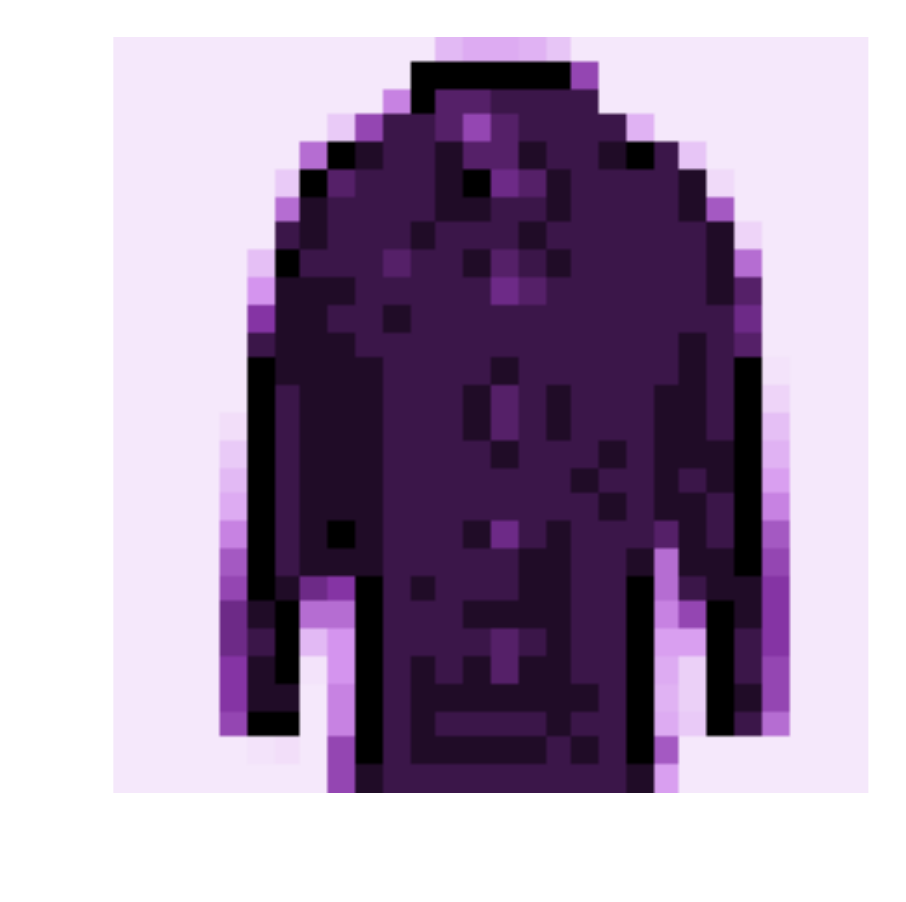}
        \\
		\rotatebox{90}{\quad~ $z$}&
        \hspace{-0.1cm} \includegraphics[width=\panelwidth]{figs/latents/appendix_fashion_latent_0.pdf}&
        \includegraphics[width=\panelwidth]{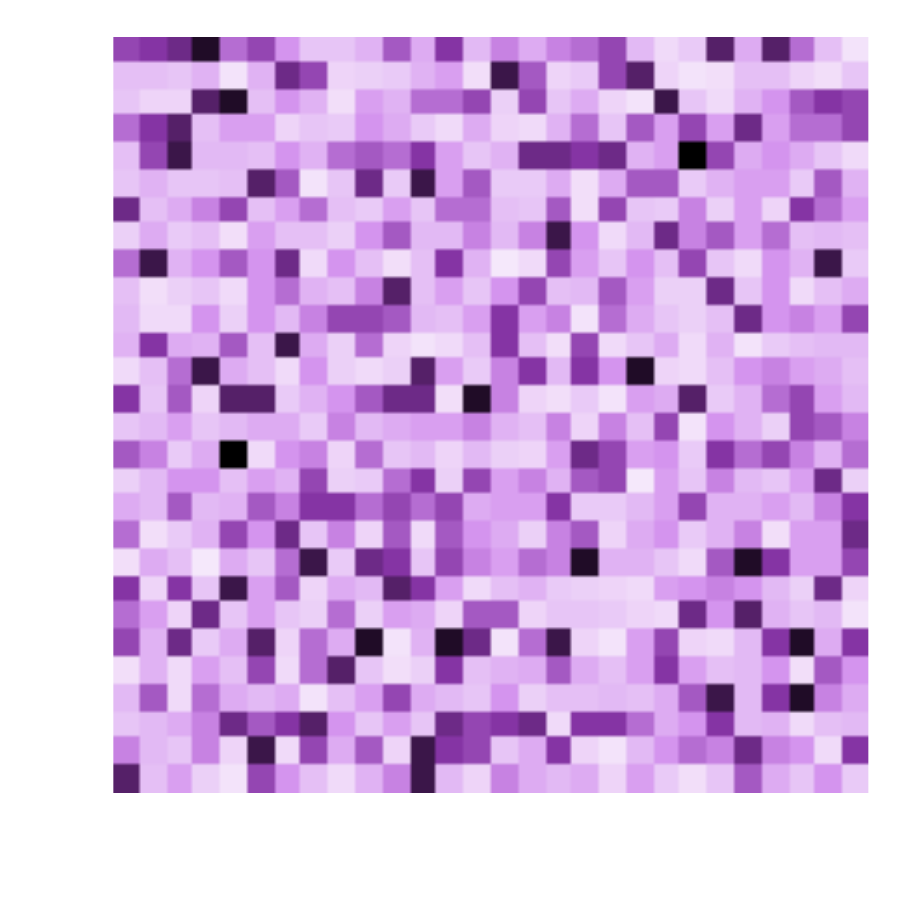}&
        \includegraphics[width=\panelwidth]{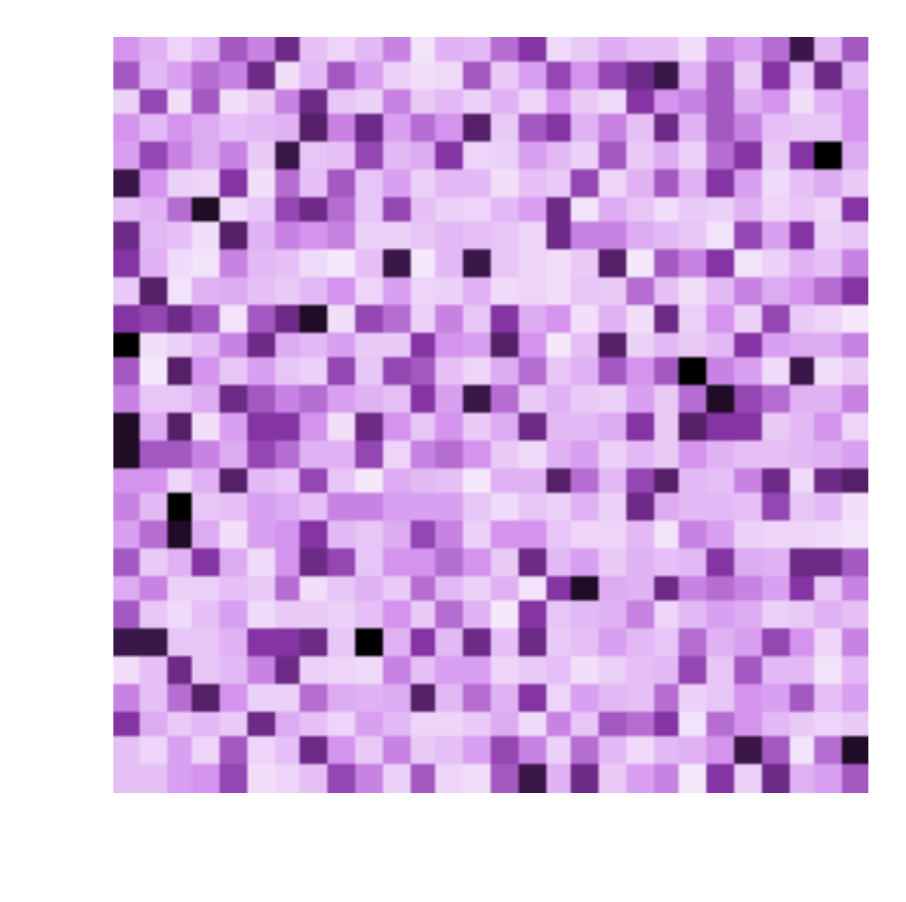}&
        \includegraphics[width=\panelwidth]{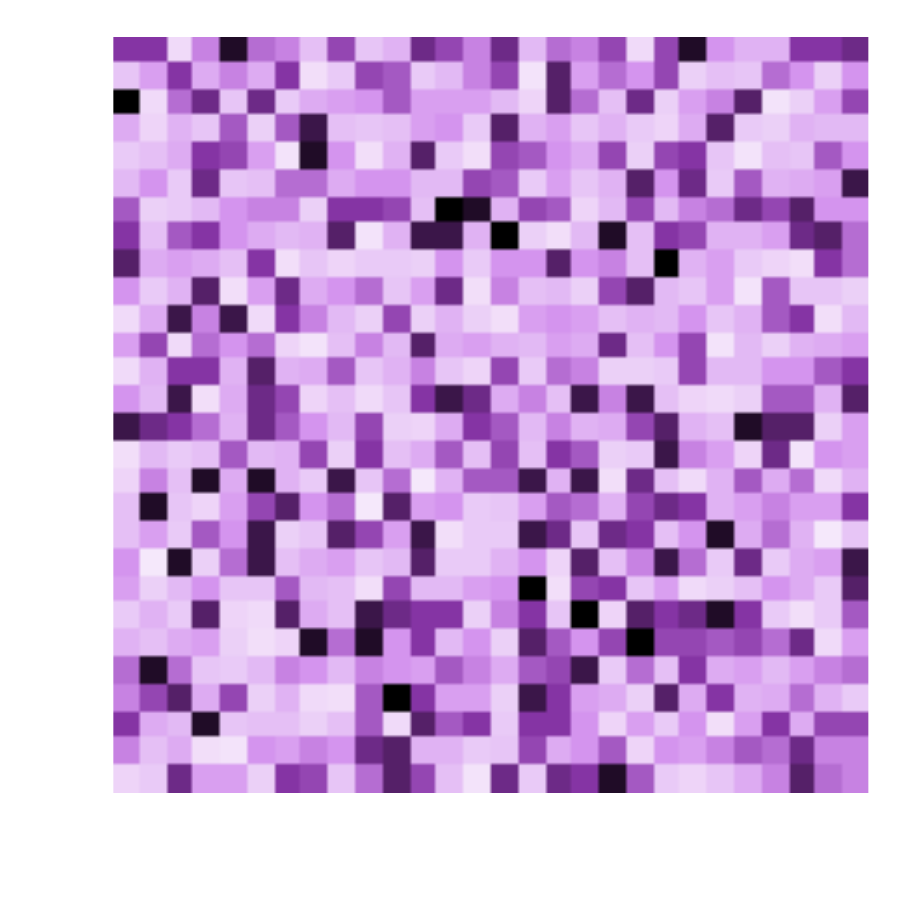}&
        \includegraphics[width=\panelwidth]{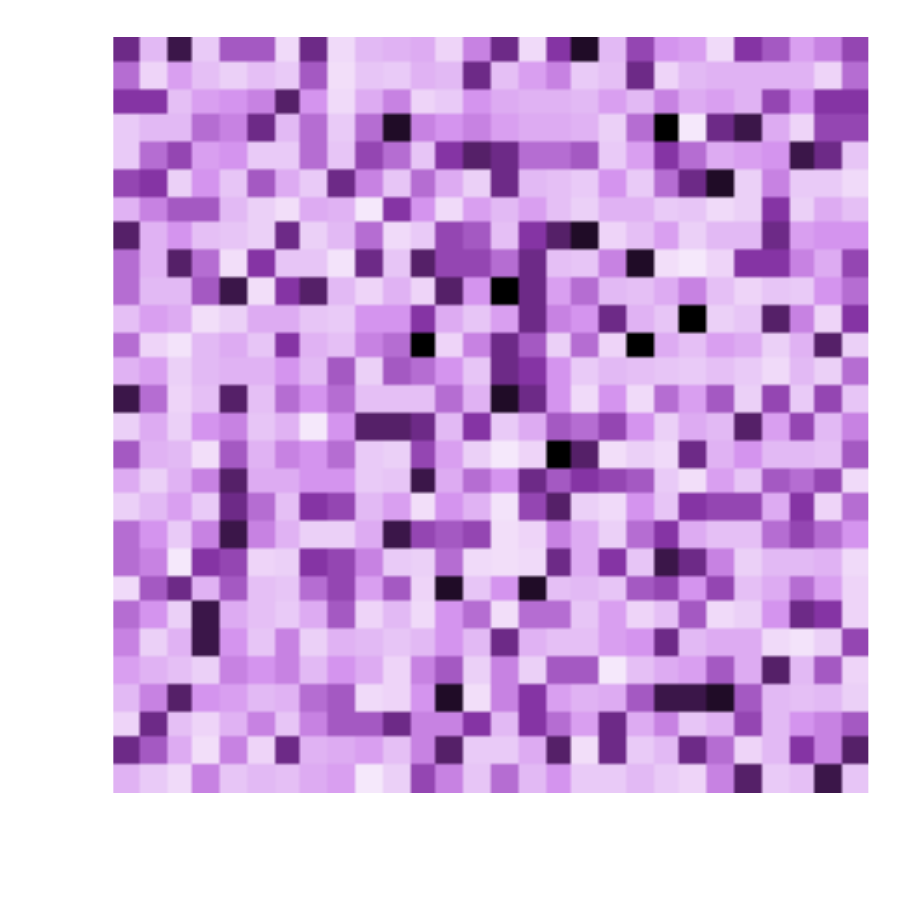}
        \\
		\rotatebox{90}{~~~Avg $z$}&
        \hspace{-0.1cm} \includegraphics[width=\panelwidth]{figs/latents/appendix_fashion_latent_avg_0.pdf}&
        \includegraphics[width=\panelwidth]{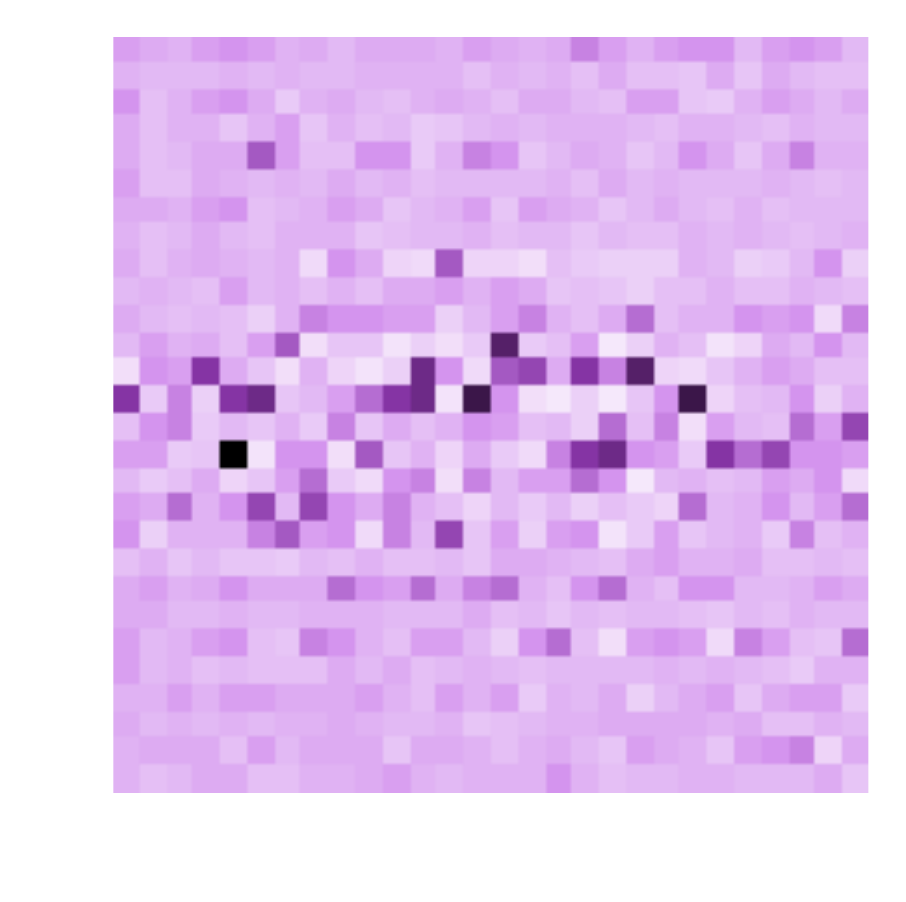}&
        \includegraphics[width=\panelwidth]{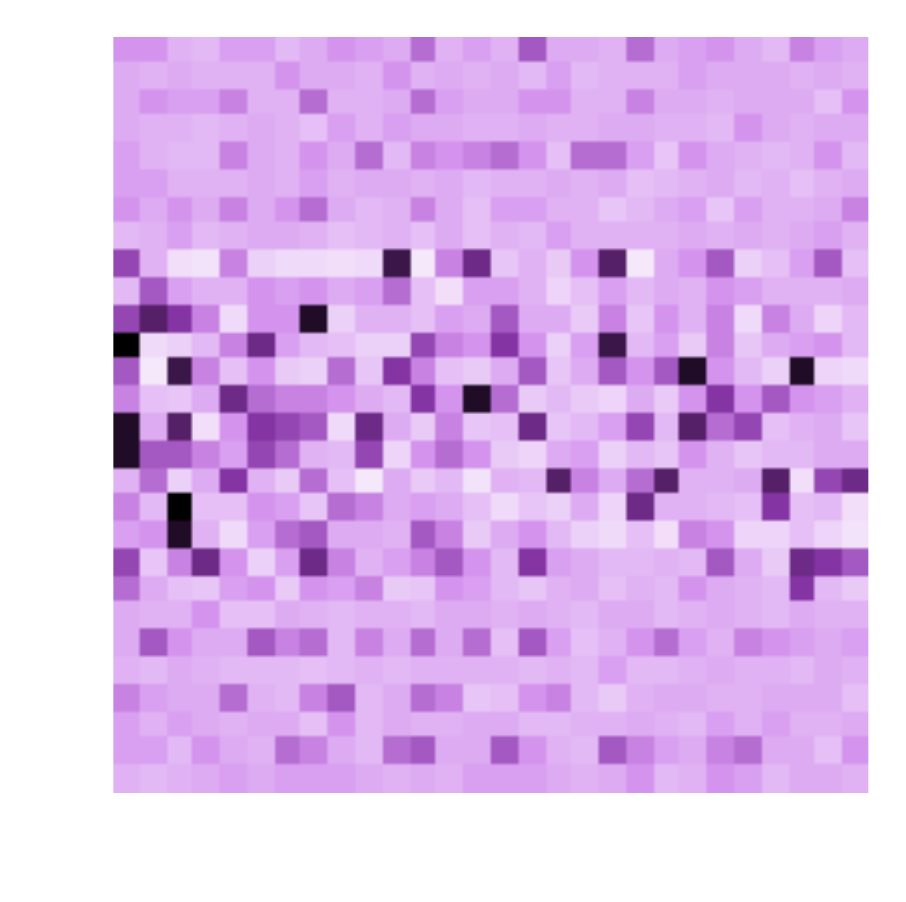}&
        \includegraphics[width=\panelwidth]{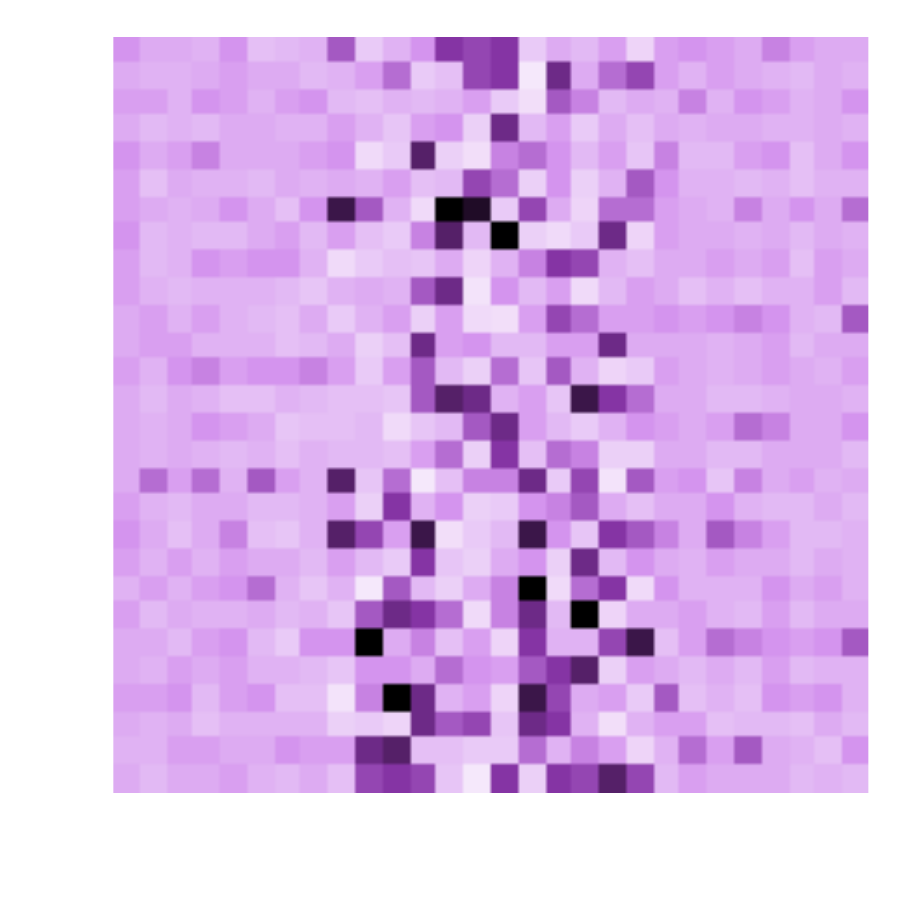}&
        \includegraphics[width=\panelwidth]{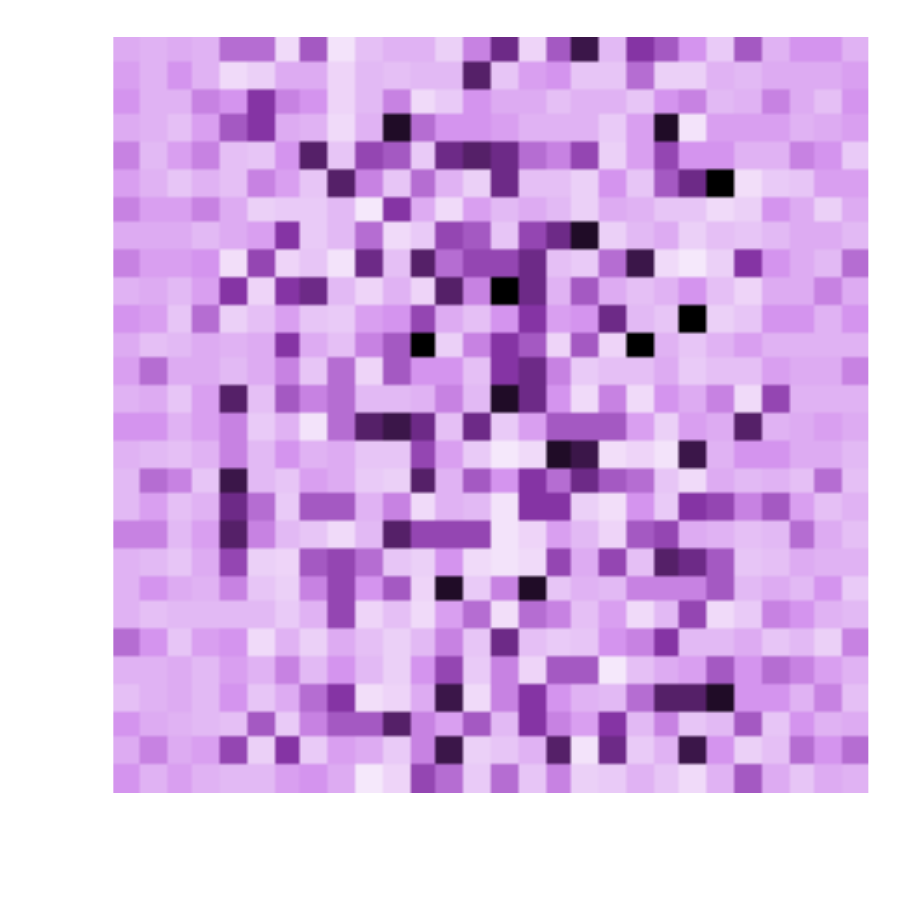}
    \end{tabular}
    \quad
    \begin{tabular}{ccccc}
        \includegraphics[width=\panelwidth]{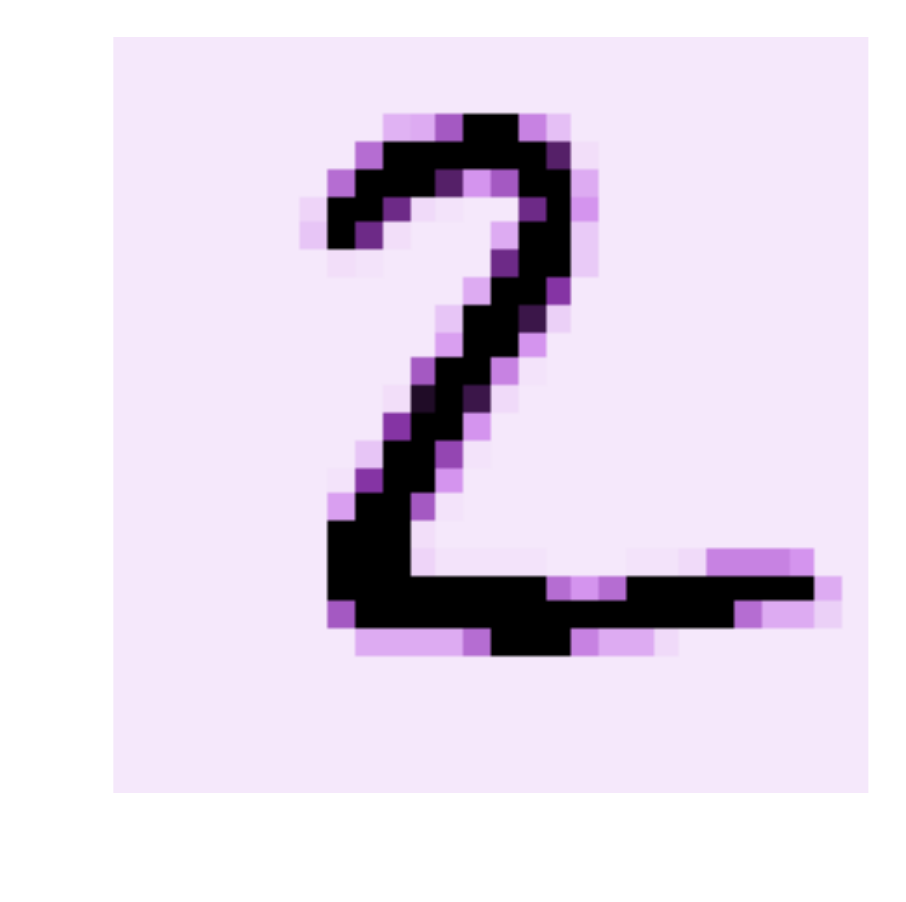}&
        \includegraphics[width=\panelwidth]{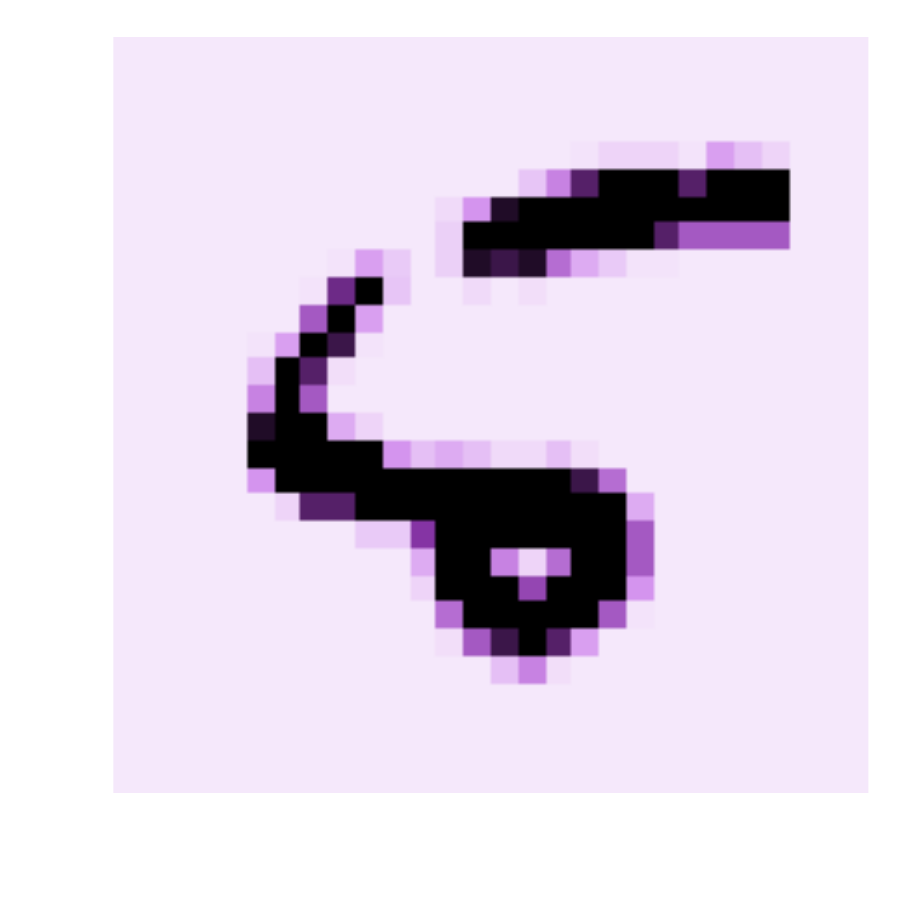}&
        \includegraphics[width=\panelwidth]{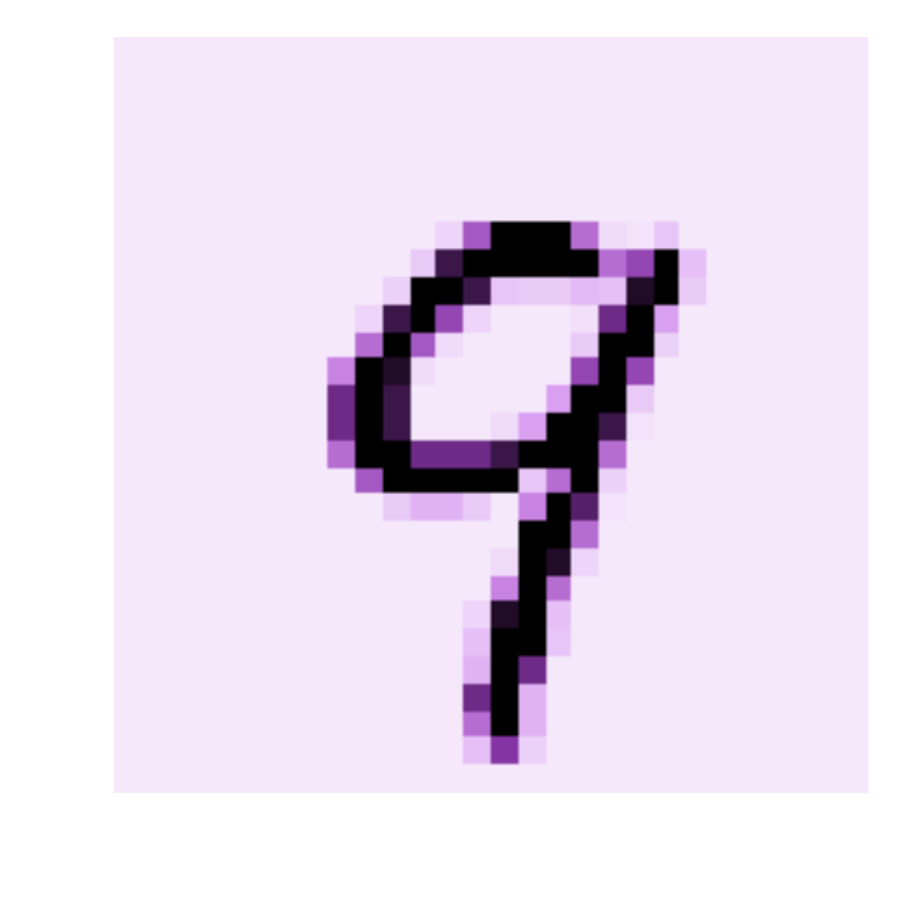}&
        \includegraphics[width=\panelwidth]{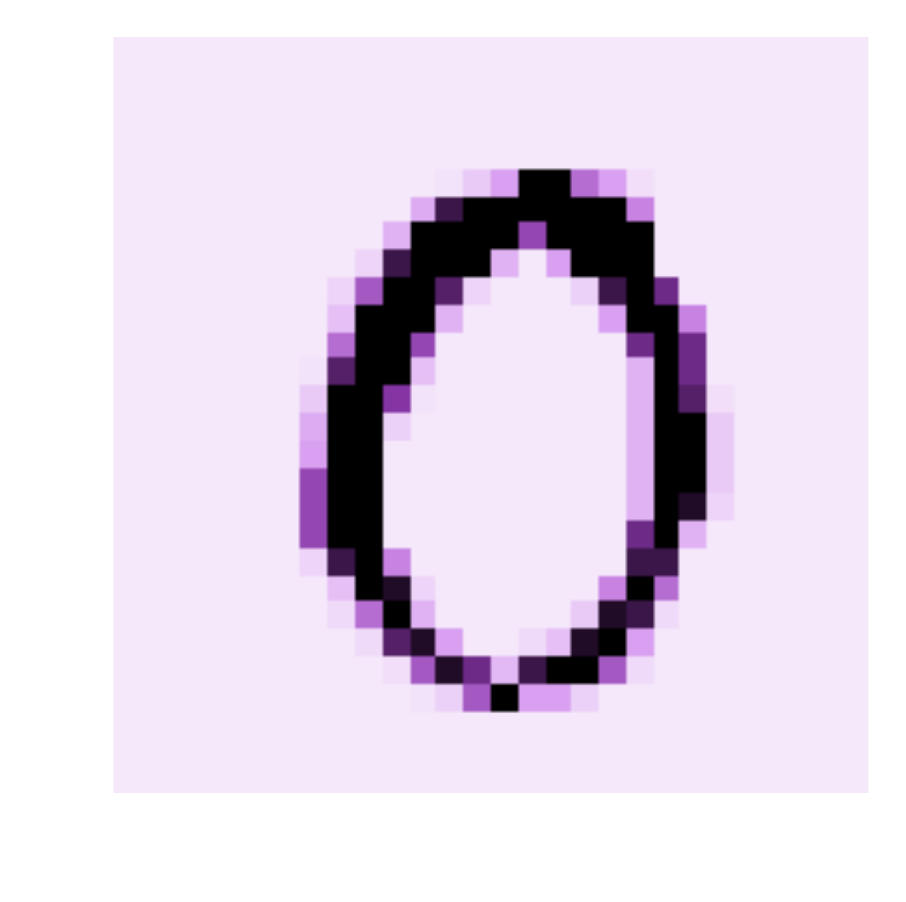}&
        \includegraphics[width=\panelwidth]{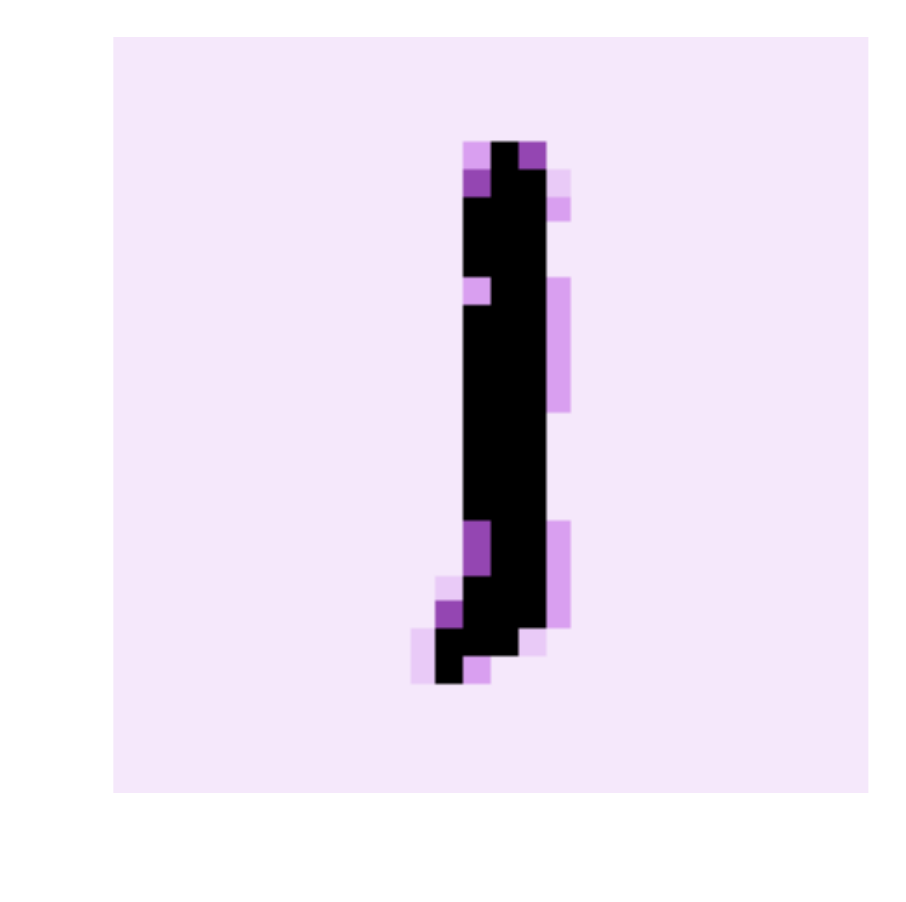}
        \\
        \includegraphics[width=\panelwidth]{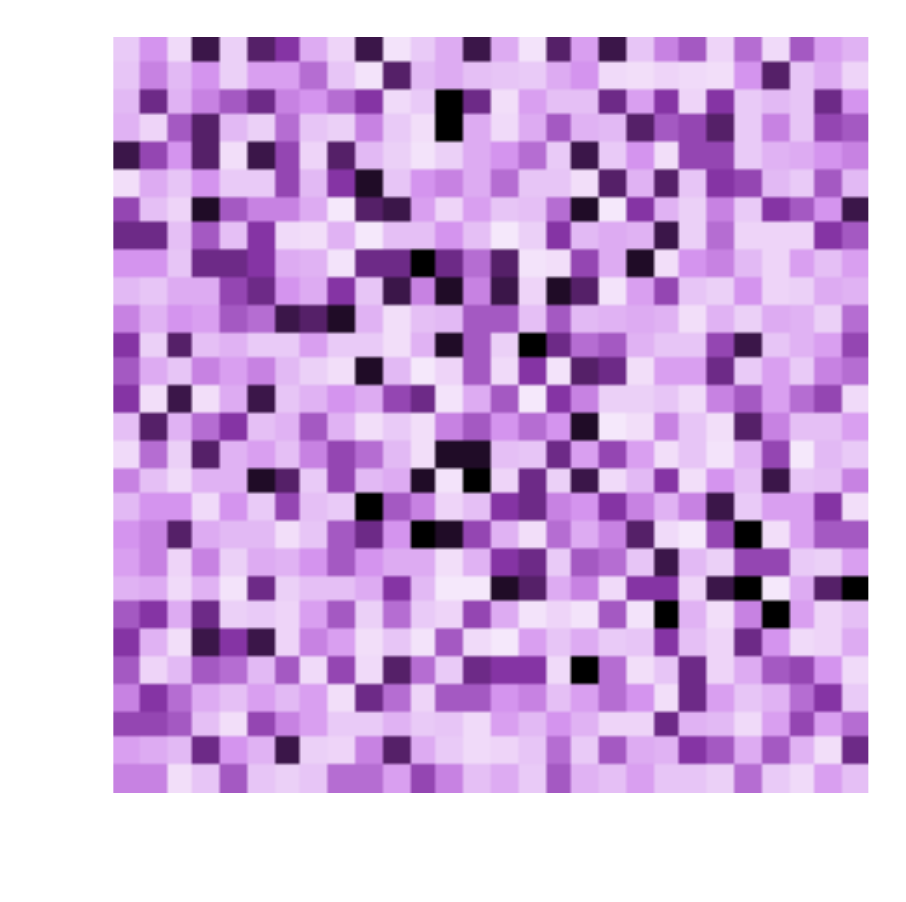}&
        \includegraphics[width=\panelwidth]{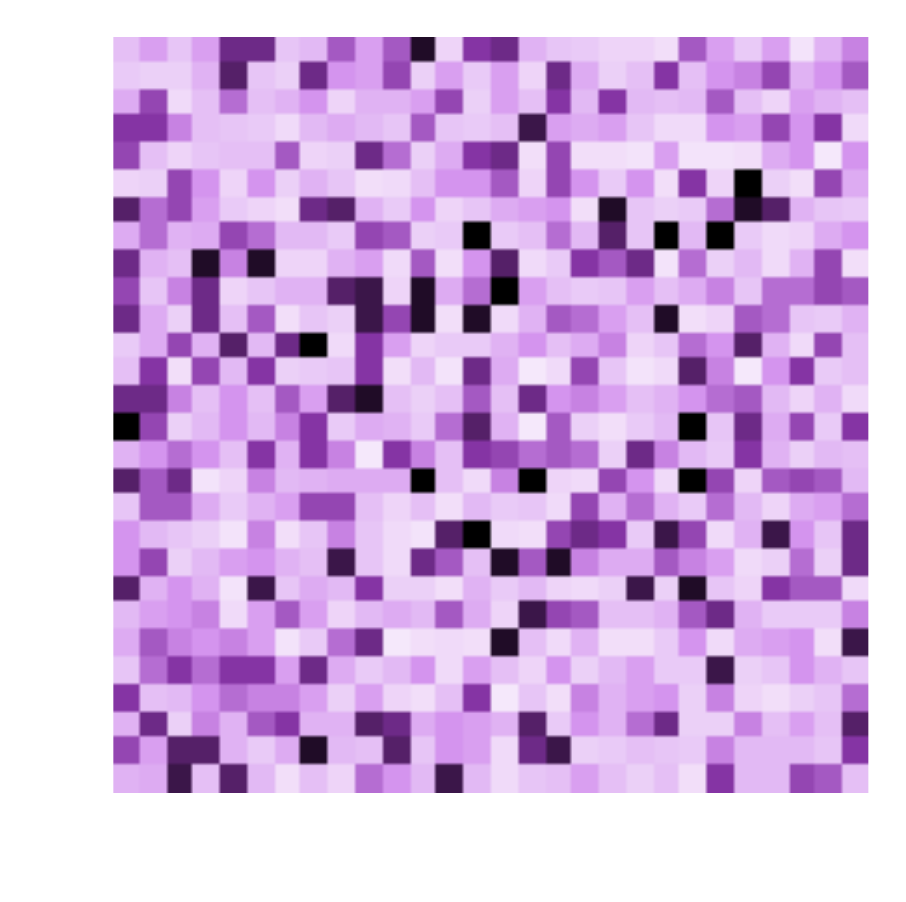}&
        \includegraphics[width=\panelwidth]{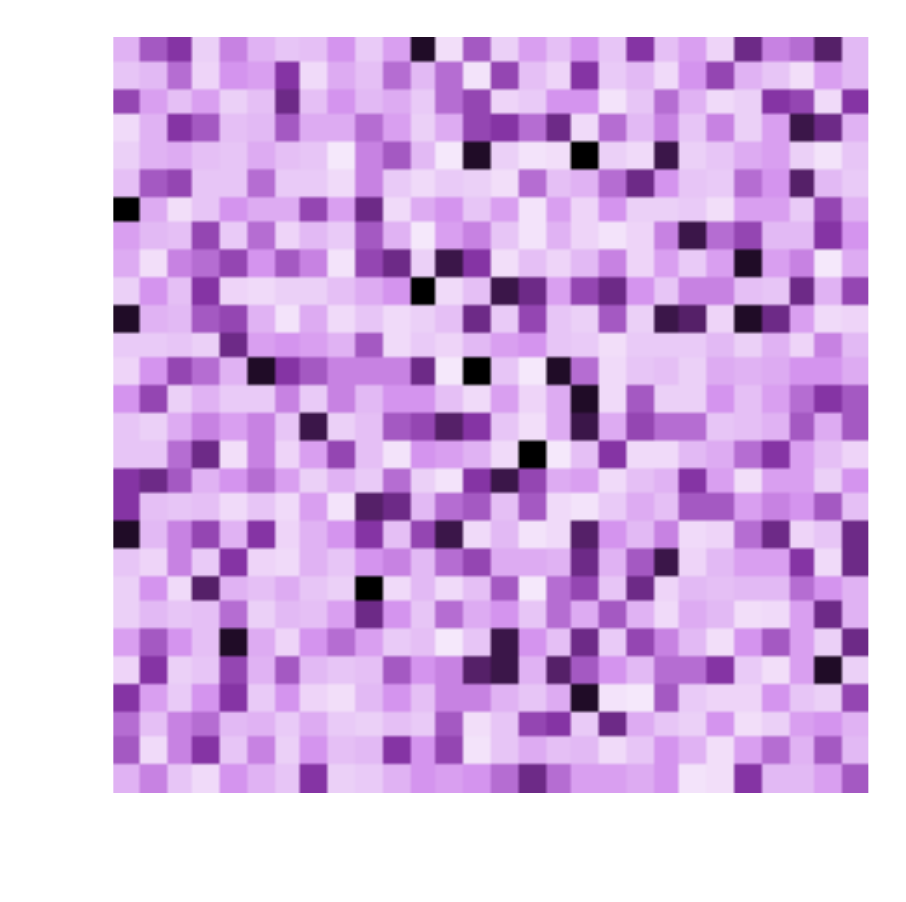}&
        \includegraphics[width=\panelwidth]{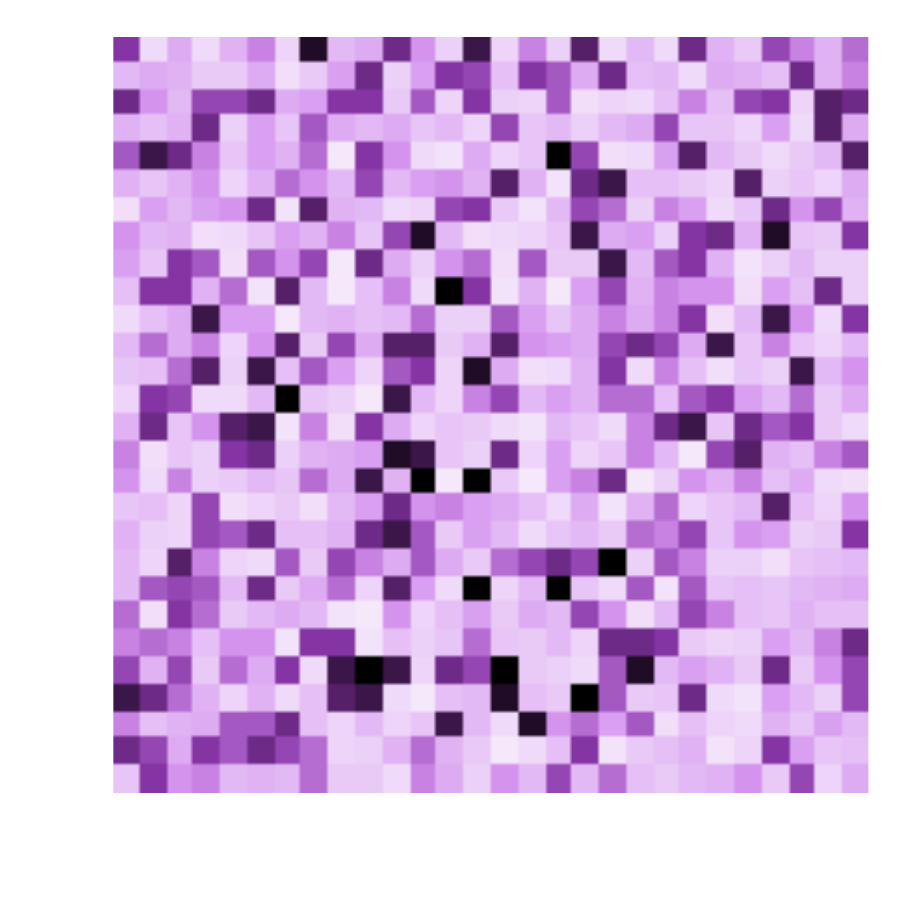}&
        \includegraphics[width=\panelwidth]{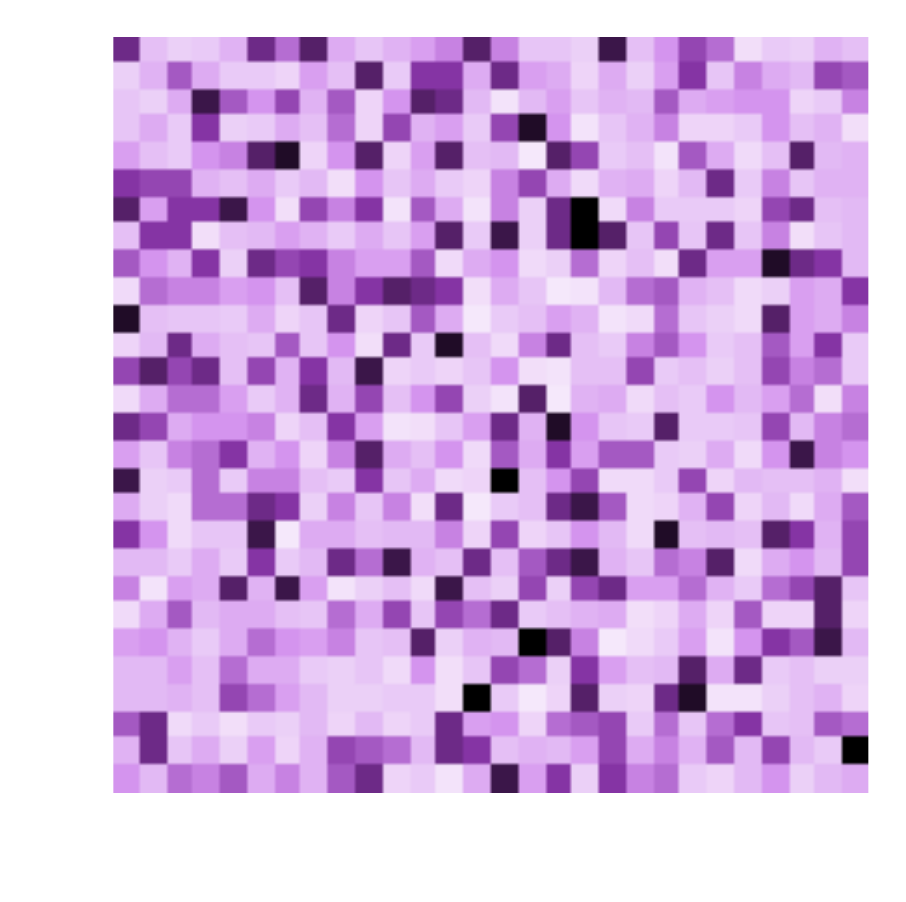}
        \\
        \includegraphics[width=\panelwidth]{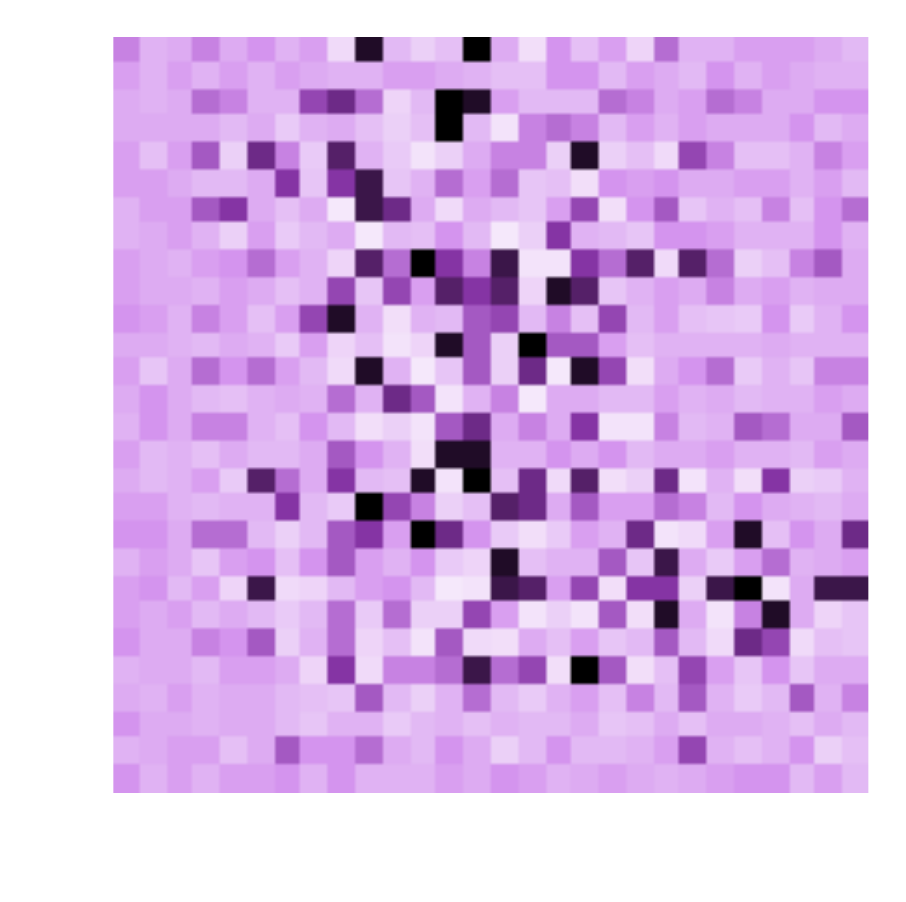}&
        \includegraphics[width=\panelwidth]{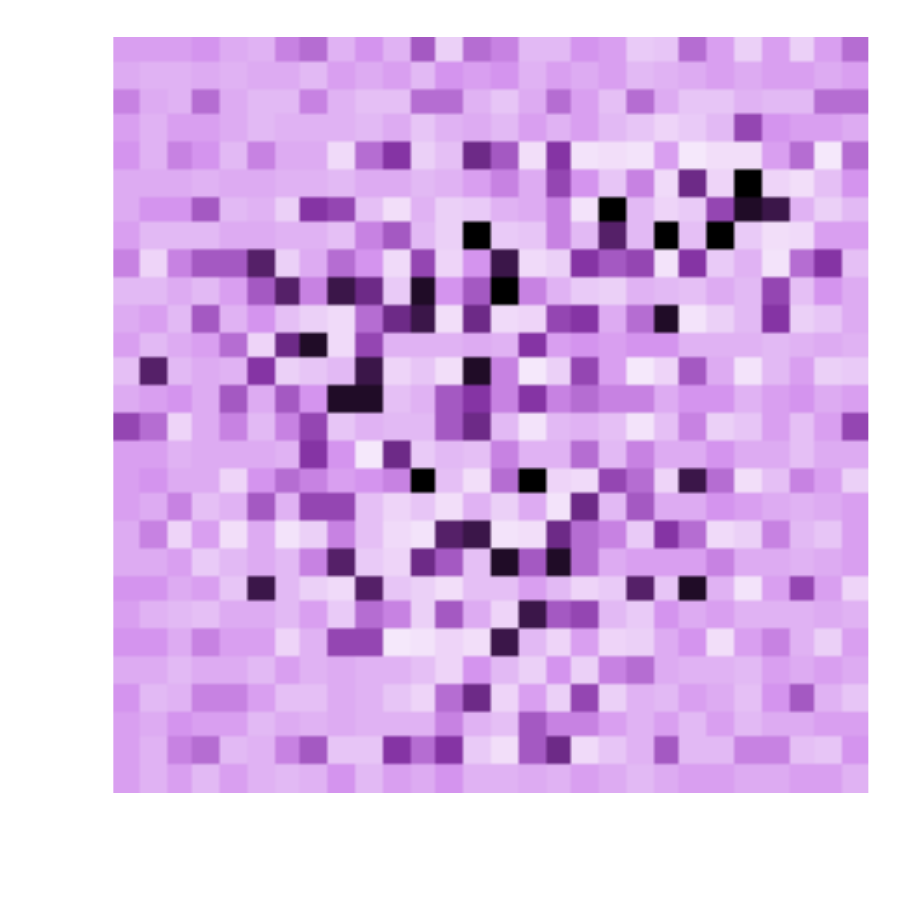}&
        \includegraphics[width=\panelwidth]{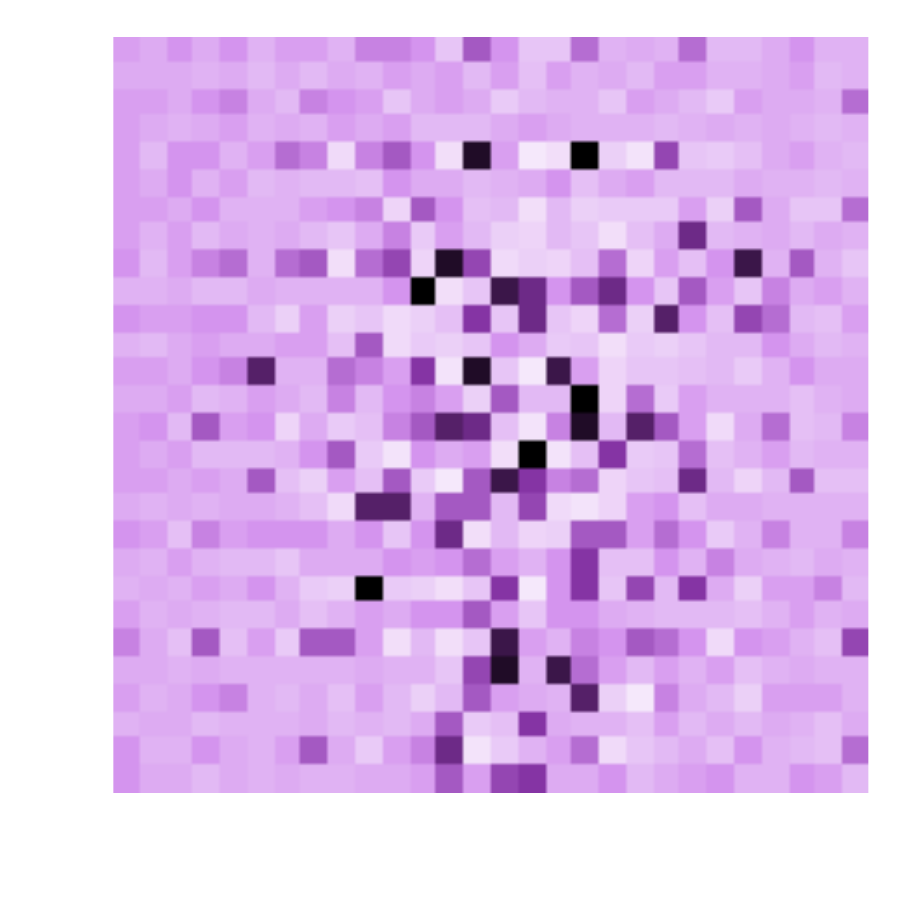}&
        \includegraphics[width=\panelwidth]{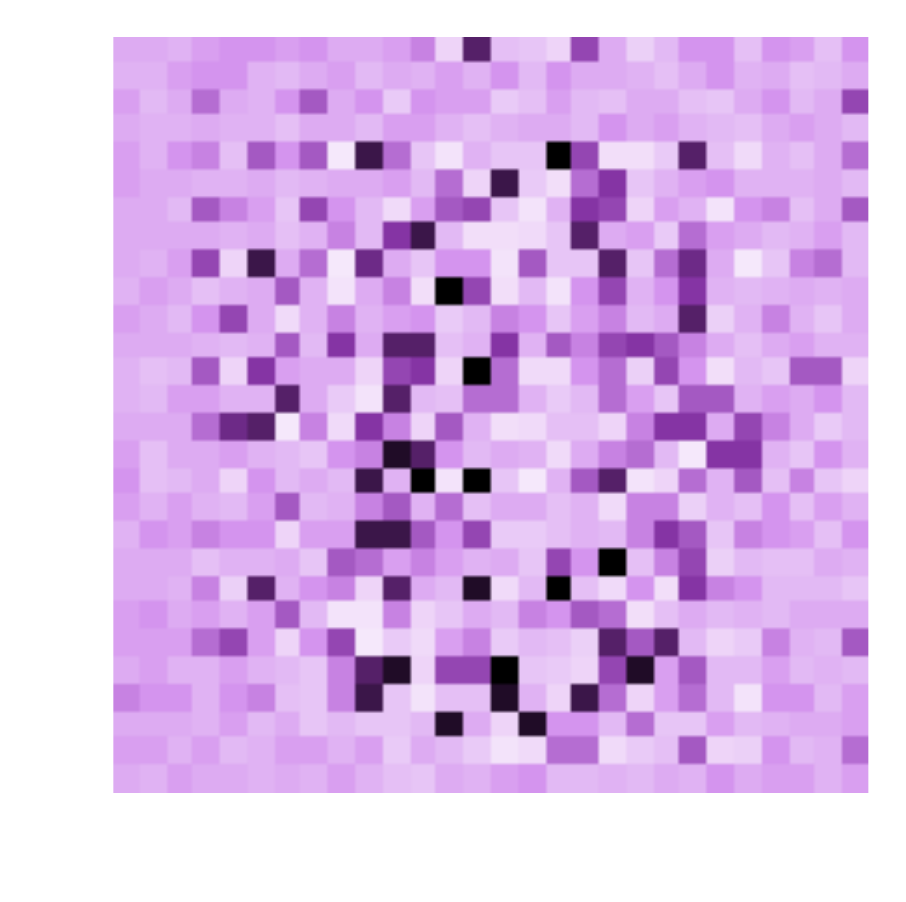}&
        \includegraphics[width=\panelwidth]{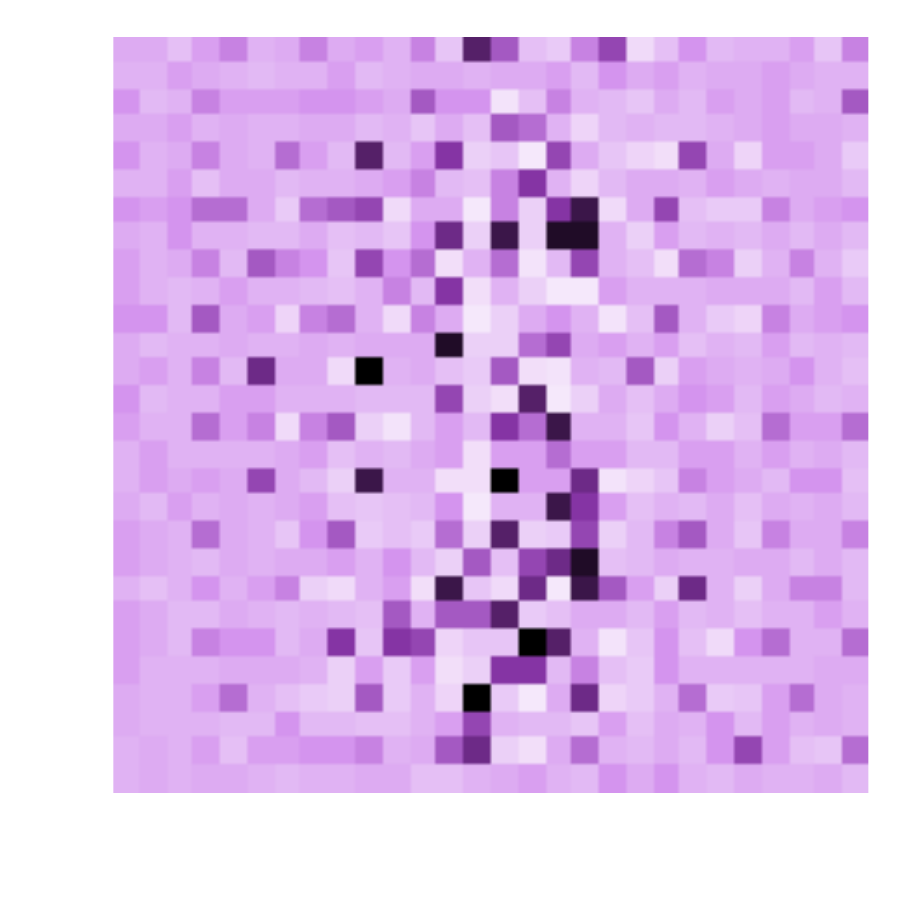}
    \end{tabular}
    }

    \subfigure[RealNVP trained on CelebA]{
    \begin{tabular}{cccccc}
		\rotatebox{90}{\quad~ $x$}&
        \hspace{-0.1cm} 
        \includegraphics[width=\panelwidth]{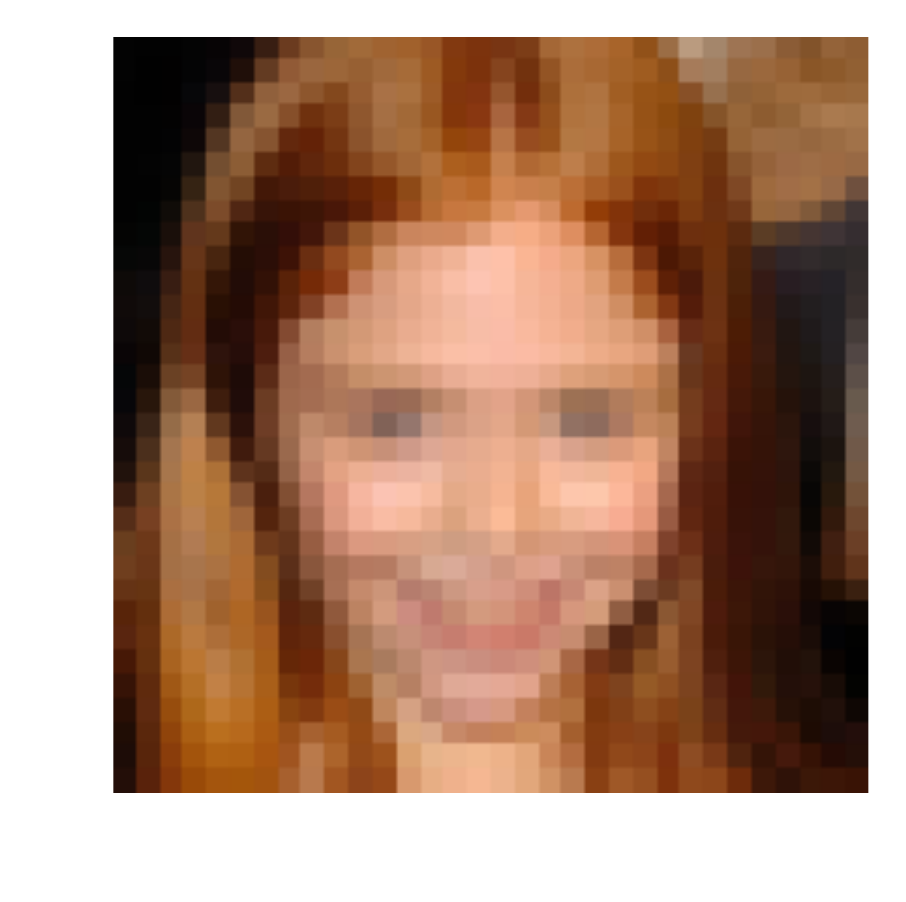}&
        \includegraphics[width=\panelwidth]{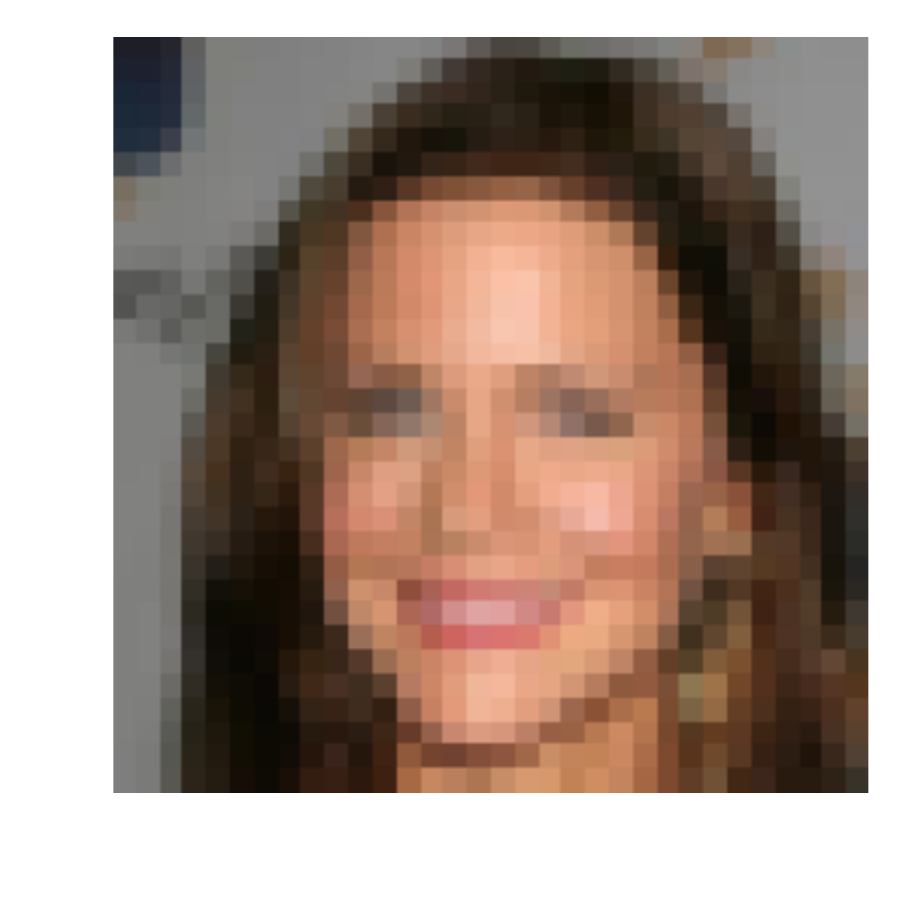}&
        \includegraphics[width=\panelwidth]{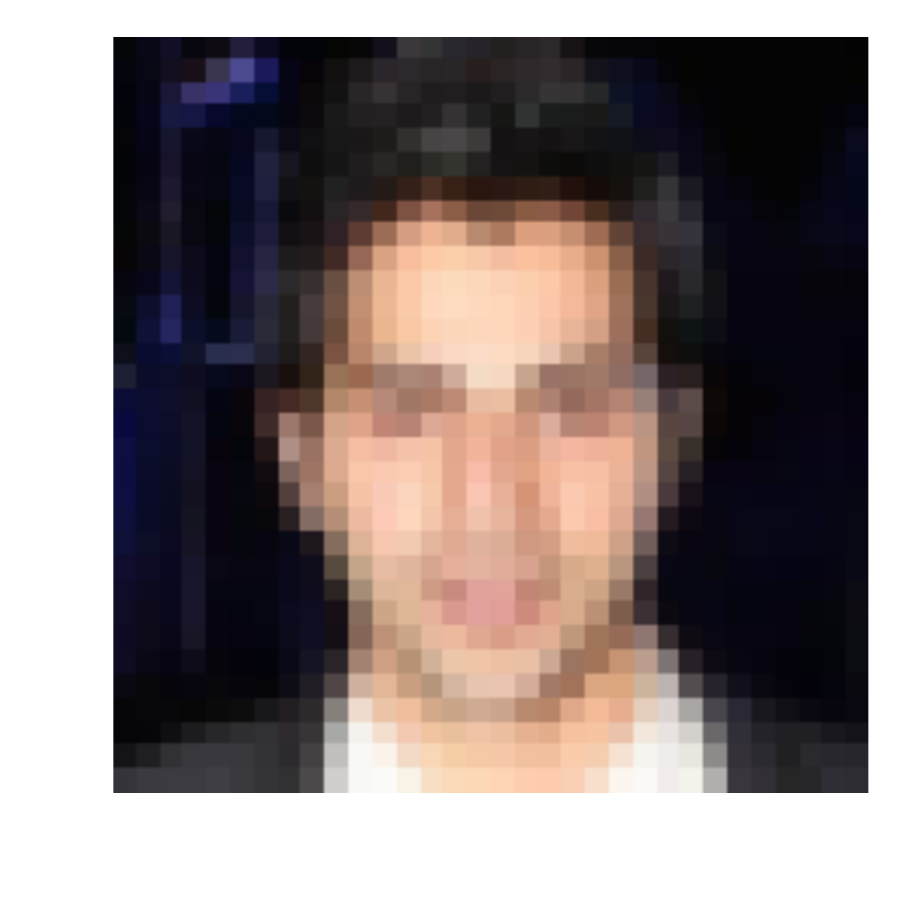}&
        \includegraphics[width=\panelwidth]{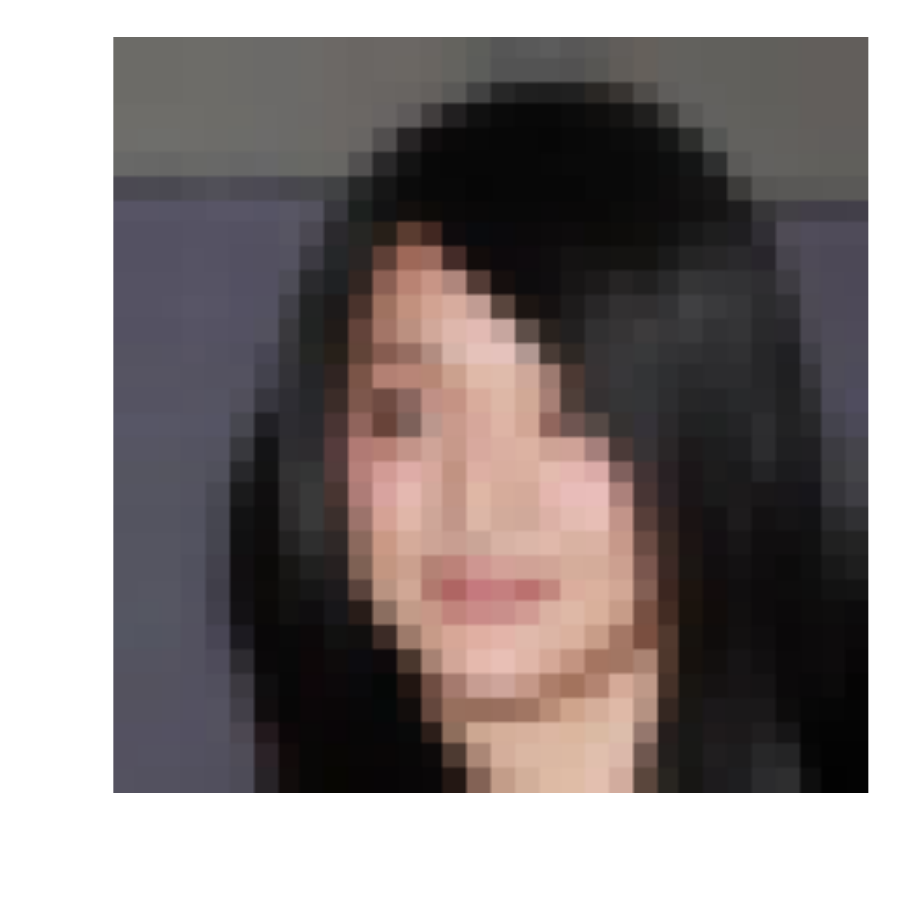}&
        \includegraphics[width=\panelwidth]{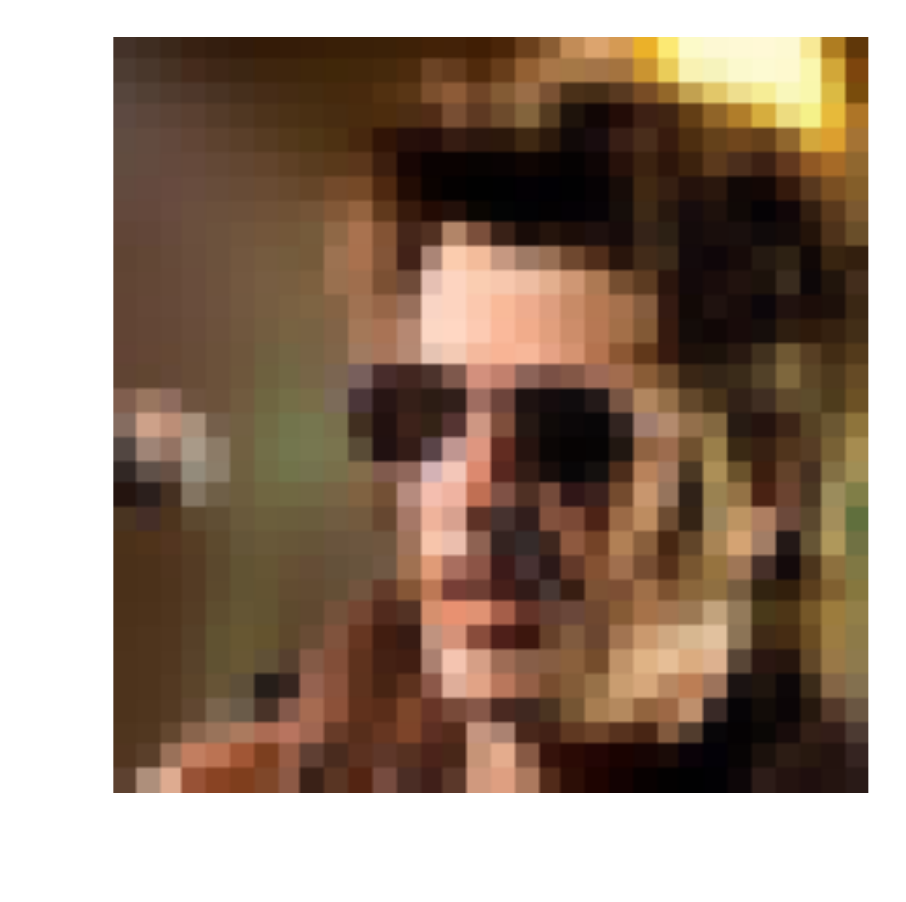}
        \\
		\rotatebox{90}{\quad~ $z$}&
        \hspace{-0.1cm}
        \includegraphics[width=\panelwidth]{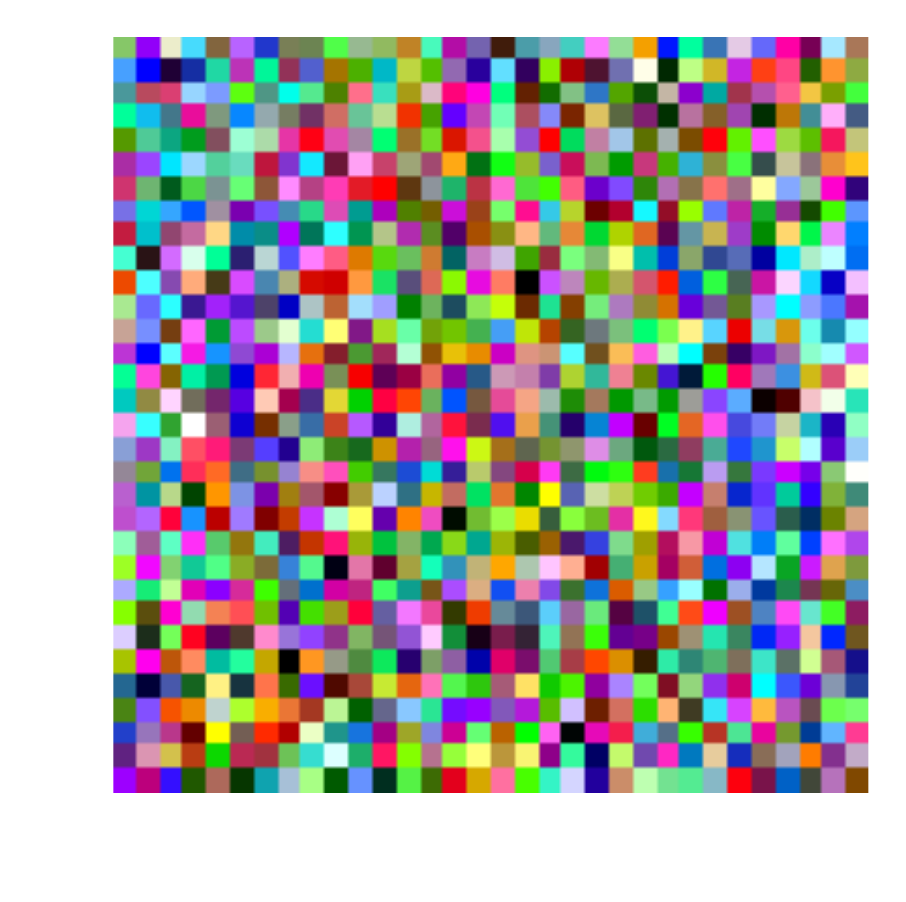}&
        \includegraphics[width=\panelwidth]{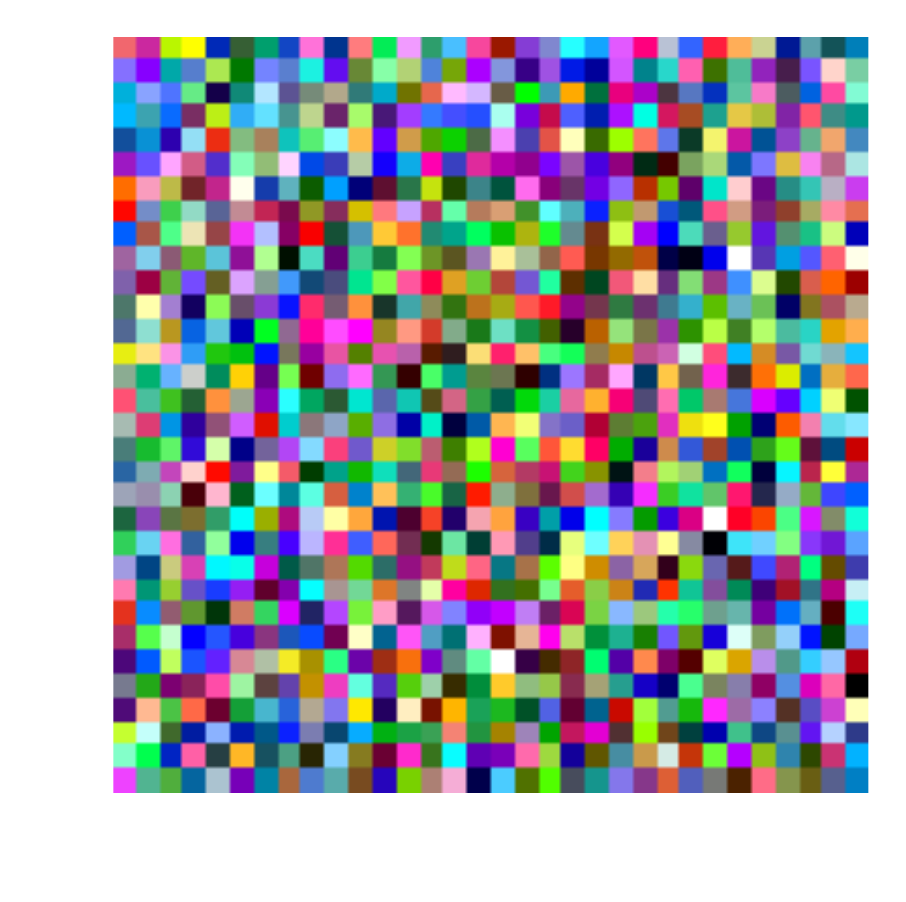}&
        \includegraphics[width=\panelwidth]{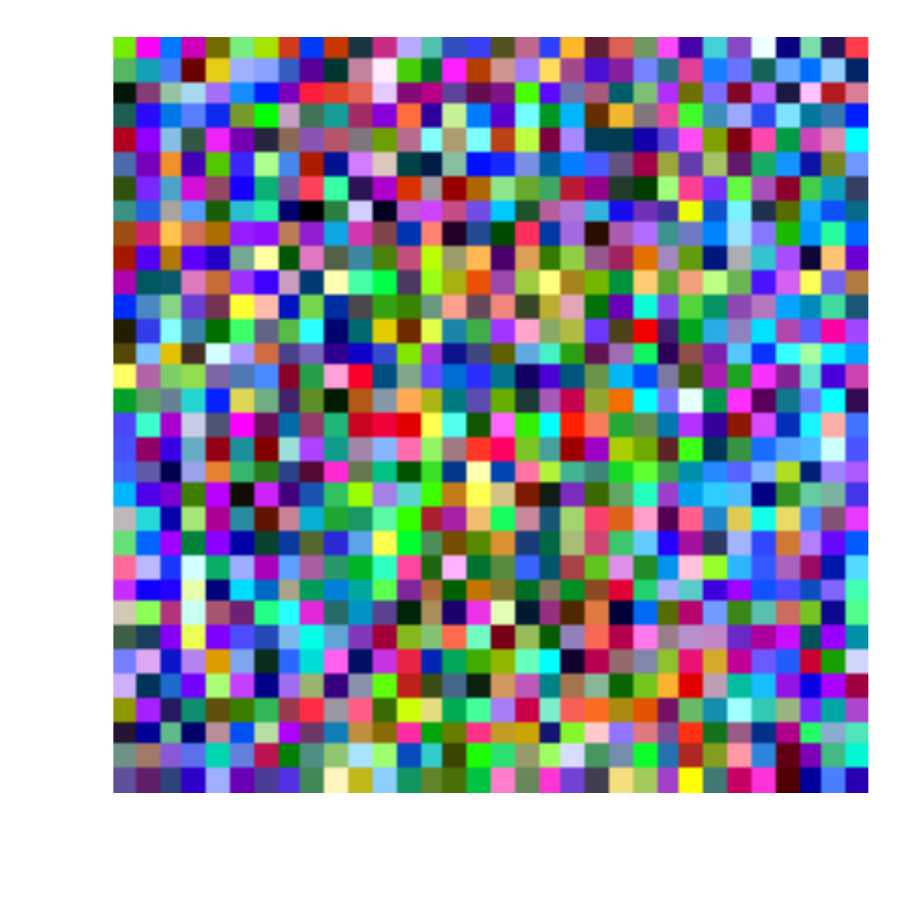}&
        \includegraphics[width=\panelwidth]{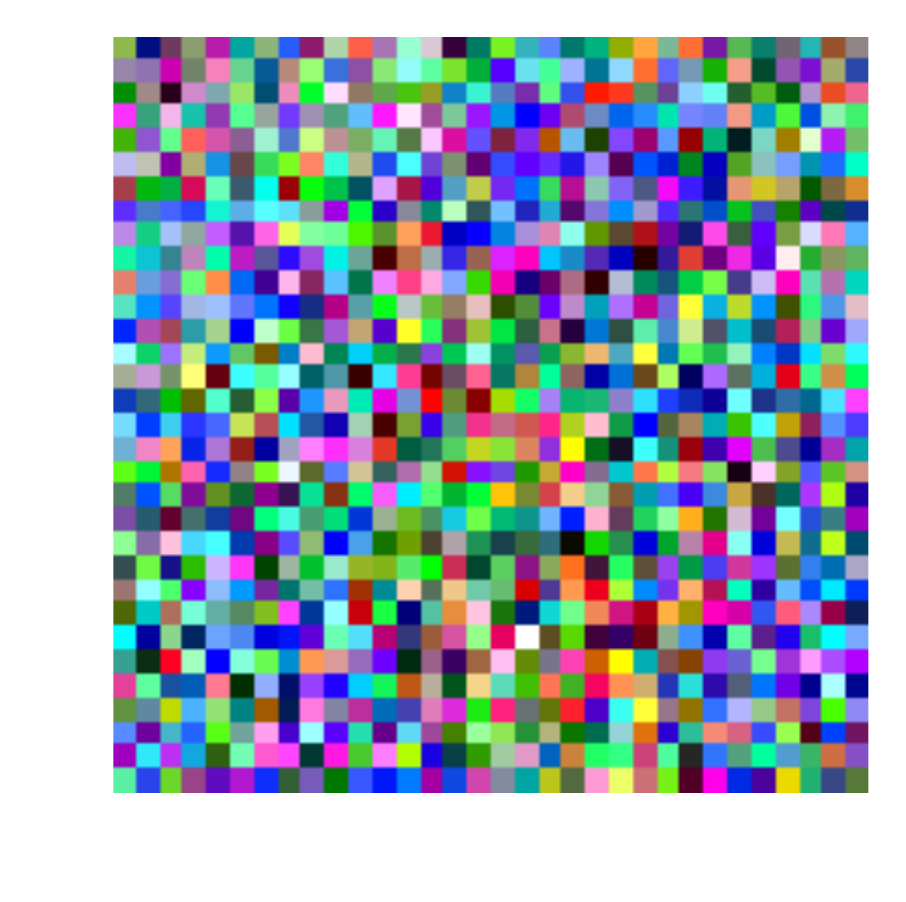}&
        \includegraphics[width=\panelwidth]{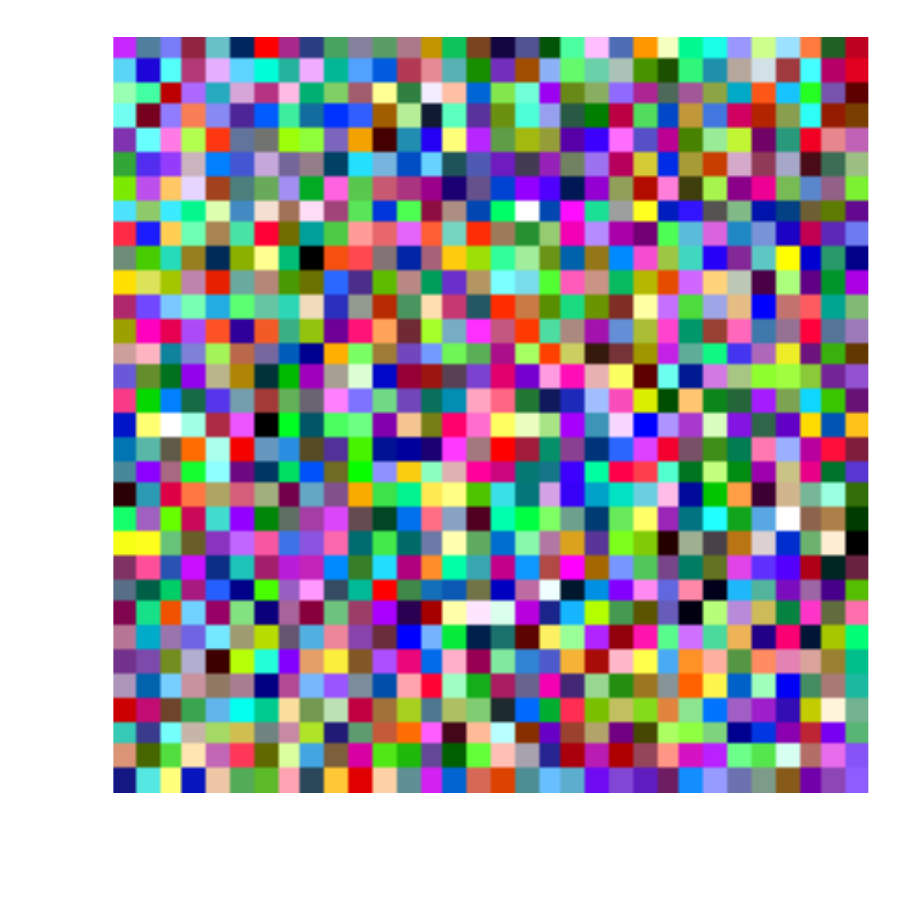}
        \\
		\rotatebox{90}{~$z$ Blue}&
        \hspace{-0.1cm}
        \includegraphics[width=\panelwidth]{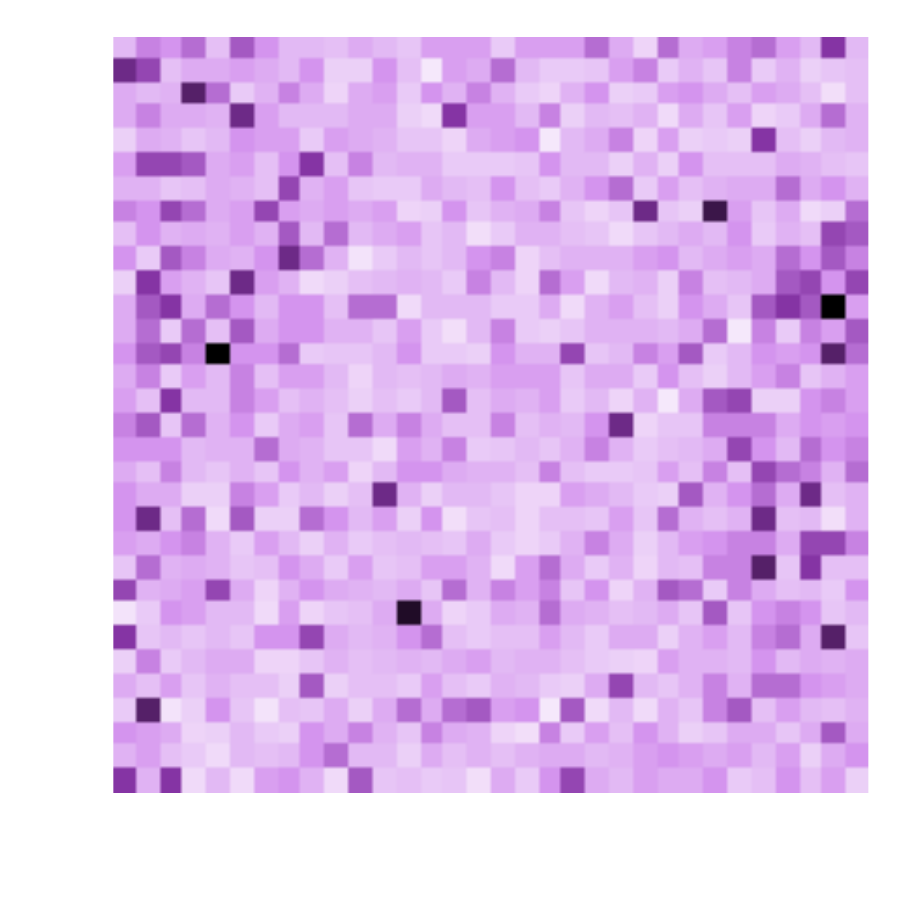}&
        \includegraphics[width=\panelwidth]{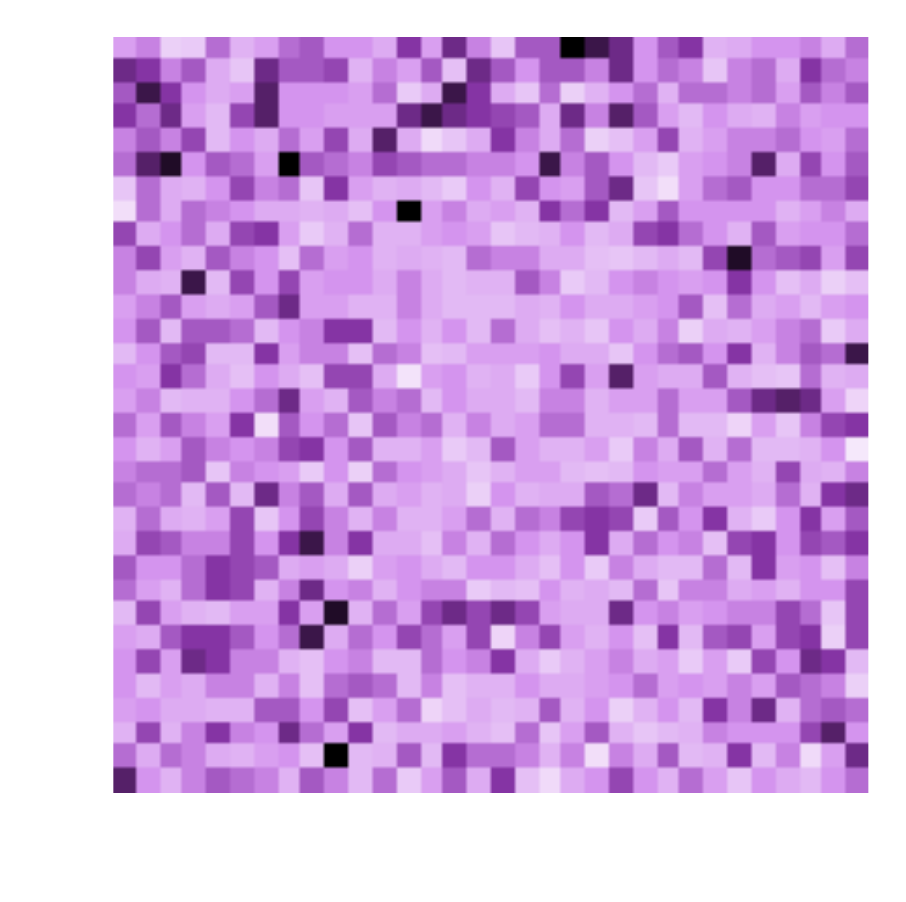}&
        \includegraphics[width=\panelwidth]{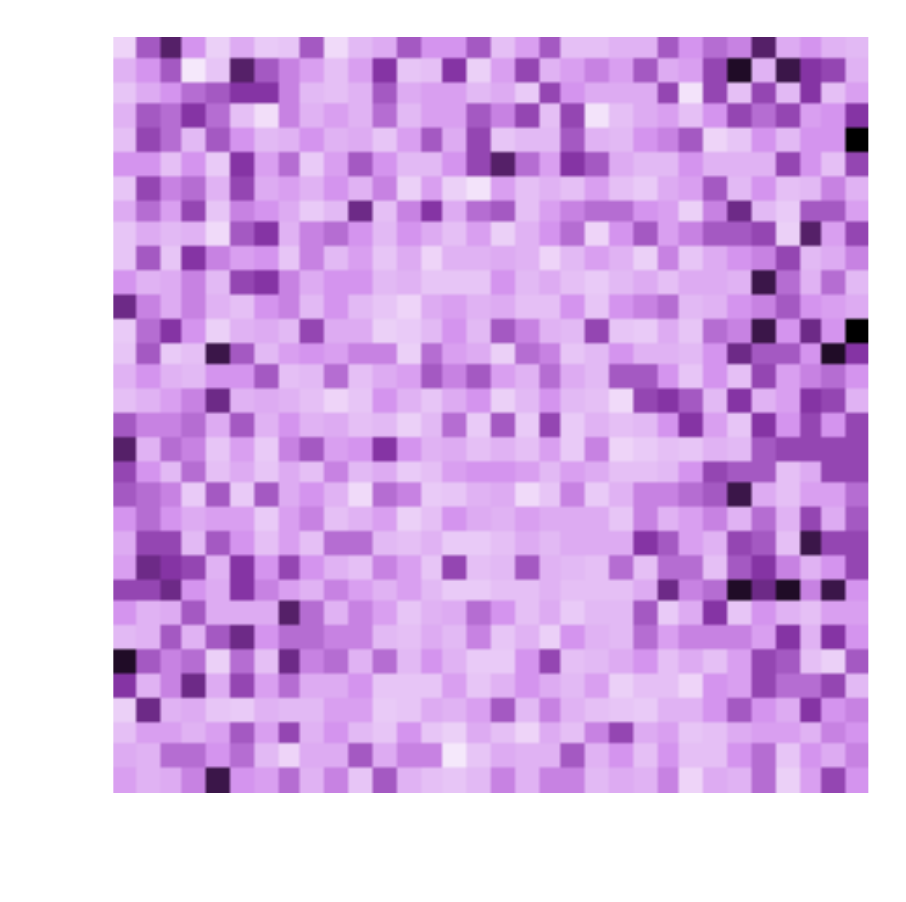}&
        \includegraphics[width=\panelwidth]{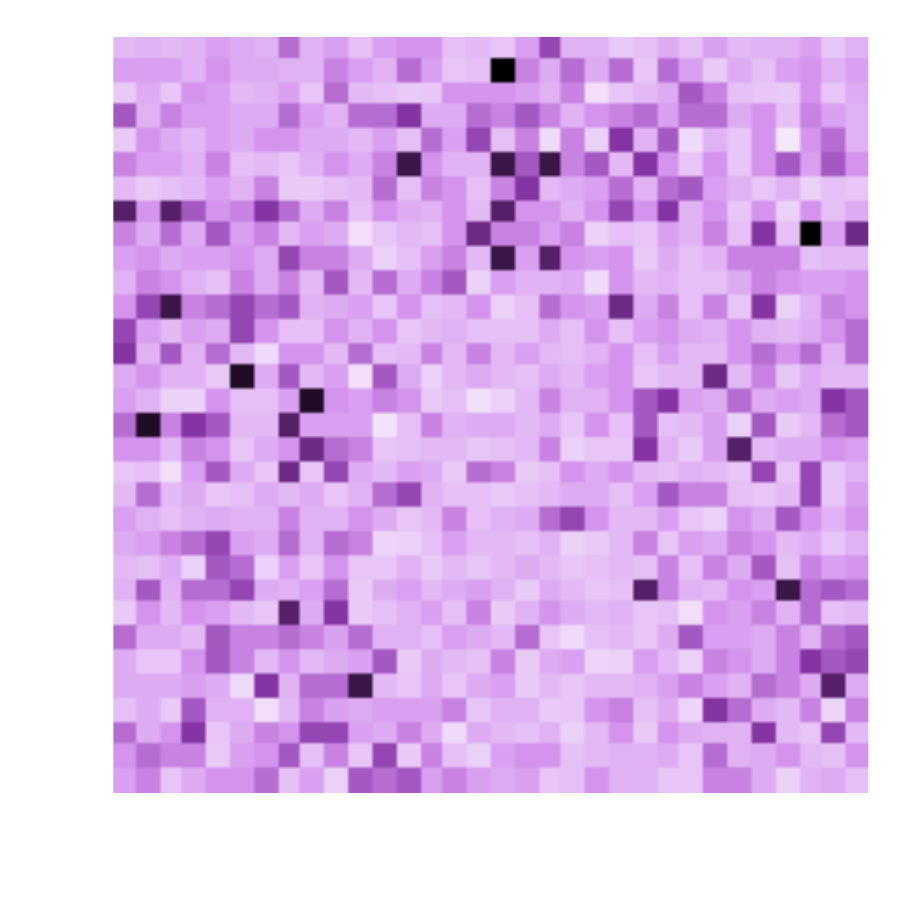}&
        \includegraphics[width=\panelwidth]{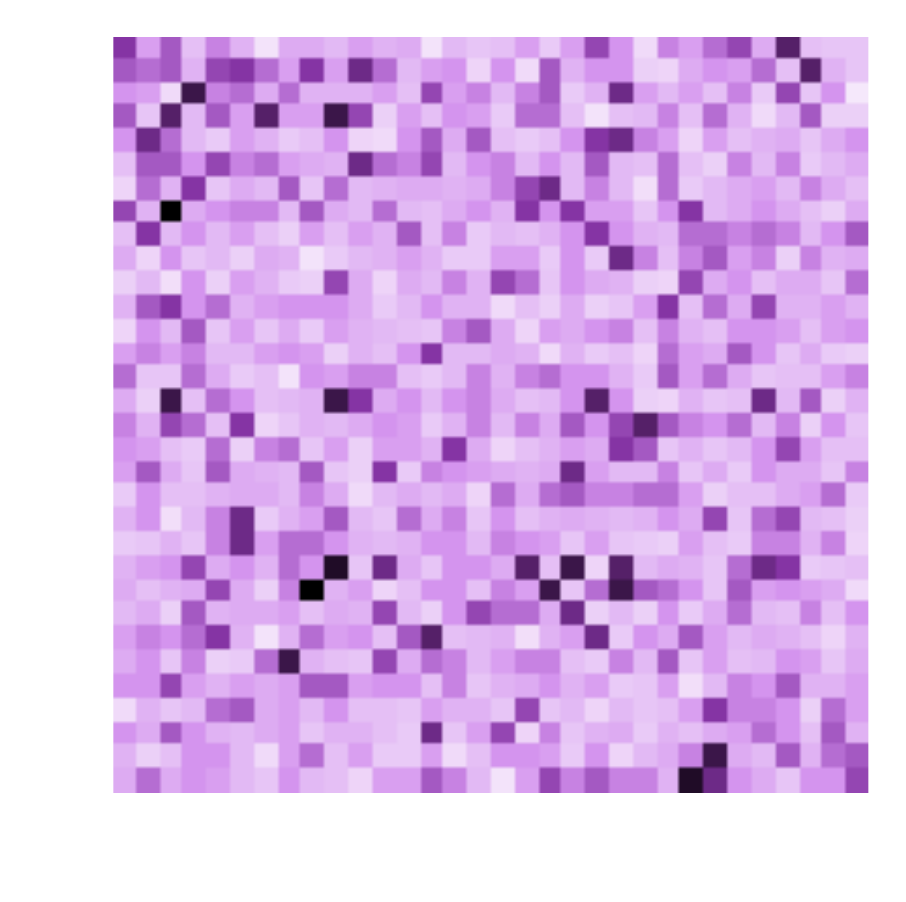}
        \\
		\rotatebox{90}{~~~BN $z$}&
        \hspace{-0.1cm}
        \includegraphics[width=\panelwidth]{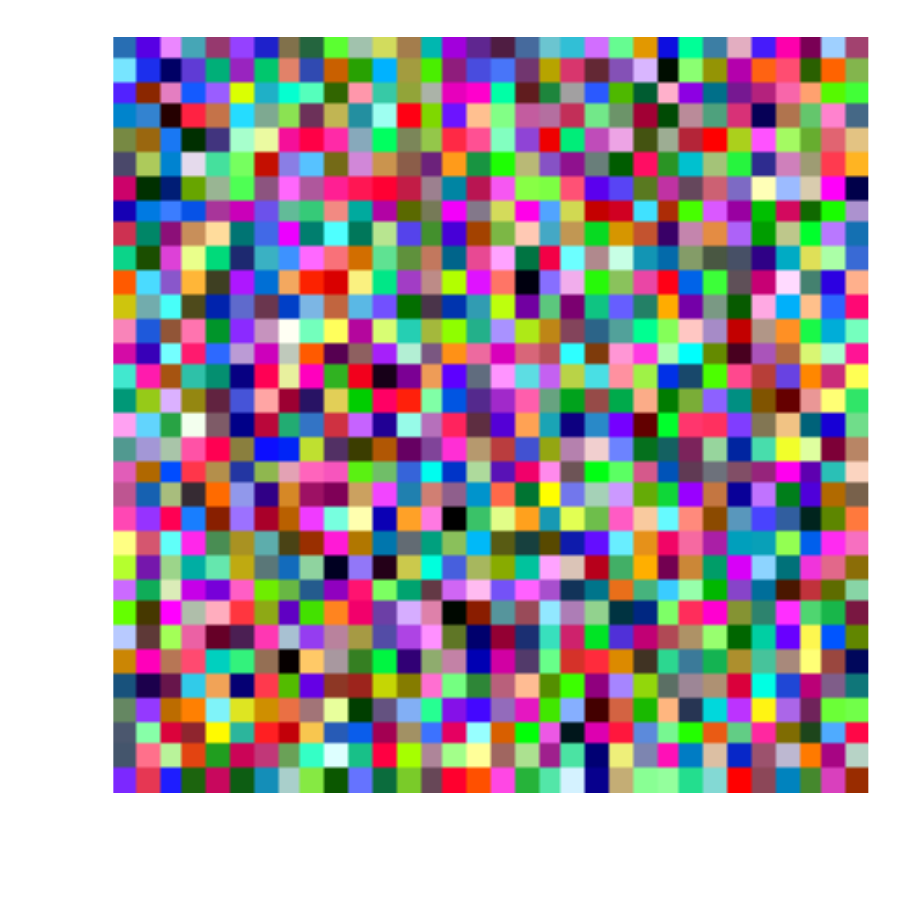}&
        \includegraphics[width=\panelwidth]{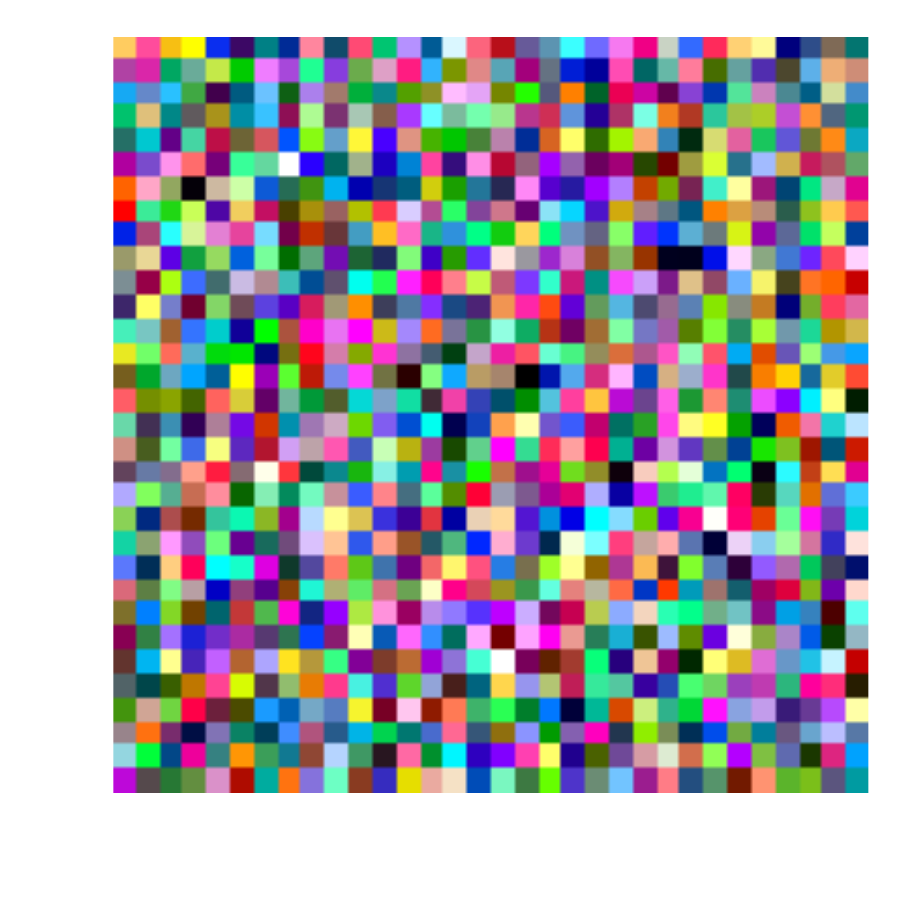}&
        \includegraphics[width=\panelwidth]{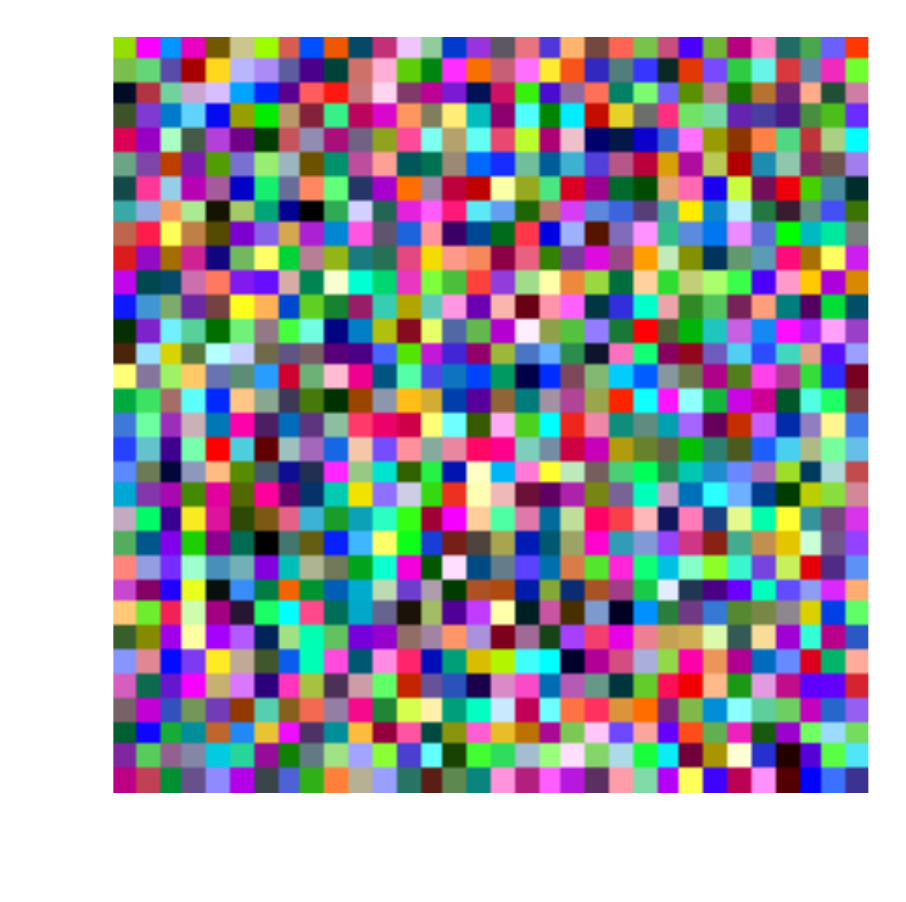}&
        \includegraphics[width=\panelwidth]{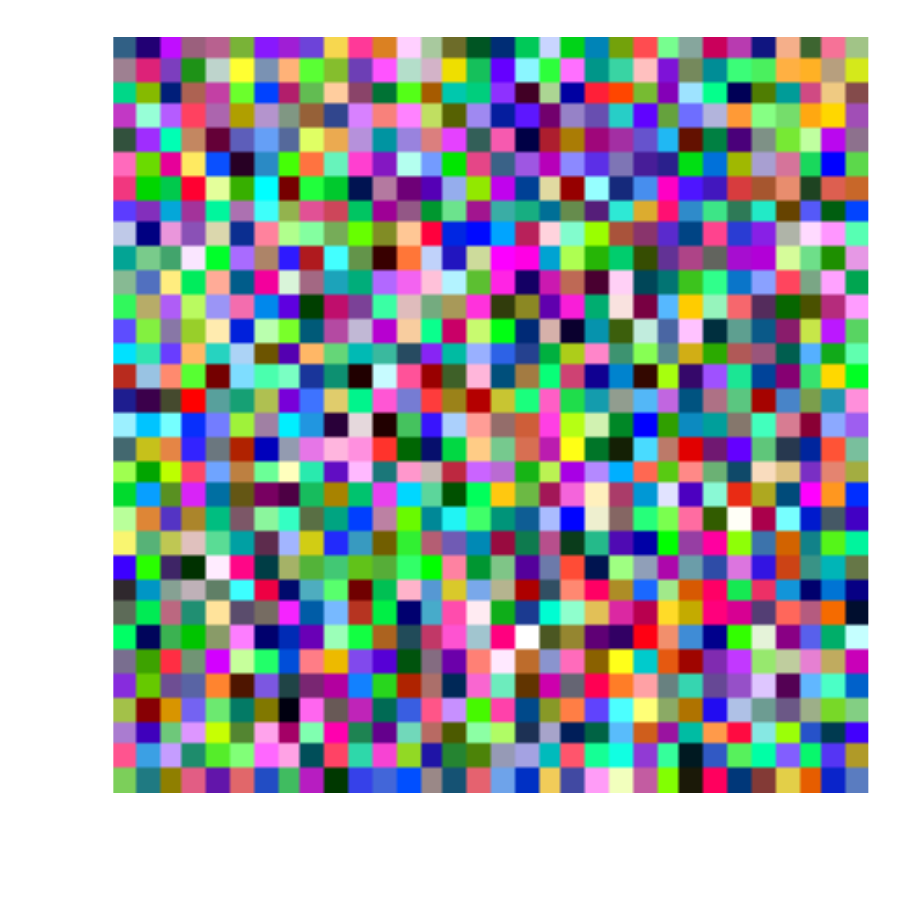}&
        \includegraphics[width=\panelwidth]{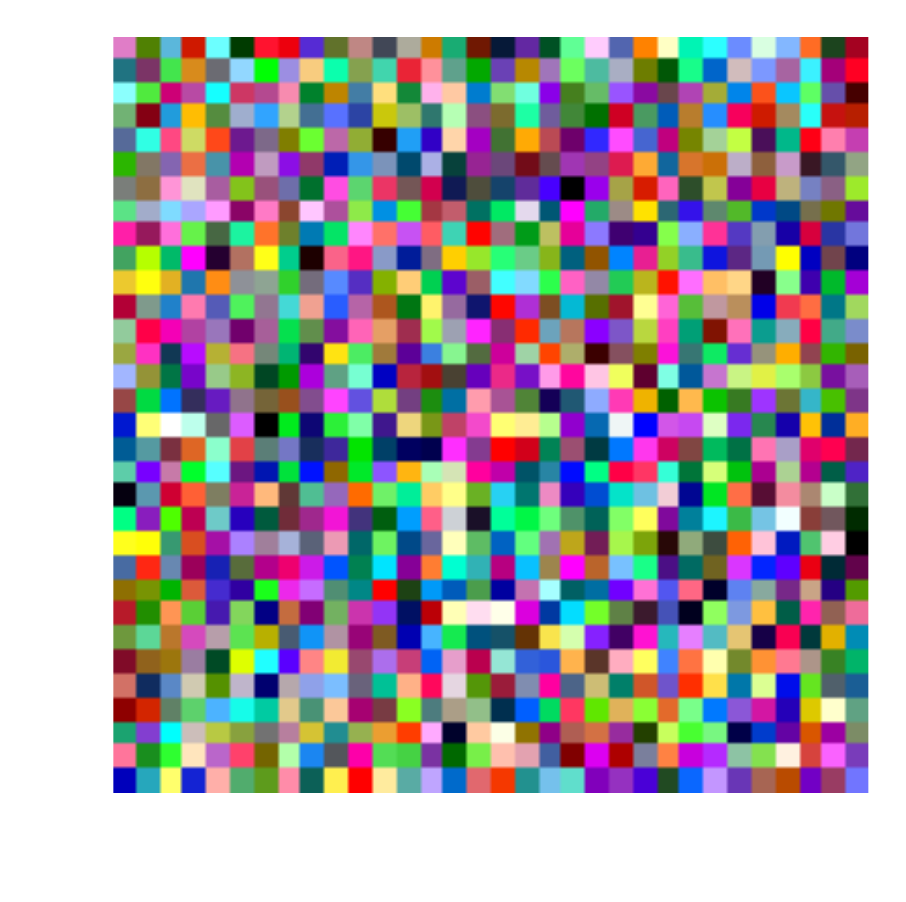}
    \end{tabular}
    \quad
    \begin{tabular}{ccccc}
        \includegraphics[width=\panelwidth]{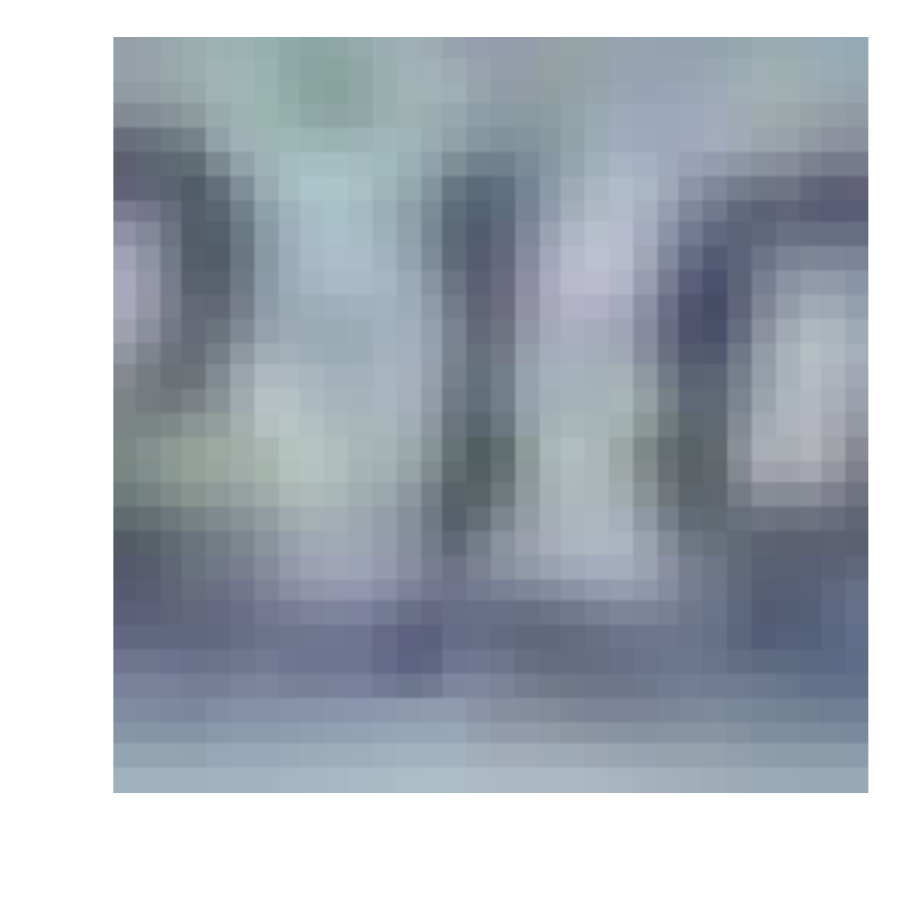}&
        \includegraphics[width=\panelwidth]{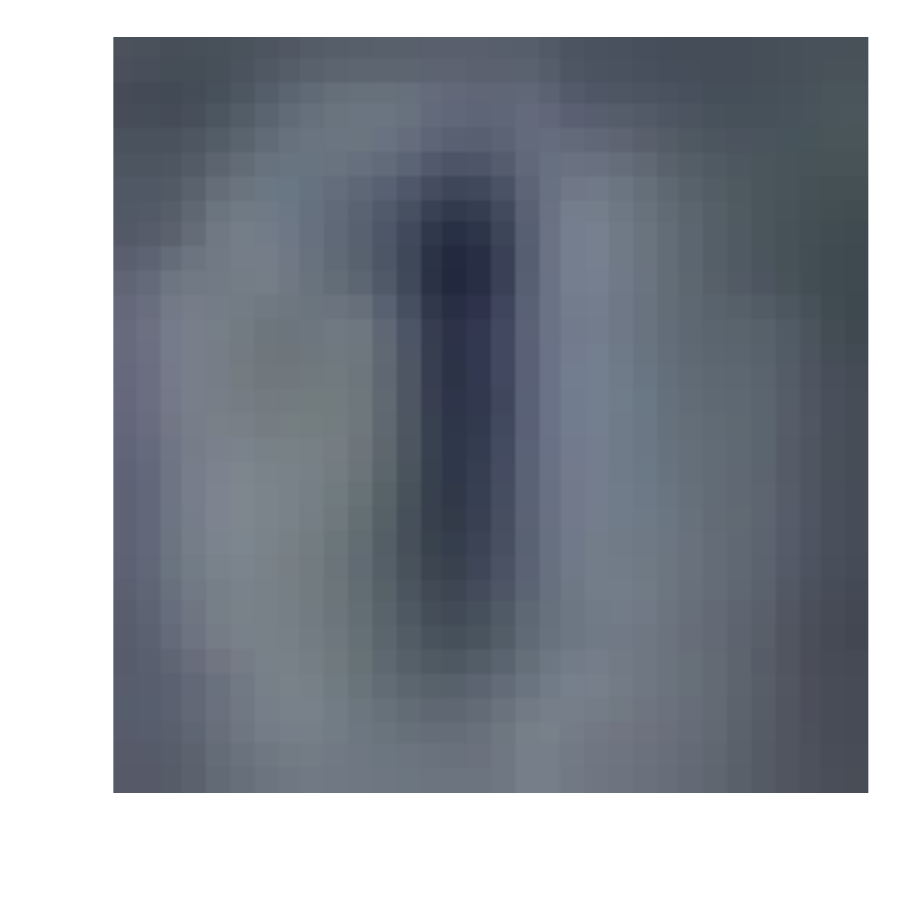}&
        \includegraphics[width=\panelwidth]{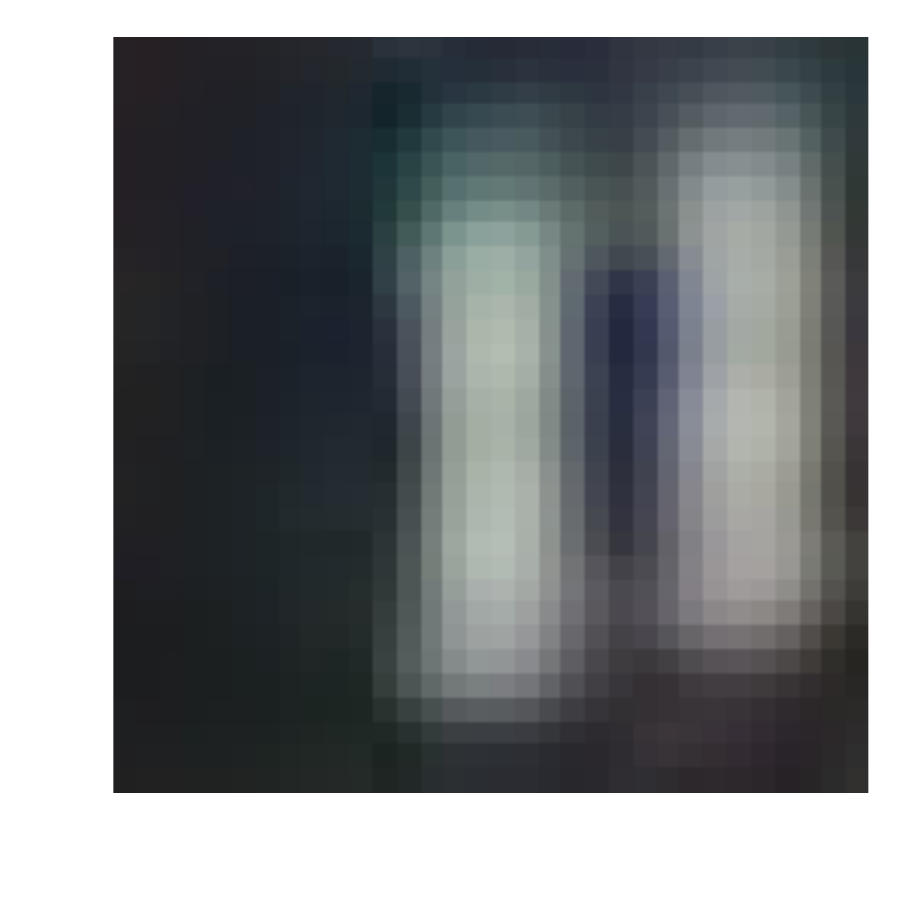}&
        \includegraphics[width=\panelwidth]{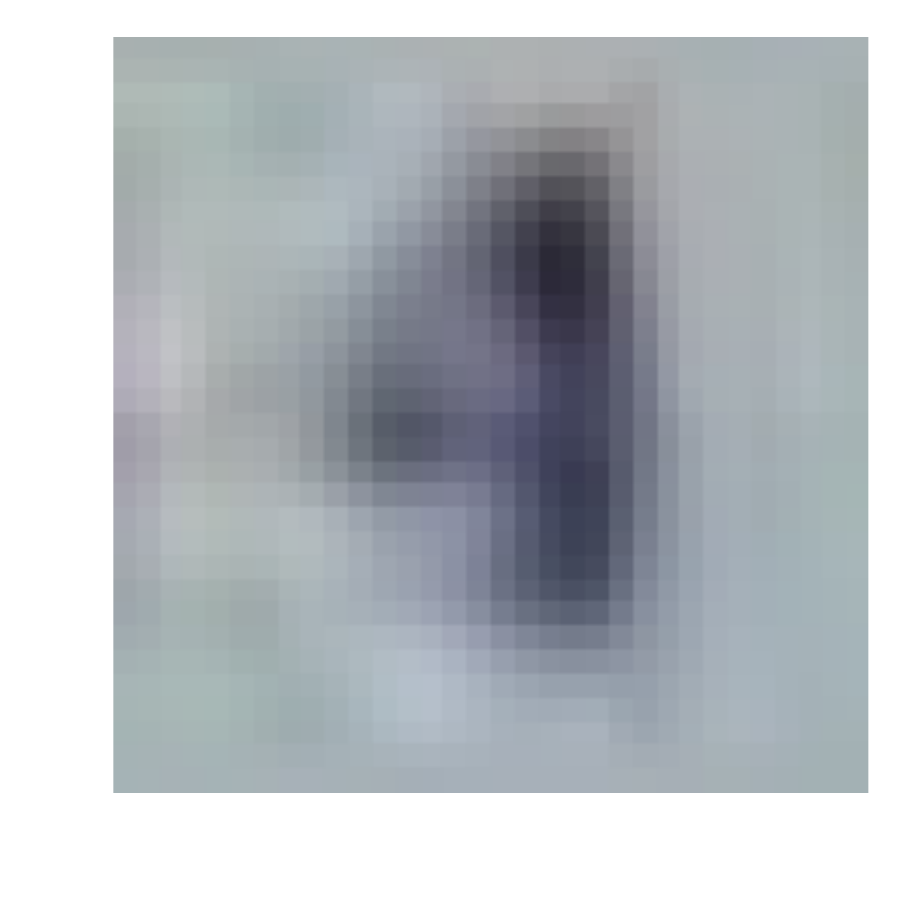}&
        \includegraphics[width=\panelwidth]{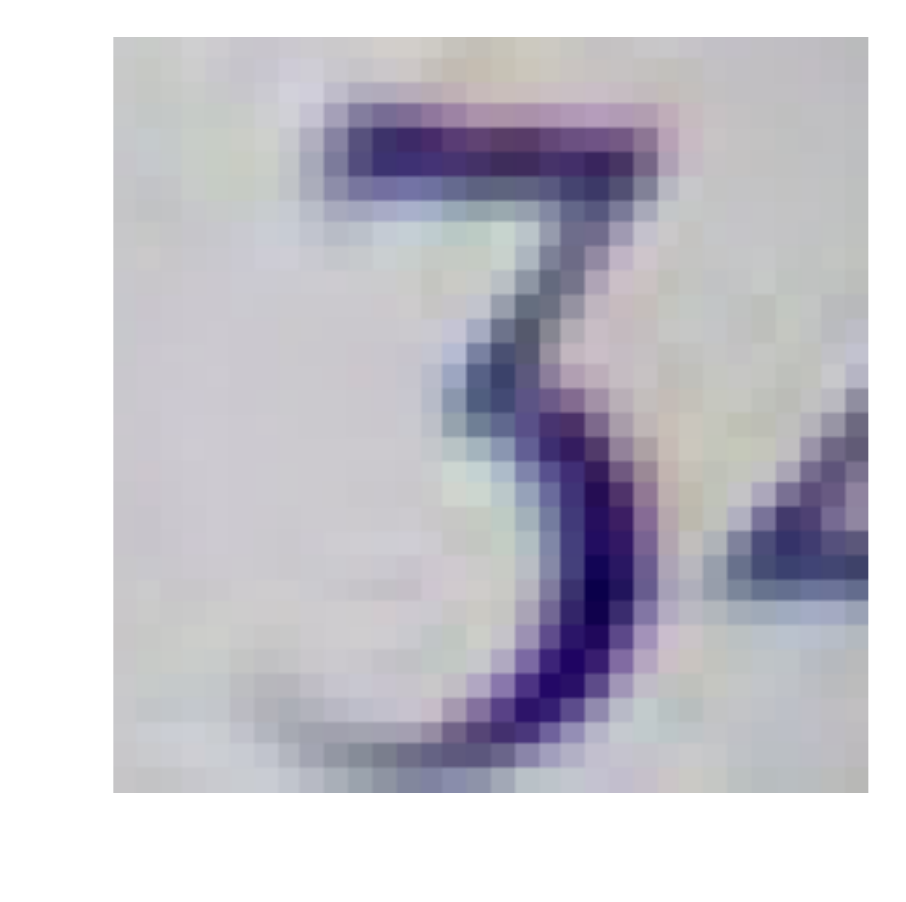}
        \\
        \includegraphics[width=\panelwidth]{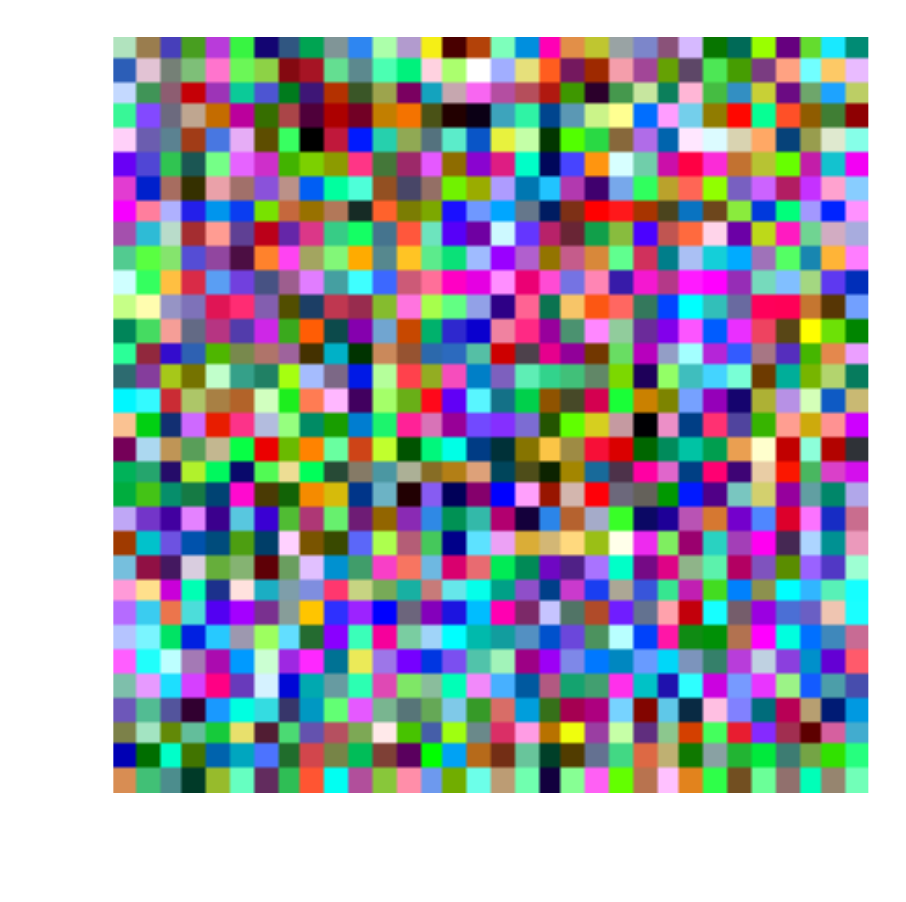}&
        \includegraphics[width=\panelwidth]{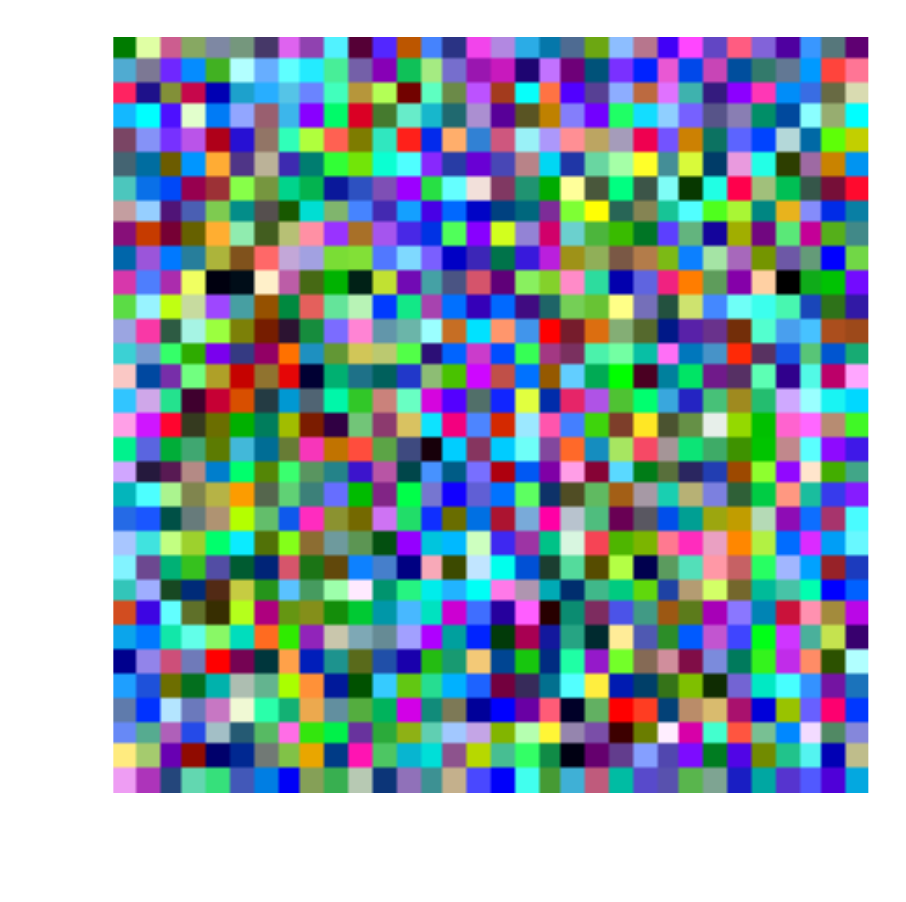}&
        \includegraphics[width=\panelwidth]{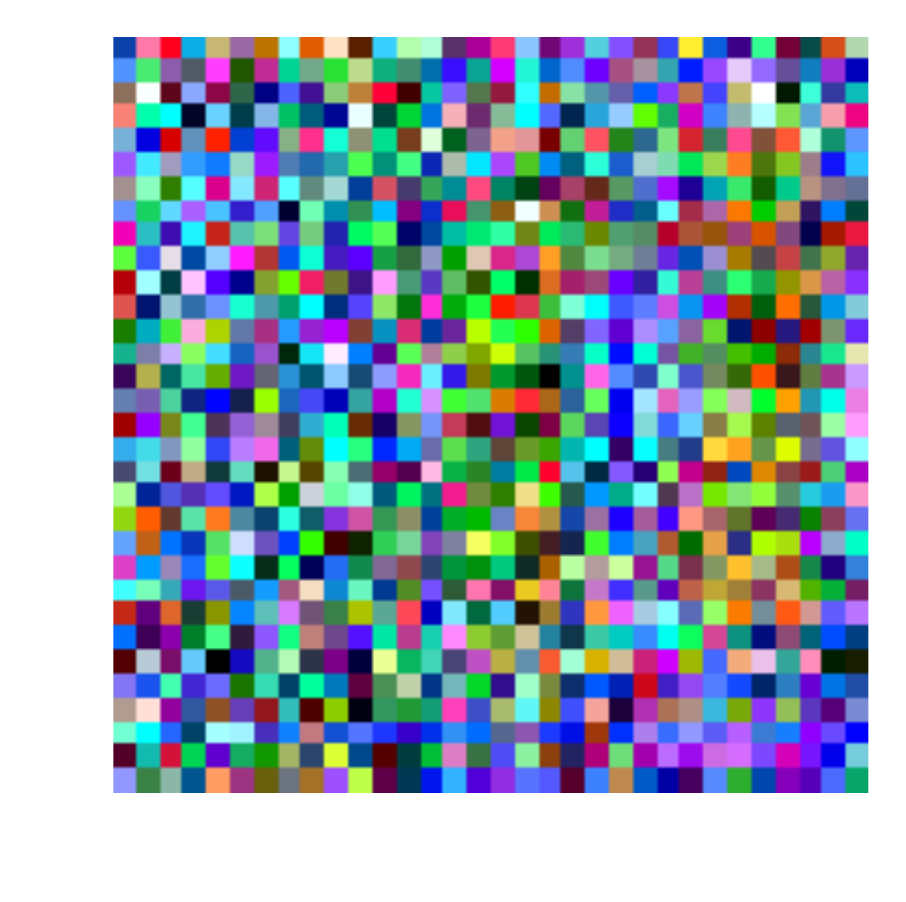}&
        \includegraphics[width=\panelwidth]{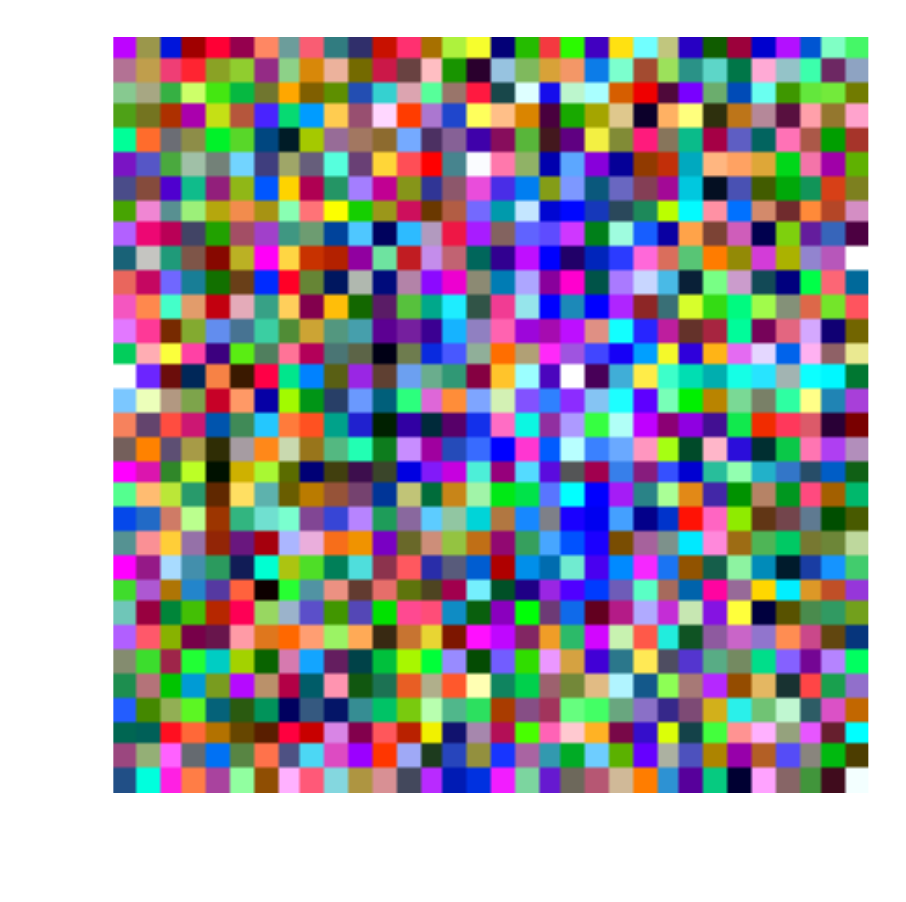}&
        \includegraphics[width=\panelwidth]{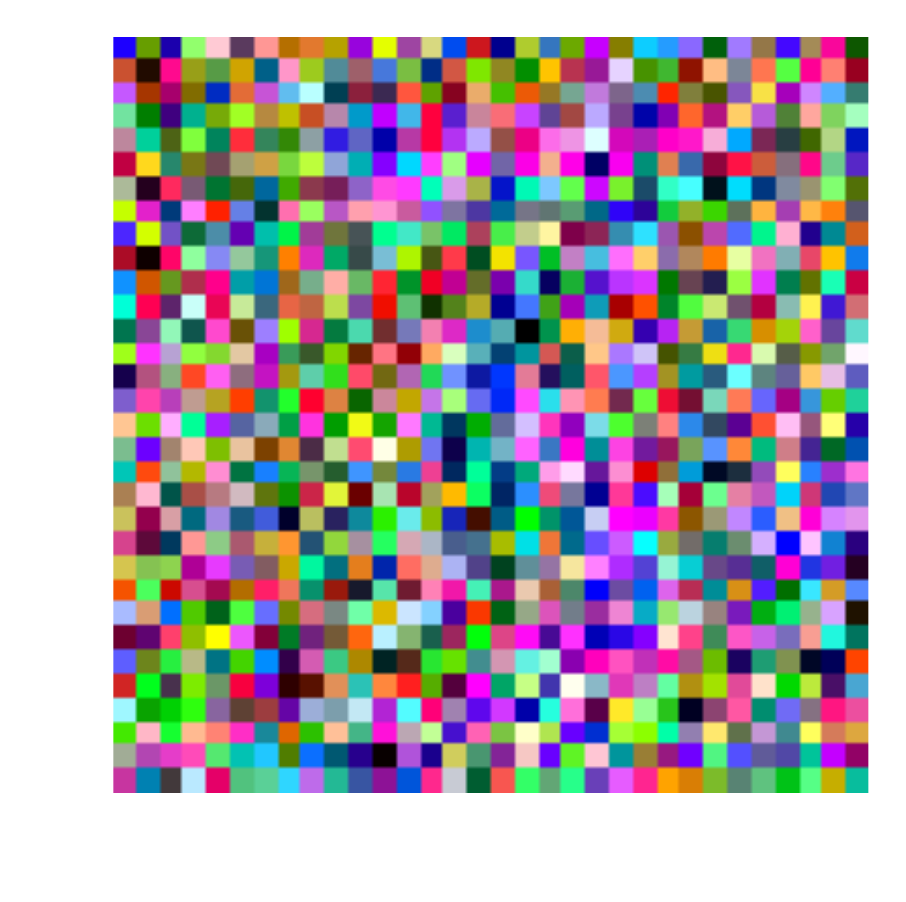}
        \\
        \includegraphics[width=\panelwidth]{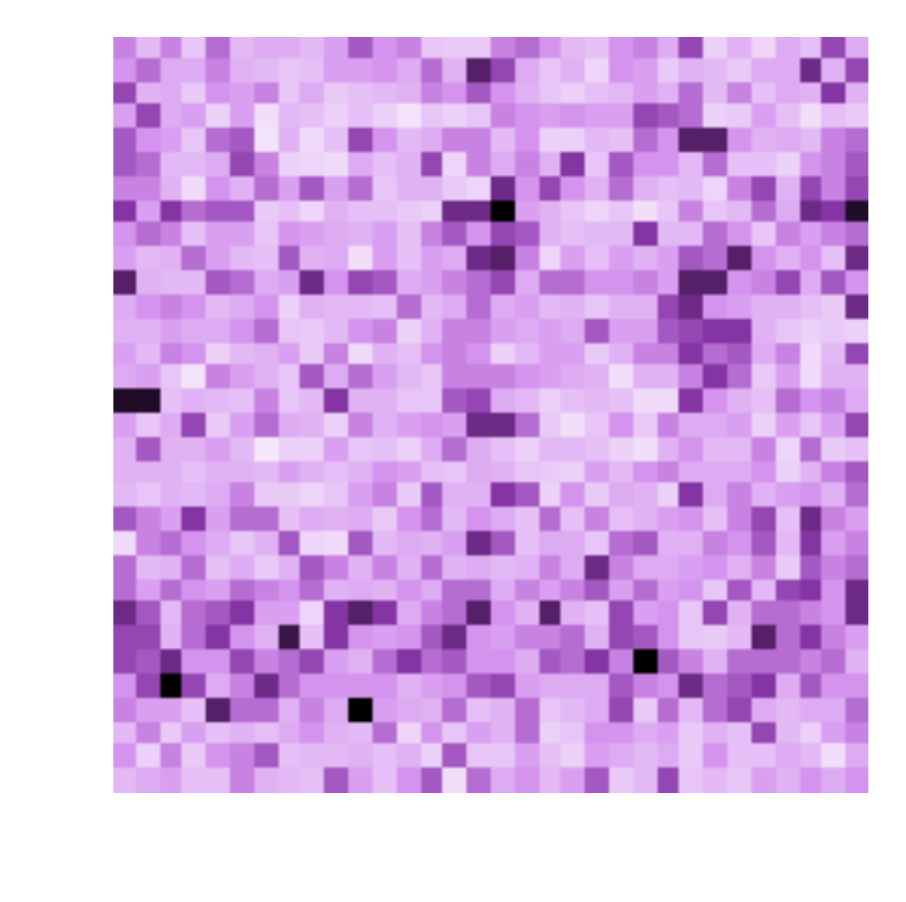}&
        \includegraphics[width=\panelwidth]{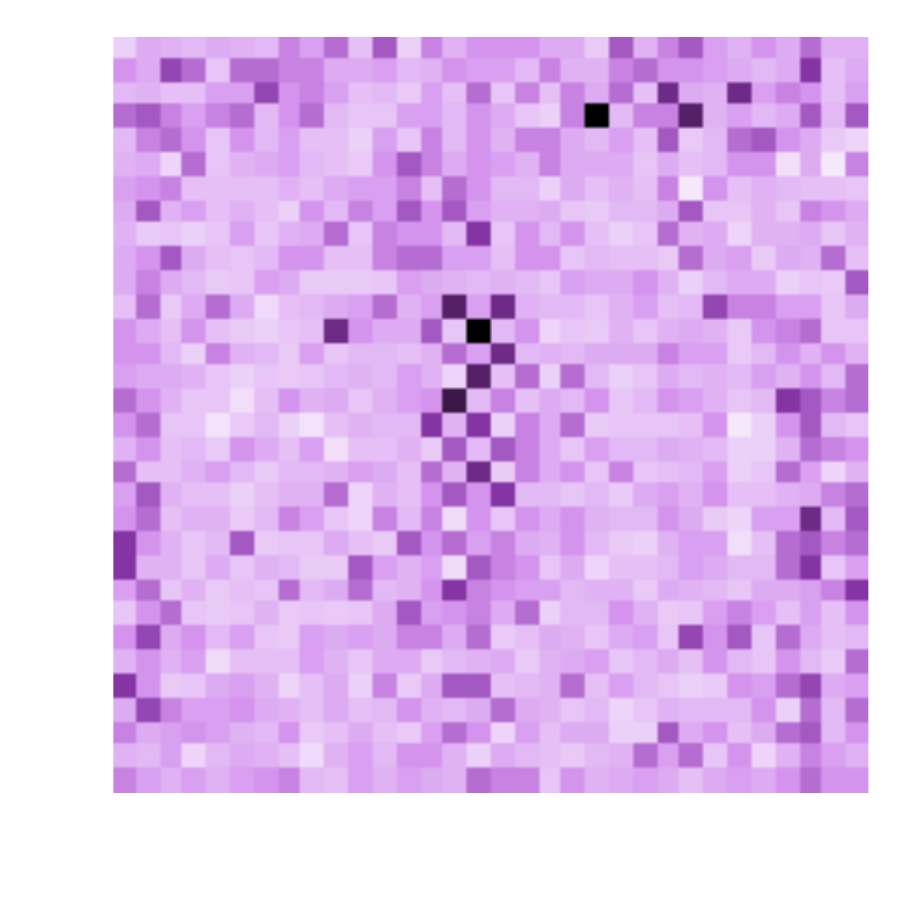}&
        \includegraphics[width=\panelwidth]{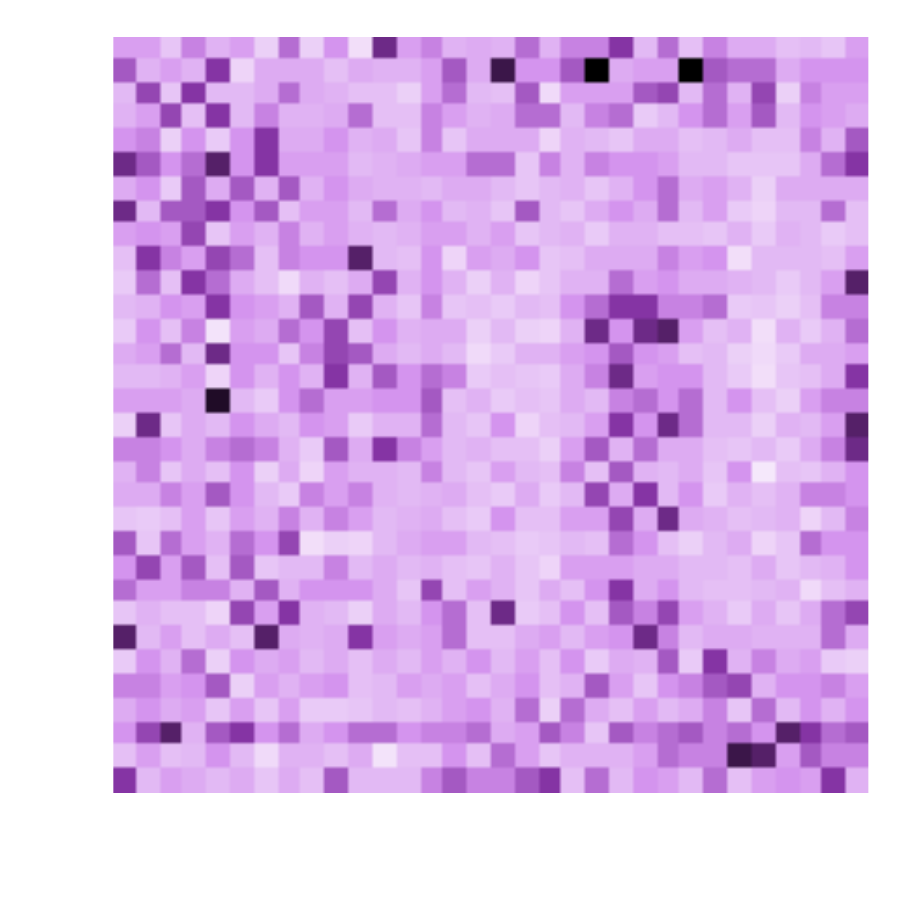}&
        \includegraphics[width=\panelwidth]{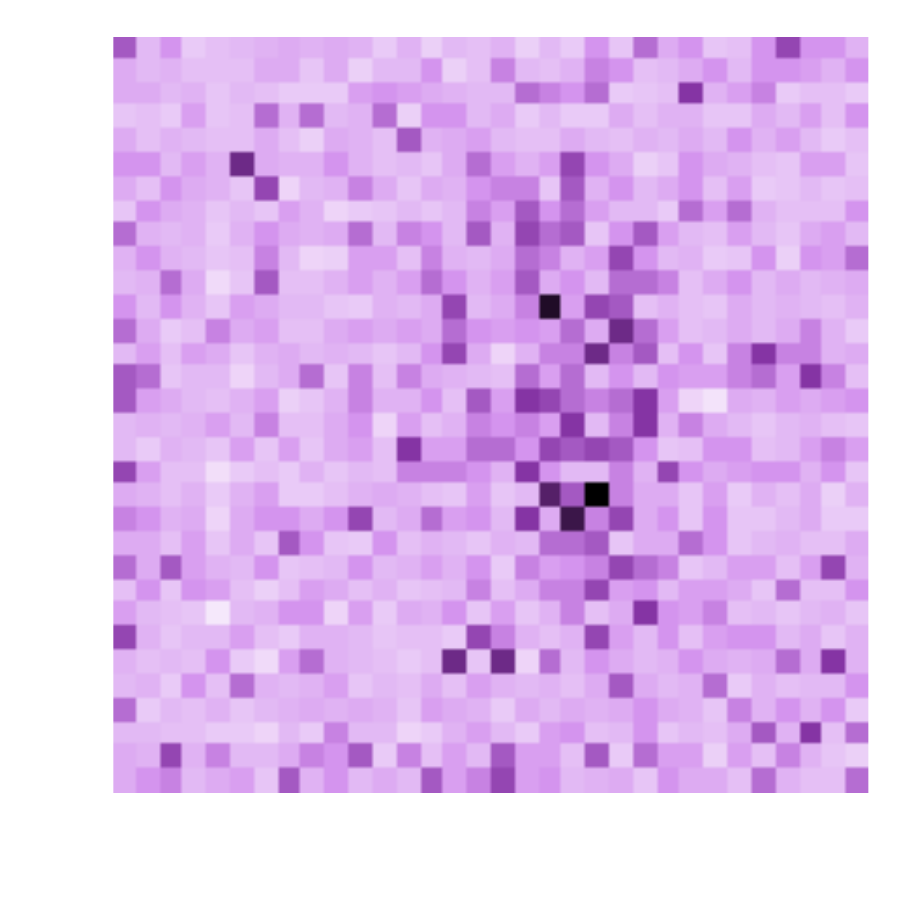}&
        \includegraphics[width=\panelwidth]{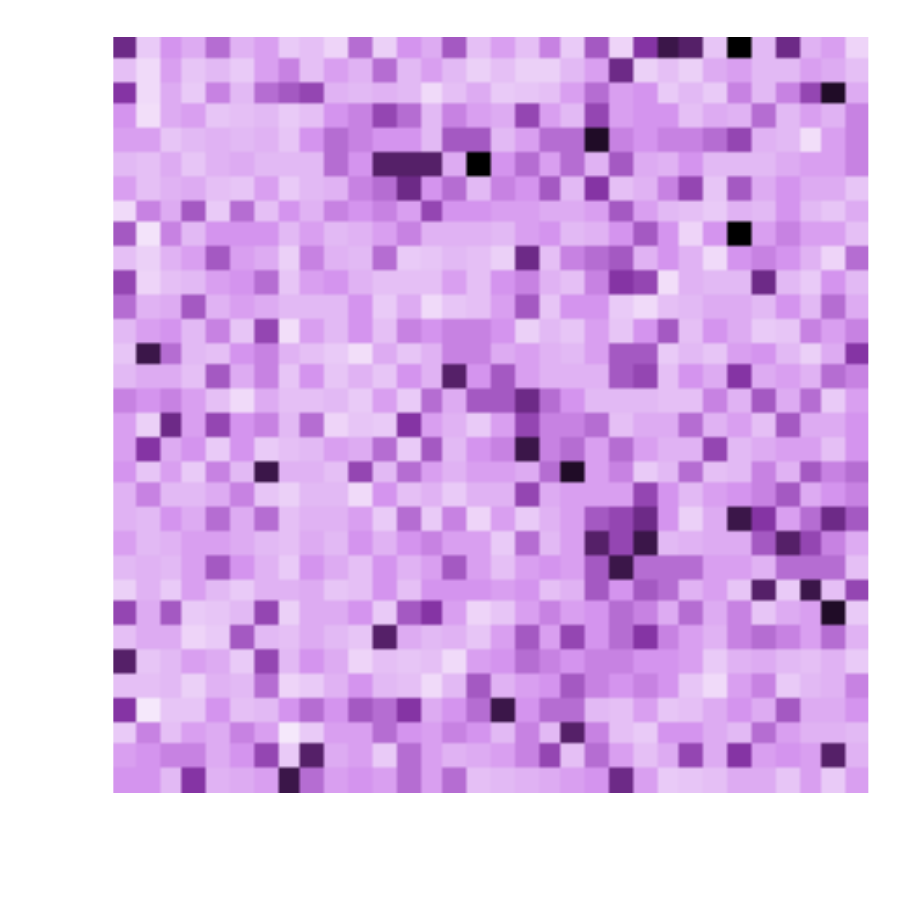}
        \\
        \includegraphics[width=\panelwidth]{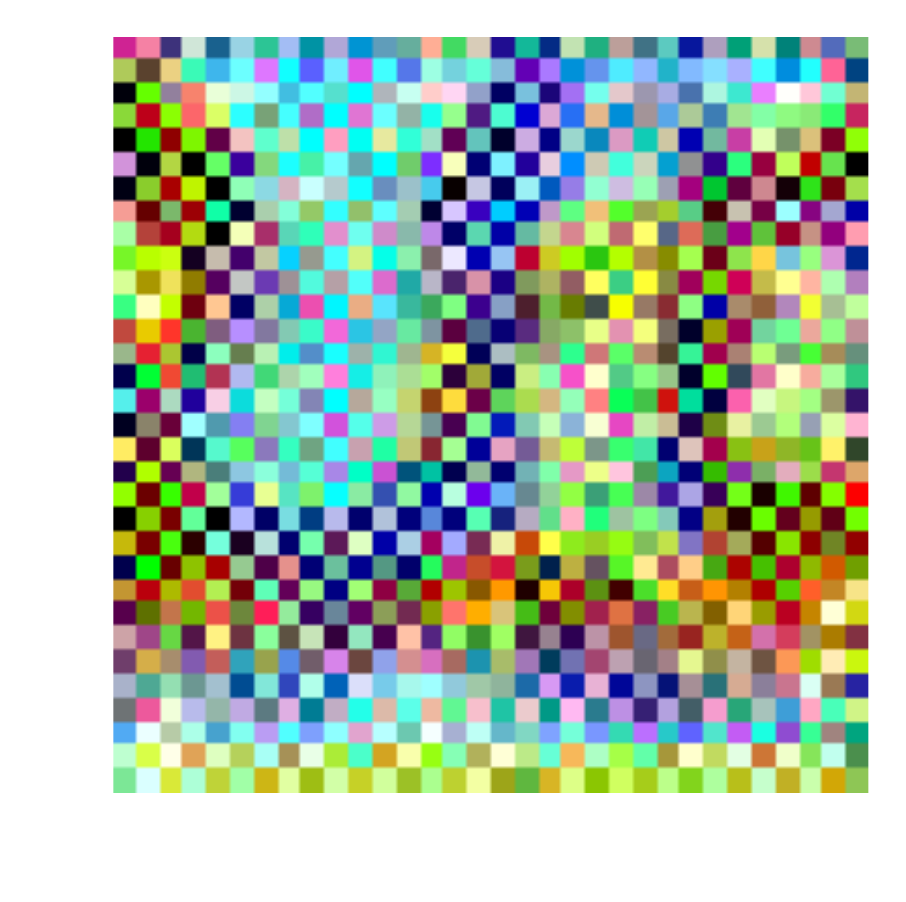}&
        \includegraphics[width=\panelwidth]{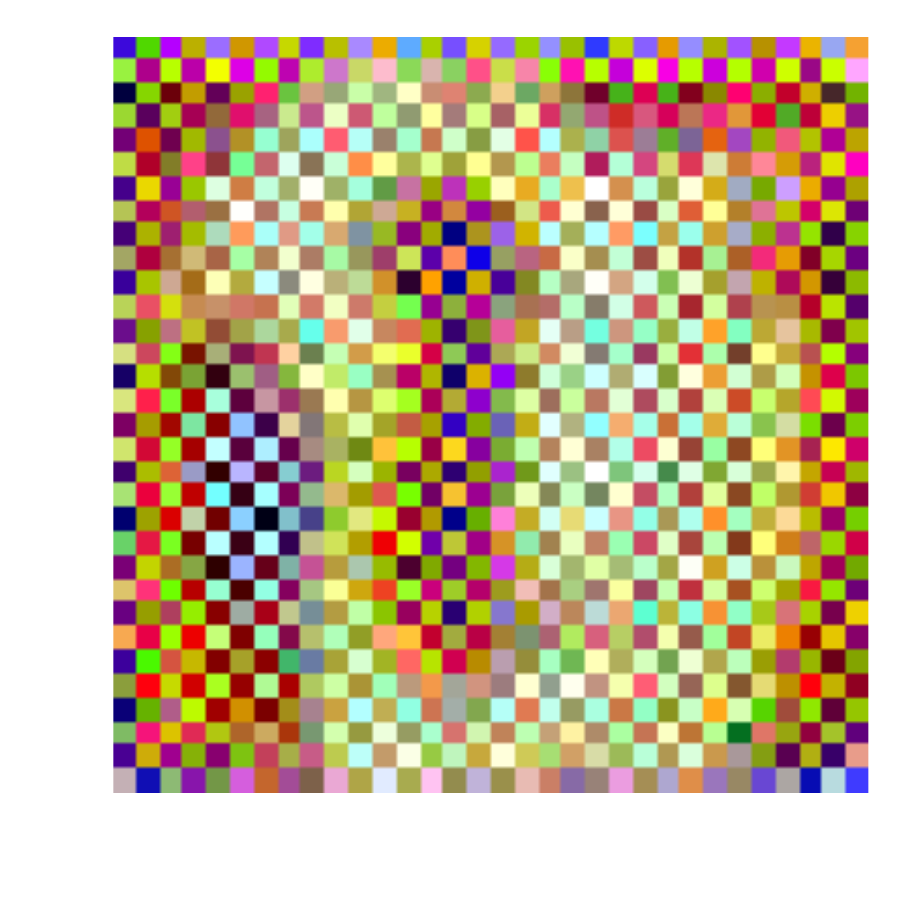}&
        \includegraphics[width=\panelwidth]{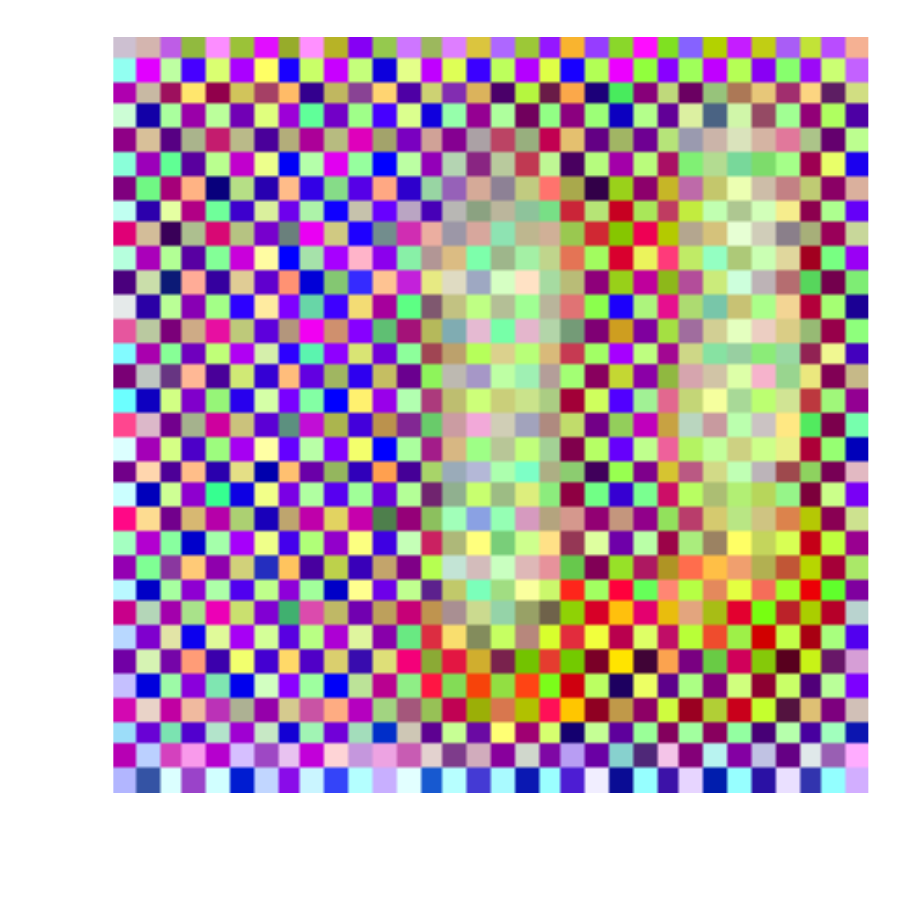}&
        \includegraphics[width=\panelwidth]{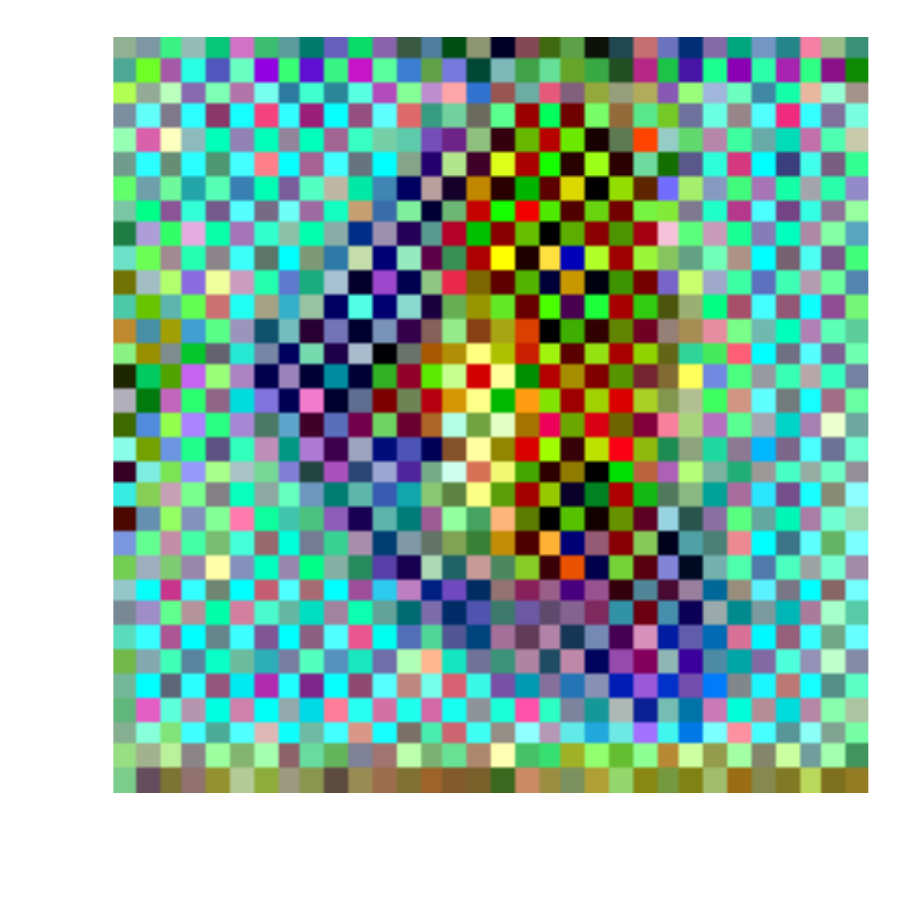}&
        \includegraphics[width=\panelwidth]{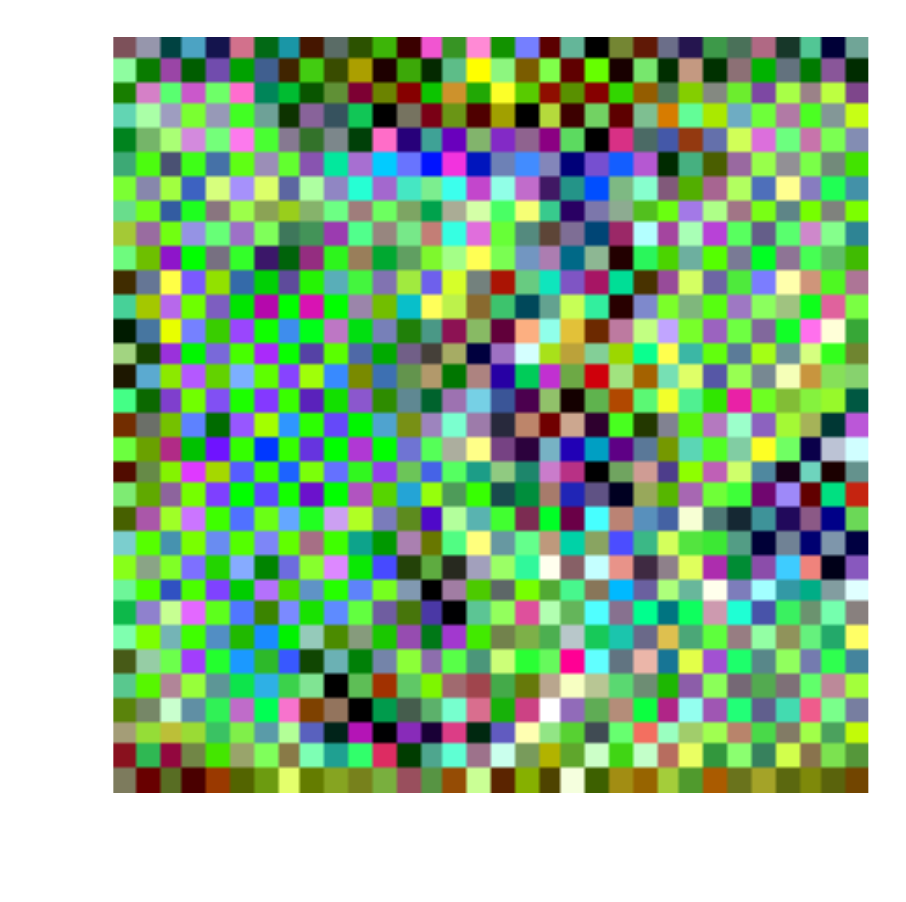}
    \end{tabular}
    }

	\caption{
	\textbf{Latent spaces.}
	Visualization of latent representations for RealNVP and Glow models on
	in-distribution and out-of-distribution inputs.
	\textbf{Rows 1-3 in (a) and (b)}: original images, latent representations, 
	latent representation averaged over $40$ samples of dequantization noise for RealNVP and Glow model trained on FashionMNIST and using MNIST for OOD data.
	\textbf{Row 4 in (a)}: latent representations for batch normalization in train mode.
	\textbf{Rows 1-4 in (c)}: original images, latent representations, 
	the blue channel of the latent representation, and the latent representations for batch normalization in train mode for a RealNVP model trained on CelebA and using SVHN as OOD data.
	For both dataset pairs, we can recognize the shape of the input image in the latent representations.
	The flow represents images based on their graphical appearance rather than semantic content.
		\vspace{-2mm}
	}
	\label{fig:app_latent_reprs}
\end{figure*}

in Figure \ref{fig:app_latent_reprs}, we plot additional latent representations for 
RealNVP and Glow trained on FashionMNIST with MNIST as OOD dataset, RealNVP trained on CelebA with SVHN as OOD. 
The results agree with Section \ref{sec:latent_space}: we can recognize edges from the original
inputs in their latent representations.

\begin{figure}[t]
    \centering
    \hspace{-.3cm}
    \subfigure[RealNVP trained on FashionMNIST]{
    \begin{tabular}{c}
	    \includegraphics[height=0.2\linewidth]{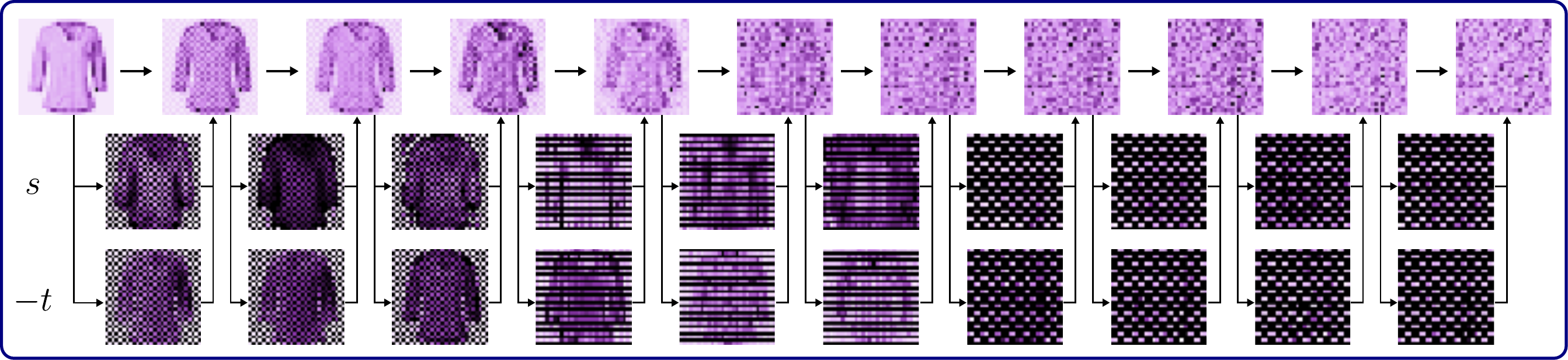}
        \\
	    \includegraphics[height=0.2\linewidth]{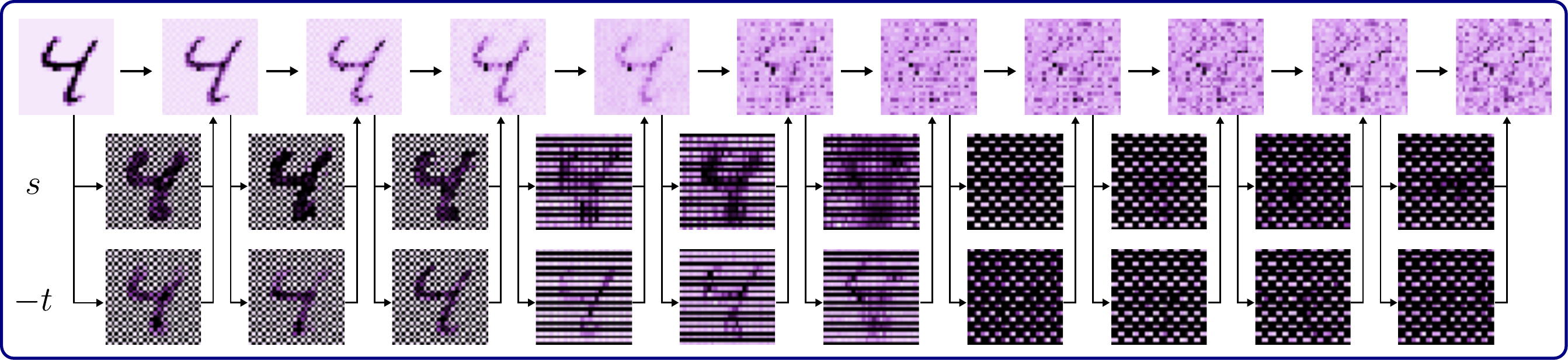}
    \end{tabular}
    }
    \subfigure[Glow trained on FashionMNIST]{
    \begin{tabular}{c}
	    \includegraphics[height=0.2\linewidth]{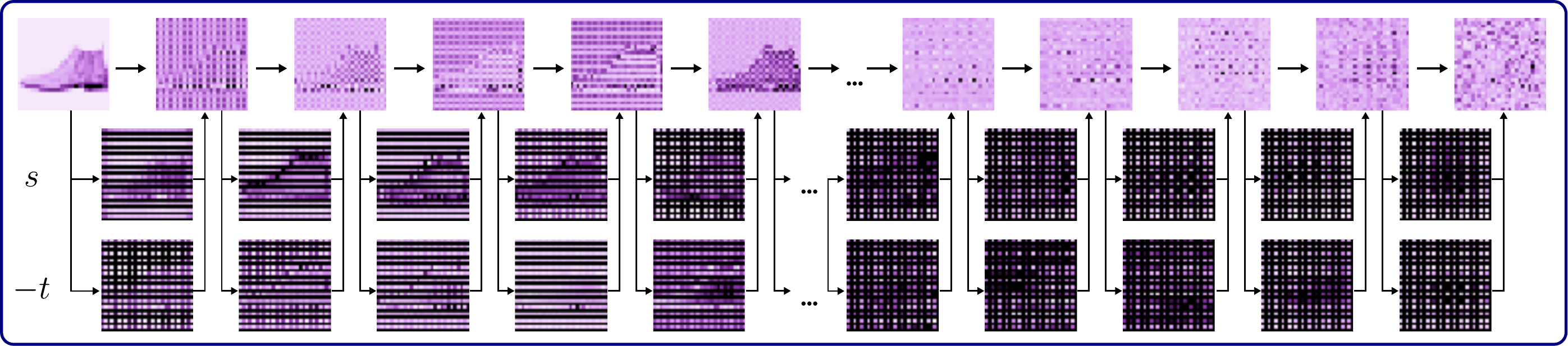}
        \\
	    \includegraphics[height=0.2\linewidth]{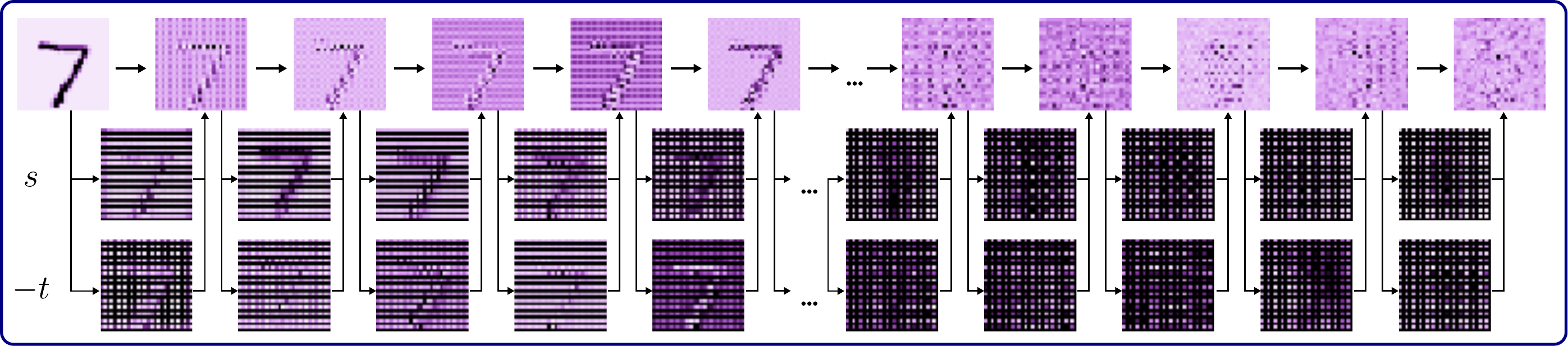}
    \end{tabular}
    }
    \subfigure[RealNVP trained on CelebA]{
    \begin{tabular}{c}
	    \includegraphics[height=0.2\linewidth]{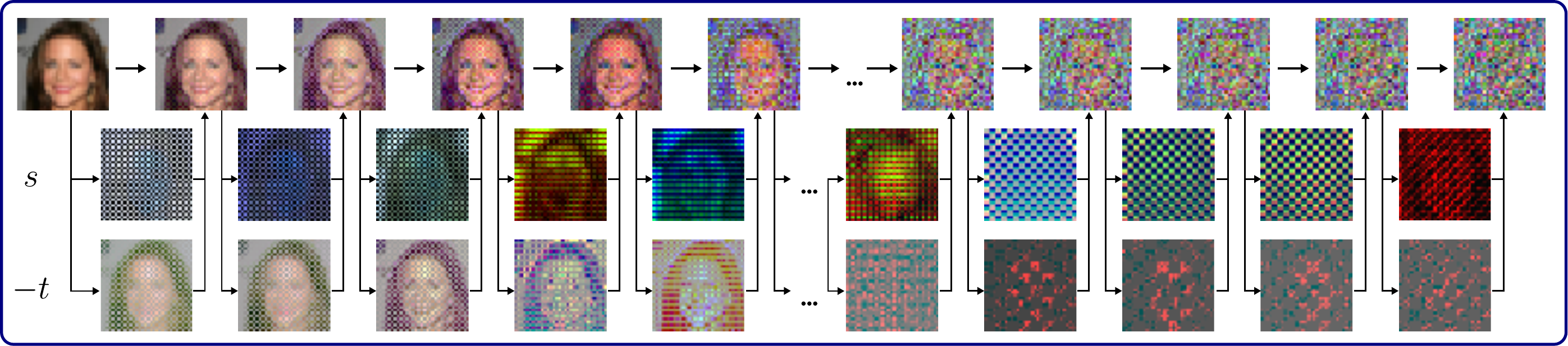}
        \\
	    \includegraphics[height=0.2\linewidth]{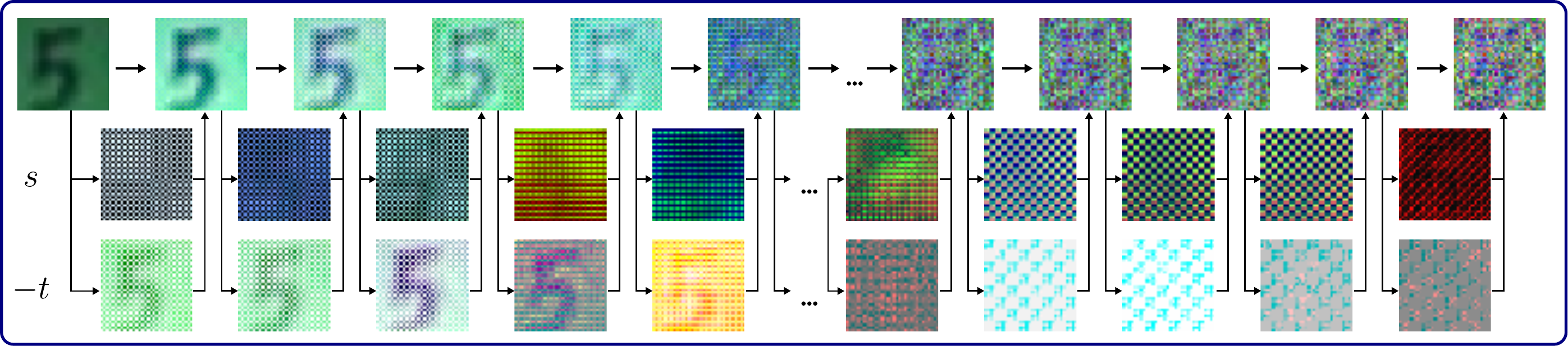}
    \end{tabular}
    }
	\caption{
    \textbf{Coupling layer visualizations.}
    Visualization of intermediate coupling layer activations and $st$-network predictions for 
    \textbf{(a)}: RealNVP trained on FashionMNIST; 
    \textbf{(b)}: Glow trained on FashionMNIST; 
    \textbf{(c)}: RealNVP trained on CelebA. 
    The top half of each subfigure shows the visualizations for an in-distribution image 
    (FashionMNIST or CelebA) while the bottom half shows the visualizations for an OOD image (MNIST or SVHN).
    For all models, the shape of the input both for in- and out-of-distribution 
    image is clearly visible in $s$ and $t$ predictions of the coupling layers.
    }
	\label{fig:app_coupling_layers}
    \vspace{-.3cm}
\end{figure}

\FloatBarrier

\section{Masking strategies}
\label{sec:app_masks}

In Figure \ref{fig:app_squeeze}, we visualize checkerboard, channel-wise masks and horizontal masks on a single-channel image.
The checkerboard and channel-wise masks are commonly used in RealNVP, Glow and other coupling layer-based flows for image data.
We use the horizontal mask to better understand the transformations learned by the coupling layers
in Section \ref{sec:coupling_layers}.

\section{Additional coupling layer visualizations}
\label{sec:app_coupling}

In Figure \ref{fig:app_coupling_layers}, we plot additional visualizations of coupling layer
activations and scale $s$ and shift $t$ parameters predicted by $st$-networks.
In Figure \ref{fig:coadaptation} we visualize the coupling layer activations for the
small flow with horizontal mask from Section \ref{sec:coadaptation} on several 
additional OOD inputs. 
These visualizations provide additional empirical support for Section \ref{sec:coupling_layers}.

\begin{figure}[!h]
    
	\def \panelwidth {0.21\textwidth}
	\def \panelskip {-0.75cm}

    \centering
    \subfigure{
    	\includegraphics[width=\panelwidth]{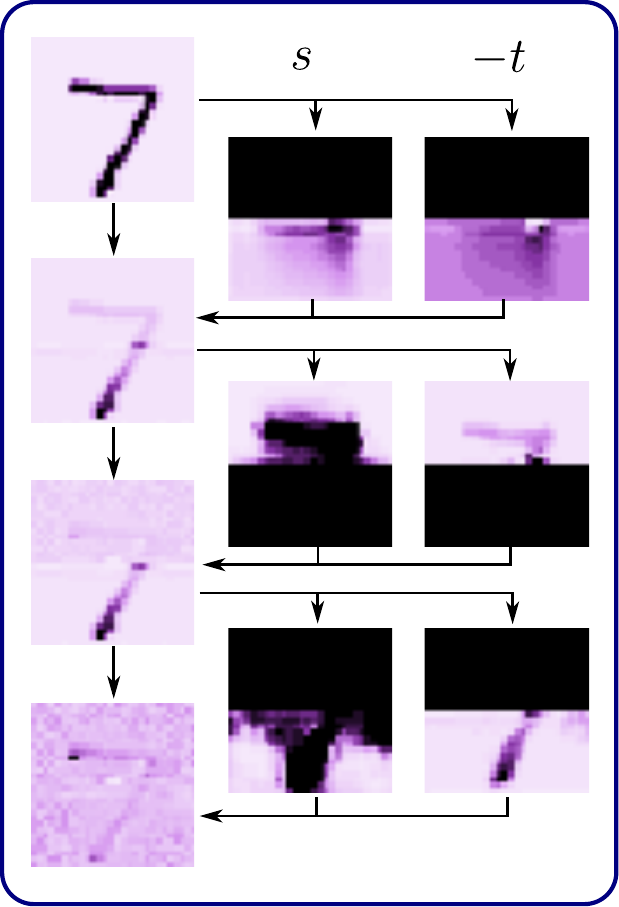} 
    }~~~
    \subfigure{
    	\includegraphics[width=\panelwidth]{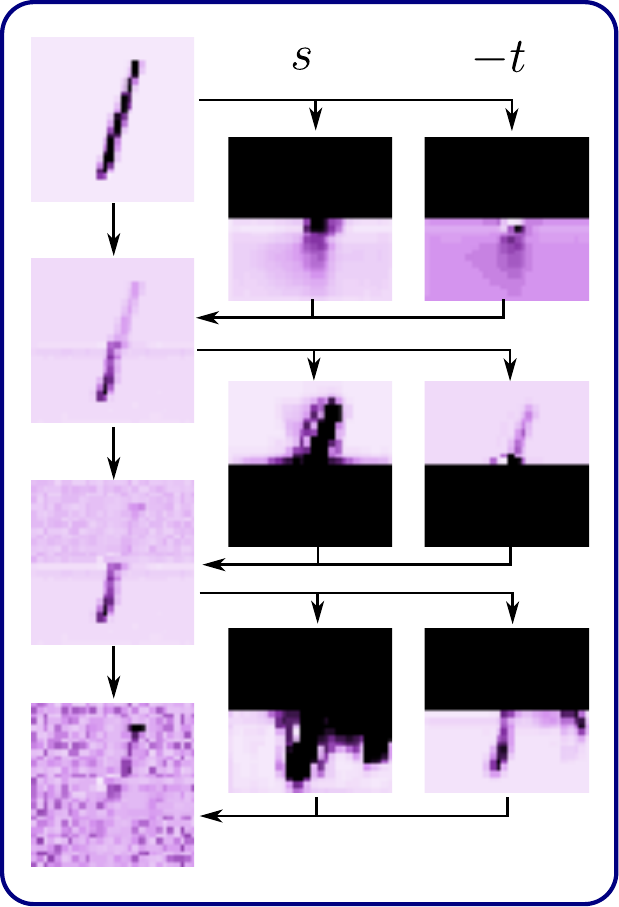} 
    }~~~
    \subfigure{
    	\includegraphics[width=\panelwidth]{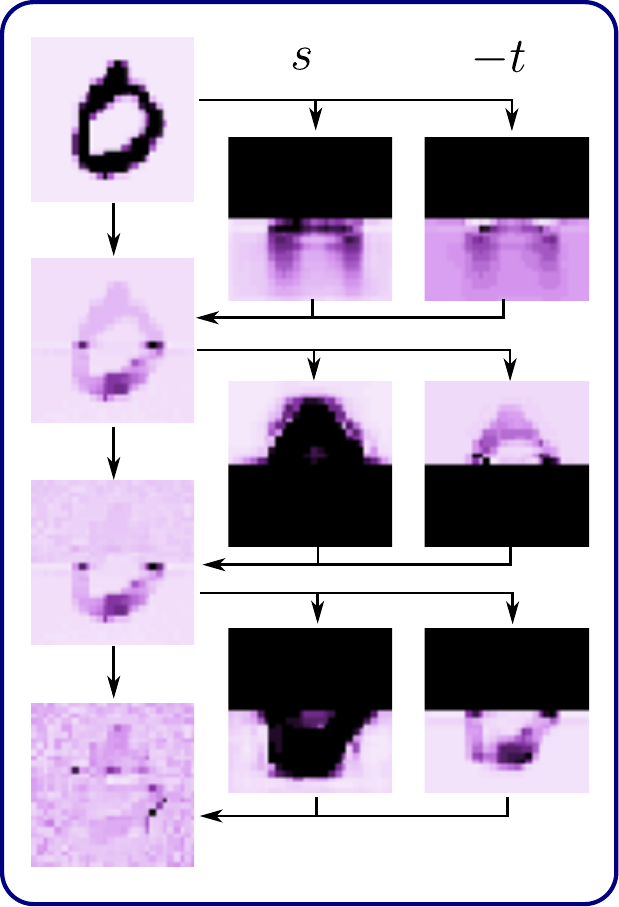} 
    }~~~
    \subfigure{
    	\includegraphics[width=\panelwidth]{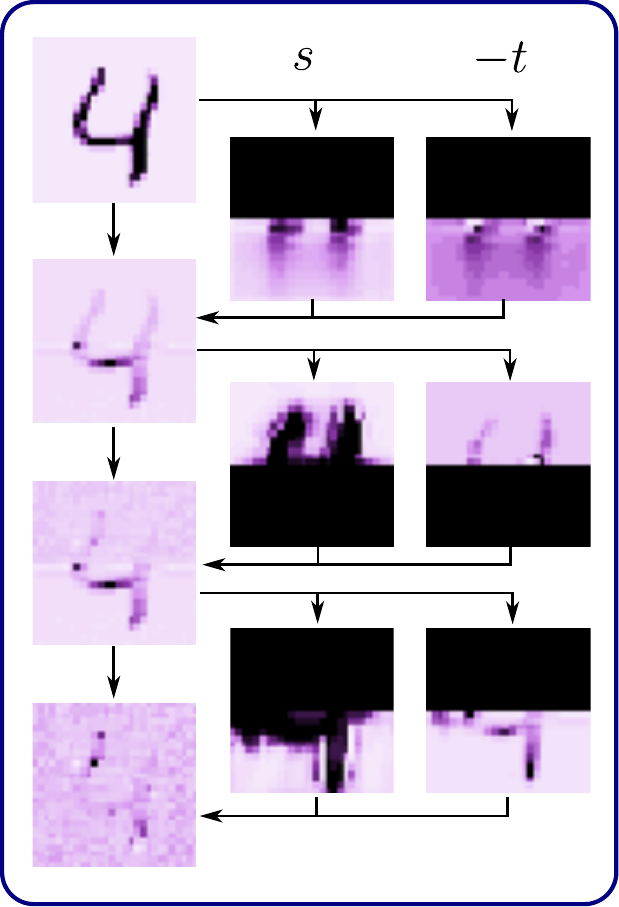} 
    }

	\caption{ 
	\textbf{Coupling layer co-adaptation.}
	Visualization of intermediate coupling layer activations, as well as scales $s$ and shifts $t$ 
	predicted by each coupling layer of a RealNVP model with a horizontal mask on 
    out-of-distribution MNIST inputs. 
    Although RealNVP was trained on FashionMNIST, the $st$-networks are able to correctly
	predict the bottom half of MNIST digits in the second coupling layer due to coupling layer co-adaptation.
	}
	\label{fig:coadaptation}
    \vspace{-5mm}
\end{figure}

\section{Changing biases in flow models for better OOD detection} \label{sec:app_biases}

\subsection{Cycle-mask}
\label{sec:app_cyclemask}

In Section \ref{sec:coupling_layers} we identified two mechanisms through which normalizing flows learn to predict masked pixels from observed pixels on OOD data: leveraging local color correlations and coupling layer co-adaptation.
We reduce the applicability of these mechanisms with \textit{cycle-mask}:
a new masking strategy for the coupling layers illustrated in Figure \ref{fig:cyclemask}.

With cycle-mask, the coupling layers do not have access to neighbouring pixels 
when predicting the masked pixels, similarly to the horizontal mask. 
Furthermore, cycle mask reduce the effect of coupling layer co-adaptation: the information about a part of the image has to travel through $4$ coupling layers before it can be used to update the same part of the image.

\textbf{Changing masking strategy}\quad
In Figure \ref{fig:app_mask_effect} we show the log-likelihood histograms and samples for a RealNVP of a fixed size with checkerboard, horizontal and cycle-mask.

\begin{figure*}[h]
    \vspace{-1cm}
    \centering
    \includegraphics[width=0.8\textwidth]{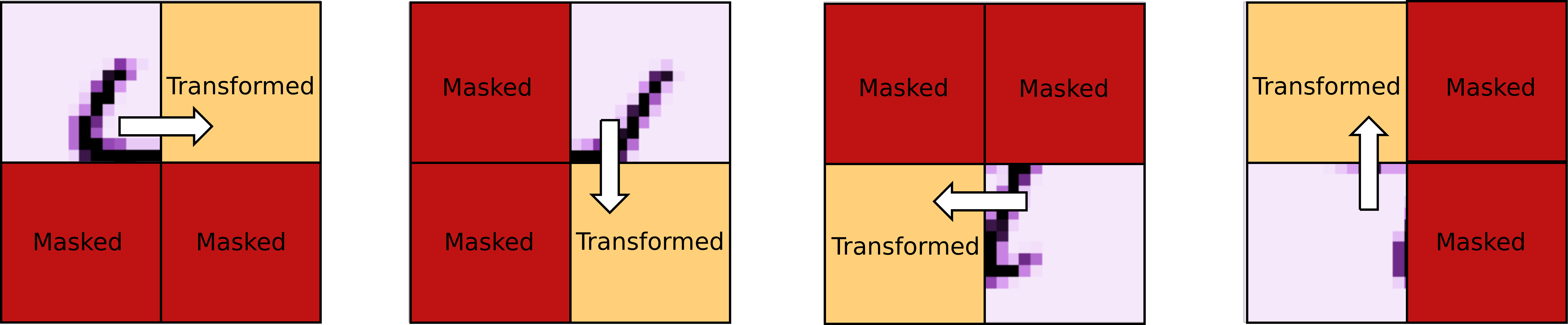}
    \vspace{-0.2cm}
	\caption{
    \textbf{Cycle-mask.}
	A new sequence of masks for coupling layers in RealNVP that we evaluate in Section \ref{sec:changing_biases}.
	We separate the input image space of size $c \times h \times w$ into four quadrants of size $c \times h / 2 \times w / 2$ each.
	Each coupling layer transforms one quadrant based on the previous quadrant.
	Cycle-mask prevents co-adaptation between subsequent coupling layers discussed in Section \ref{sec:coupling_layers}:
    the information from a quadrant has to propagate through four coupling layers before reaching the same quadrant.
	}
	\label{fig:cyclemask}
\end{figure*}

\begin{figure}[t]
    \vspace{-.7cm}
	\def \panelwidth {0.27\textwidth}
	\def \panelskip {-0.3cm}
    \centering

	\subfigure[~Checkerboard Mask]{
        \begin{tabular}{c}
		\rotatebox{90}{\quad\quad FashionMNIST}\quad
        \includegraphics[width=\panelwidth]{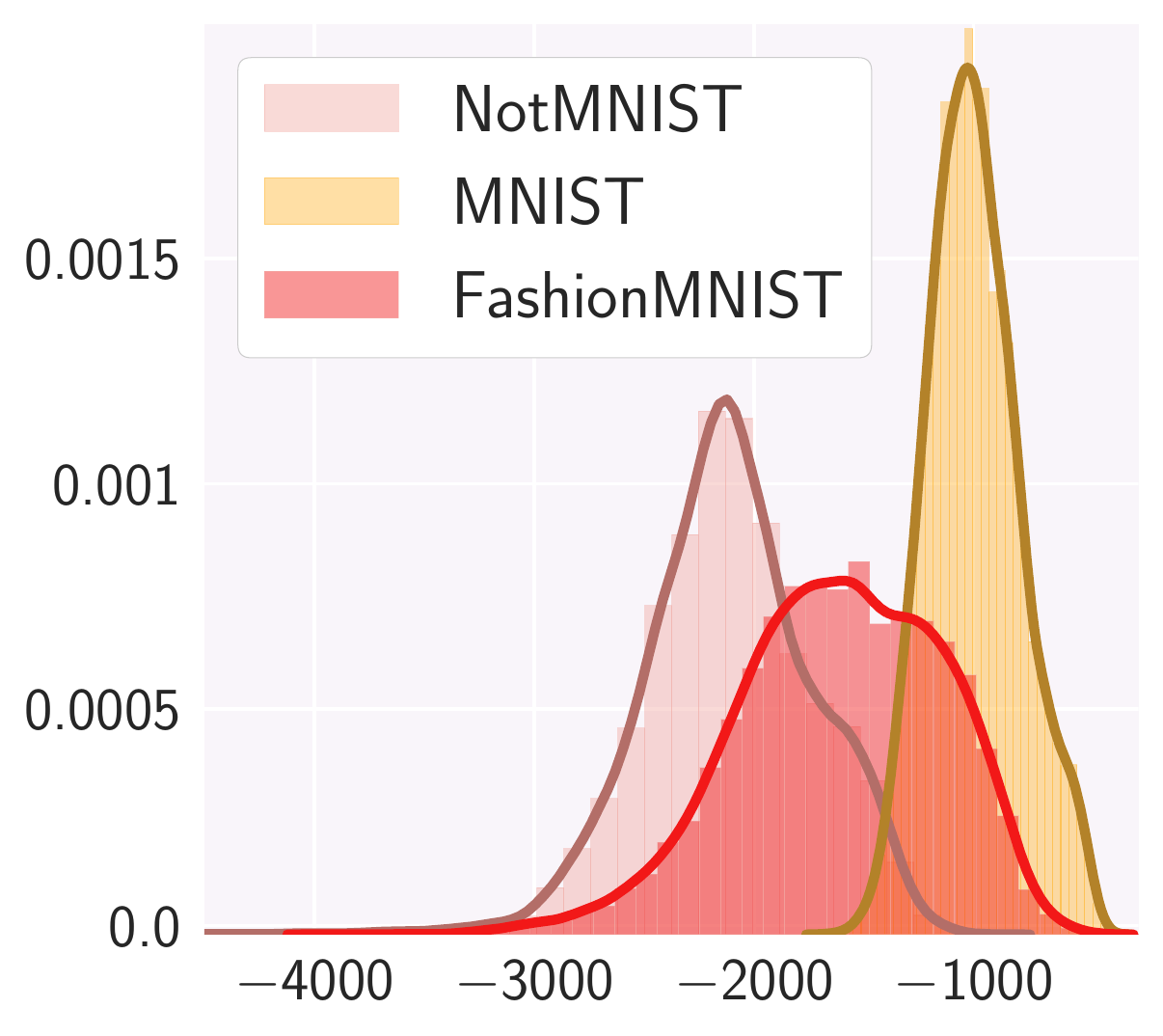}\\
		\rotatebox{90}{\quad\quad\quad CelebA}\quad
        \includegraphics[width=\panelwidth]{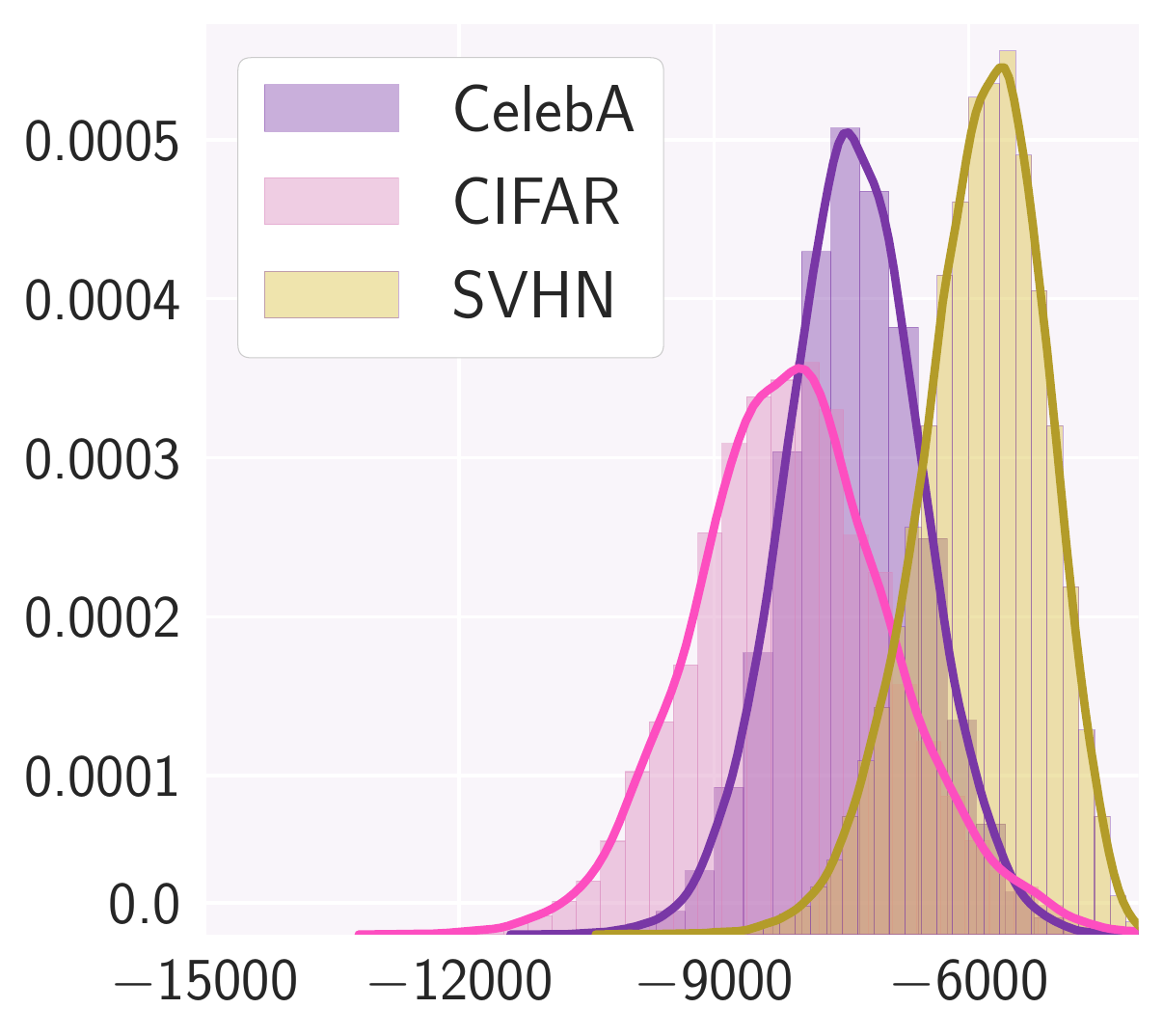}\\
        \end{tabular}
	}
    \hspace{\panelskip}
	\subfigure[Horizontal Mask]{
        \begin{tabular}{c}
		\includegraphics[width=\panelwidth]{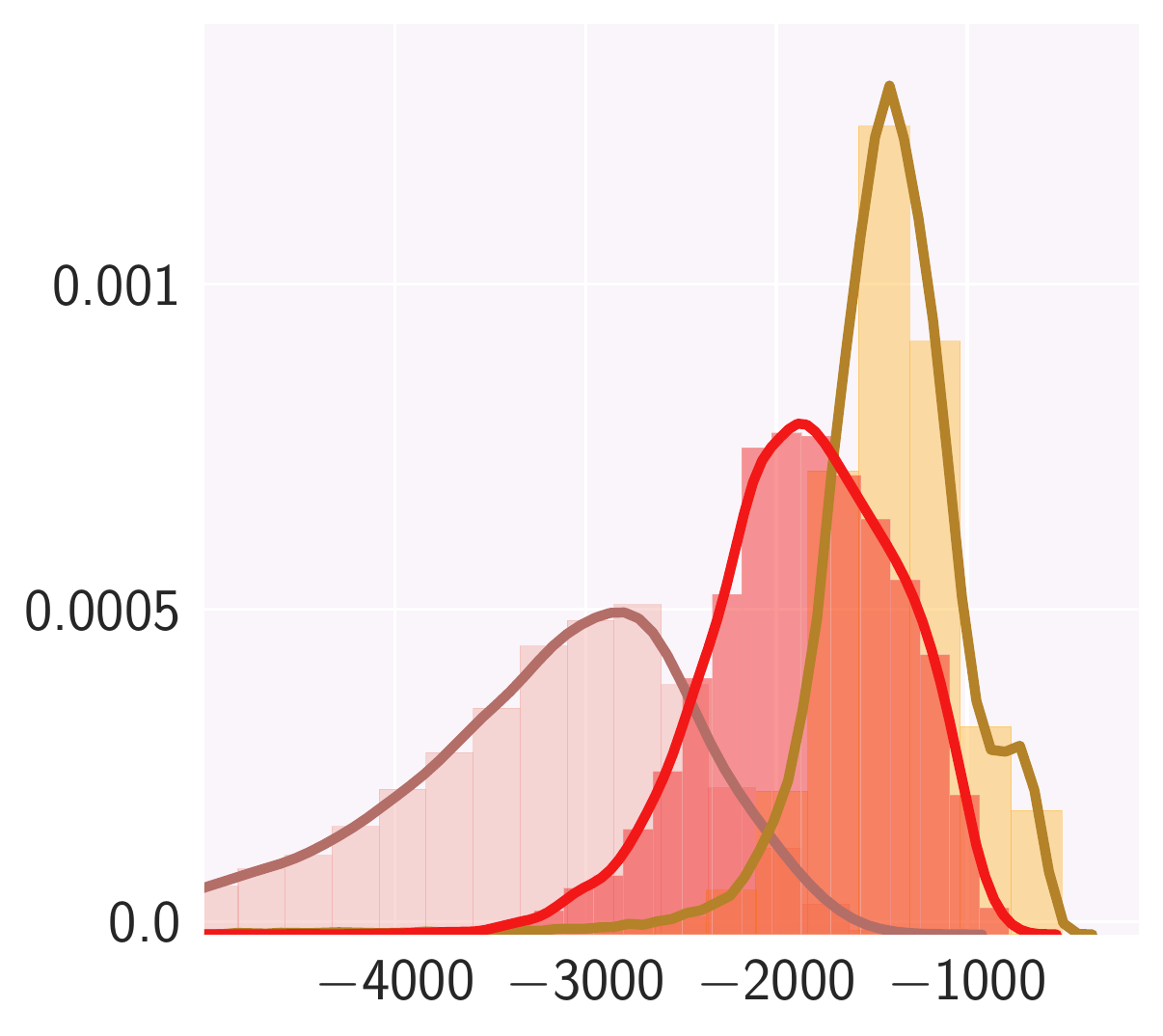}\\
		\includegraphics[width=\panelwidth]{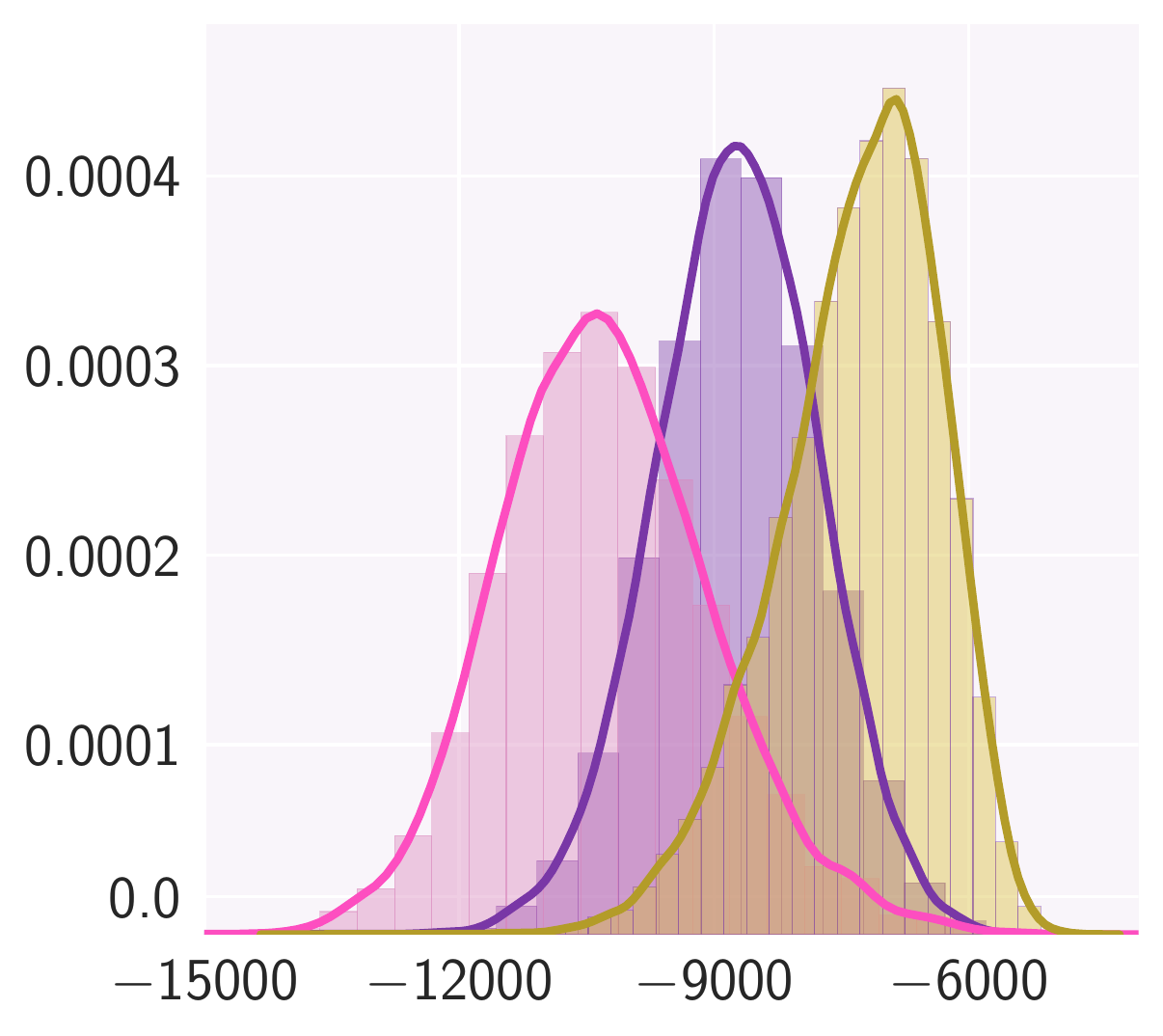}\\
        \end{tabular}
	}
    \hspace{\panelskip}
	\subfigure[Cycle-Mask]{
        \begin{tabular}{c}
		\includegraphics[width=\panelwidth]{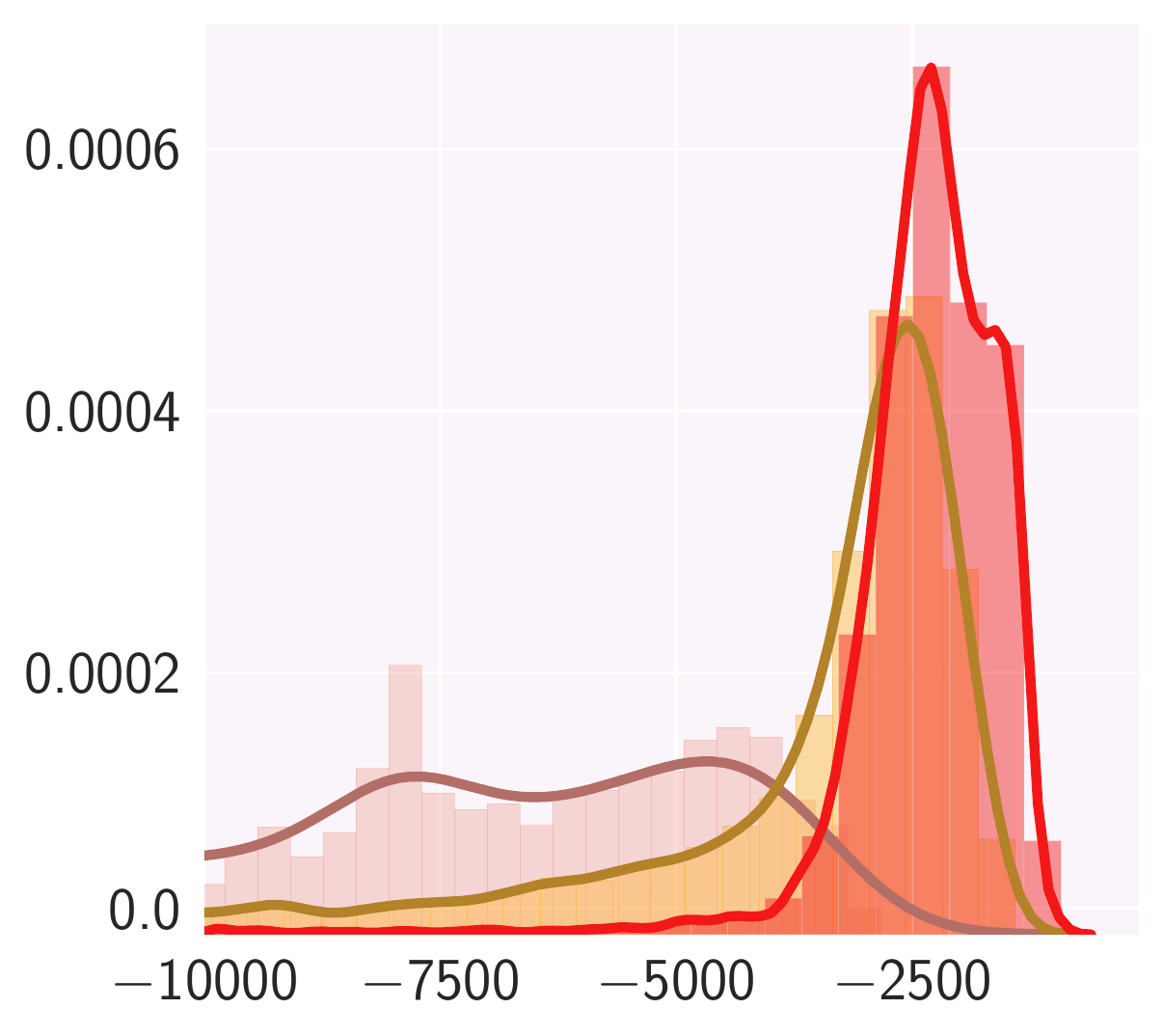}\\
		\includegraphics[width=\panelwidth]{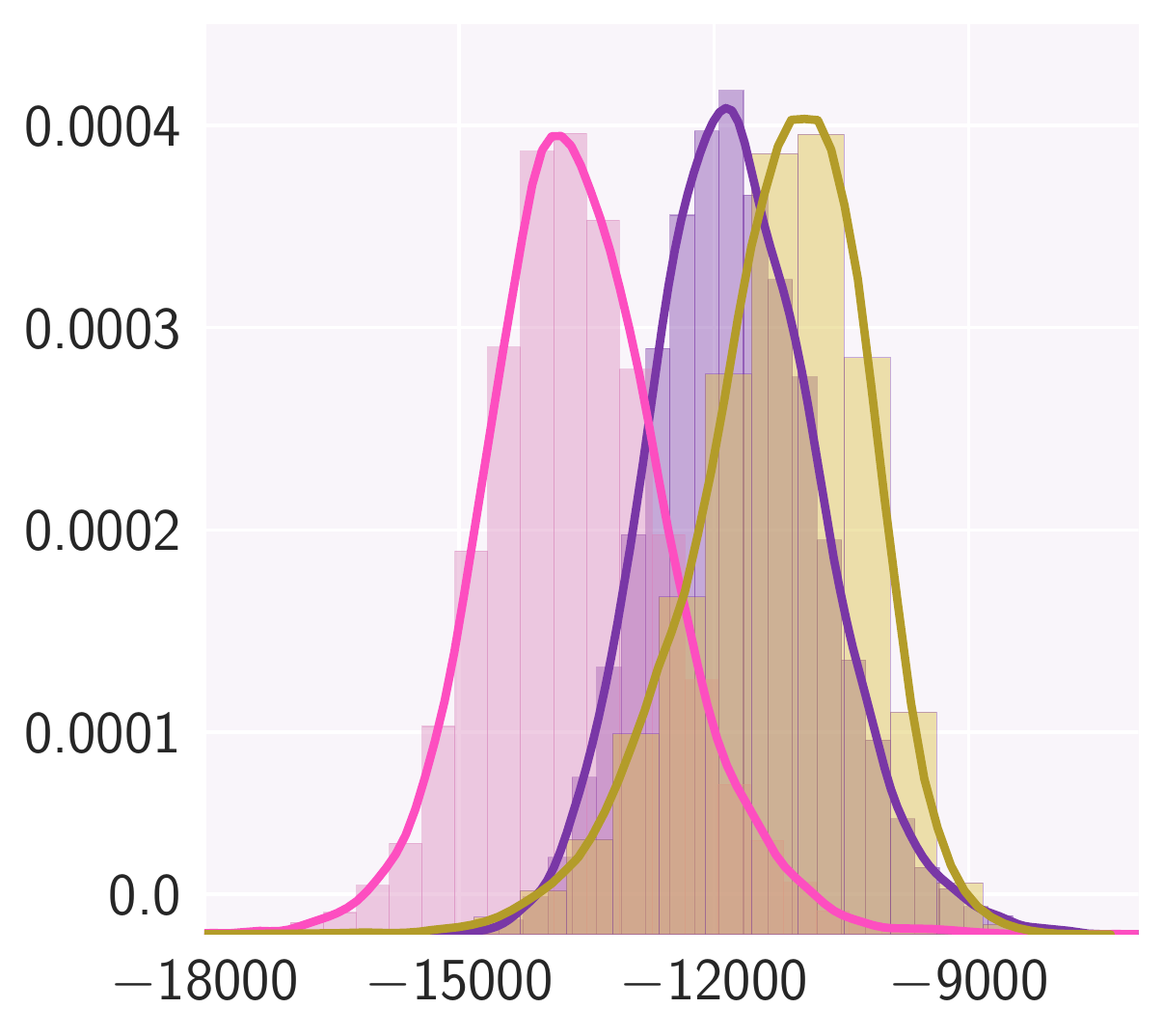}\\
        \end{tabular}
	}

    \subfigure[Checkerboard  Mask]{
        \begin{tabular}{c}
            \rotatebox{90}{\quad\quad\quad CelebA}~~
	        \includegraphics[width=\panelwidth]{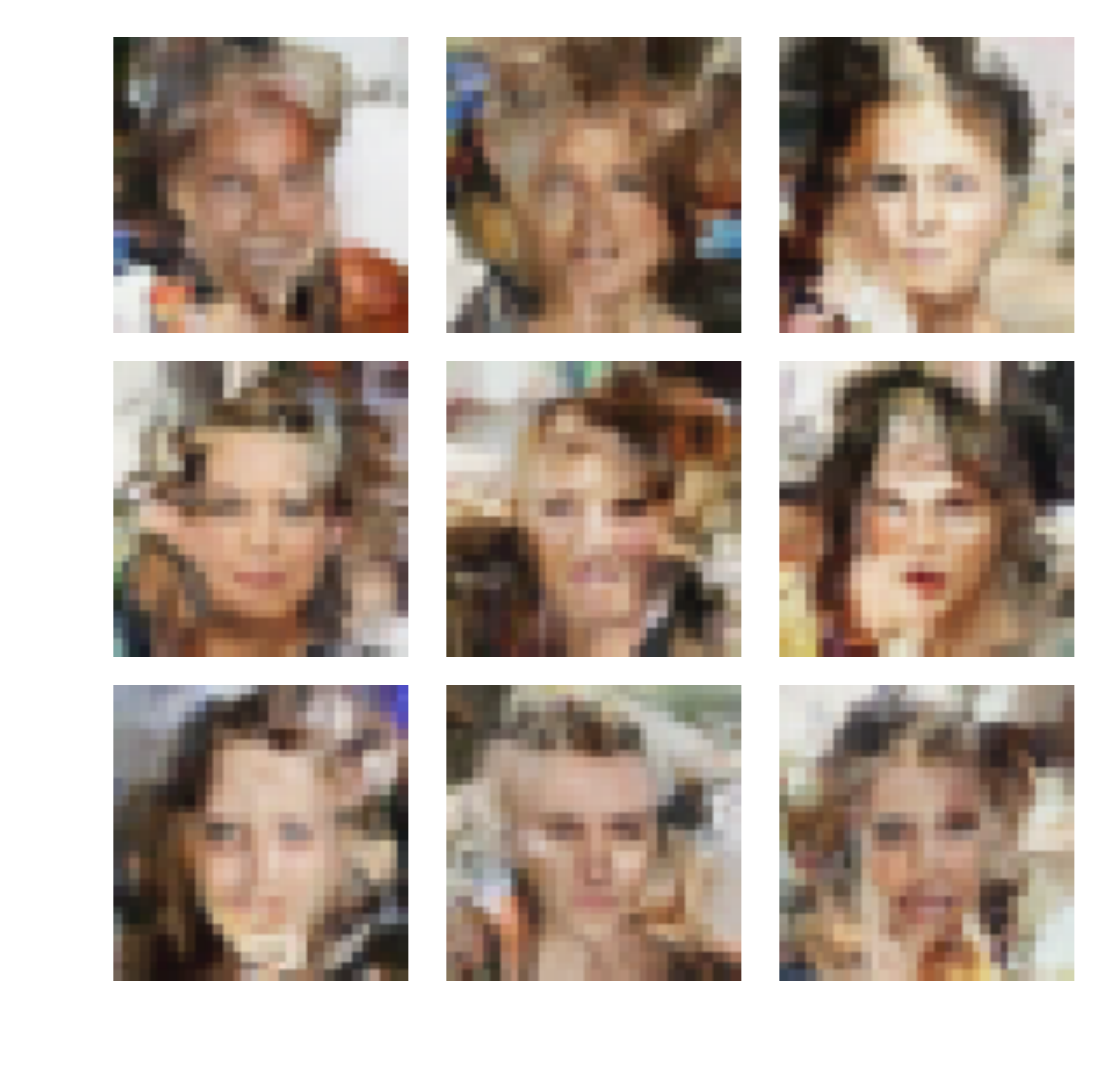} \\
            \rotatebox{90}{\quad~~ FashionMNIST}~~
	        \includegraphics[width=\panelwidth]{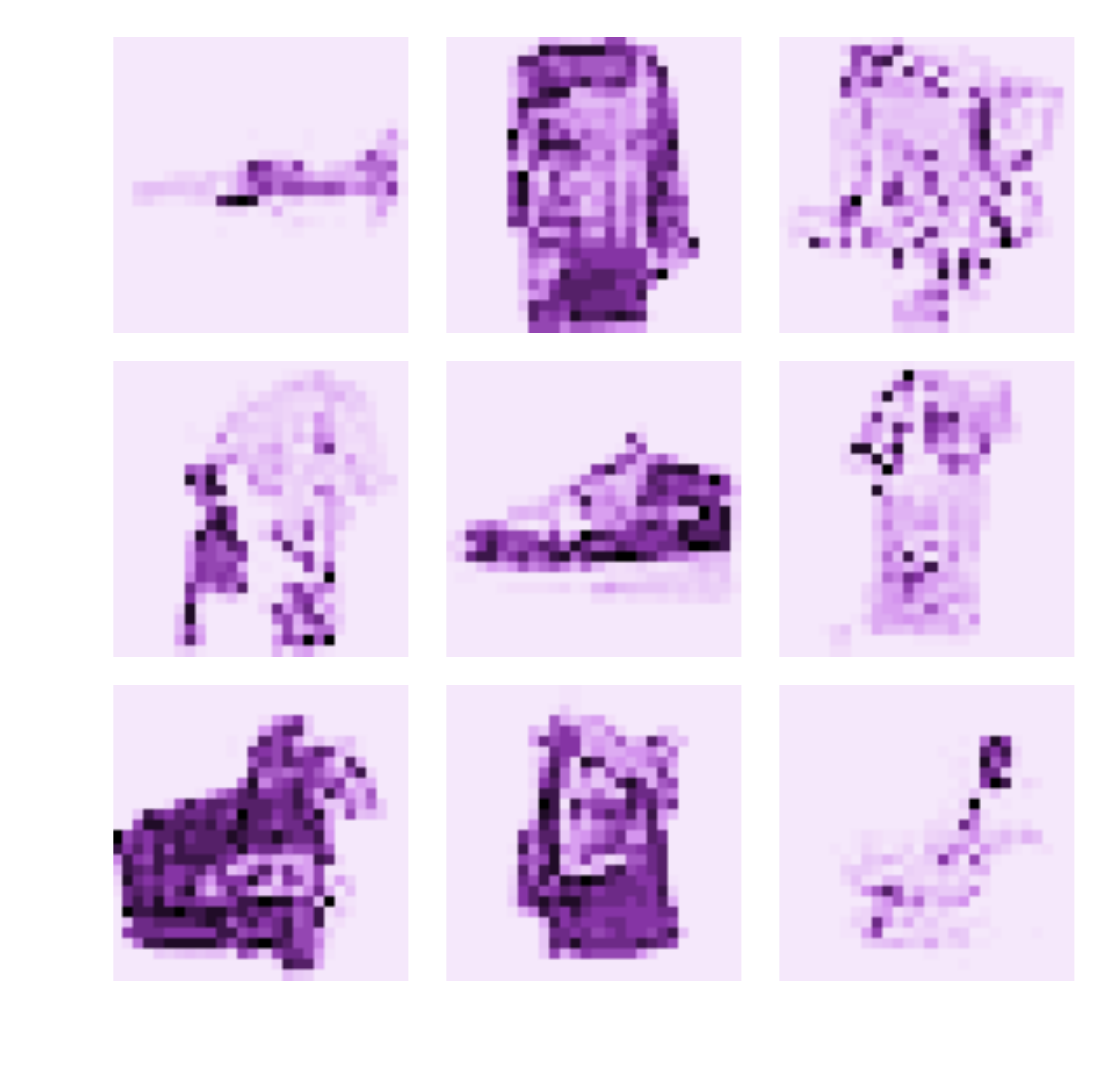}
        \end{tabular}
	}
    \hspace{\panelskip}
    \subfigure[Horizontal Mask]{
        \begin{tabular}{c}
	        \includegraphics[width=\panelwidth]{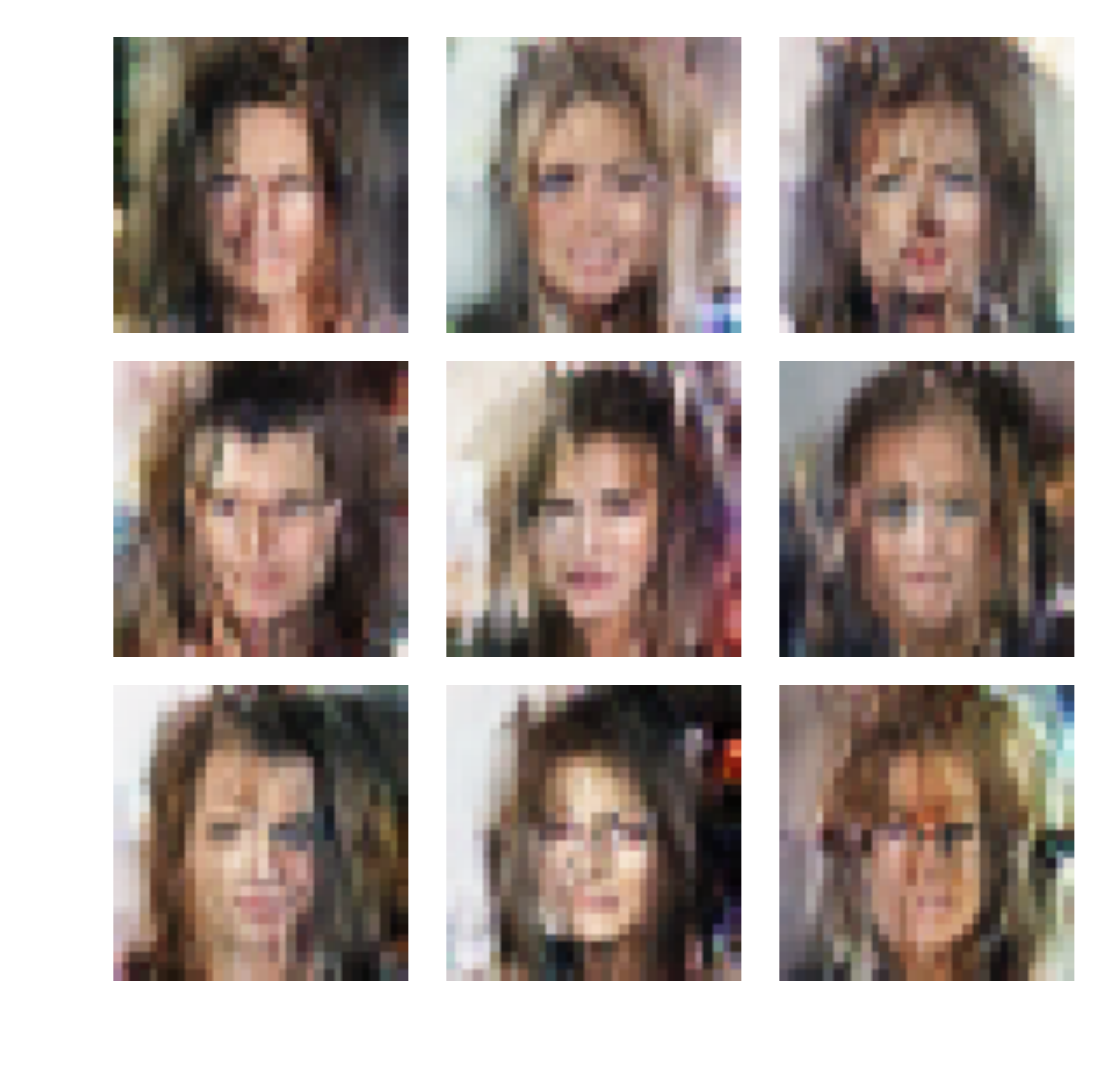} \\
	        \includegraphics[width=\panelwidth]{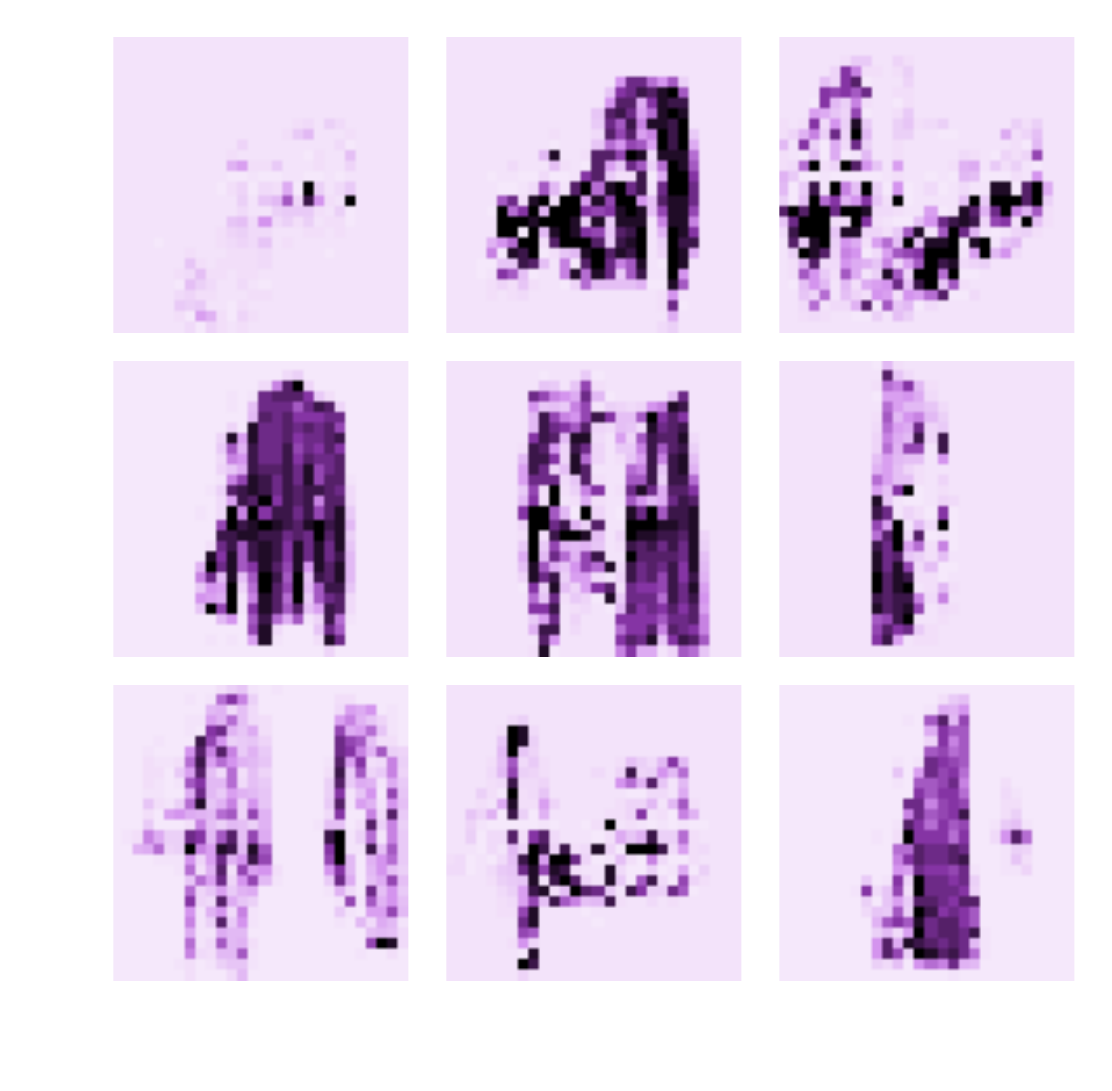}
        \end{tabular}
	}
    \hspace{\panelskip}
    \subfigure[Cycle-Mask]{
        \begin{tabular}{c}
	        \includegraphics[width=\panelwidth]{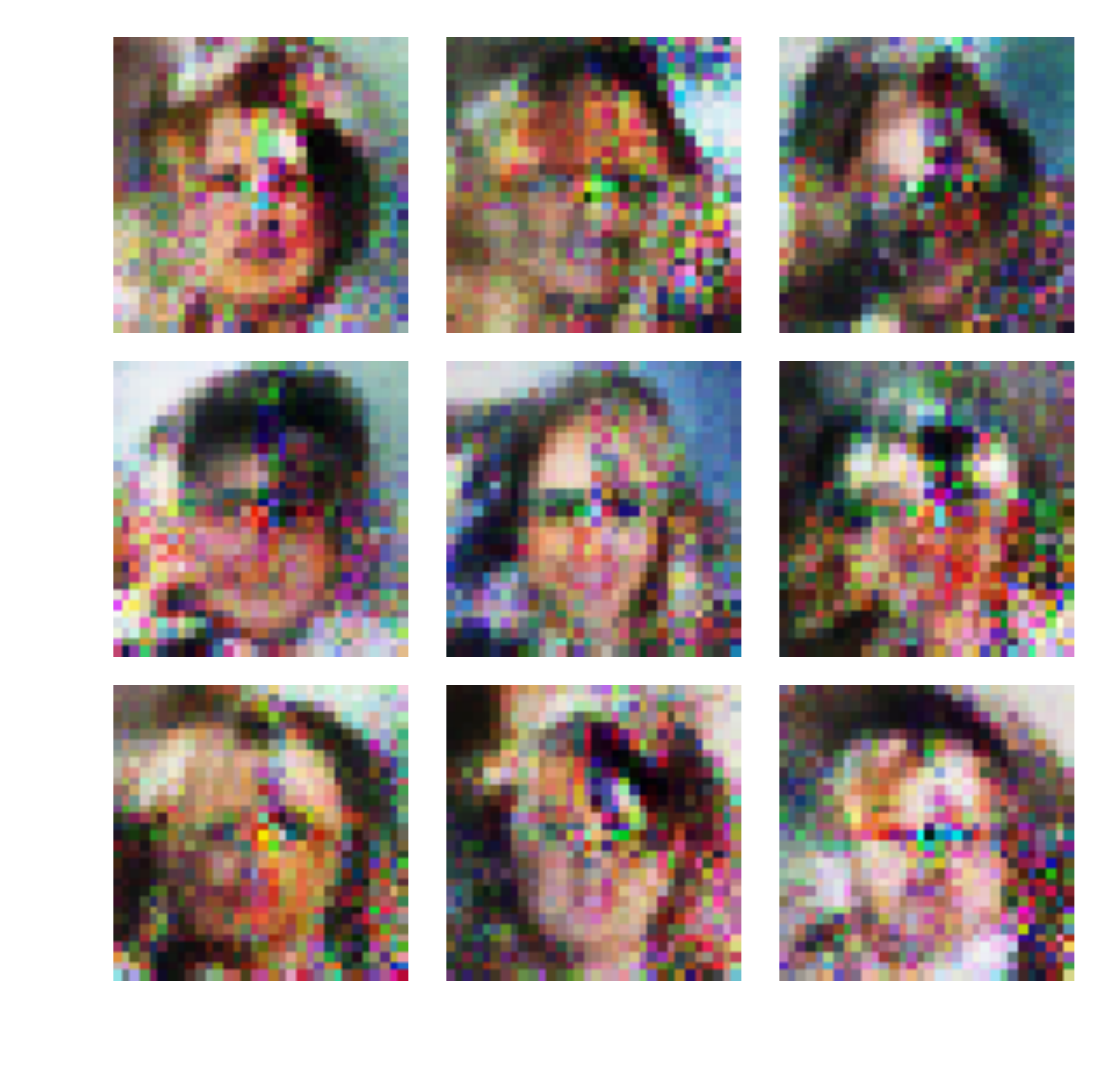} \\
	        \includegraphics[width=\panelwidth]{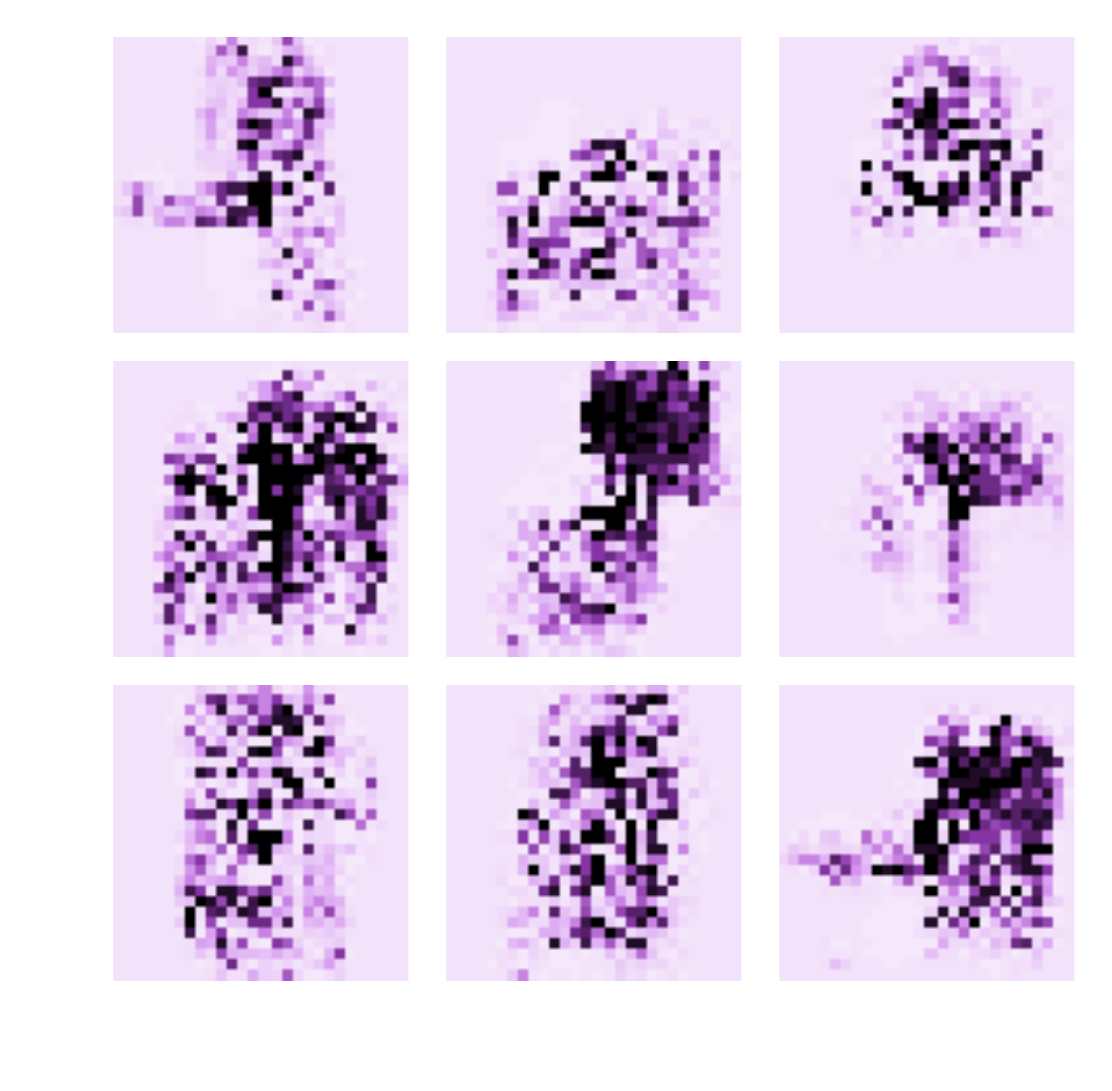}
        \end{tabular}
	}
    \vspace{-0.2cm}
	\caption{
        \textbf{Effect of masking strategy}
        The first two rows show log likelihood distribution for RealNVP models trained on FashionMNIST and CelebA with (a) checkerboard mask; (b) horizontal mask; and (c) cycle-mask. The third and the fourth rows show samples produced by the corresponding models.
	}
	\label{fig:app_mask_effect}
\end{figure}

\begin{figure}[h!]
	\def \panelwidth {0.22\textwidth}
	\def \panelskip {-0.3cm}
    \centering

    \vspace{-1cm}
	\subfigure[Baseline]{
        \rotatebox{90}{\quad\quad CelebA~~ LLs}~~
		\includegraphics[width=\panelwidth]{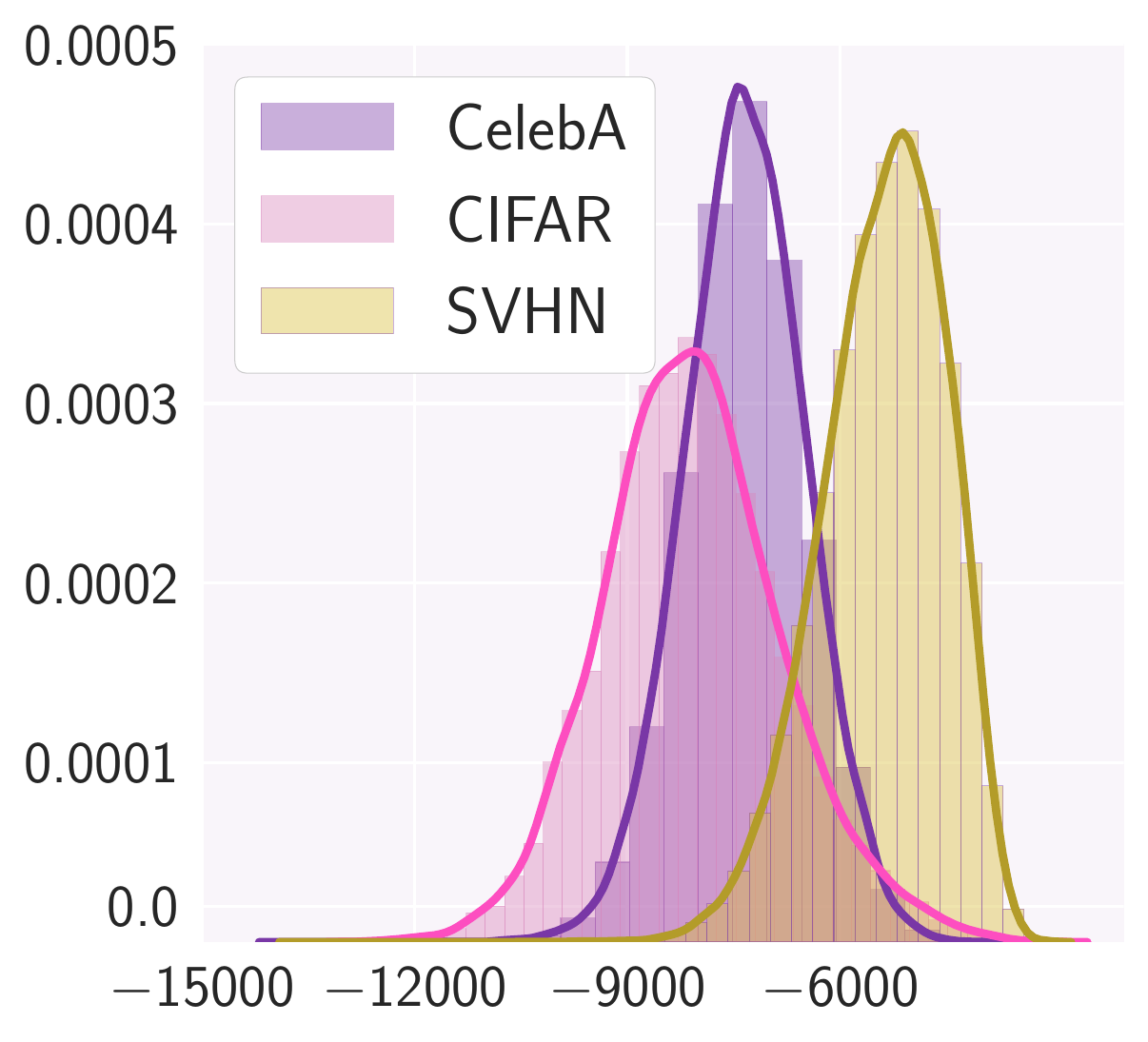}
	}
    \hspace{\panelskip}
	\subfigure[$l = 150$]{
		\includegraphics[width=\panelwidth]{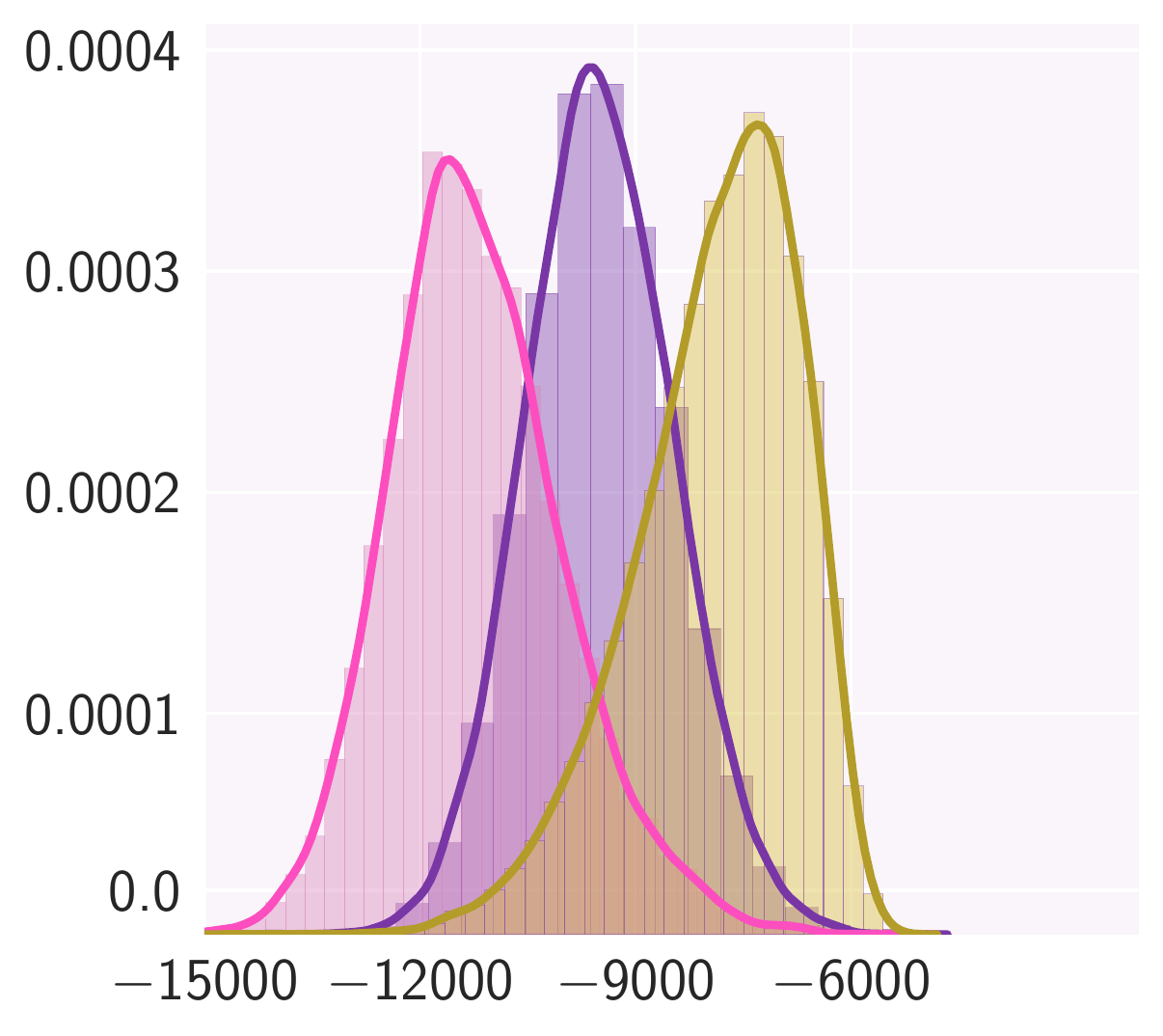}
	}
    \hspace{\panelskip}
	\subfigure[$l = 80$]{
		\includegraphics[width=\panelwidth]{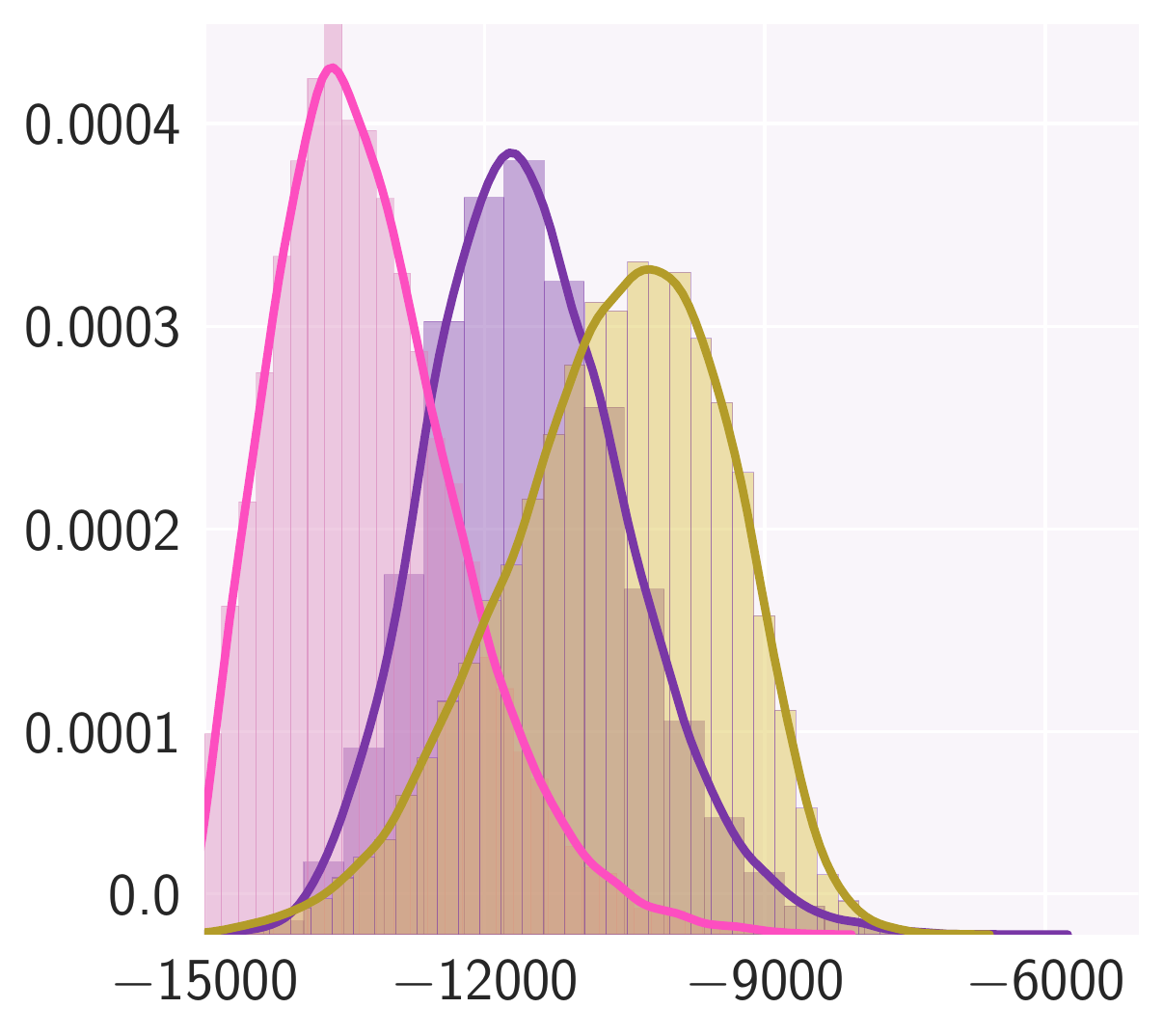}
	}
    \hspace{\panelskip}
	\subfigure[$l = 30$]{
		\includegraphics[width=\panelwidth]{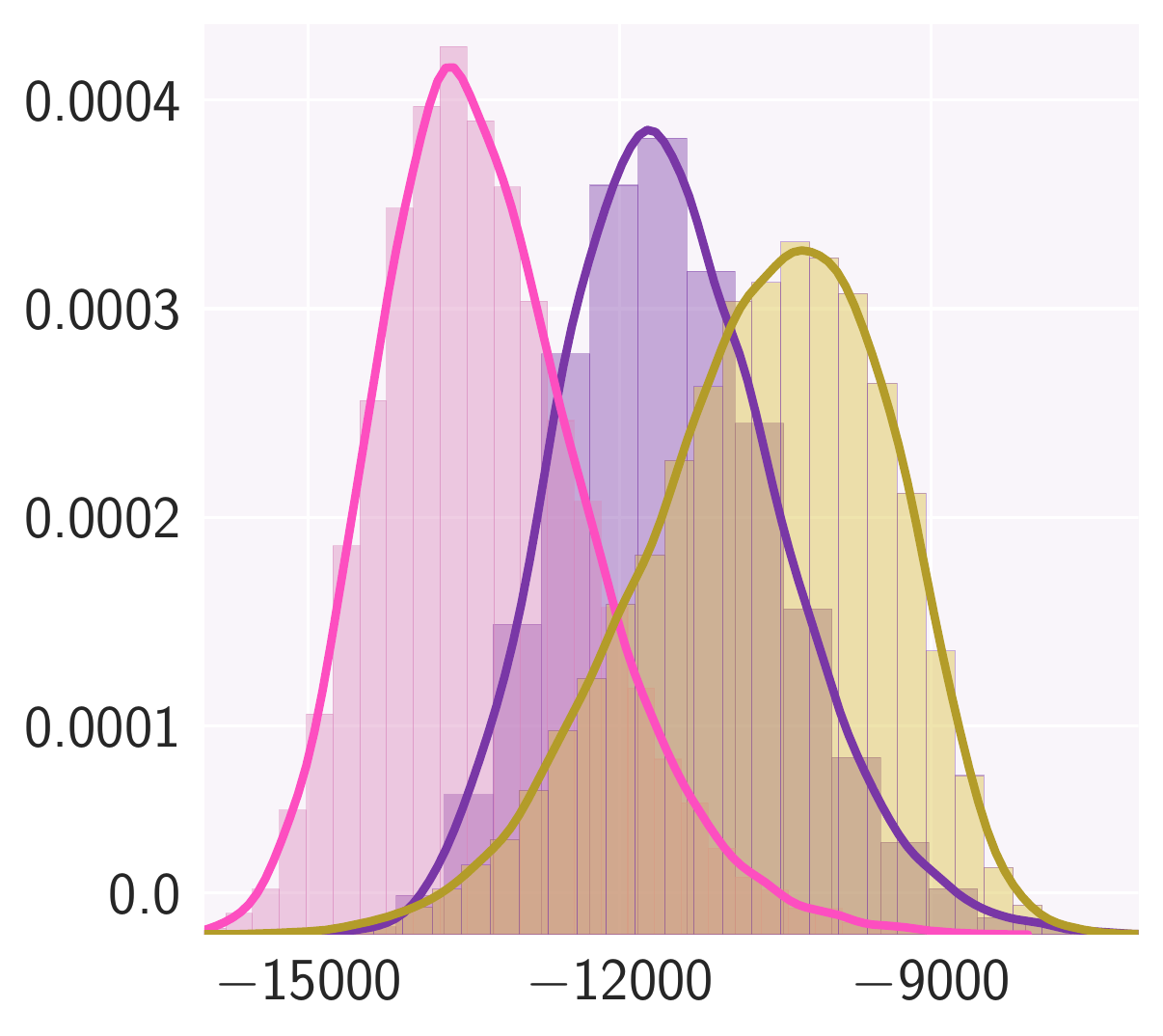}
	}

    \subfigure[Baseline]{
    \rotatebox{90}{\quad\quad\quad~ CelebA}~~
	\includegraphics[width=\panelwidth]{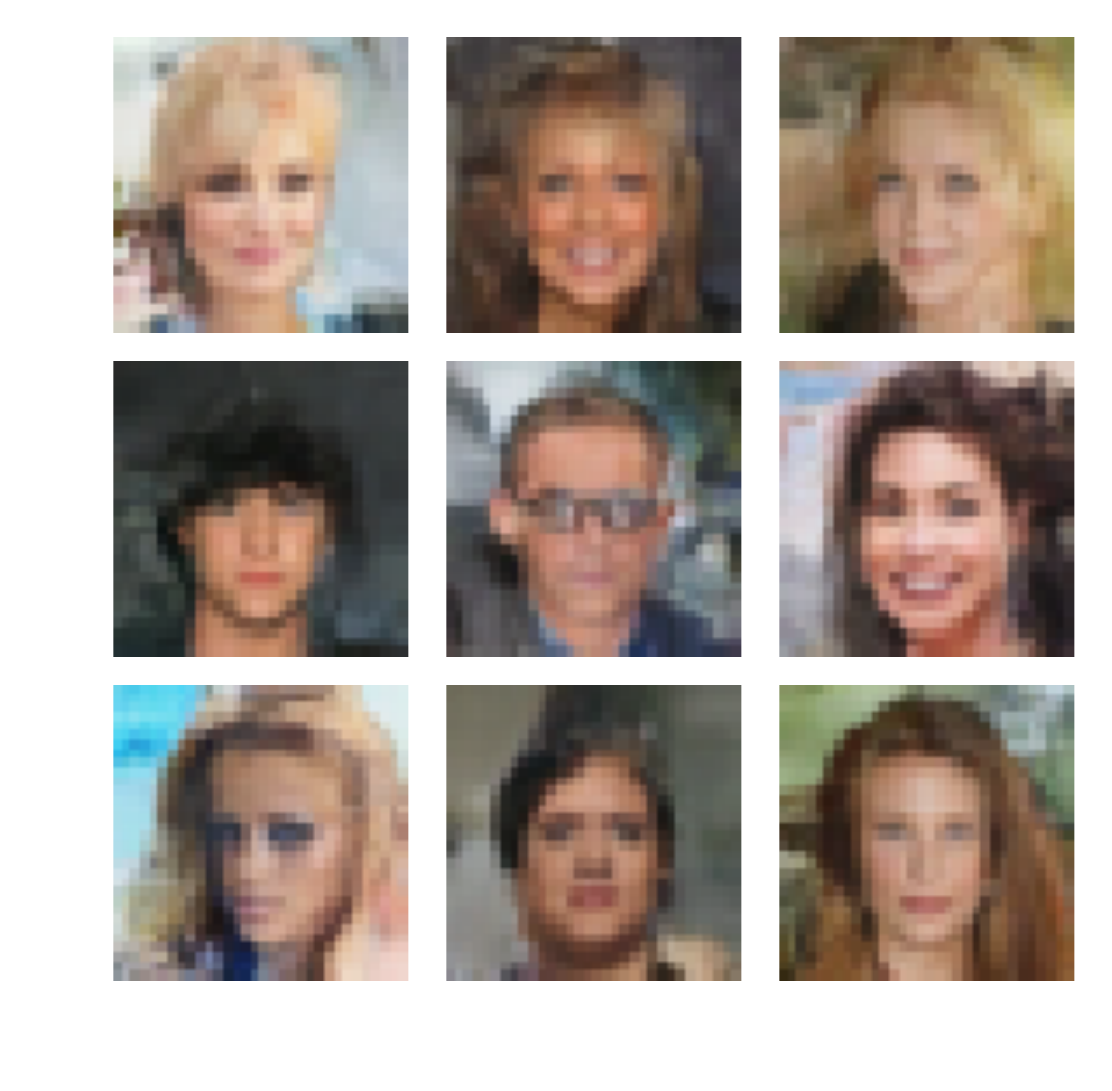}
	}
    \hspace{\panelskip}
	\subfigure[$l = 30$]{
	\includegraphics[width=\panelwidth]{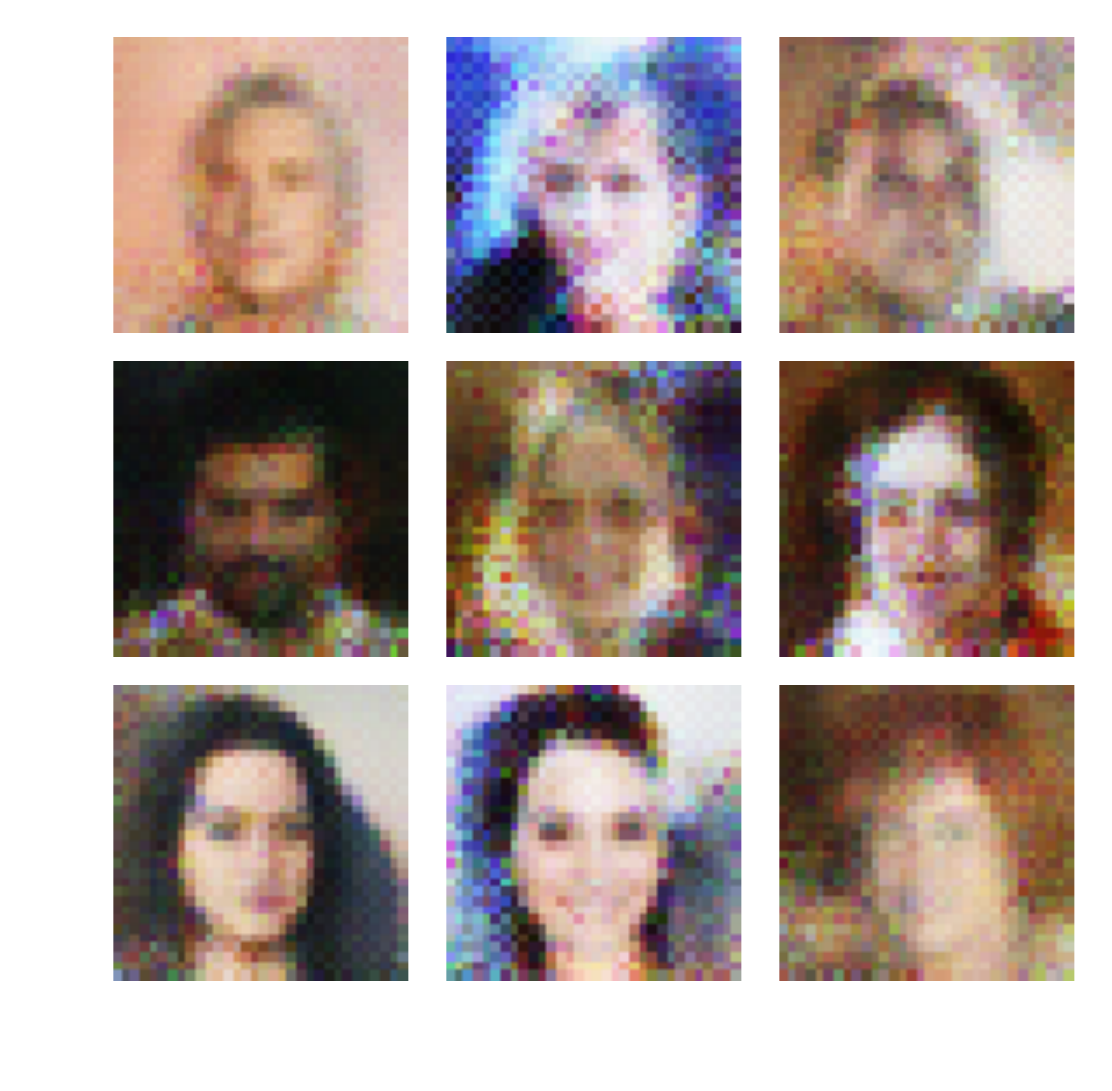}
	}
    \hspace{\panelskip}
	\subfigure[$l = 80$]{
	\includegraphics[width=\panelwidth]{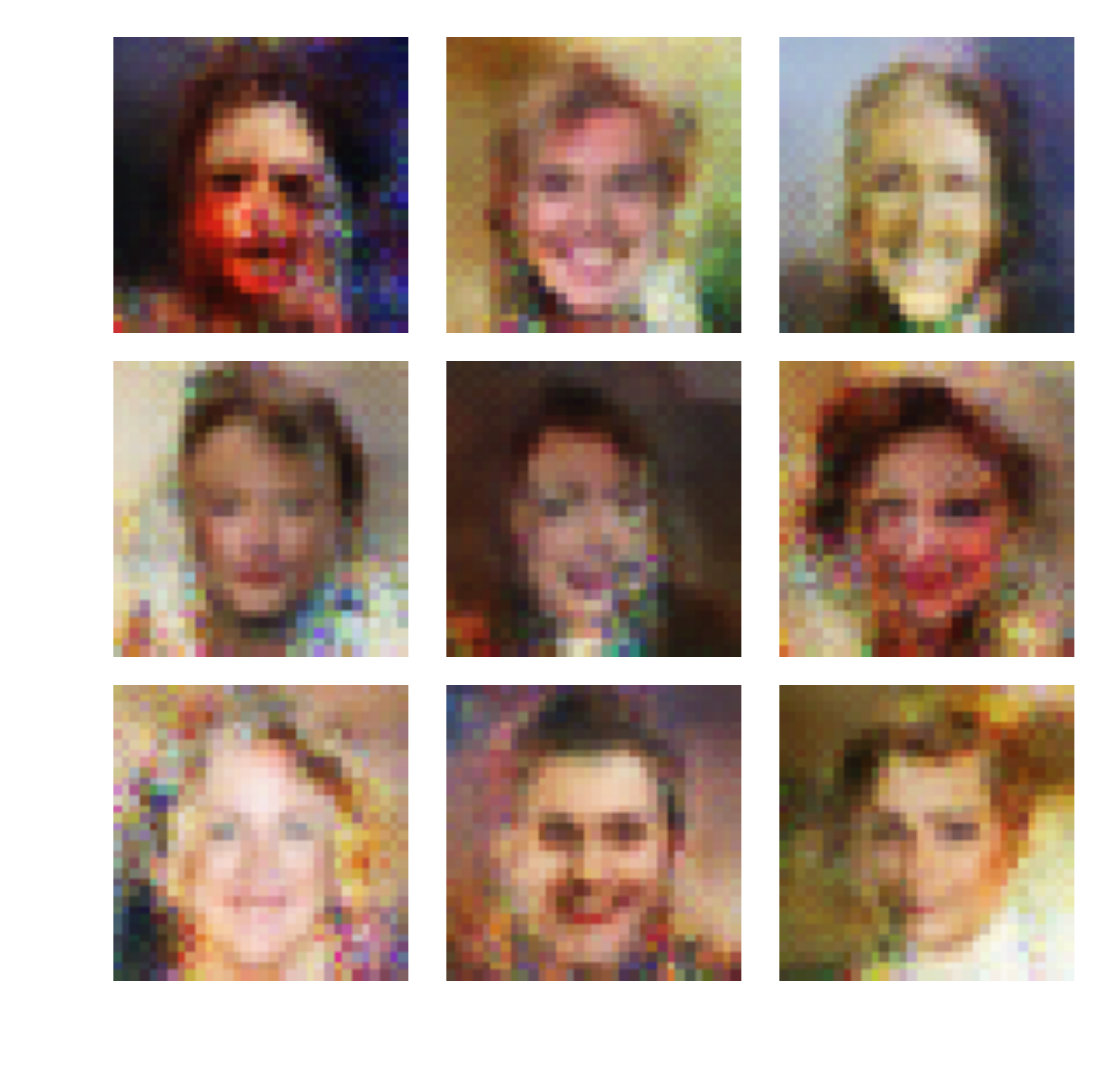}
	}
    \hspace{\panelskip}
	\subfigure[$l = 150$]{
	\includegraphics[width=\panelwidth]{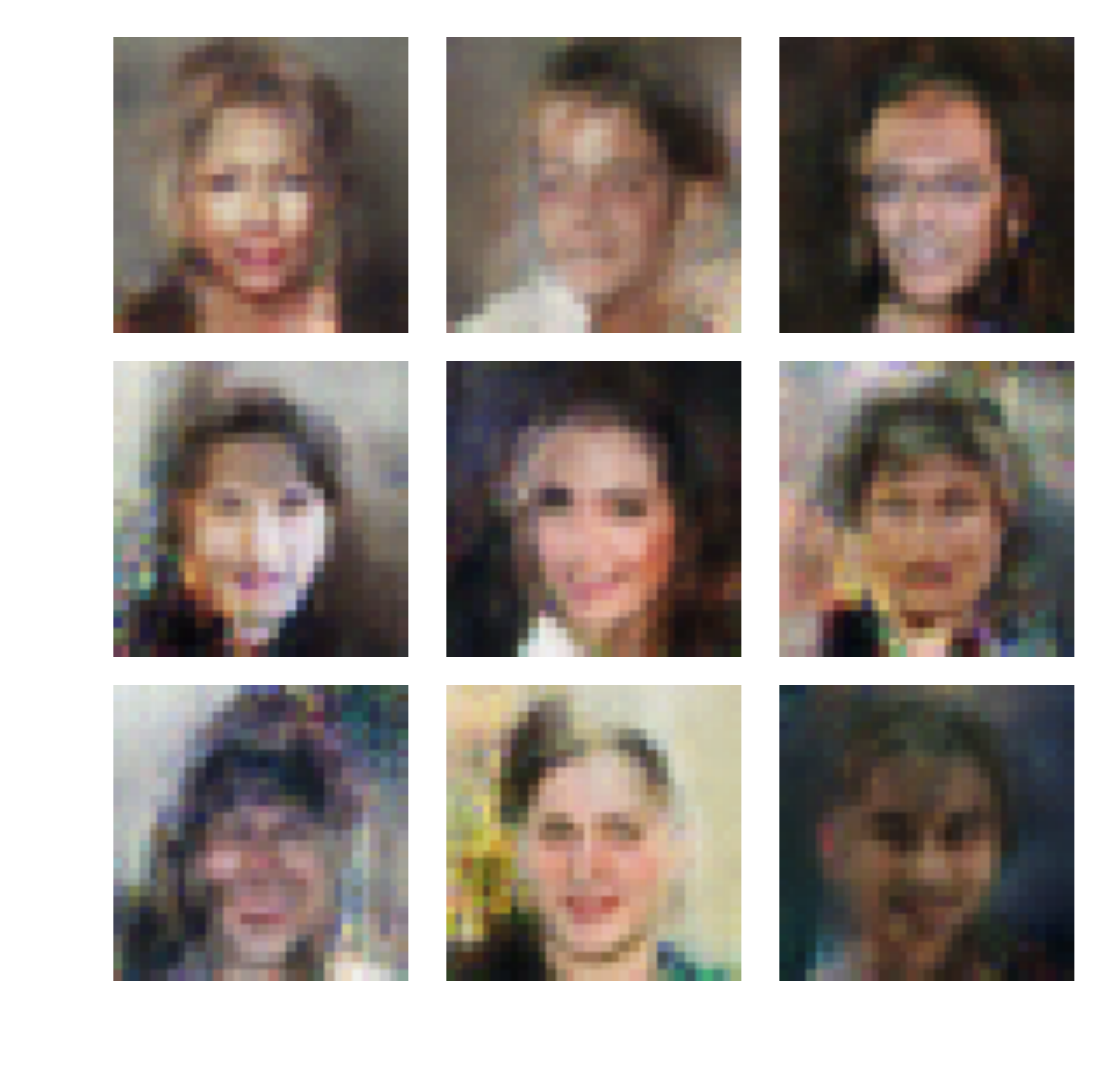}
	}

    \subfigure[Baseline]{
    \rotatebox{90}{\quad~~ FashionMNIST}~~
	\includegraphics[width=\panelwidth]{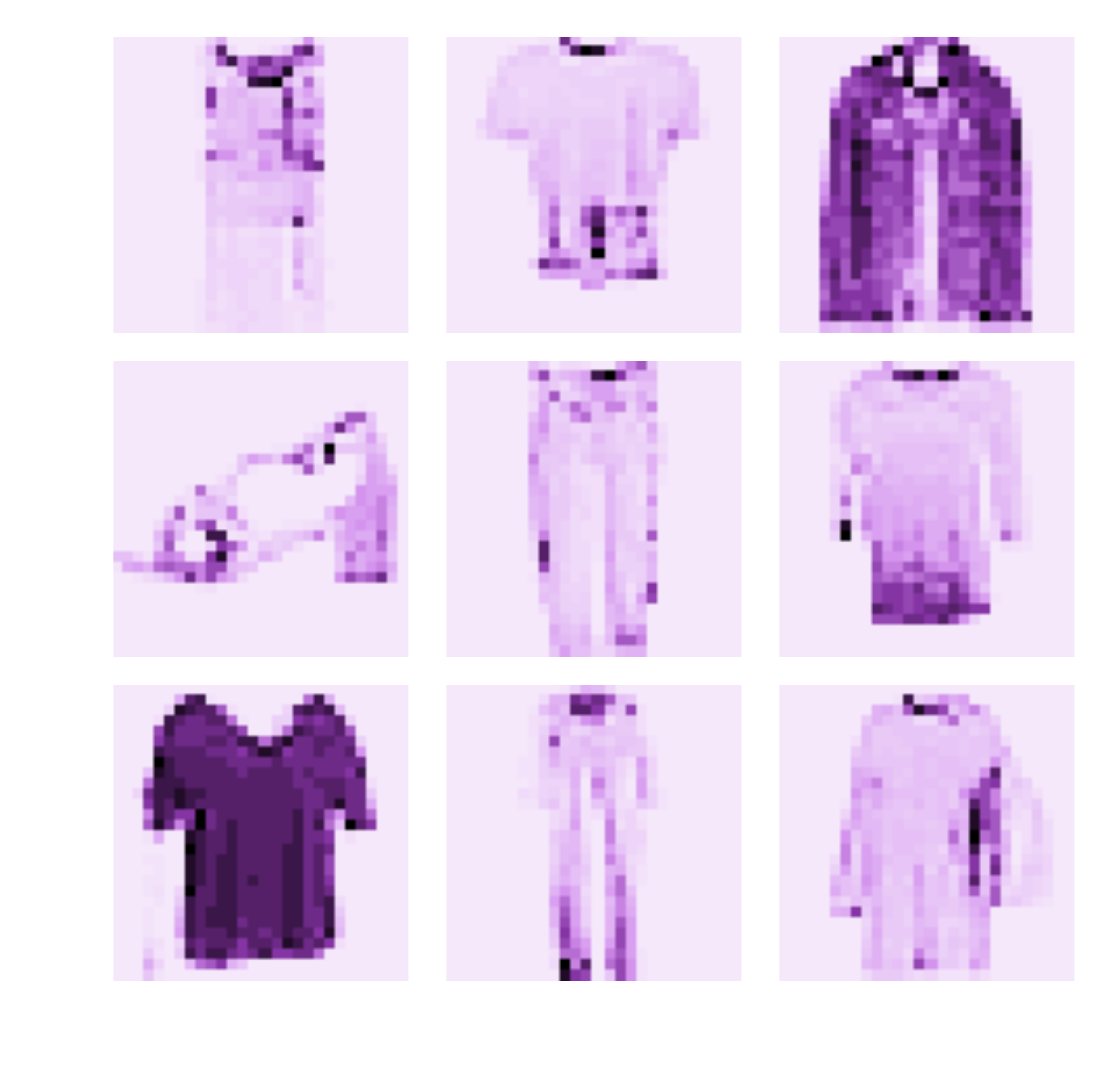}
	}
    \hspace{\panelskip}
	\subfigure[$l = 10$]{
	\includegraphics[width=\panelwidth]{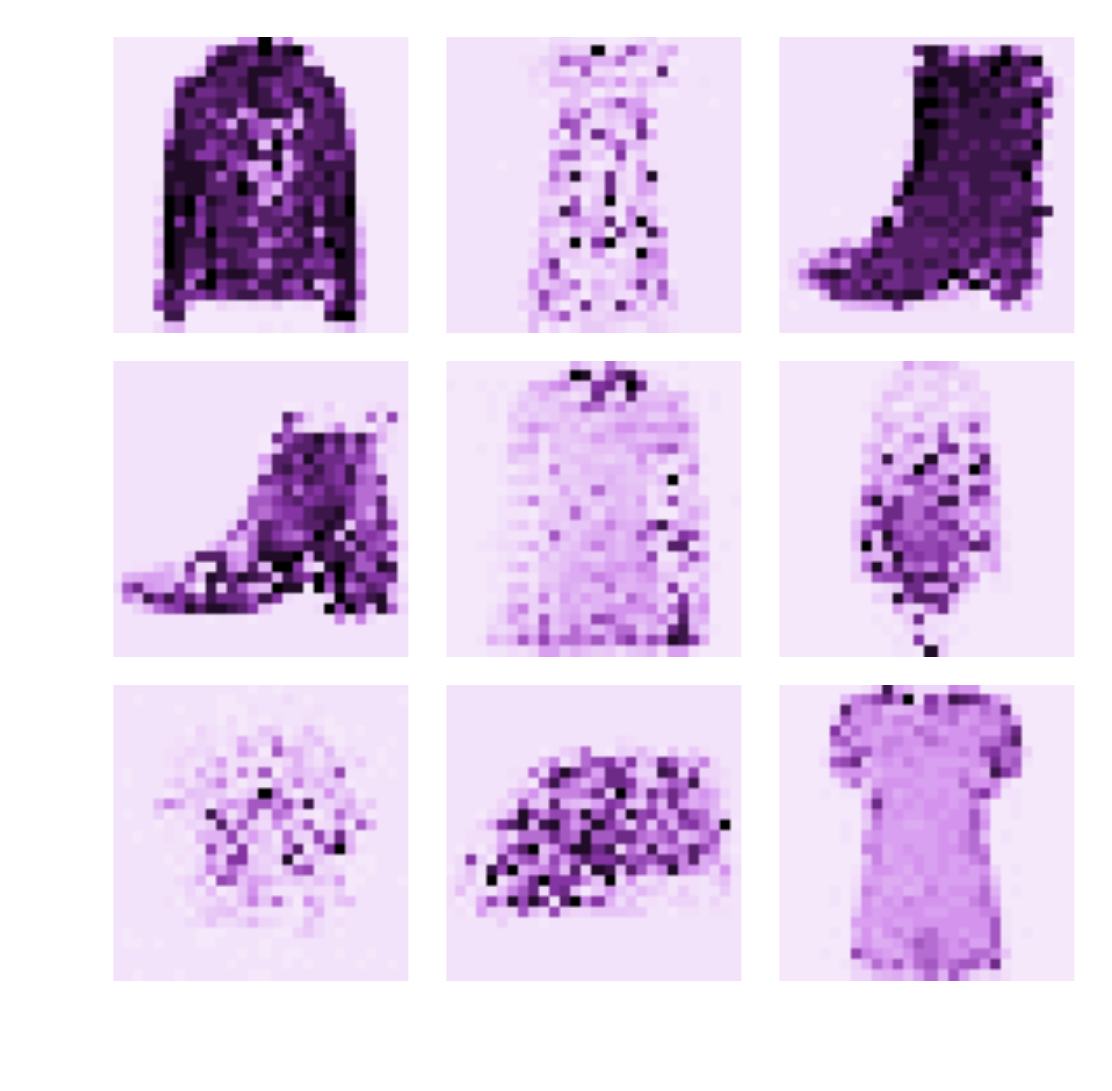}
	}
    \hspace{\panelskip}
	\subfigure[$l = 50$]{
	\includegraphics[width=\panelwidth]{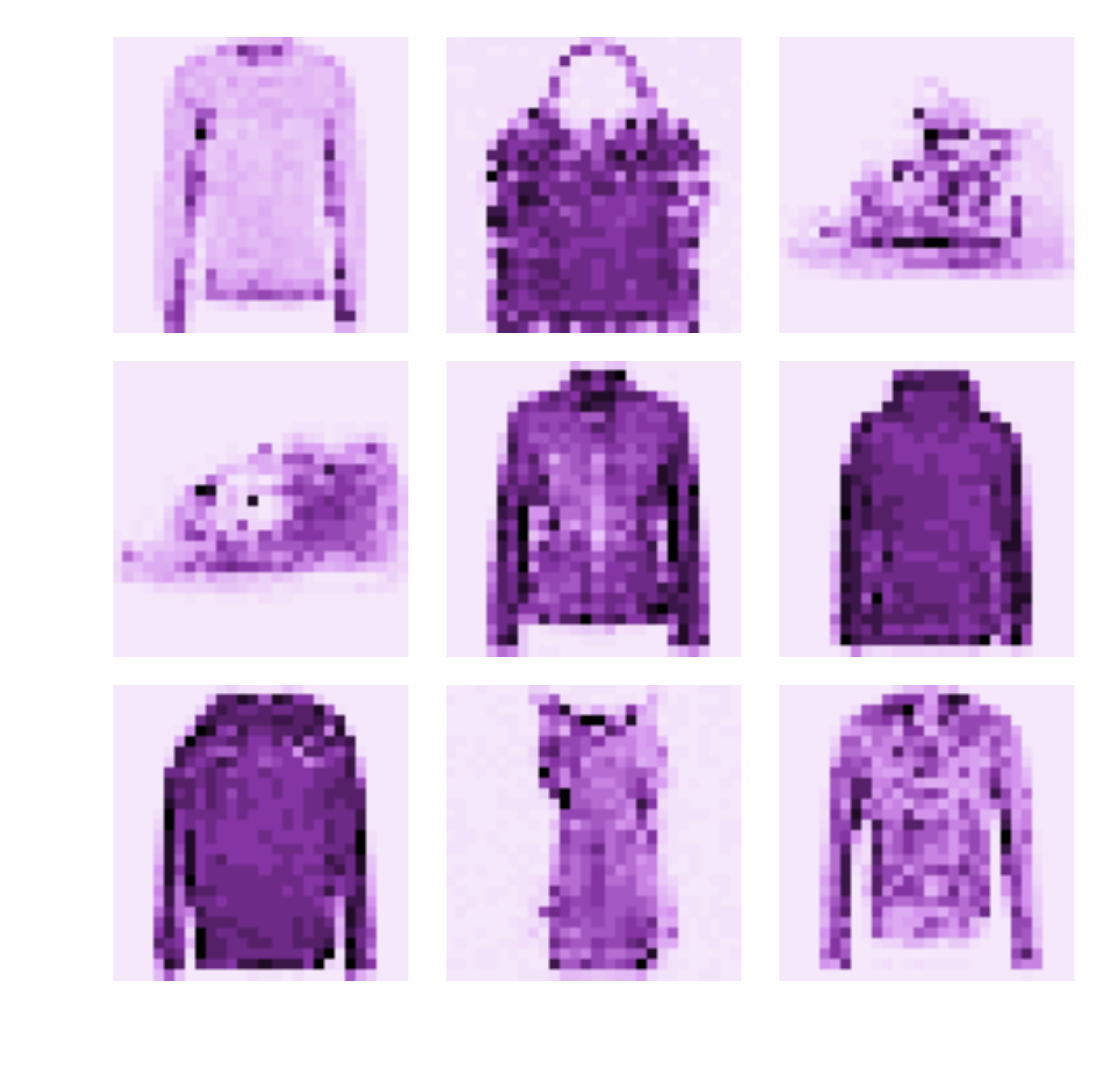}
	}
    \hspace{\panelskip}
	\subfigure[$l = 100$]{
	\includegraphics[width=\panelwidth]{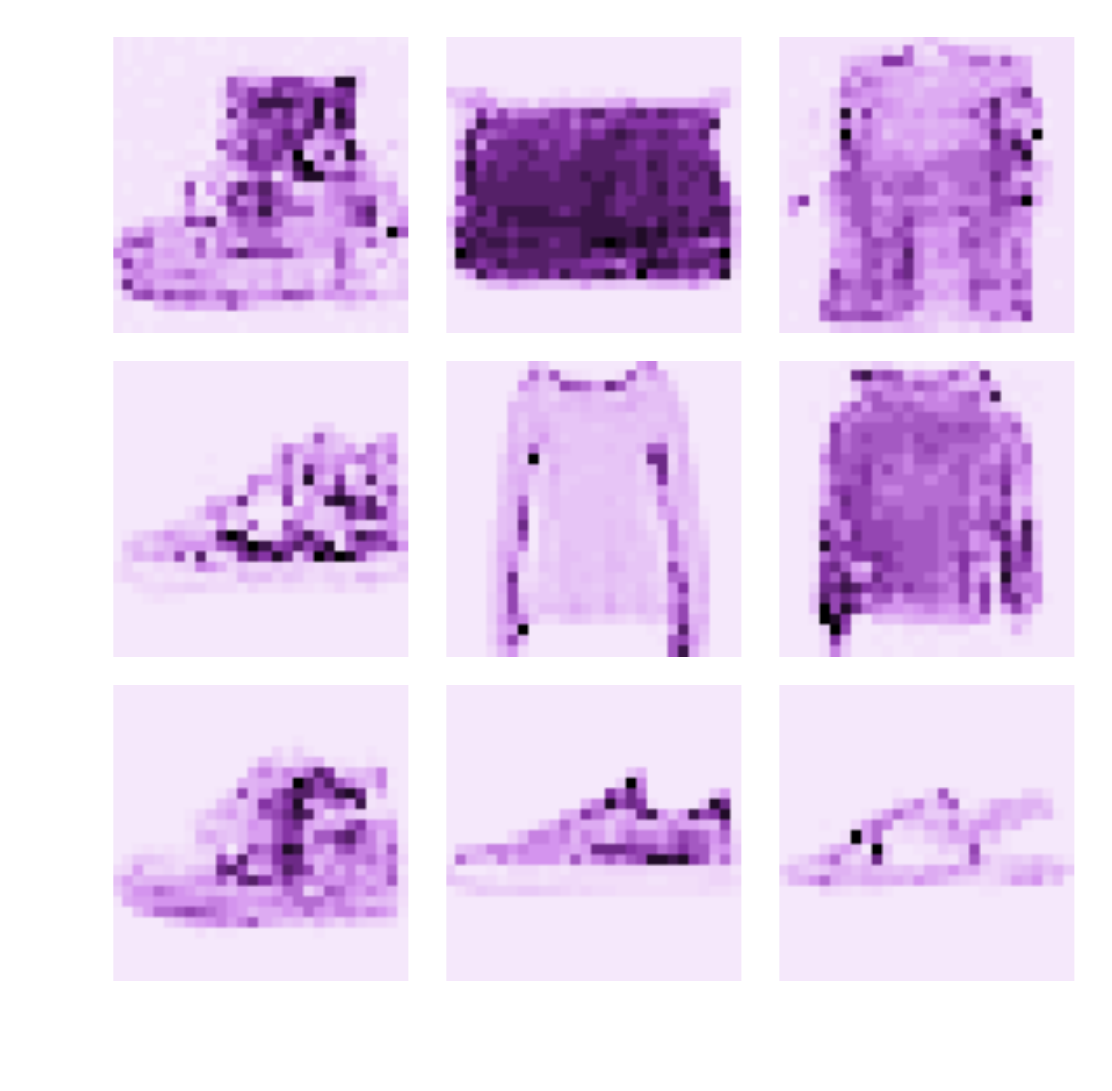}
	}

    \subfigure[RealNVP with $l=10$ trained on FashionMNIST]{
    \begin{tabular}{c}
	    \includegraphics[height=0.23\linewidth]{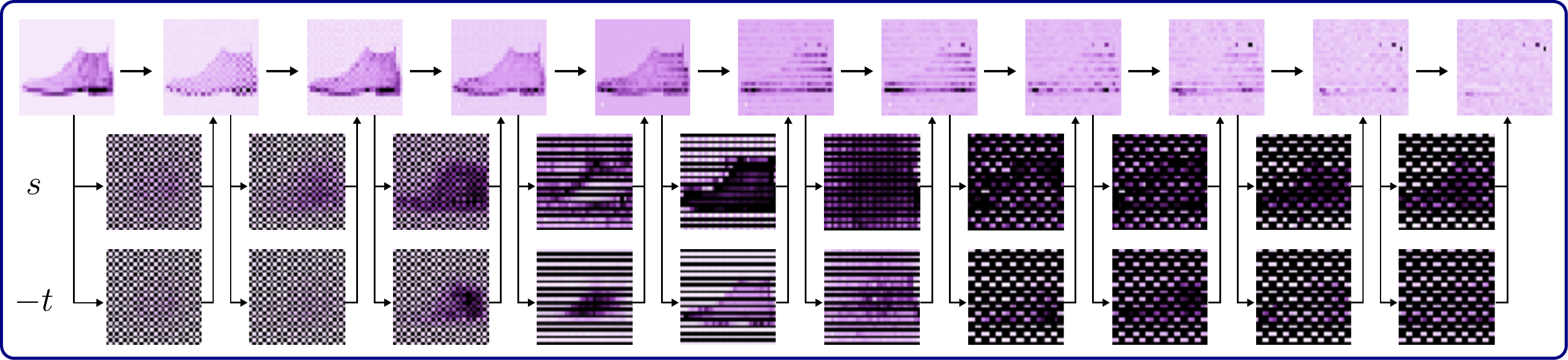}
        \\
	    \includegraphics[height=0.23\linewidth]{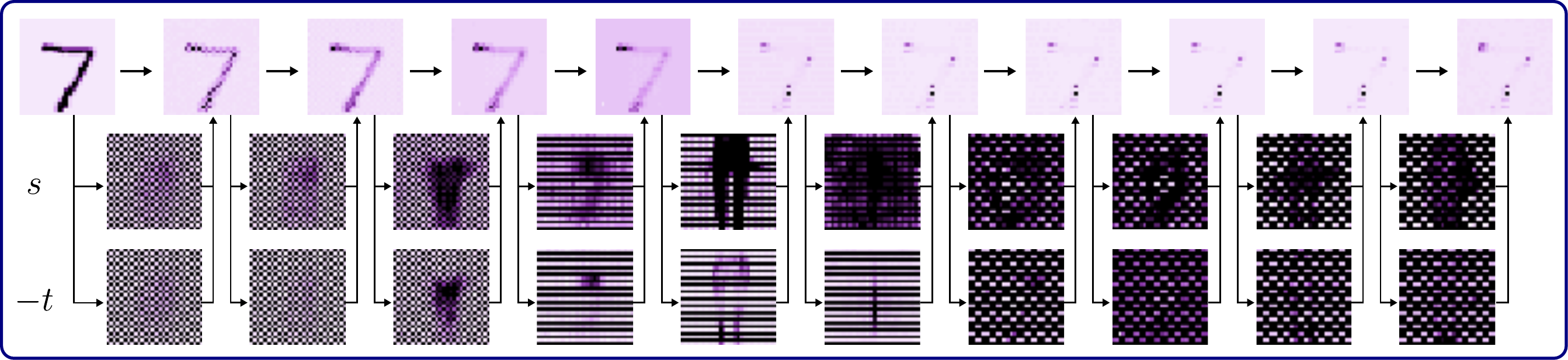}
    \end{tabular}
    }
	\caption{
        \textbf{Effect of $st$-network capacity.}
	    \textbf{The first row} shows the histogram of log likelihoods for a RealNVP model trained on CelebA dataset: 
        \textbf{(a)} for a baseline model, and \textbf{(b)-(d)} for models with different bottleneck dimensions $l$ in $st$-network. 
	    \textbf{The second and third rows} show samples from RealNVP model trained on CelebA and FashionMNIST respectively: 
        \textbf{(e)} and \textbf{(i)} for baseline models, and 
        \textbf{(f)-(h)} and \textbf{(j)-(l)} for models with different bottleneck dimensions $l$.
	    In \textbf{(m)}, we show the visualization of the coupling layer activations and $st$-network predictions for a RealNVP model trained on FashionMNIST with a bottleneck of dimension $l = 10$.
	    The top half shows the visualizations for an in-distribution FashionMNIST image 
        while the bottom half shows the visualizations for an OOD MNIST image. 
        $st$-network with restricted capacity cannot accurately predict masked 
        pixels of the OOD image in the intermediate coupling layers. 
        Moreover, in the middle coupling layers for the MNIST input the activations
        resemble FashionMNIST images in $s$ and $t$ predictions.
	}
	\label{fig:st_coupling_samples}
    \vspace{-.3cm}
\end{figure}

\FloatBarrier

\textbf{Changing the architecture of $st$-networks}\quad
In Figure \ref{fig:st_coupling_samples}, we show likelihood distributions, samples 
and coupling layer visualization for RealNVP model with $st$-network with a 
bottleneck trained on FashionMNIST and CelebA datasets. 
The considered bottleneck dimensions for FashionMNIST are $\{10, 50, 100\}$, and
for CelebA the dimensions are $\{30, 80, 150\}$. 
In the baseline RealNVP model, we use a standard deep convolutional residual 
network without additional skip connections from the intermediate layers to the 
output which were used in \citet{dinh2016density}.

\section{Samples}
\label{sec:app_samples}

In Figure \ref{fig:cm_samples}, we show samples for RealNVP and Glow models trained on CelebA, CIFAR-10, SVHN, FashionMNIST and MNIST,
and a RealNVP model trained on ImageNet $64 \times 64$ and CelebA $64 \times 64$.

\subsection{Latent variable resampling}

To further understand the structure of the latent representations learned by the flow, 
we study the effect of resampling part of the latent representations corresponding
to images from different datasets from the base Gaussian distribution.
In Figure \ref{fig:resampling}, using a RealNVP model trained on CelebA we compute 
the latent representations corresponding to input images from CelebA, SVHN, and CIFAR-10 datasets, 
and randomly re-sample the subset of latent variables corresponding to a $10 \times 10$ 
square in the center of the image (to find the corresponding latent variables we apply the
squeeze layers from the flow to the $32\times32$ mask).
We then invert the flow and compute the reconstructed images from the altered latent representations.

Both for in-distribution and out-of-distribution data, the model almost ideally 
preserves the part of the image other than the center, confirming the alignment 
between the latent space and the original input space discussed in Section \ref{sec:latent_space}.
The model adds a face to the resampled part of the image, preserving the consistency with the background to some extent.

\begin{figure}[h]
    \centering
    \subfigure[Celeb-A]{
	\includegraphics[width=0.25\linewidth]{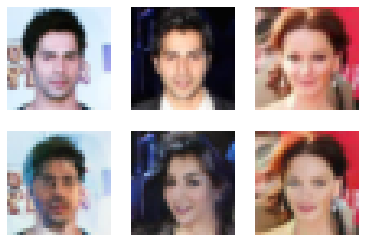}
	}
	\subfigure[CIFAR-10]{
	\includegraphics[width=0.25\linewidth]{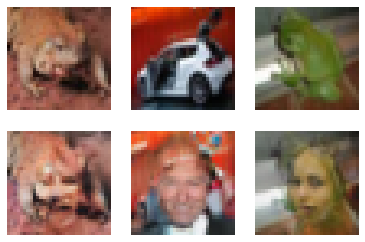}
	}
	\subfigure[SVHN]{
	\includegraphics[width=0.25\linewidth]{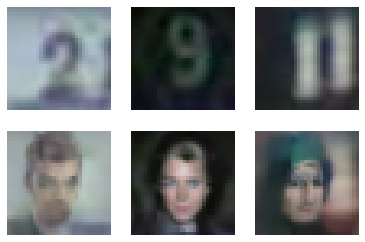}
	}
	\caption{
    \textbf{Latent variable resampling.}
	Original images (\textbf{top row}) and reconstructions with the latent variables corresponding to a $10 \times 10$ square in the center of the image randomly re-sampled for a RealNVP model trained on Celeb-A (\textbf{bottom row}).
	The model adds faces (as it was trained Celeb-A) to the part of the image that is being re-sampled.
	}
	\label{fig:resampling}
    \vspace{-.3cm}
\end{figure}

\begin{figure}[t]
    \centering
	\def \panelwidth {0.22\textwidth}
	\def \panelskip {-0.1cm}
    \hspace{\panelskip}

	\subfigure[RNVP, CelebA]{
	\includegraphics[width=\panelwidth]{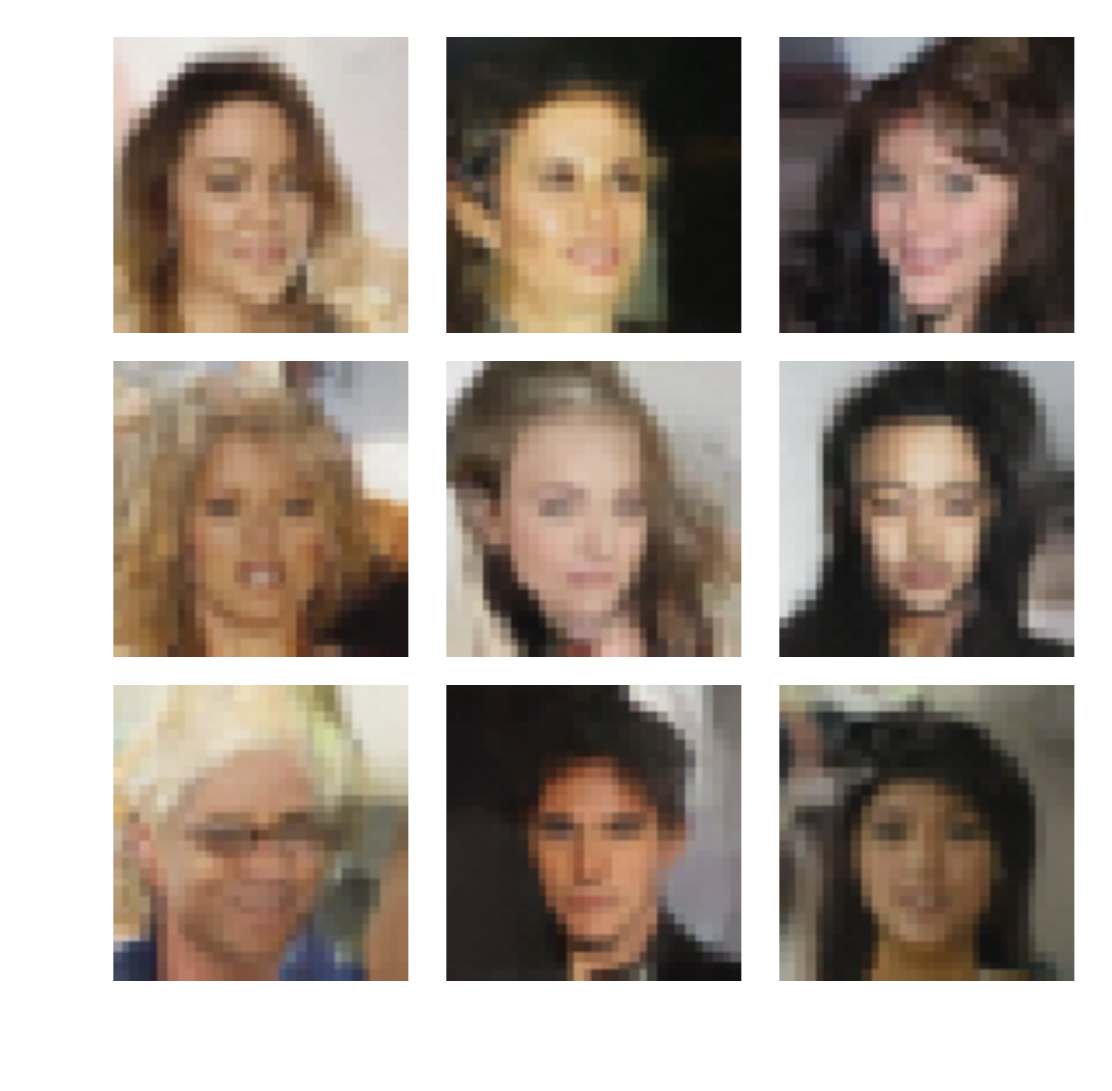}
	}
    \hspace{\panelskip}
	\subfigure[RNVP, CIFAR-10]{
	\includegraphics[width=\panelwidth]{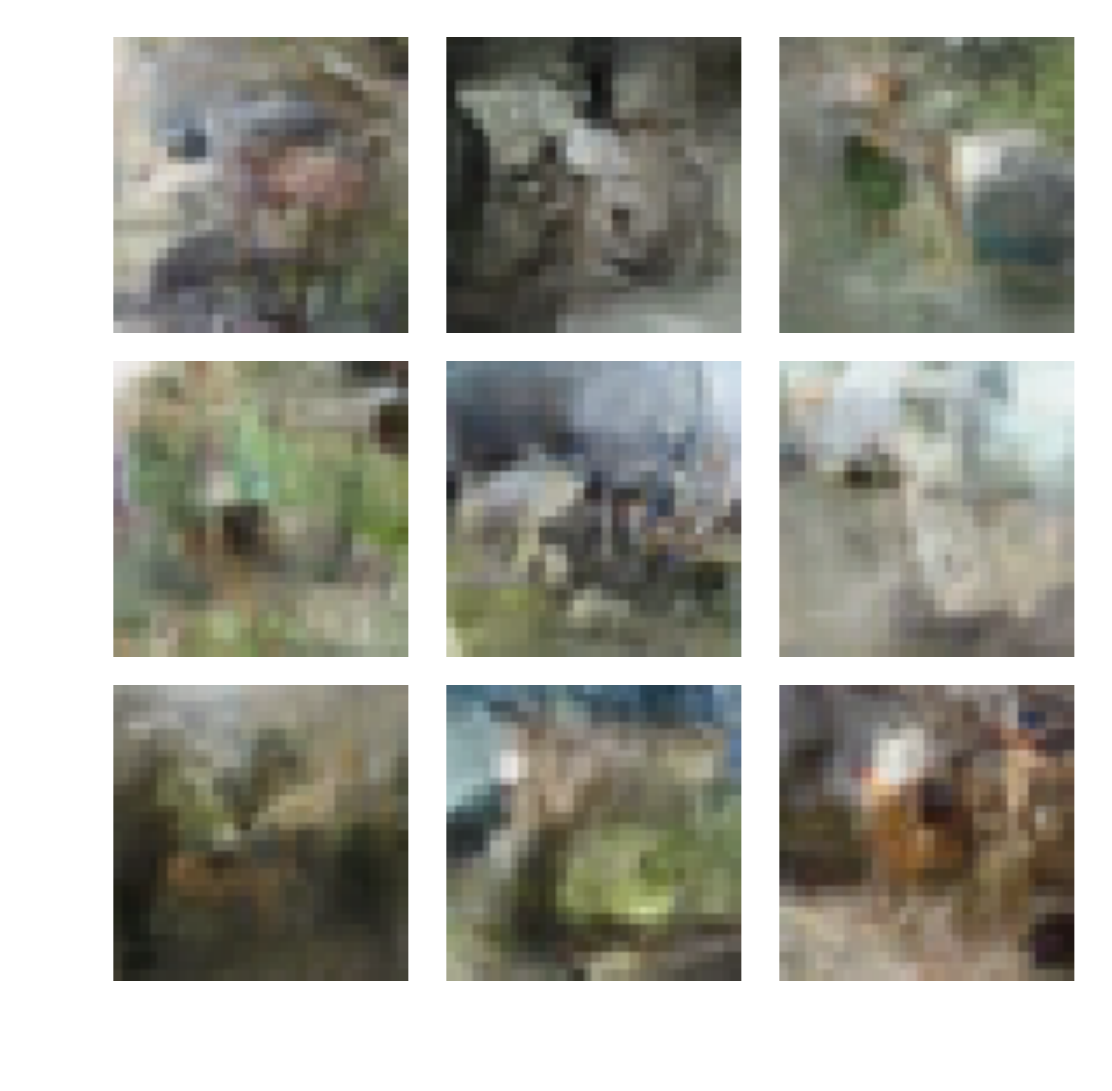}
	}
	\subfigure[RNVP, SVHN]{
	\includegraphics[width=\panelwidth]{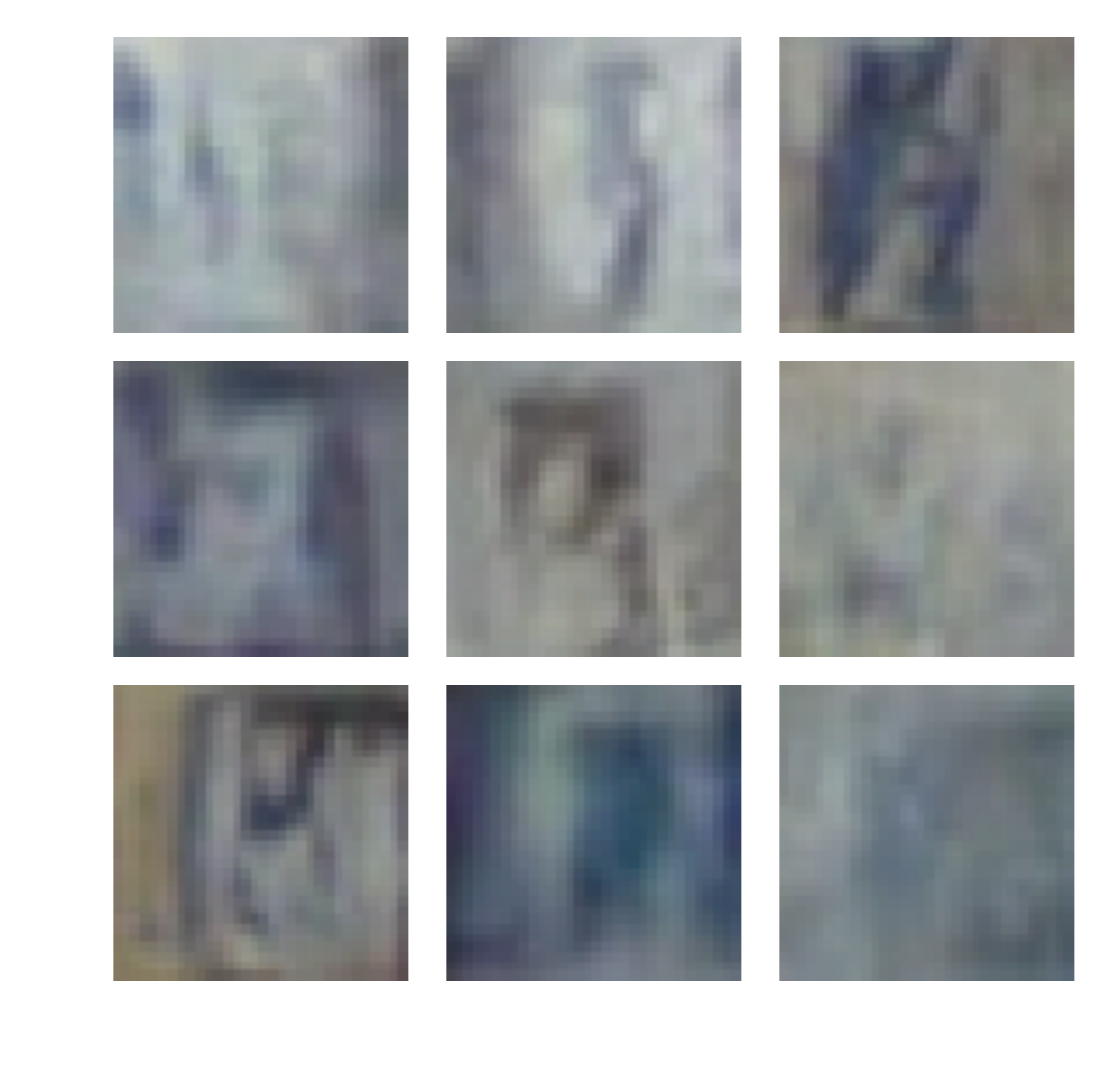}
	}
    \subfigure[RNVP, FashionMNIST]{
	\includegraphics[width=\panelwidth]{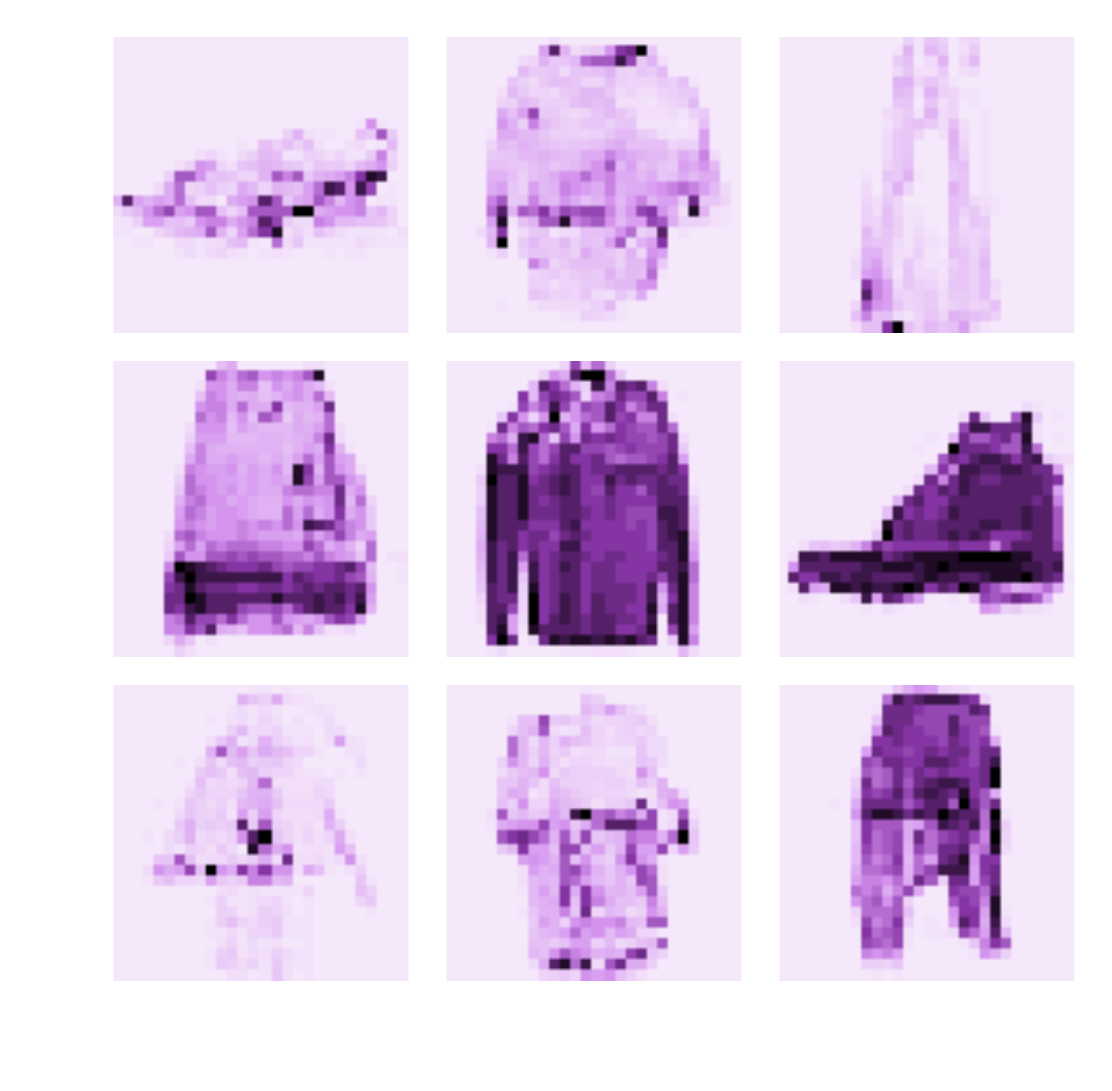}
	}

    \subfigure[RNVP, MNIST]{
	\includegraphics[width=\panelwidth]{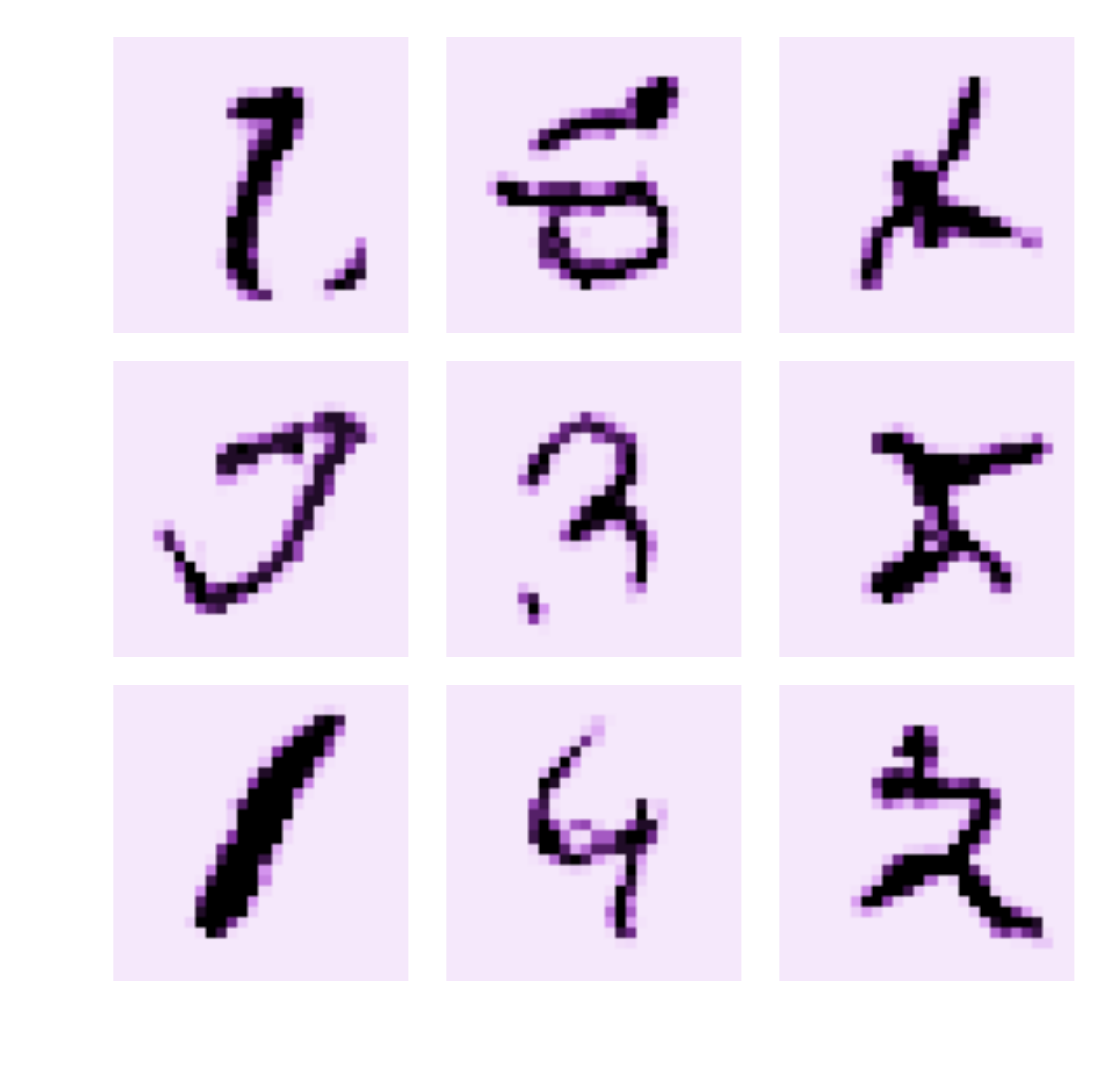}
	}
    \hspace{\panelskip}
	\subfigure[RNVP, CelebA-HQ]{
	\includegraphics[width=\panelwidth]{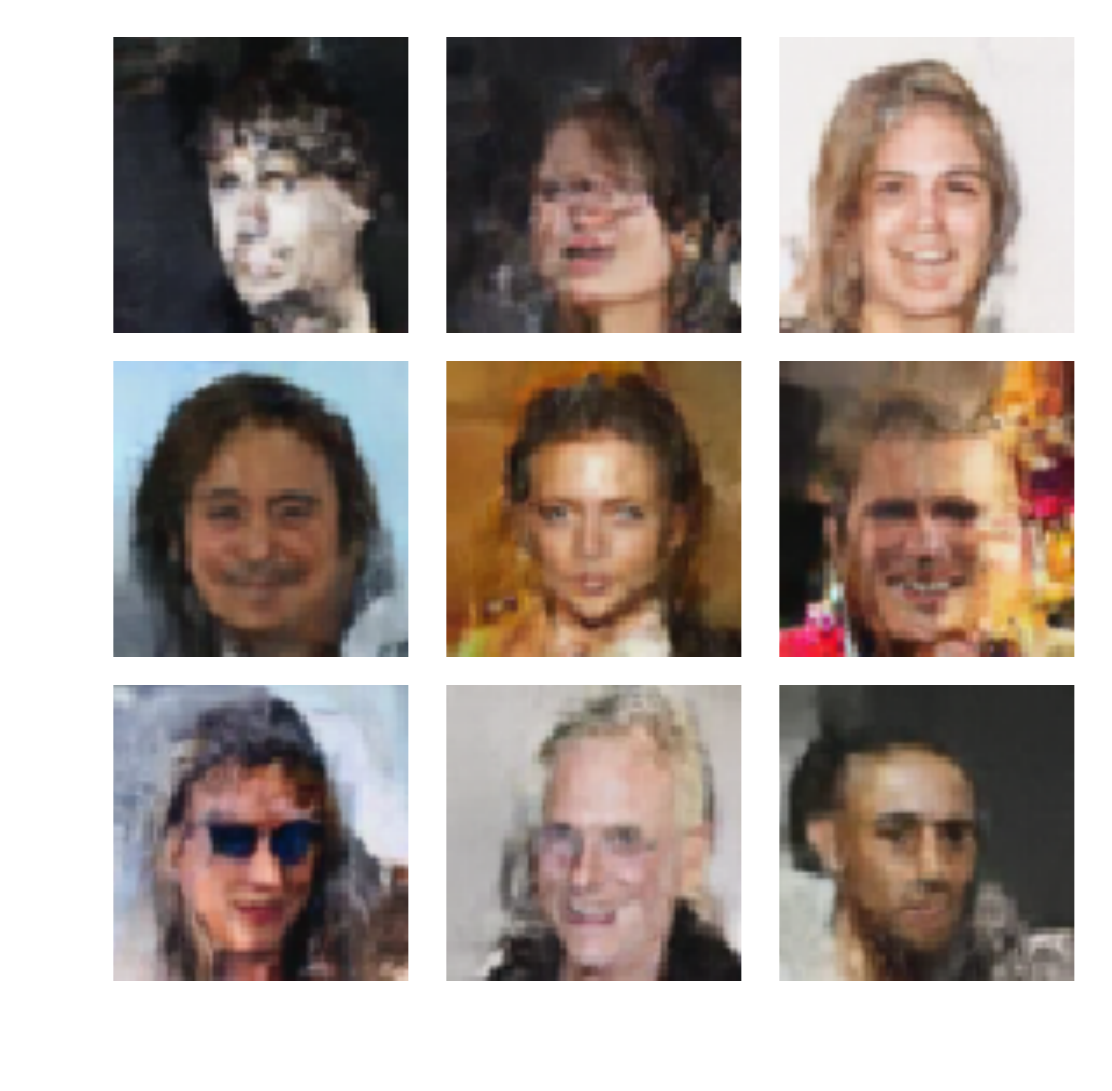}
	}
    \hspace{\panelskip}
	\subfigure[RNVP, ImageNet]{
	\includegraphics[width=\panelwidth]{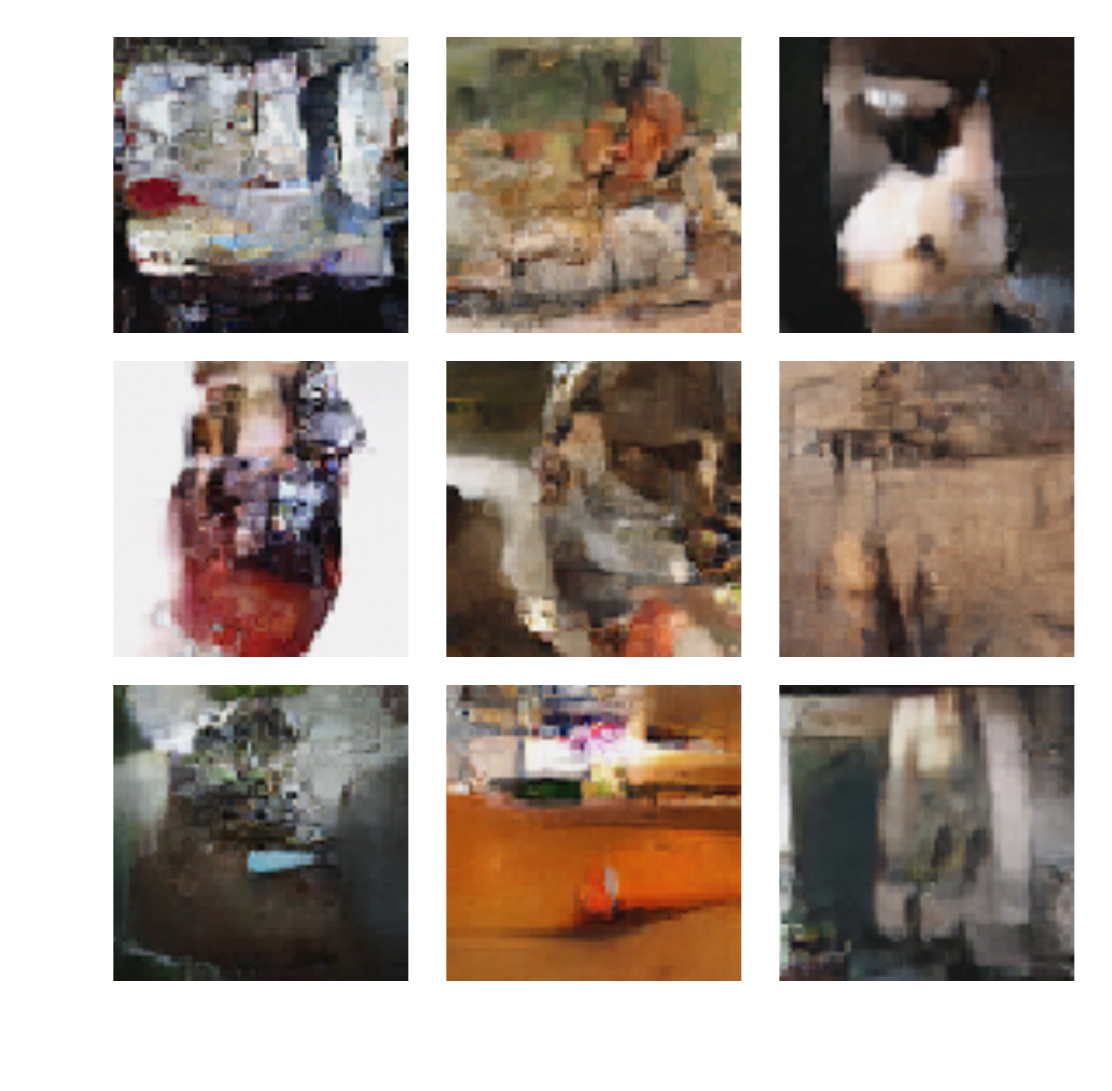}
	}
    \hspace{\panelskip}
	\subfigure[Glow, CelebA]{
	\includegraphics[width=\panelwidth]{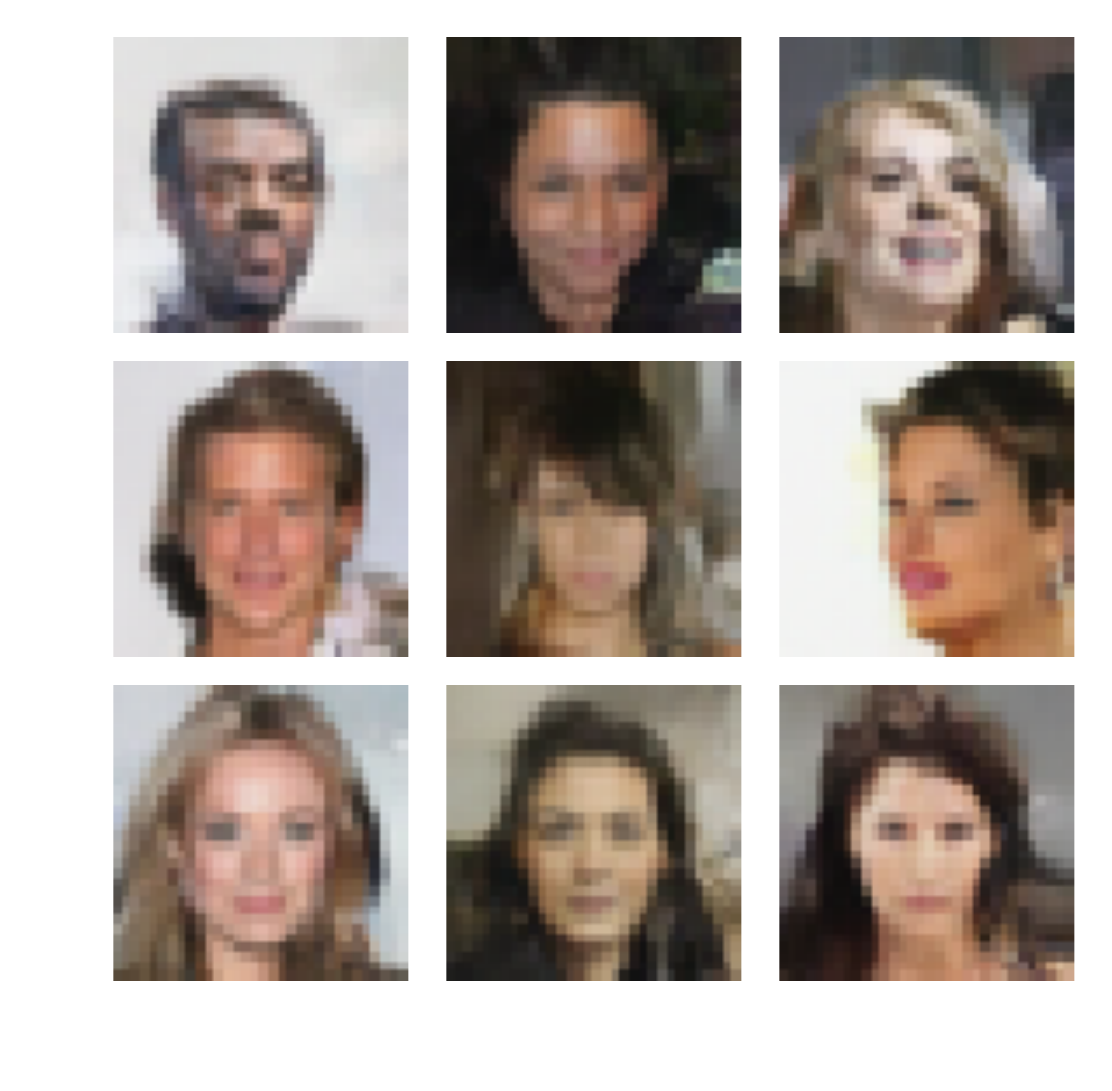}
	}

	\subfigure[Glow, CIFAR-10]{
	\includegraphics[width=\panelwidth]{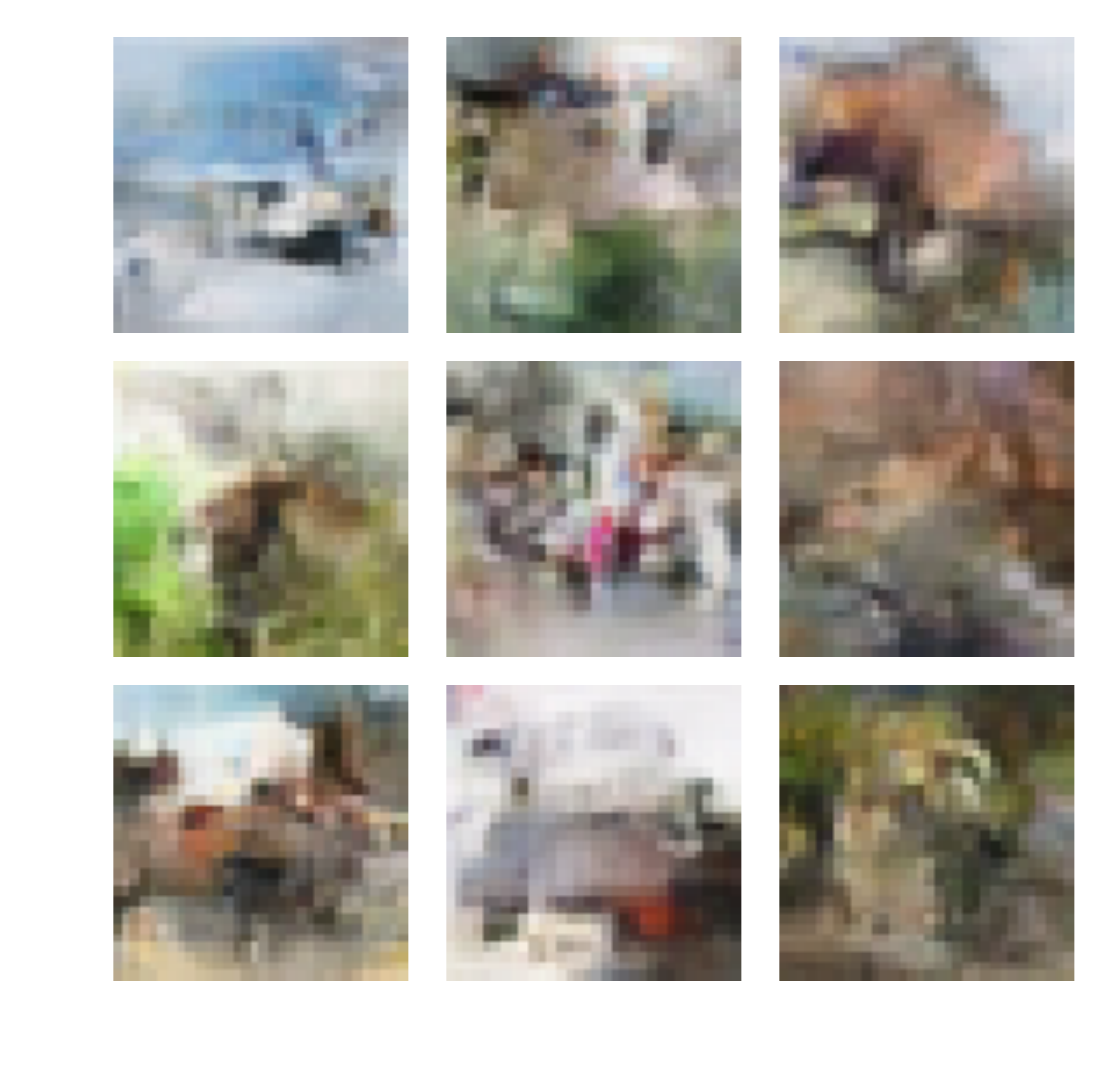}
	}
    \hspace{\panelskip}
	\subfigure[Glow, SVHN]{
	\includegraphics[width=\panelwidth]{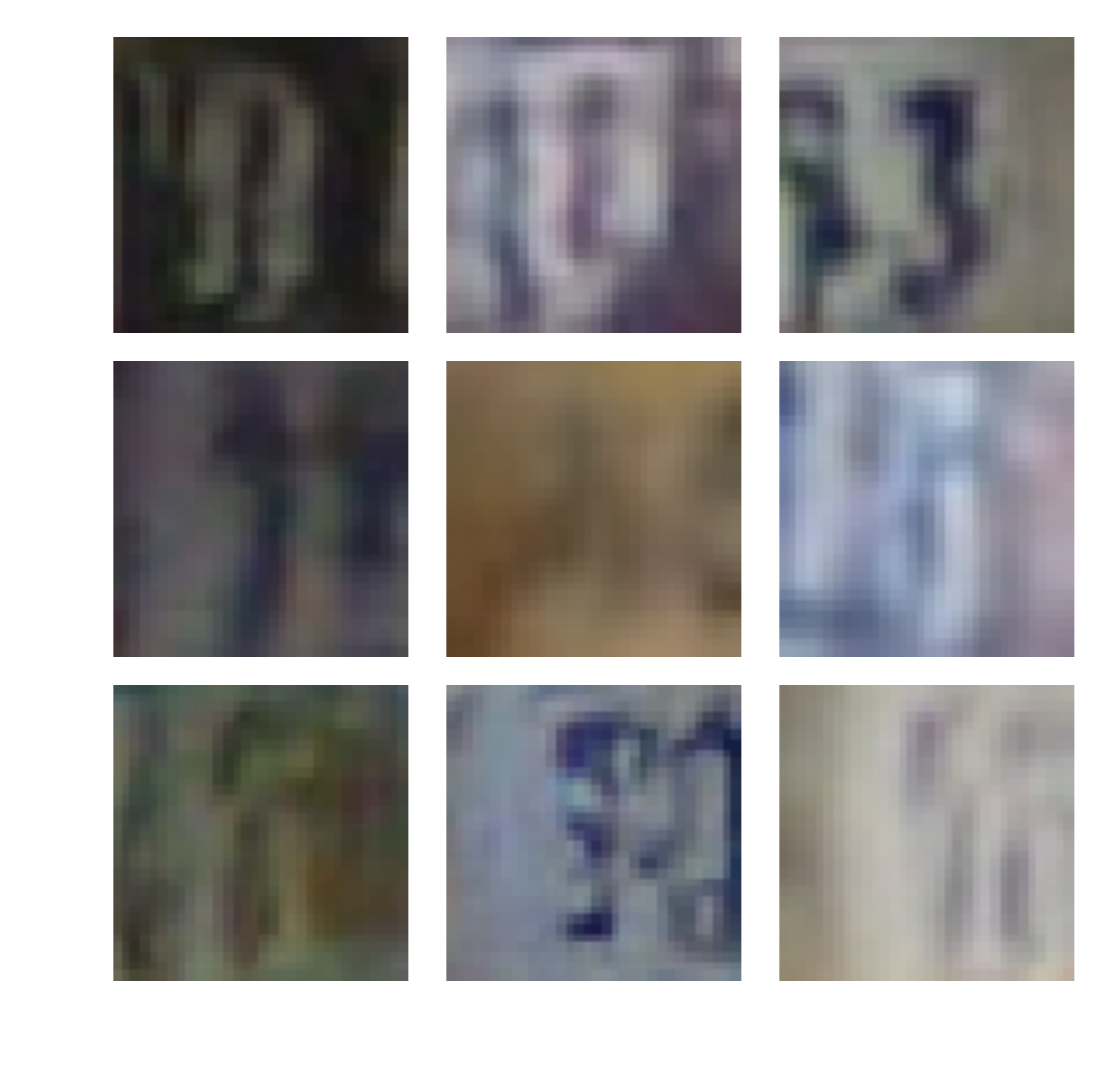}
	}
    \hspace{\panelskip}
	\subfigure[Glow, Fashion]{
	\includegraphics[width=\panelwidth]{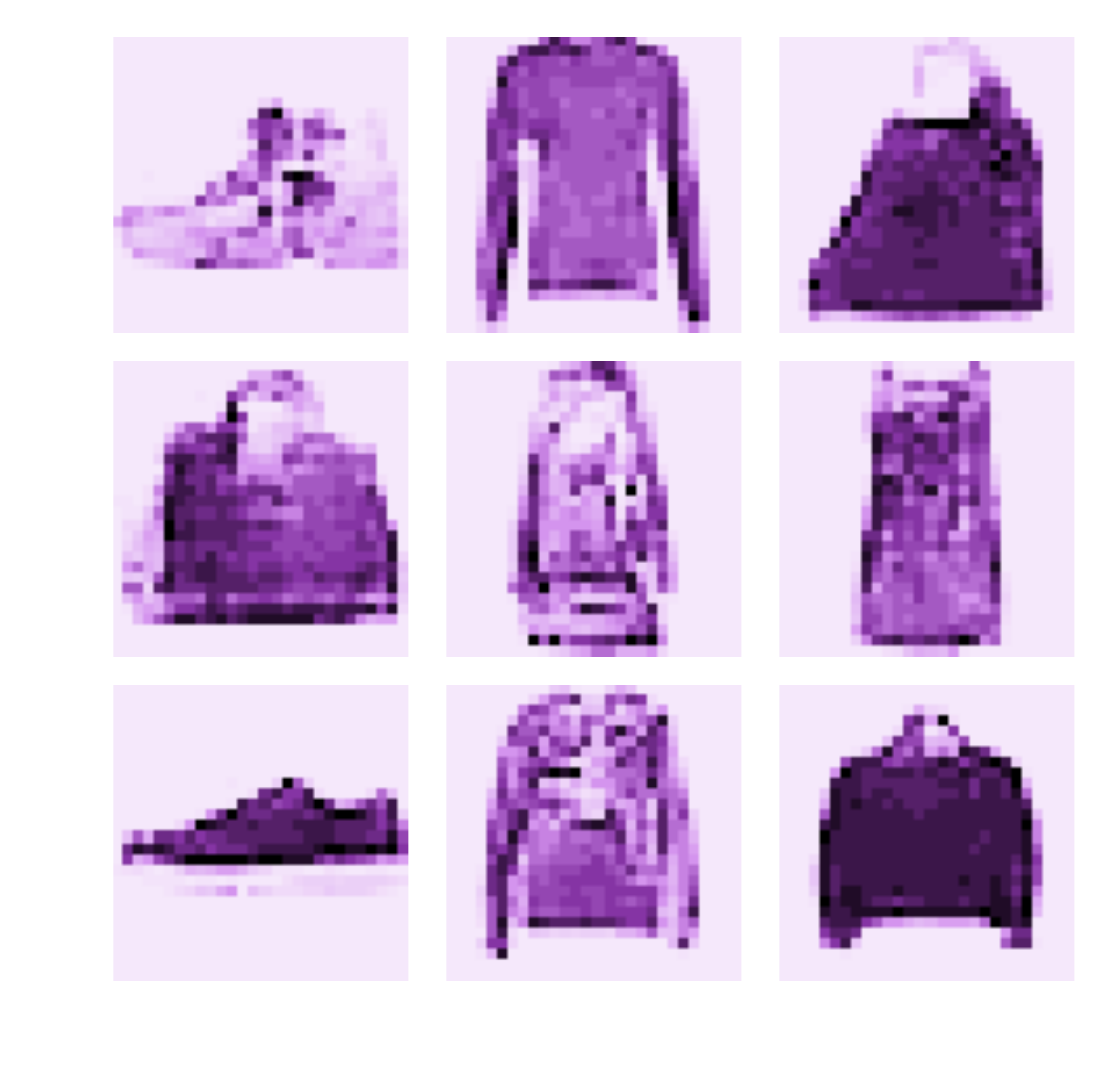}
	}
    \hspace{\panelskip}
	\subfigure[Glow, MNIST]{
	\includegraphics[width=\panelwidth]{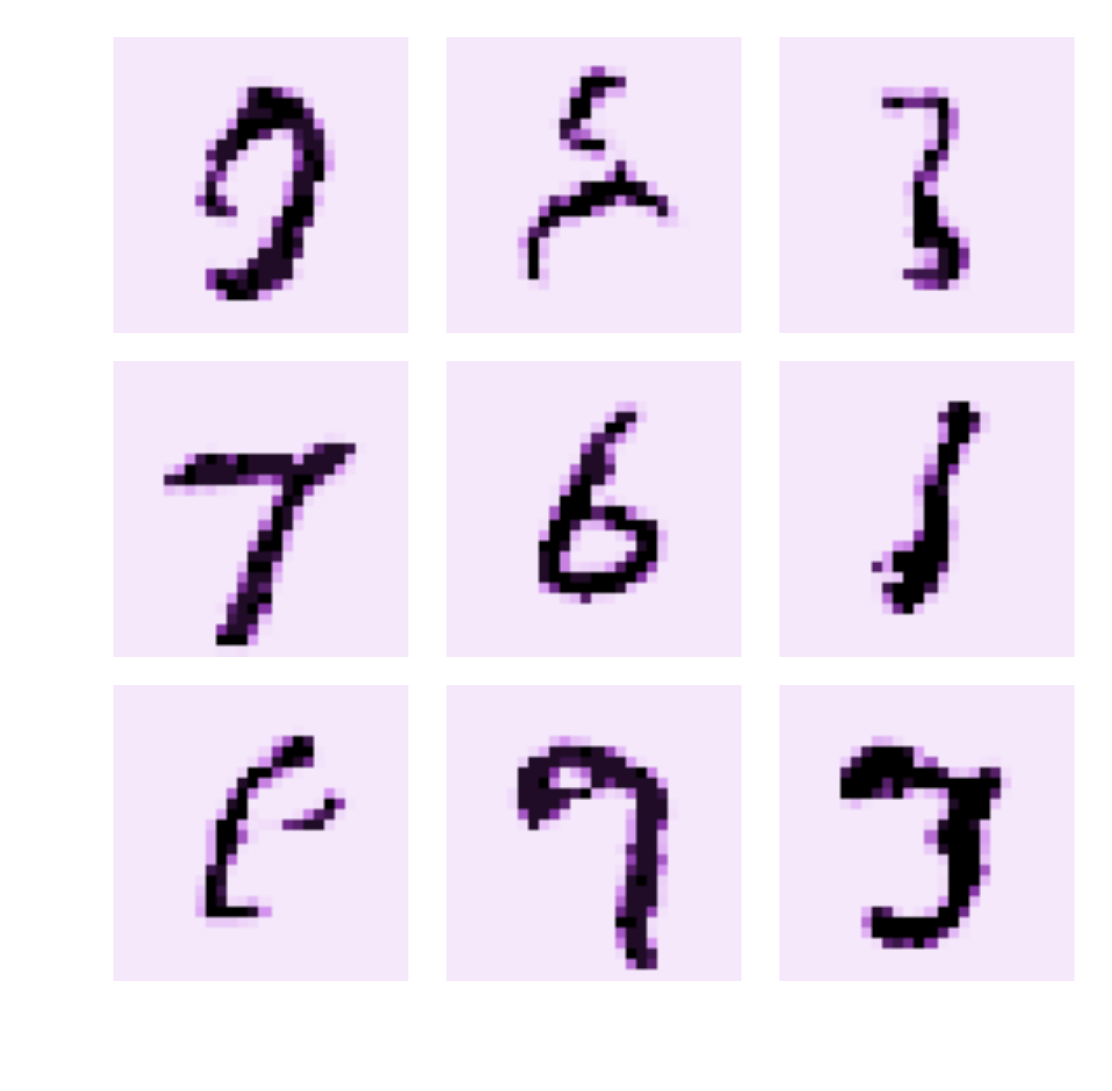}
	}
	\caption{
        \textbf{Baseline Samples.}
	    Samples from baseline RealNVP and Glow models.
        For ImageNet and CelebA-HQ we used datasets with $(64\times64)$ definition.
	}
	\label{fig:cm_samples}
    \vspace{-.3cm}
\end{figure}

\FloatBarrier

\section{Out-of-distribution detection on tabular data}
\label{appendix:tabular}

\begin{table}[h!]
    \centering
    \scriptsize
    \subfigure[Image embeddings]{
    \begin{tabular}{cccc}
        \toprule
        Train data & \multicolumn{3}{c}{OOD data} \\
        \cmidrule(r){2-4} 
        & CelebA & CIFAR-10 & SVHN \\
        \midrule
        CelebA & -- & 99.99 & 99.99  \\
        CIFAR-10 & 99.99 & -- & 73.31 \\
        SVHN & 100.0 & 99.98 & -- \\
         \bottomrule
    \end{tabular}
    }
    \quad
    \subfigure[Tabular data]{
    \begin{tabular}{ccc}
        \toprule
        Train class (OOD class) & \multicolumn{2}{c}{Dataset} \\
        \cmidrule(r){2-3} 
        & HEPMASS & MINIBOONE \\
        \midrule        
        Background (Signal) & 83.78 & 72.71 \\
        Signal (Background) & 70.73 & 87.56 \\
        \bottomrule
    \end{tabular}
    }
    \vspace{0.5cm}
    \caption{
    \textbf{Image embedding and UCI AUROC.}
    \textbf{(a)}:
    AUROC scores on OOD detection for RealNVP model trained on image embeddings extracted from EfficientNet.
    The model is trained on one of the embedding datasets while the remaining two are considered OOD.
    The models consistenly assign higher likelihood to in-distribution data, and
    in particular AUROC scores are significantly better compared to flows trained on  the
    original images (see Table \ref{tab:baseline_auroc}).
    \textbf{(b)}:
    AUROC scores on OOD detection for RealNVP trained on one class of Hepmass and
    Miniboone datasets while the other class is treated as OOD data.
    }
    \label{tab:efficientnet_auroc}
\end{table}

\begin{figure}[h!]
    \centering
    \subfigure[Miniboone Dataset]{
    \begin{tabular}{cc} 
	    \includegraphics[height=0.18\linewidth]{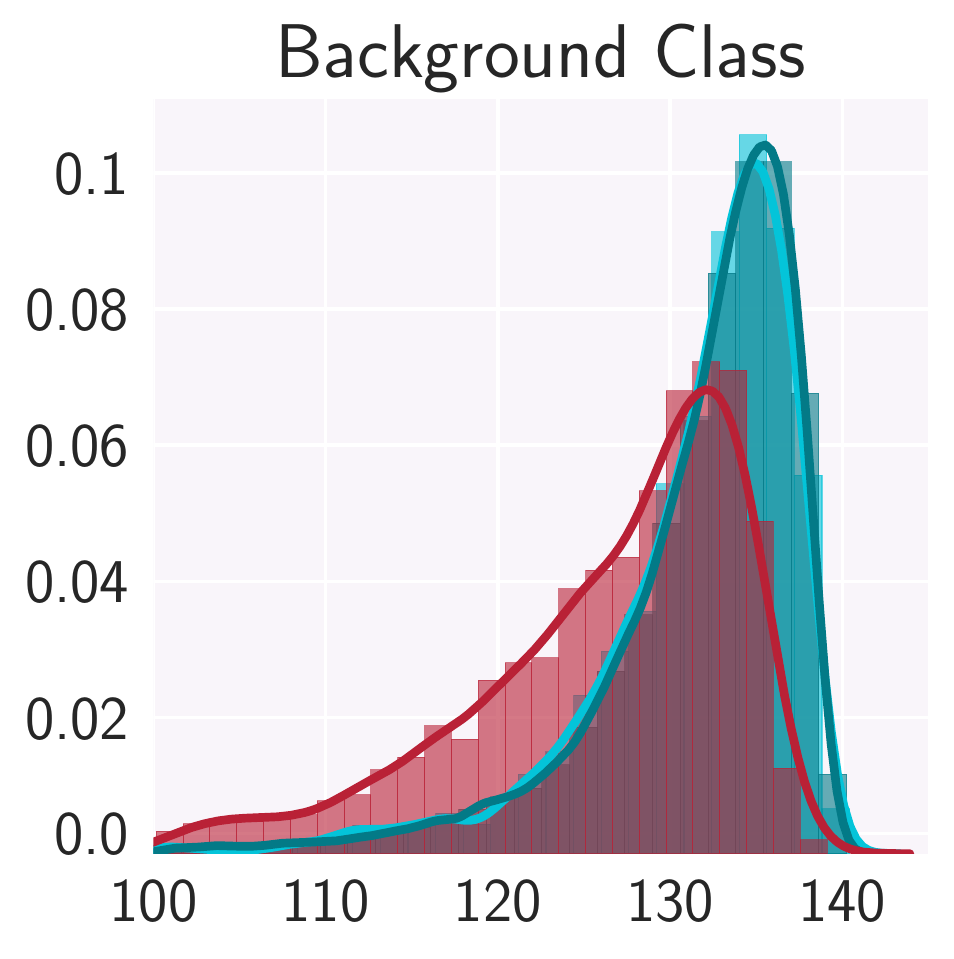}&
        \hspace{-0.4cm}
	    \includegraphics[height=0.18\linewidth]{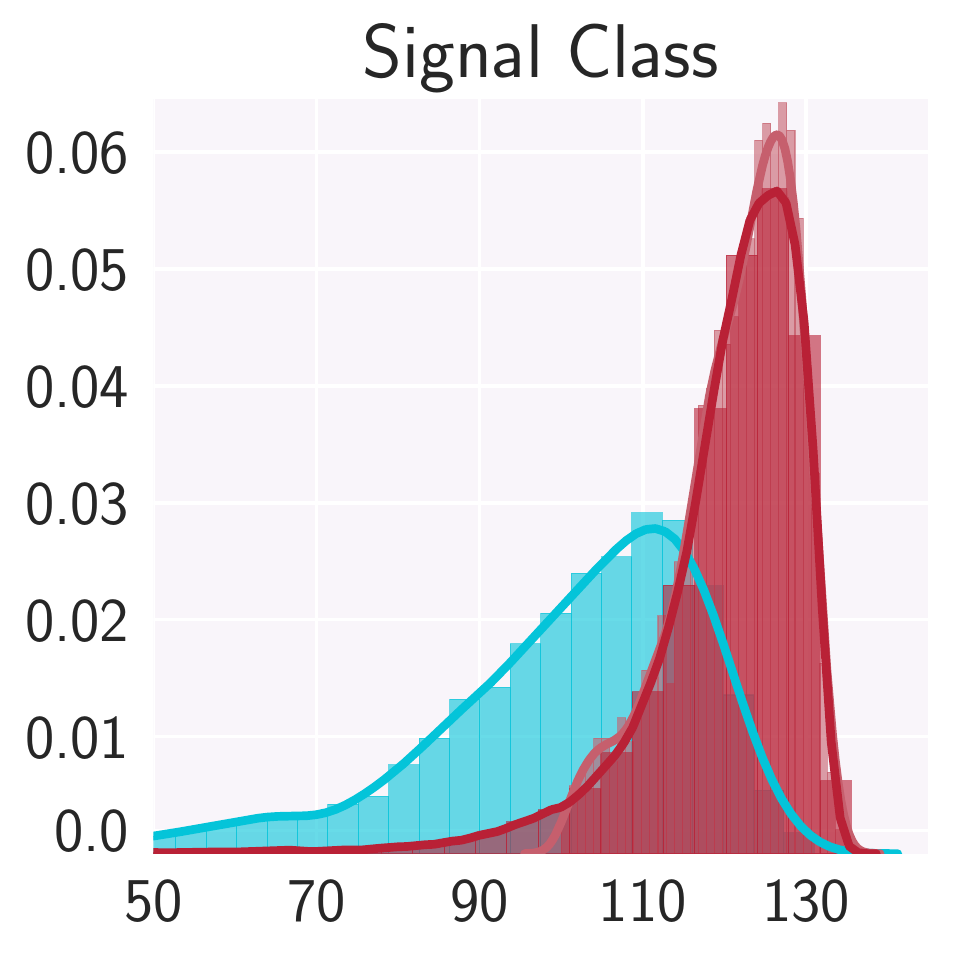}
    \end{tabular}
    }
    \hspace{-0.3cm}
    \subfigure[Hepmass Dataset]{
    \begin{tabular}{ccc} 
	    \includegraphics[height=0.18\linewidth]{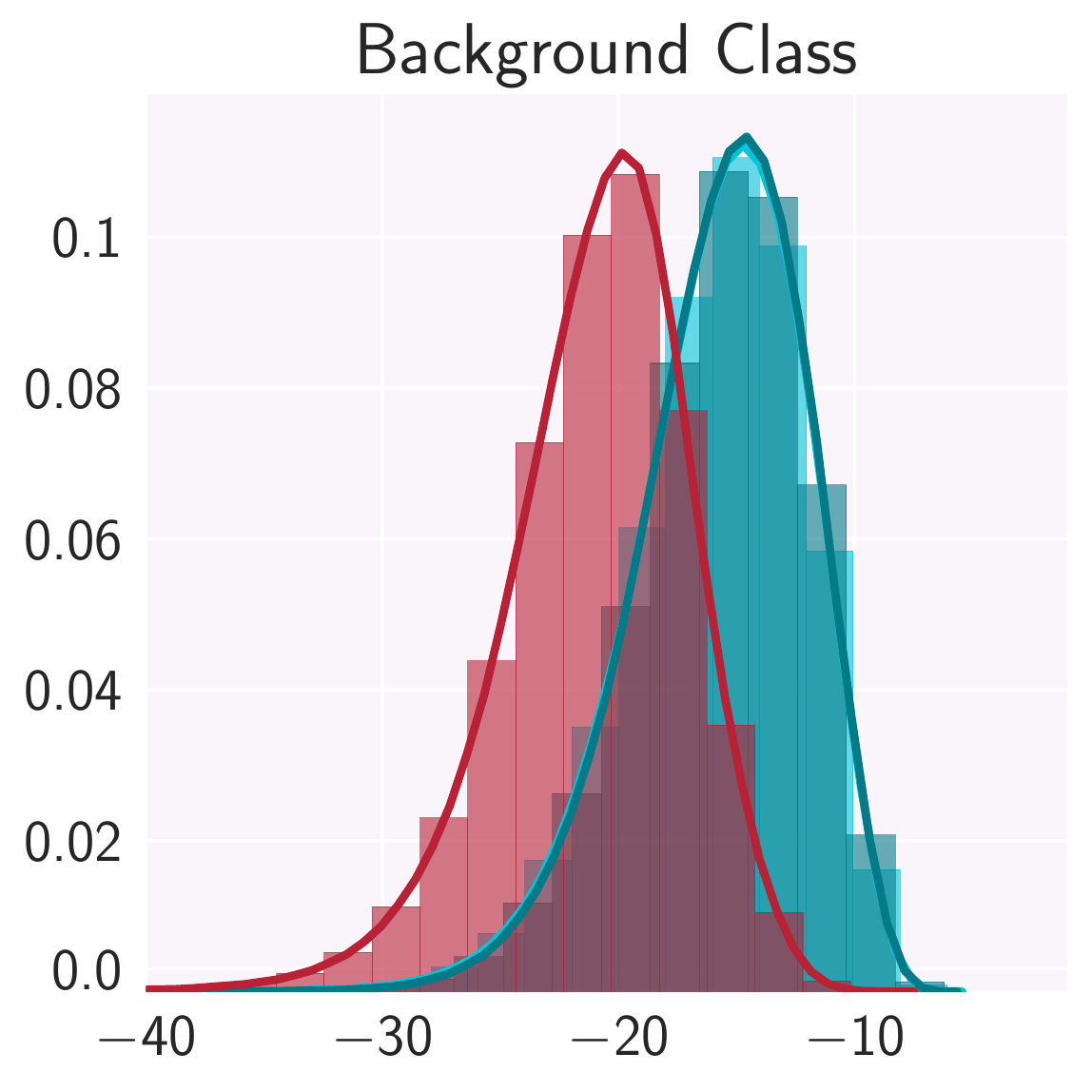}&
	    \hspace{-0.4cm}
	    \includegraphics[height=0.18\linewidth]{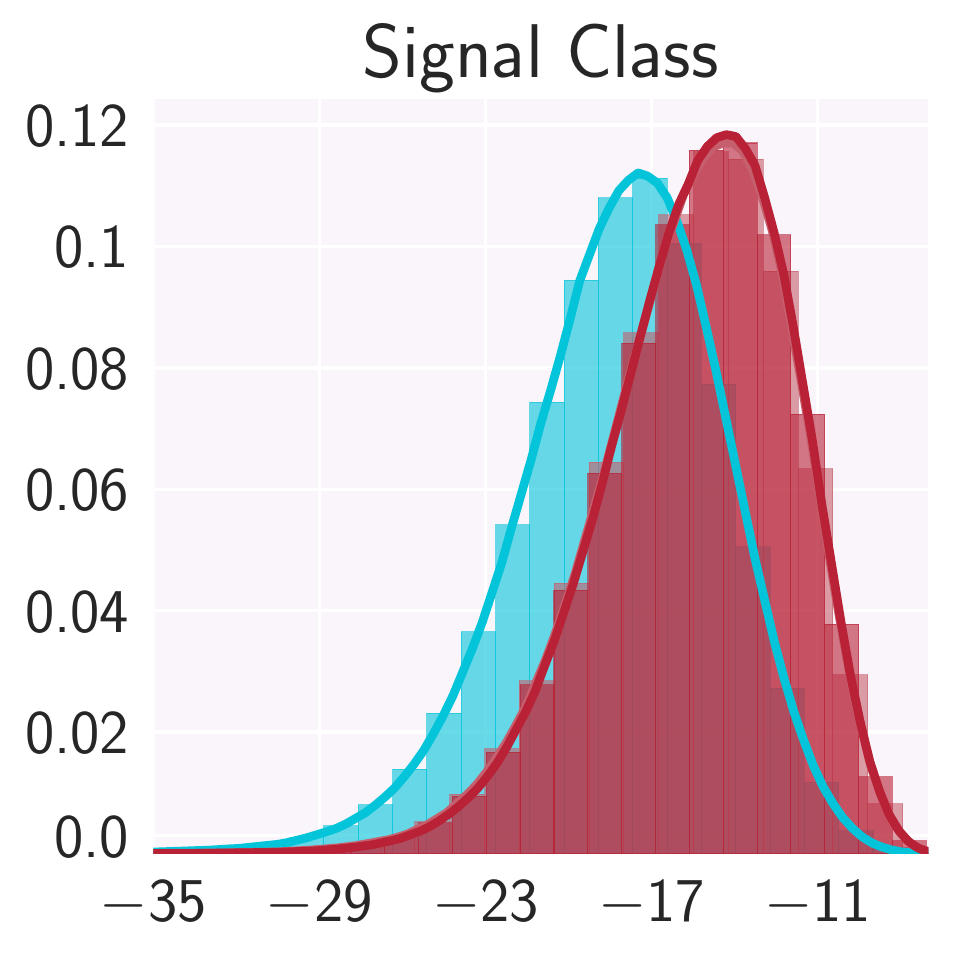}&
	    \includegraphics[height=0.18\linewidth]{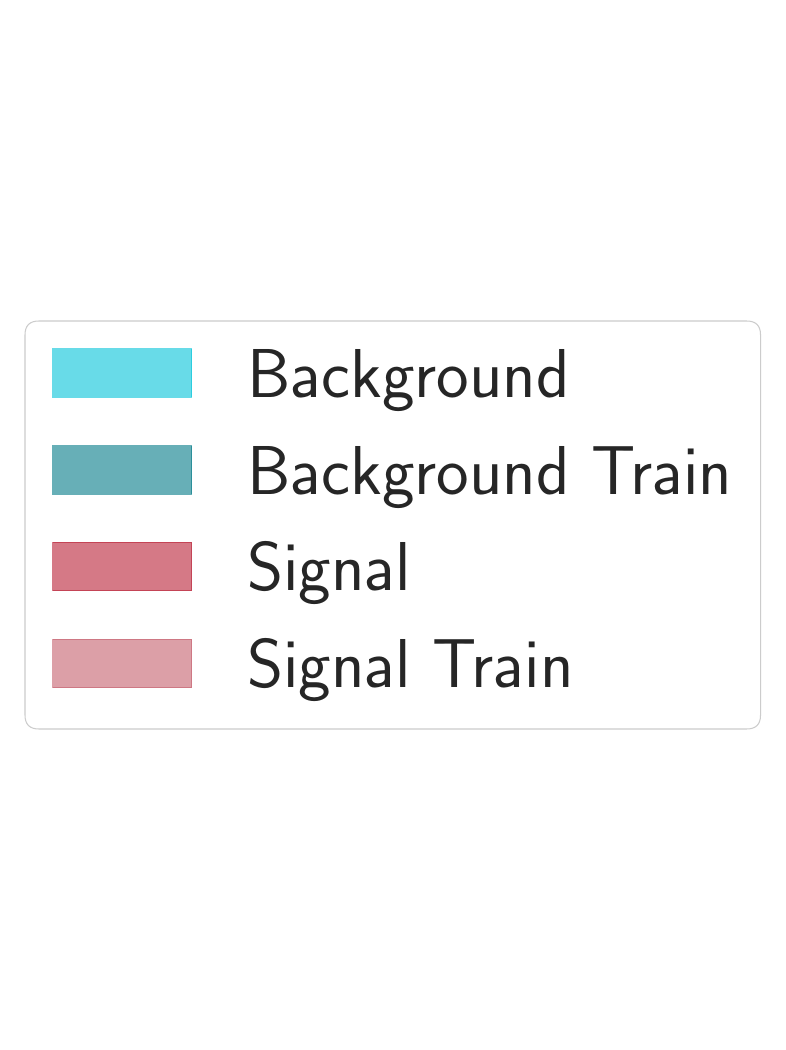}
    \end{tabular}
    }
	\caption{
    \textbf{UCI datasets.}
	The histograms of log-likelihood for RealNVP on Hepmass and Miniboone tabular datasets when trained on
    one class and the other class is viewed as OOD.
    The train and test likelihood distributions are almost identical when trained
    on either class, and the OOD class receives lower likelihoods on average.
    There is however a significant overlap between the likelihoods for in- and out-of-distribution
    data.
    }
	\label{fig:ll_hists_tabular_hepmass_miniboone}
    \vspace{-.3cm}
\end{figure}

\subsection{Model}
We use RealNVP with 8 coupling layers, fully-connected $st$-network and masks which split input vector by half in an alternating manner. For UCI experiments, we use 1 hidden layer and 256 hidden units in $st$-networks, learning rate $10^{-4}$, batch size 32 and train the model for 100 epochs. For image embeddings experiments, we use 3 hidden layer and 512 hidden units in $st$-networks, learning rate $10^{-3}$, batch size 1024 and train the model for 120 epochs. For all experiments, we use the AdamW optimizer \citep{loshchilov2017decoupled} and weight decay $10^{-3}$.

\subsection{EfficientNet embeddings}
We train RealNVP model on image embeddings for CIFAR-10, CelebA and SVHN extracted from EfficientNet train on ImageNet, and report AUROC scores in Table \ref{tab:efficientnet_auroc}(a). 

\subsection{UCI datasets}
\label{appendix:uci}

In this experiment, we use 2 UCI classification datasets which were used for unsupervised 
modeling in prior works on normalizing flows \citep{papamakarios2017masked, durkan2019neural, grathwohl2018ffjord}: 
HEPMASS \citep{baldi2016parameterized} and MINIBOONE \citep{roe2005boosted}. 
HEPMASS and MINIBOONE are both binary classification datasets originating from physics, 
and the two classes represent \textit{background} and \textit{signal}.
We follow data preprocessing steps of \citet{papamakarios2017masked}. We filter features which have too many reoccurring values, after that the dimenionality of the data is 15 for HEPMASS and 50 for MINIBOONE. For HEPMASS, we use the ``1000'' dataset which contains subset of particle signal with mass 1000.
For MINIBOONE data, for each class we take a random split of 10\% for a test set. 

To test OOD detection performance, for each dataset we train a model on one class while treating the second class as OOD data. 
We plot the resulting train, test and OOD likelihood distributions for each dataset in Figure \ref{fig:ll_hists_tabular_hepmass_miniboone}. 
We also report AUROC scores for each setup in Table \ref{tab:efficientnet_auroc}(b). 
While test and OOD likelihoods overlap, the in-distribution class has higher average 
likelihood in all cases, and AUROC values are ranging between 70\% and 87\% which is a 
significantly better result compared to the results for image benchmarks reported in \citet{nalisnick2018deep}.

\end{document}